\documentclass[a4paper]{article}

\usepackage[utf8]{inputenc}
\usepackage{setspace}
\usepackage{multirow}
\usepackage{array}
\usepackage{amsfonts,amsthm}

\usepackage{dsfont}
\usepackage{booktabs}
\usepackage{tikz}
\usepackage{vcell}

\usepackage{hyperref}
\hypersetup{
	colorlinks=true,
	linkcolor=red,
	citecolor=blue,
	urlcolor=violet,
	pdftitle={Scalable Learning of Item Response Theory Models},
	linktocpage=true,
}

\usepackage{enumitem}
\usepackage{mathtools}
\usepackage{xspace}
\usepackage{array}
\usepackage{url}

\newtheorem{theorem}{Theorem}[section] 
\newtheorem{definition}[theorem]{Definition}
\newtheorem{lemma}[theorem]{Lemma}
\newtheorem{corollary}[theorem]{Corollary}

\newcommand{\nnz}{\operatorname{nnz}}

\makeatletter
\newsavebox{\@brx}
\newcommand{\llangle}[1][]{\savebox{\@brx}{\(\m@th{#1\langle}\)}%
  \mathopen{\copy\@brx\kern-0.5\wd\@brx\usebox{\@brx}}}
\newcommand{\rrangle}[1][]{\savebox{\@brx}{\(\m@th{#1\rangle}\)}%
  \mathclose{\copy\@brx\kern-0.5\wd\@brx\usebox{\@brx}}}
\makeatother

\renewcommand{\varsigma}{\sigma}

\DeclareMathOperator*{\argmin}{arg\,min}

\newcommand{\seclab}[1]{\label{sec:#1}}
\newcommand{\secref}[1]{Section~\ref{sec:#1}}
\newcommand{\subsecref}[1]{Subsection~\ref{sec:#1}}
\newcommand{\thmlab}[1]{{\label{theo:#1}}}
\newcommand{\thmref}[1]{Theorem~\ref{theo:#1}}
\newcommand{\lemlab}[1]{\label{lemma:#1}}
\newcommand{\lemref}[1]{Lemma~\ref{lemma:#1}}

\newcommand{\deflab}[1]{\label{def:#1}}
\newcommand{\defref}[1]{Definition~\ref{def:#1}}

\newcommand{\corlab}[1]{\label{cor:#1}}
\newcommand{\corref}[1]{Corollary~\ref{cor:#1}}

\providecommand{\Oh}[1]{\ensuremath{{O\left(#1\right)}}}

\renewcommand{\Pr}[1]{\ensuremath{\mathbf{Pr}\left[#1\right]}}

\newcommand{\REAL}{\ensuremath{\mathbb{R}}}

\newcommand{\removed}[1]{}

\newcommand{\eps}{\varepsilon}
\providecommand\phantomsection{}

\usepackage[capitalise,noabbrev]{cleveref}

\allowdisplaybreaks[1]

\usepackage{fullpage}
\usepackage[round]{natbib}

\begin{document}

\title{Scalable Learning of Item Response Theory Models \thanks{The conference version of this paper appeared in the 27th International Conference on Artificial Intelligence and Statistics (AISTATS) 2024, pages~1234--1242.}
}
\author{
   Susanne Frick
   \thanks{Department of Statistics, TU Dortmund, Germany; 
} 
   \and
   Amer Krivo\v{s}ija%
   \thanks{Department of Statistics, TU Dortmund, Germany; 
} 
   \and
   Alexander Munteanu%
   \thanks{Department of Statistics and Center for Data Science \& Simulation (DoDaS), TU Dortmund, Germany; 
} 
}

\date{\today}

\maketitle

\begin{abstract}
  Item Response Theory (IRT) models aim to assess latent abilities of $n$ examinees along with latent difficulty characteristics of $m$ test items from categorical data that indicates the quality of their corresponding answers. Classical psychometric assessments are based on a relatively small number of examinees and items, say a class of $200$ students solving an exam comprising $10$ problems. More recent global large scale assessments such as PISA, or internet studies, may lead to significantly increased numbers of participants. Additionally, in the context of Machine Learning where algorithms take the role of examinees and data analysis problems take the role of items, both $n$ and $m$ may become very large, challenging the efficiency and scalability of computations. To learn the latent variables in IRT models from large data, we leverage the similarity of these models to logistic regression, which can be approximated accurately using small weighted subsets called \emph{coresets}. We develop coresets for their use in alternating IRT training algorithms, facilitating scalable learning from large data.
\end{abstract}

\clearpage

\tableofcontents


\clearpage

\section{INTRODUCTION}

Item Response Theory (IRT) is a paradigm often employed in psychometrics to estimate the ability of tested persons, called \emph{examinees}, through tests comprising multiple questions, called \emph{items}. The probability $p_{ij}$ that an item $i\in [m]:=\{ 1,\ldots,m\}$ will be solved by a person $j\in [n]$, depends on characteristic parameters of the item as well as on an ability parameter of the examinees.

The number of tested persons can be very large in contemporary global large scale assessments. For instance, the Programme for International Student Assessment (PISA) evaluates the education quality across $38$ OECD countries by measuring the literacy of $15$ year old students in reading, mathematics, and sciences. In this and other large scale (meta-)studies, nearly $n\approx 600\,000$ examinees are being tested regularly \citep{MuncerHGCWH21,PISA}. The number of items in the case of PISA is, however, comparatively small, $m\approx 10-30$ in each category. Beyond educational applications, IRT can be applied to benchmark studies where the examinees are artificial intelligence agents or machine learning algorithms, and the items are various problems. Then, the number of both, items and examinees, can in principle become arbitrarily large \citep{MartinezPMH19}. 
When the input data dimensions, $n$ and $m$, become large as motivated above, the computational effort to learn the parameters of IRT models grows. Sometimes it is not even possible to store the entire input or all latent variables simultaneously in main memory, which limits the applicability of IRT algorithms in large scale settings. 

A basic algorithmic pattern for learning IRT models is an alternating optimization procedure akin to EM algorithms. This is a classical approach taught in standard undergraduate courses in psychology, and thus it is highly significant. Given fixed values for the ability parameters, we optimize the item specific difficulty characteristics. Then, the updated difficulty characteristics are fixed while the abilities are being optimized. These two steps constitute one phase that is iterated over and over again until some termination criterion is met, such as convergence or exhaustion of an iteration budget.

To make this algorithmic pattern scalable to large data, we note that especially learning the item parameters from a huge number of examinees takes considerable time and space to be processed. In automated settings with a large number of test items, the same situation appears in the second step of each phase. Here, we note that in simple so called 1PL and 2PL (one/two parameter logistic) IRT models, each step consists of solving a set of logistic regression problems, where only the labels differ for each examinee or item. For logistic regression, it is known how to handle large data in a time and memory efficient way using a succinct summary as a replacement for the data. Such a proxy is commonly known as a \emph{coreset} that provably preserves the negative log-likelihood up to little errors \citep{MunteanuS18}.

\subsection{Our Contributions}
We review and motivate IRT models for various tasks and from different perspectives, ranging from the educational and social sciences to machine learning, where scalable IRT algorithms become important. From this starting point
\begin{enumerate}
\item we leverage the similarity of 2PL IRT models to logistic regression and adapt previous coresets to facilitate scalable learning of 2PL models,
   \item we develop new coresets for the more general and more challenging class of 3PL IRT models,
   \item we empirically evaluate the computational benefits of coresets for IRT algorithms while preserving their statistical accuracy up to little distortions.
\end{enumerate}
To our knowledge, our work provides the first sublinear approximation to the IRT subproblems considered in the alternating optimization steps with proven mathematical guarantees. 

\subsection{Related Work}

\paragraph{Development of IRT}
The history of IRT began with the formulation of the Rasch model \citep{rasch1960}. This was soon extended to modeling items with several parameters such as the 2PL and 3PL models \citep{Birnbaum68}. IRTs became popular in the United States through the book of \citet{lord1968}. Other extensions include models for items with several ordered categories~\citep{masters1982,samejima1969}, and models with continuous data such as the 2PL model with beta distributions~\citep{noel2007beta}. By now, IRT models are widely used for developing and scoring tests. For instance, large-scale assessments such as PISA~\citep{OECD2009,PISA} and the Trends in International Mathematics and Science Study (TIMSS)~\citep{vondavier2020} use IRT models for scoring responses, making them comparable between students who received different sets of items.

\paragraph{IRT in Machine Learning} To the best of our knowledge there are no rigorous theoretical guarantees on algorithms for learning the latent parameters of IRT models. Recently, IRT models have been used as a tool for analyzing machine learning classifiers~\citep{MartinezPMH19}. An extension building on beta distributions is the $\beta^3$-model by \citet{ChenFPDF19} introduced and applied to assess the ability of machine learning classifiers. IRT was also introduced to ensemble learning~\citep{chen2020item}. Recently, an IRT based analysis of regression algorithms and problems was suggested by \citet{munoz2021instance}.  \citet{martinez2022ai} proposed an empirical estimation for the difficulty of AI tasks using IRT models.

\paragraph{Coresets for Logistic Regression} \citet{ReddiPS15} used gradient-based methods to construct coresets for logistic regression, though without a bound on their size. Later, \citet{HugginsCB16} applied the framework of sensitivity sampling \citep{LangbergS10} noting that there are instances that require linear size to be approximated. \citet{MunteanuSSW18} proved that compression below $\Omega(n)$ is not possible in general. They developed the first provably sublinear coresets for logistic regression on \emph{mild} inputs $X$ of size $n$ and dimension $d$, introducing a data dependent parameter $\mu(X)$ to capture the complexity of compressing the data. This enabled a parameterized analysis giving a coreset, which for a given parameter $\eps\in(0,1/2)$ provides a multiplicative approximation factor of $(1+\varepsilon)$ within size $\tilde O(\mu^3 d^3 / \eps^4)$, hiding polylogarithmic terms in $n$. This was recently improved to $\tilde O(\mu^2 d/\eps^2)$ \citep{MaiMR21} by importance subsampling using $\ell_1$ Lewis weights as a replacement for the previous square root of $\ell_2$-leverage scores. More recently, it was extended to a single pass online algorithm along with a lower bound claiming linear dependence on $\mu$ \citep{WoodruffY23}. Coresets for logistic regression were recently extended to $p$-generalized probit models \citep{MunteanuOP22} giving the first coresets in this line whose size are independent of $n$. There are further extensions to a certain class of near-convex functions \citep{TukanMF20} and to monotonic functions \citep{TolochinskyJF22}.

\section{PRELIMINARIES}
\seclab{preliminaries}

\paragraph{IRT Models}
There are various IRT models that are employed in the literature, mainly differing in their number of parameters used to describe the characteristics of examinees and items, respectively. Although an examinee can in principle be described using multiple parameters, a common choice is only one \emph{ability} parameter, denoted $\theta_j$ for examinee $j\in [n]$. The number of parameters describing item characteristics varies more distinctively across IRT models, building or generalizing one over the other. The simplest of all is the Rasch model, named after its inventor \citep{rasch1960}, and is mathematically equivalent to the 1PL model. Here, one only takes into account how the {ability} $\theta_j$ differs from the \emph{difficulty} $b_i$ of solving item $i$, expressed in units of the ability parameter $\theta_j$ \citep{BakerKim04}. The 2PL-model, introduced by \citet{Birnbaum68}, is a basic model that is most commonly used. It describes item $i$ introducing a \emph{discrimination} or scale parameter $a_i$ in addition to its difficulty. The next step in this sequence of generalizations is adding to each item a \emph{default guessing} parameter $c_i$, which leads us to the 3PL model. We note that there exist even more general 4PL models \citep{Barton1981upper}. In this paper, however, we do not go into details about more general models than 3PL. Putting all parameters together in a probabilistic model, we arrive at the item characteristic curve (ICC)\footnote{The exponent in the ICC is often defined as $-a_i(\theta_j-b'_i)$. Rescaling $b'_i = b_i/a_i$ (note $a_i>0$) yields our definition.} 
specifying the probability of passing test item $i$ depending on the ability para\-meter $\theta_j$:
\begin{align}
\label{eqn:3pl:def}
p_i(\theta_j) = c_i + \frac{1-c_i}{1+\exp (-a_i\theta_j+b_i)},
\end{align}
The probability of an incorrect answer is consequently
\begin{align}
\label{eqn:3pl:def:comp}
1-p_i(\theta_j) = \frac{1-c_i}{1+\exp (a_i\theta_j-b_i)}.
\end{align}
We note that this defines a logistic sigmoid curve, see \cref{fig:ICC}, with a lower asymptote of $c_i\geq 0$. 

\noindent 
We describe the interpretation of the parameters corresponding to an item $i$:
\begin{itemize}
    \item The discrimination parameter $a_i$ specifies how flat or steep the curve ascends from $c_i$ to $1$. For example, a very steep ascend indicates that the item is nearly unsolvable unless the examinee has gained a special competence or knowledge. A knowledgeable examinee, however, is nearly guaranteed to pass the item. A flat curve indicates that the examinee needs to learn the necessary competences and gain some 'experience' in solving the task.
    \item The difficulty parameter $b_i$ specifies the threshold where passing or failing the item have equal $0.5$ probability (when $c_i=0$). Examinees with a significantly smaller ability $\theta_j$ have a low probability of passing, while those with a much larger ability have a high probability of passing.
    \item Finally, the guessing parameter $c_i$ indicates the probability of passing, say a multiple choice item, by randomly answering the question without having any knowledge or ability for solving the task.
\end{itemize}

In the special case of $c_i=0$ for all $i$, \cref{eqn:3pl:def} simplifies to the 2PL model and further constraining $a_i=1$ for all $i$ yields the 1PL (Rasch) model.

\begin{figure}[ht!]
\caption{Item Characteristic Curve examples
}
\label{fig:ICC}
\begin{center}
    \includegraphics[width=0.65\linewidth]{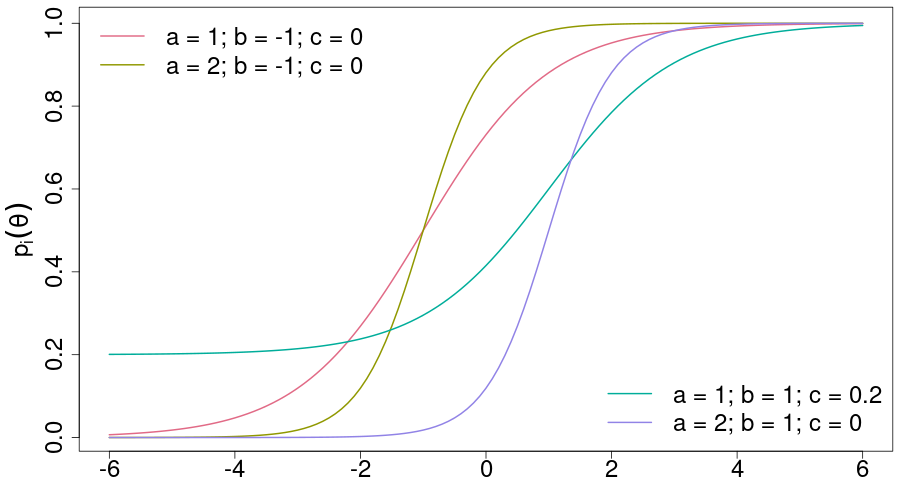} 
\end{center}
\end{figure}
The 2PL parameters are in principle unbounded, i.e., $a_i, b_i\in \mathbb{R}$, though we may safely assume that $a_i>0$ to account for the reasonable fact that with growing ability it becomes more likely to solve an item, but the reverse situation never occurs. Another prior knowledge that we may assume for the additional guessing probability is that $c_i \in [0,0.5)$ since we do not want a randomly answered item to be solved with higher probability than a coin flip. In practical settings where we encounter multiple choice items we may often assume a lower bound such as $c_i>c_{\min}=1/\kappa$, where $\kappa$ is the number of offered choices.

The difficulty in learning IRT models as introduced above comes from the fact that all parameters are unobserved latent variables, meaning that they are neither given nor explicitly observed. The data only consists of binary observations\footnote{Some literature specifies labels in \{0,1\}.}
$Y_{ij}\in\{-1,1\}$, indicating for item $i\in[m]$ and examinee $j\in[n]$ whether the item was answered correctly $Y_{ij}=1$ or not  $Y_{ij}=-1$. For notational convenience, we let the data be arranged in a matrix $Y=(Y_{ij})_{i\in [m], j\in [n]}\in \{-1,1\}^{m\times n}$.

We stress that our coreset results are quite general in that they approximate the IRT model, and their use is not restricted to a specific algorithm. Nevertheless, we choose to build and evaluate our coresets on the following classical approach due to its high significance in standard undergraduate courses in psychology.
Learning the latent parameters of IRT models involves a non-convex joint maximum likelihood optimization problem that encounters identifiability problems \citep{sanmartin2015}. Due to the fact that the parameter space increases with the sample size, we need to condition on one set of parameters to optimize for the other. This yields an alternating two-step optimization approach that operates as follows \citep[cf.][]{BakerKim04}:

\bigskip

\phantomsection
\label{algo}
\noindent
\textbf{General Algorithmic IRT Framework}
\begin{enumerate}
    \item Initialize all latent parameters.
    \item While termination criterion is not met:
    \begin{enumerate}
        \item Learn the ability parameters, given fixed item characteristics.
        \item Learn the item characteristics, given fixed ability parameters.
    \end{enumerate}
\end{enumerate}

Starting from a proper initialization, the algorithm optimizes one set of parameters given the other until convergence (to a local optimum) is detected or a given iteration budget is exhausted.
It is noteworthy that in the case of a 2PL IRT model, the two conditional optimization subproblems are not only convex but correspond exactly to standard logistic regression problems in two dimensions. The 3PL model, however, is more challenging, since it involves optimization over a combination of unbounded logistic loss functions as well as bounded non-convex sigmoid functions. We will elaborate on this in \secref{coresets:models} below.

\paragraph{Coresets for the IRT Framework}
Given massively large input data and a potential solution to an optimization problem, it is often already prohibitively expensive to evaluate or even to optimize the loss function with respect to the entire input. In such situations, it is preferable to have a much smaller subset of the data, such that solving the optimization problem on this small summary gives us an accurate approximate solution compared to the result obtained from analyzing the entire data.

This leads us to the concept of \emph{coresets} that we want to compute in order to make the optimization steps 2(a) and 2(b) scalable to large data. Both can be treated similarly. For the sake of presentation, we thus focus on the optimization in step 2(b) since in most natural settings the number of examinees exceeds the number of items, i.e. $n\gg m$. The optimization step 2(b) can be decomposed into $m$ independent instances, indexed by $i\in[m]$, of the following form, each summing over the huge number of $n$ examinees:
$f_w(X\eta_i) = \sum_{j\in [n]} w_j g(x_j\eta_i),$
where $X$ is an $n\times d$ matrix comprising the currently fixed ability parameters as \emph{row} vectors $x_j \in \mathbb{R}^{d}$, along with their corresponding labels $Y_{ij}$ from the data matrix, $\eta_i \in \mathbb{R}^{d}$ are vectors comprising the item characteristic parameters to be optimized in the current iteration, and $w \in \mathbb{R}^{n}$ is a vector of non-negative weights that is dropped from the notation whenever all weights equal $w_j=1$. 

A significantly smaller subset $K\subseteq X, k:=|K|\ll|X|$ together with corresponding weights $u\in \mathbb{R}^{k}$ is a $(1+\varepsilon)$-coreset for $X$ if it satisfies that 
\begin{equation}
\label{eqn:coreset:property}
\forall \eta\in\REAL^d\colon \lvert f_{w} (X \eta) - f_{u}(K \eta)\rvert\leq \varepsilon\cdot f_{w}(X \eta).
\end{equation}
We refer to \defref{coresets} in the appendix for details. 
Intuitively, a coreset evaluates for each possible solution to the same value as the original point set up to a factor of $(1\pm\eps)$, and moreover it implies that the minimum obtained from optimizing over the coreset is within a $(1+O(\eps))$ approximation to the original optimum (see \lemref{coreset:error:approx}), while the memory and computational requirements are significantly reduced.

Unfortunately, $(1+\varepsilon)$-coresets of size $k\ll n$ cannot be obtained for the logistic regression problem in general. Thus, such coresets can neither exist for 2PL IRT models, nor for 3PL models. To facilitate an analysis beyond the worst case, a data dependent parameter $\mu$ was introduced by \citet{MunteanuSSW18}, which can be used to bound the size of data summaries with the above accuracy guarantees and thus it enables a formal analysis and construction of small coresets for the logistic regression problem, as well as for other related problems.
Their original definition will suffice for the 2PL model.

Here, we extend the definition slightly to impose that additionally to the $\ell_1$-norm ratio between the positive and the negative entries, also their fraction in terms of $\ell_0$-norm\footnote{The case $p=0$ is often abusively referred to as a norm in the literature. 
} is bounded, i.e., the ratio of the number of positive and negative entries. This will be needed in our extension to 
the 3PL model. We let for $p\in\{0,1\}$\footnote{We note that $\mu$-complexity has been generalized to arbitrary $p \in \{0\} \cup [1,\infty)$ \citep{MunteanuOP22,TukanMF20}. Here, we require only the cases $p\in\{0,1\}$.}
 \begin{align*}
 \mu_p(X)=\sup_{\eta\in\REAL^d\setminus\lbrace 0\rbrace} \frac{\sum_{x_i\eta \geq 0 }|x_i\eta|^p}{\sum_{x_i\eta <0 }|x_i\eta|^p} = \sup_{\eta\in\REAL^d\setminus\lbrace 0\rbrace} \frac{\|(X\eta)^+\|_p}{\|(X\eta)^-\|_p}
\end{align*}
and say $X$ is $\mu_p$-complex if $\mu_p(X)\leq \mu_p$ for a bounded $1\leq \mu_p < \min\{m,n\}$. We say $X$ is $\mu$-complex if $\max\{\mu_0,\mu_1\}\leq \mu < \min\{m,n\}$.
It follows that 
\begin{align}
\label{mu:property}
\|(X\eta)^-\|_p /\mu \leq \|(X\eta)^+\|_p \leq \mu\cdot \|(X\eta)^-\|_p.
\end{align}
For the left hand side inequality, note that for every $\eta$ the supremum also considers $-\eta$, for which the roles of positive and negative entries are reversed.

\paragraph{Constructing Coresets}
Recall that the loss functions that we encounter when we train IRT models are defined as sums of individual point-wise losses. It is well-known from the related work on logistic regression that the multiplicative approximation guarantees provided by coresets cannot be obtained by uniform sampling. We elaborate on this with a focus on IRT in~Appendix~\ref{appendix:uniform} for completeness of presentation.
 
A common method for obtaining coresets to approximate such functions by importance sampling is called the \emph{sensitivity framework} that was introduced by \citet{LangbergS10}. They defined the sensitivity of an input point as their worst case individual contribution to the entire loss function. The sensitivity of a point $x_j$ for the function $f_w(X\eta)=\sum_{j\in [n]} w_j g(x_j\eta)$ is 
\begin{align*}
\sigma_j = \sup_\eta \frac{w_j g(x_j\eta)}{f_w(X\eta)}.
\end{align*}
This was subsequently combined with the theory of VC dimension to obtain a meta-theorem. It states that we can take a properly reweighted subsample using sampling probabilities that are proportional to the sensitivities. This yields a $(1+\eps)$-coreset if its size is taken to be $k=O(\frac{S}{\eps^2}(\Delta\log S + \log \frac{1}{\delta}))$. Here $S=\sum_{j\in[n]} \sigma_j$ denotes the total sensitivity, $\Delta$ denotes the VC dimension of a set system derived from the functions $g(x_i\eta)$, and $\delta$ is the failure probability \citep{FeldmanSS20}. One complication, however, is that computing the exact sensitivities is usually as hard as solving the problem under study. Fortunately, any upper bounds on the sensitivities suffice as a replacement. However their overestimation should be controlled carefully since the total sensitivity grows and is an important parameter that determines the coreset size. Further details on the sensitivity framework are in \cref{sec:sensitivityframework}. In the following we can assume that the problem of constructing coresets reduces to bounding the VC dimension and estimating the sensitivities for the functions under study.

\section{CORESETS FOR IRT MODELS}
\seclab{coresets:models}
\subsection{2PL Models}
\seclab{2PL:model}

For a suitable presentation of our technical results on coresets for IRT models, we use the following notation. For the item parameters, we define vectors $\alpha_i=(a_i,b_i)^T, i\in[m]$ and similarly we define for the examinees $\beta_j=(\theta_j,-1)^T, j\in[n]$ and collect them in matrices $A= \begin{bmatrix} \alpha_1&\ldots&\alpha_m\end{bmatrix}\in \REAL^{2\times m}$ and $B=\begin{bmatrix}\beta_1&\ldots&\beta_n\end{bmatrix}\in \REAL^{2\times n}$. Now, given the item characteristics and the ability parameters, the probability of observing the data matrix $Y$ can be rewritten as
\begin{align}
\label{eqn:2pl:prob}
\Pr{Y|A,B} = \prod\nolimits_{i\in [m], j\in [n]} \frac{1}{1+\exp( - Y_{i j} \alpha_i^T \beta_j )}.
\end{align}
To compute a joint maximum likelihood estimate of the item and ability parameters, a basic approach is to fix one set, say the item parameters $A$, and optimize over the ability parameters $B$, and then switch their roles. This process is repeated in an alternating manner \citep{BakerKim04} as we introduced in the general algorithmic IRT framework, see \cref{algo}. This leads us to minimizing the following negative log-likelihood function switching back and forth between the roles of data and variables: 
\begin{align*}
f(A\mid B) = \sum\nolimits_{i\in[m], j\in [n]} \ln( 1+\exp( -Y_{i j} \alpha_i^T \beta_j)) = f(B\mid A ).
\end{align*}
In particular, for a given fixed $B\in \REAL^{2\times n}$, we can write $x_j=-Y_{i j}\beta_j^T$ for every $j\in [n]$, and then set $X_{(i)}=( x_j)_{j\in[n]} \in \REAL^{n\times 2}$ for each $i\in [m]$ to optimize for
\begin{align}
\label{eqn:2pl:logistic:b}
\min\nolimits_{\alpha_i\in \REAL^2} \sum\nolimits_{j\in [n]} \ln ( 1+ \exp( x_j \alpha_i)).
\end{align}
By symmetry, for a given fixed $A\in \REAL^{2\times m}$, we can write $x_i=-Y_{i j}\alpha_i^T$ for every $i\in [m]$, and set $X_{(j)}=( x_i)_{i\in[m]} \in \REAL^{m\times 2}$ for each $j\in [n]$ to optimize for
\begin{align}
\label{eqn:2pl:logistic:a}
\min\nolimits_{\beta_j\in \REAL^2} \sum\nolimits_{i\in [m]} \ln ( 1+ \exp( x_i \beta_j)).
\end{align}
Note that the objective functions given in \cref{eqn:2pl:logistic:b,eqn:2pl:logistic:a} 
are equivalent to plain logistic regression \citep[cf.][]{MunteanuSSW18}, where coresets for logistic regression were constructed using the sensitivity framework.
To obtain an upper bound on the sensitivity of the input, the authors related the single contributions of input points $x_j$ to the square root of the so called $\ell_2$-leverage scores: $l_j = \sup_{\eta\in \REAL^d \setminus \lbrace 0\rbrace} {|x_j \eta|^2}/{\| X\eta\|_2^2}\;,$
a measure that can be derived from the row norms of an orthonormal basis for the space spanned by the data matrix, see \defref{leverage:scores} and \lemref{lem:leverageequiv} for details.

However, in \citep{MunteanuSSW18}, the label vector $Y$ was a fixed vector in $\REAL^{n}$, while here, $Y$ is a matrix in $\REAL^{m\times n}$, i.e., we have to deal with a different label vector for each item, respectively for each ability parameter, that is fixed in one iteration, and thus the matrices $X_{(i)}$ differ across a large number of iterations. Fortunately, the leverage scores -- only depending on the spanned subspace, not on its representation -- are invariant to sign flips as we show in the next lemma.
\begin{lemma}
\lemlab{leverage:invariant}
Suppose we are given a matrix $X\in \REAL^{m\times n}$ (for any $m,n\in \mathbb{N}$) and an arbitrary diagonal matrix $D=(d_{i j})_{i\in [m], j\in [m]}$, with $d_{i j}\in \lbrace -1,1\rbrace$ if $i=j$, and $d_{i j}=0$ otherwise. Then 
the leverage scores of $X$ are the same as the leverage scores of $DX$.
\end{lemma}
This insight allows us to use the square root of the $\ell_2$-leverage scores of $A$, respectively $B$, as a fixed importance sampling distribution across \emph{all} iterations where the same latent parameter matrix is involved as a fixed 'data set' even though the signs may arbitrarily change in each iteration.
Let us consider the optimization problem in \cref{eqn:2pl:logistic:b}\footnote{The subsequent discussion also applies verbatim to the problem in \cref{eqn:2pl:logistic:a}.}. Here, we are given the ability parameter matrix $B\in \REAL^{2\times n}$ and the label matrix $Y\in\REAL^{m\times n}$.
We can directly use Theorem 15 of 
\citep{MunteanuSSW18}, for logistic regression in $d=2$ dimensions (with uniform weights) to get a small reweighted coreset for each optimization of an $\alpha_i\in \REAL^2$.
To this end, we approximate the $\ell_2$-leverage scores $l_j,j\in[n]$ of $B$ and sample a coreset proportional to $\sqrt{l_j} + 1/n$, where $\sqrt{l_j}$ captures the importance of coordinates with a large linear contribution, and the augmented uniform $1/n$ is useful to capture small elements near zero that can dominate when their number is large, since their logistic loss is bounded below by a nonzero constant. As in \citep{MunteanuSSW18}, this yields a coreset whose size is dominated by an $O(\sqrt{n})$ factor which can be repeated recursively $O(\log\log n)$ times to decrease the dependence to $\operatorname{polylog}(n)$.
Moreover, by \lemref{leverage:invariant} it suffices to sample one single coreset that is valid across all iterations $i\in [m]$ optimizing for $\alpha_i$ and whose size is only inflated by an additive $\log(m)$ term to control the overall failure probability using a union bound over the $m$ iterations. This yields the following theorem. 

\begin{theorem}
\thmlab{2pl:main:result}
Let $X_{(i)}=( -Y_{ij} \beta_j^T)_{j\in[n]}\in\REAL^{n\times 2}$ be $\mu_1$-complex, for each $i\in[m]$. Let $\varepsilon\in (0,1/2)$. There exists a weighted set $K\in \REAL^{k\times 2}$ of size\footnote{The $\tilde{O}$ notation omits $o(\log n)$ terms for a clean presentation. The full statements can be found in the appendix.} $k\in \tilde O(\frac{\mu^3}{\varepsilon^4}(\log(n)^4 + \log(m))$, that is a $(1+\varepsilon)$-coreset simultaneously for all $X_{(i)}$, $i\in[m]$ for the 2PL IRT problem. The coreset can be constructed with constant probability and in $\tilde O(n)$ time.  
\end{theorem}
We note that despite the fact that there are more recent theoretical improvements such as \citep{MaiMR21,MunteanuOP22}, we build our results on the techniques of 
\citet{MunteanuSSW18}. Even though an analogue of 
\lemref{leverage:invariant} can be proven for the scores of these references, the practical performance of the classic result is often better or only slightly worse than the competitors and at the same time it is significantly faster to compute \citep[cf.][]{MaiMR21,MunteanuOP22}. 
Recent advances \citep{WoodruffYasudaICML23} also improve theoretical bounds for the \emph{root leverage scores} of \citet{MunteanuSSW18}, which partially explain and corroborate their success in practical applications, though in a different setting from ours.

\subsection{3PL Models}
\seclab{3PL:model}

An often addressed concern about 3PL IRT models is the difficulty to properly estimate the guessing parameter $c_i$ \citep{BakerKim04}, since it is hard to distinguish between sufficiently high abilities, and a large guessing probability. Different to the 2PL model, the subproblem of optimizing the item characteristics, conditioned on fixed ability parameters is already non-convex. Thus, parameter estimation is significantly more challenging\footnote{Indeed, parameters are not identifiable \citep{sanmartin2015}.} 
and can greatly benefit from an input size reduction. To this end, we now develop coresets for the 3PL model.

We would like to reduce the 3PL model to solving logistic regression problems, as we have done for the 2PL model, by first fixing the additional parameter $c_i$ in order to learn all other parameters $(a_i,b_i,\theta_j)_{i\in[m], j\in[n]}$ as before, and at the end of one iteration of the main loop fix the other parameters in the model to optimize only for $c_i,i\in[m]$. Unfortunately, if we would optimize the guessing parameter $c_i$ in this way, the optimizer would conclude that either\footnote{This can be any other upper or lower bound on $c_i \in \{c_{\min},c_{\max}\}$, but the problem remains the same.} $c_i=0$ or $c_i=1$ since the objectives are monotonic in $c_i$. Thus, we would never reach a realistic estimate for $c_i$. 

Using the notation of \secref{2PL:model}, we cannot rewrite \cref{eqn:3pl:def,eqn:3pl:def:comp}  
in a uniform way to express the probability of observing the label matrix $Y$ as in \cref{eqn:2pl:prob}. Although the guessing parameters $c_i$ are inseparable from the corresponding $a_i,b_i$ parameters during optimization, we denote them in a separate vector $C= ( c_1, \ldots, c_m)^T$. Then, we have that
\begin{align}
\label{eqn:3pl:prob}
\Pr{Y | A,B,C} = \prod\nolimits_{i\in [m], j\in [n]}^{[Y_{i j}=-1]} \left( \frac{1-c_i}{1+\exp( \alpha_i^T \beta_j )}\right) 
\times \prod\nolimits_{i\in [m], j\in [n]}^{[Y_{i j}=1]} \left( c_i + \frac{1-c_i}{1+\exp( -\alpha_i^T \beta_j )}\right),
\end{align}
where the products iterate only over all indexes in the subscript, that satisfy the condition in the superscript. Similar notations are used for the sums below.
Let $g_i(z)=-\ln( \frac{1-c_i}{1+\exp( z )}) = \ln( 1+\exp( z )) - \ln( 1-c_i)$ and $h_i(z)= - \ln( c_i + \frac{1-c_i}{1+\exp( -z )})$. The general algorithmic IRT framework with an alternating optimization, see \cref{algo}, that we already dealt with for the 2PL models, can be applied to the 3PL models as well for the following objective function 
\begin{align*}
f( A,C\mid B) = f(B\mid A,C )= \sum_{i\in [m], j\in [n]}^{[Y_{i j}=-1]} g_i(-Y_{i j}\alpha_i^T \beta_j ) + \sum_{i\in [m], j\in [n]}^{[Y_{i j}=1]} h_i( -Y_{i j}\alpha_i^T \beta_j ) .
\end{align*}
Let us assume that $A$ and $C$ are fixed, the other case will be addressed later. As in the case of 2PL we can write $x_i=-Y_{i j} \alpha_i^T$, for each $i\in[m]$, and $X_{(j)}=\left( x_i \right)_{i\in [m]}\in \REAL^{m \times 2}$. Then, we aim at minimizing for each $j\in [n]$ over $\beta_j\in\REAL^2$, the objective 
\begin{align}
\label{eqn:3pl:a}
f(\beta_j \mid A,C) = \sum\nolimits_{i\in [m]}^{\left[Y_{i j}=-1\right]} g_i( x_i \beta_j ) + \sum\nolimits_{i\in [m]}^{\left[Y_{i j}=1\right]} h_i( x_i \beta_j ).
\end{align}

For all $z$ it holds that $g_i(z)>0$ and $h_i(z)>0$. The functions $g_i(z)$ and $h_i(z)$ have different shapes and cannot be represented as a single function. In particular, all functions $g_i(z)$ are similar to the logistic regression loss up to an additive shift of $-\ln(1-c_i)$, with $0\leq -\ln(1-c_i)<\ln 2$, since $c_i\in [0,0.5)$.   
The others, $h_i(z)$, are sigmoid functions satisfying $0<h_i(z)<\ln\left(1/c_i\right)$, for all values of $z$. 

In the 3PL case, assuming that each matrix $X_{(j)}$ is $\mu_1$-complex does not give sufficient bounds for the distribution of input points to the two different types of functions. Therefore we split $X_{(j)}$ into submatrices $X_{(j)}'$, containing the rows indexed by $i$ with labels $Y_{ij}=-1$, and $X_{(j)}''$ containing the rows with $Y_{ij}=1$. Now, we assume that $X_{(j)}'$ and $X_{(j)}''$ are both $\mu$-complex, and $\sup_{\eta\in\mathbb{R}\setminus\{0\}}{\|X'_{(j)}\eta\|_1}/{\|X''_{(j)}\eta\|_1} \leq 2\mu_1.$

The detailed technical analysis is deferred to the appendix due to page limitations. Here, we only give a high level description. We first upper bound the sensitivities for both types of functions separately and show that the total sensitivity over all functions remains sublinear. To this end consider the set of (shifted) logistic functions $g_i$. Those can be handled using the $\mu_1$-complexity of $X'_{(j)}$ as in \citep{MunteanuSSW18} up to technical modifications and adjusting constants.

For the second set of sigmoid functions $h_i$ we use the $\mu_0$-complexity property of both sets to bound the total number of elements in $X'_{(j)}$ and $X''_{(j)}$ from below. This is needed to obtain uniform upper bounds for the sensitivities across all labelings, which together with \lemref{leverage:invariant} assures that one coreset suffices across all iterations $j\in[n]$. We further leverage the $\mu_0$-complexity of $X''_{(j)}$ to conclude that the fraction of positive elements in $X''_{(j)}\beta_j$ is sufficiently large. 

The final open issue is to bound the VC dimension. Again, we handle both sets of functions separately. Since both types of functions are strictly monotonic and invertible tranformations of a dot product, they can be related to a set of affine separators that have bounded VC dimension of $d+1=3$ \citep{Kearns1994introduction}. By a classic result of~\citet{BlumerEHW89} the VC dimension of the union of both sets of functions can be bounded by $O(d+1)$. Leveraging the disjointness of our sets, we can give a simpler proof that leads to a bound of $2(d+1)=6$. Another union over $O(\log m)$ weight classes concludes the VC dimension bound of $O(\log m)$.
This yields our second main result:
\begin{theorem}
\thmlab{3pl:main:result}
Let each $X_{(j)}=( -Y_{ij} \alpha_i^T)_{i\in[m]}\in\REAL^{m\times 2}$. Let $X'_{(j)}$ contain the rows $i$ of $X_{(j)}$ where $Y_{ij}=-1$ and let $X''_{(j)}$ comprise the rows with $Y_{ij}=1$. Let $X'_{(j)}$ and $X''_{(j)}$ be $\mu$-complex, and 
$\sup_{\eta\in\mathbb{R}\setminus\{0\}}\frac{\|X'_{(j)}\eta\|_1}{\|X''_{(j)}\eta\|_1} \leq 2\mu_1$ for each $j\in[n]$. Let $\varepsilon\in (0,1/2)$. There exists a weighted set $K\in \REAL^{k\times 2}$ of size 
$k\in O(\frac{\mu^2 \sqrt{m}}{\varepsilon^2} (\log(m)^2 + \log(n)))$, that is a $(1+\varepsilon)$-coreset for all $X_{(j)}$, $j\in[n]$ simultaneously for the 3PL IRT problem. The coreset can be constructed with constant probability and in $O(m)$ time.
\end{theorem}

The remaining case $f( A,C\mid B)$ requires another $\frac{\mu^2}{\eps}$ factor. The analysis is deferred to \cref{appendix:theory} due to page limitations. The discussion starts above \lemref{logistic:sum:lowerbound}. In addition, we provide a parameter estimation guarantee for $\tau$-PL, with 
$\tau\in\{2,3\}$:

\begin{theorem}[Informal version of \cref{thm:quality:coreset} in \cref{sec:quality:coreset}]\label{thm:l1error}
Assume the conditions of 
\cref{theo:2pl:main:result} resp. \cref{theo:3pl:main:result}. Then the optimal solutions for the $\tau$-PL problem, for $\tau\in\{2,3\}$, on the full input ($\eta_{\text{opt}}$) and on the coreset ($\eta_{\text{core}}$) satisfy $$ \|\eta_{\text{opt}} - \eta_{\text{core}} \|_1 \leq O(\mu^{\tau-1}) \cdot f(X \eta_{\text{opt}}). $$
\end{theorem}

\section{EXPERIMENTS}

All experiments were run on a HPC workstation with AMD Ryzen Threadripper PRO 5975WX, 32 cores at 3.6GHz, 512GB DDR4-3200.
Our Python code\footnote{Our Python code is available at \url{https://github.com/Tim907/IRT}.} implements the IRT framework introduced in \cref{algo} where Steps 2(a) and 2(b) solve Eq.~(\ref{eqn:2pl:logistic:b})~and~(\ref{eqn:2pl:logistic:a}), resp. their 3PL variants.  
The coreset is only computed in step 2(b) for reducing the number of examinees, i.e., the dominating dimension $n$, since the number of items $m$ is relatively small in our data; the coreset construction would dominate over analyzing the complete data.

\paragraph{Experimental Setup}
We focus on 2PL models, which can be estimated more stably, as discussed before. We generate synthetic 2PL/3PL data by drawing item and ability parameters for each $j\in[n]$, $i\in[m]$ from the following distributions: $a_i \sim N(2.75, 0.3)$ truncated at $0$, $b_i \sim N(0,1)$ and $\theta_j \sim N(0,1)$. For 2PL, we fix $c_i = 0$, and for 3PL, we truncate $c_i \sim N(0.1, 0.1)$ within $[0,0.5)$. The response probabilities $p_{ij}:=p_i(\theta_j)$ are computed as in  \cref{eqn:3pl:def}.  
Each label is drawn from a Bernoulli distribution with the corresponding response probability, i.e., $Y_{ij}\sim \operatorname{Bernoulli}(p_{ij})$.
{We also use real world data (see their dimensions in \cref{tab:results}): SHARE~\citep{boersch2022survey}, measuring health indication of elderly Europeans, \nocite{Boersch2005a,Boersch2005b,boersch2013data} and NEPS~\citep{NEPS,NEPS-SC4}, measuring high school abilities of ninth grade students.}\footnote{While PISA serves as a motivational example, their data is not available readily analyzable in one large batch.}

In our estimation algorithm, the ability parameters $\theta_j$ and the item difficulties $b_i$ are bounded by $b_i,\theta_j\in[-6,6]$, and the item discrimination parameters are bounded by $a_i\in (0,5]$. 
Without imposing identification restrictions, the scale of estimated IRT parameters $a$, $b$, and $\theta$ is arbitrary. Therefore, we rescale them to obtain standardized parameters. To this end, we subtract the mean of $\theta$ from each $b_i$, multiply $a_i$ by the standard deviation of $\theta$ and finally standardize $\theta$ to zero mean and unit variance.

We vary the number of examinees $n$, the number of items $m$, and the size of the coreset $k$. For every combination we run 50 iterations of the main loop. Each experiment is repeated 20 times. We report results for a few selected configurations in \cref{tab:results,fig:param_exp_appendix_pareto:main,fig:param_exp}. The majority of the results is in \cref{appendix:experiments}.

Since $\mu$ is a crucial complexity parameter, we estimate its value for all different data sets in \cref{appendix:estimation:mu}. The majority and mean values for $\mu$ are small constants ranging between 2 and 20. Only in rare cases $\mu$ takes large maximum values for some label vectors. We checked the corresponding labels, and found that the large values occur only in degenerate cases, in which the maximum likelihood estimator of the model is undefined, for example, when an item is solved by all or none of the students. 

\paragraph{Computational Savings}
The parameter estimation using coresets is significantly faster than using the full input set. The coresets use only a small fraction of the memory used by the full data, while approximating the objective function very closely.

For the 2PL models on  
the synthetic data sets, the running time gains were at least $32\,\%$ and up to $66\,\%$ (see \cref{tab:results,tab:results_appendix}). At the same time, the amount of memory used never exceeds $1\,\%$ of the original size.
The largest instances we found across the literature are $n\approx 500\,000$ \citep{PISA} and $m\approx 5\;000$ \citep{munoz2021instance}. We added a synthetic example of this size whose total running time (for a single repetition) was reduced from $6.5$ to $3.8$ \emph{days}. Besides running time, the memory spent for this large experiment is larger than 5\,GB, impossible to be handled by standard psychometric tools.

For the real-world data sets, SHARE~\citep{boersch2022survey} and NEPS~\citep{NEPS-SC4}, we show that a relative error of $\hat\varepsilon=0.05$ can be achieved using less than $6\,\%$ of the memory used when working on the full data. For the (relatively small) NEPS data set, the running time gain was about $30\,\%$, except when the coreset sizes exceed half of the input size. We note that for the SHARE data set, the running time gains are small, and can even be (slightly) negative. This is due to its very small original dimensions (especially $m=10$), for which the time for the coreset construction can dominate the overall running time.

For 3PL models, solving the original problem is more difficult and thus takes longer. Indeed, the subproblems estimating the sets of parameters in each phase are non-convex and cause the computational issues discussed in \cref{sec:3PL:model}. As a consequence, reducing the input size increases the running time gain up to $86\,\%$ (see \cref{tab:results,tab:results_appendix4}). {The memory used by the coresets is between $5\,\%$ and $20\,\%$ of the original data.}

The data dimensions considered across our experiments are huge compared to data that is usually collected for IRT studies. On the other hand, even the largest data dimensions, are chosen small enough to be able to estimate the models on the full data set. However, our theoretical results prove that the subsample grows sublinearly with arbitrarily increasing data, showing the potential for larger future data.

\begin{figure}[ht!]
\caption{2PL Experiments on real world SHARE and NEPS data: Coreset sizes vs. relative error and mean absolute deviation (MAD), cf.~\cref{tab:results_appendix3:b,fig:param_exp_appendix_pareto}.
}
\begin{center}
\includegraphics[width=.65\columnwidth]{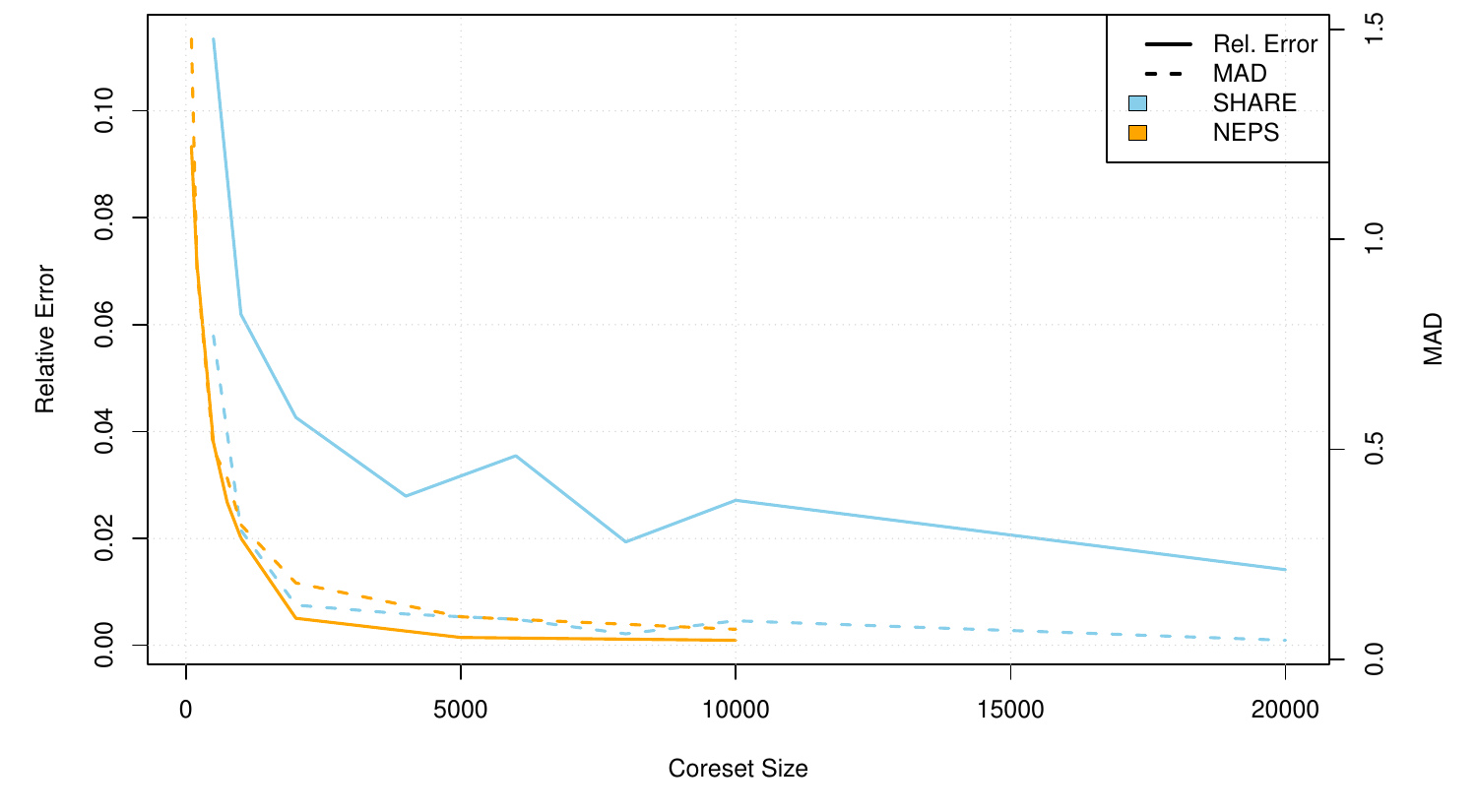} 
\label{fig:param_exp_appendix_pareto:main}
\end{center}
\end{figure}

\begin{table*}[ht!]
\caption{{Mean running times (in minutes) taken across $20$ repetitions (of $50$ iterations of the main loop) per data set 2-/3-PL, (Sy)nthetic, SH(ARE), NE(PS), for different configurations of their data dimensions: number of items $m$, number of examinees $n$, and coreset size $k$. } The (relative) gain is defined as  $(1-\mathrm{mean}_{\textsf{core}}/\mathrm{mean}_{\textsf{full}})\cdot 100\,\%$. For the quality of the solutions, let $f_{\sf full}$ and $f_{{\sf core}(j)}$ be the optimal objective values on the input and on the coreset for the $j$-th repetition, resp. Let $f_{\sf core} = \min_j f_{{\sf core}(j)}$.
Relative error: \textbf{r.err.} $\hat{\varepsilon}=|f_{\sf core} - f_{\sf full}|/f_{\sf full}$ (cf. \lemref{coreset:error:approx}).
Mean Absolute Deviation: $\textbf{mad}(\alpha)=\frac{1}{n}\sum (|a_{\textsf{full}}-a_{\textsf{core}}| + |b_{\textsf{full}}-b_{\textsf{core}}| + |c_{\textsf{full}}-c_{\textsf{core}}| )$; $\textbf{mad}(\theta)=\frac{1}{m}\sum |\theta_{\textsf{full}}-\theta_{\textsf{core}}|$, evaluated on the parameters {attaining}
the optimal $f_{\textsf{full}}$ and $f_{\sf core}$.
}
\label{tab:results}
\begin{center}
\begin{tabular}{ c l| r r r | r c c }
    {\bf data} &
		{$\mathbf n, \mathbf m, \mathbf k$}  &  {$\textbf{mean}_{\textsf{full}}(\text{min})$} & {$\textbf{mean}_{\textsf{core}}(\text{min})$} & {\bf gain} & {\bf r.err. $\hat{\varepsilon}$}  & {\bf $\text{mad}(\alpha)$} & {\bf $\text{mad}(\theta)$} \\ \hline 
    2PL-Sy & $50\,000, 500, 500$ & $136.981$ & $45.547$ & $66.749\,\%$ & $0.04803$ & $0.525$ & $0.008$ \\ \hline
    2PL-Sy & $100\,000, 200, 1\,000$ & $122.252$ & $61.459$ & $49.727\,\%$ & $0.03404$ & $0.379$ & $0.008$ \\ \hline
    2PL-Sy & $500\,000, 500, 5\,000$ & $1\,278.845$ & $591.878$ & $53.718\,\%$  & $0.01445$ & $0.171$ & $0.001$ \\ \hline
    2PL-Sy & $500\,000, 5\,000, 5\,000$ & $9\,363.750$ & $5\,536.684$ & $40.871\,\%$ & $0.00076$ & $0.120$ & $0.013$ \\ \hline    
    2PL-SH & $138\,997, 10, 8\,000$ & $28.853$ & $27.637$ & $4.216\,\%$ & $0.01935$ & $0.061$ & $0.007$ \\ \hline
    2PL-NE & $11\,532, 88, 1\,000$ & $5.968$ & $4.009$ &  $32.829\,\%$ & $0.02007$ & $0.320$ & $0.045$ \\ \hline	
    3PL-Sy & $50\,000, 100, 10\,000$ &  $211.468$ & $93.780$ & $55.653\,\%$ & $0.00212$ & $0.384$ & $0.010$ \\ \hline
    3PL-Sy & $50\,000, 200, 10\,000$ & $369.816$  & $145.674$ &  $60.609\,\%$ & $0.02186$ & $0.488$ & $0.001$ \\ \hline
    3PL-Sy & $200\,000, 100, 10\,000$ &  $893.183$ & $196.802$ & $77.966\,\%$ & $0.01789$ & $0.524$ & $0.003$ \\ \hline
\end{tabular}
\end{center}
\end{table*}

\begin{figure*}[ht!]
\caption{Parameter estimates for the coresets compared to the full data sets. The first row shows the bias for the item parameters $a,b$ (and $c$ for 3PL). The vertical axis is scaled to display $2\,{\mathrm{std.}}$ ($4\,{\mathrm{std.}}$ for 3PL) of the parameter estimate obtained from the full data set. The second row shows a kernel density estimate for the ability parameters $\theta$, standardized to zero mean and unit variance, with a LOESS regression line in dark green.}
\label{fig:param_exp}
\begin{center}
\begin{tabular}{lll}
{\tiny{\textbf{2PL Syn:}}} {\tiny{$\mathbf{n=500\,000,m=500,k=5\,000}$}}&{\tiny{\textbf{2PL NEPS:}}} {\tiny{$\mathbf{n=11\,532,m=88,k=1\,000}$}}&{\tiny{\textbf{3PL Syn:}}} {\tiny{$\mathbf{n=50\,000,m=100,k=10\,000}$}}
\\
\includegraphics[width=0.312\linewidth]{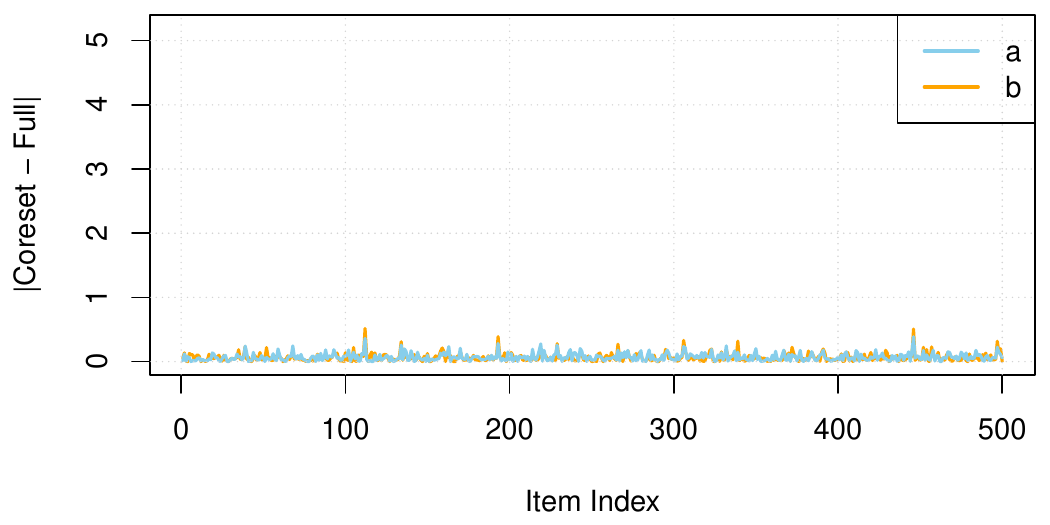}&
\includegraphics[width=0.312\linewidth]{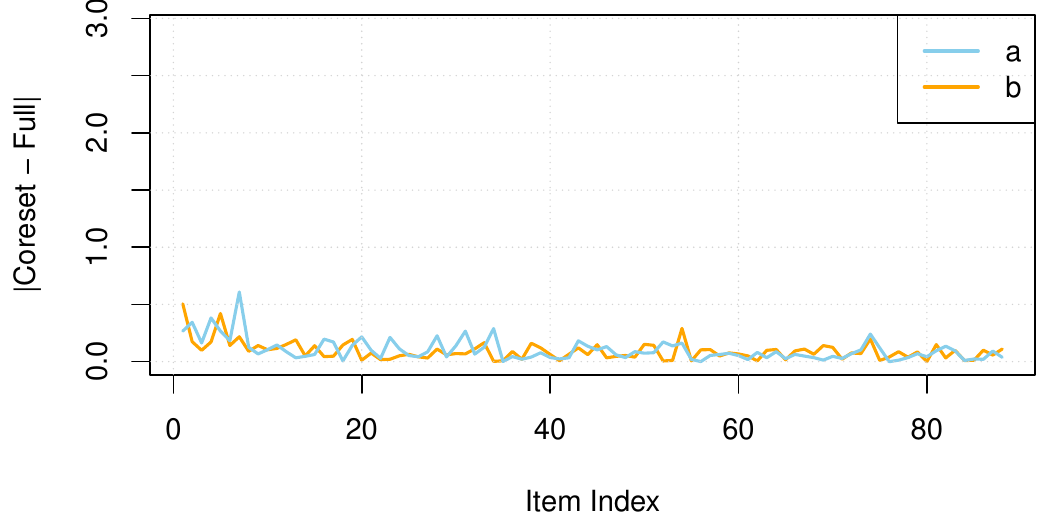}&
\includegraphics[width=0.312\linewidth]{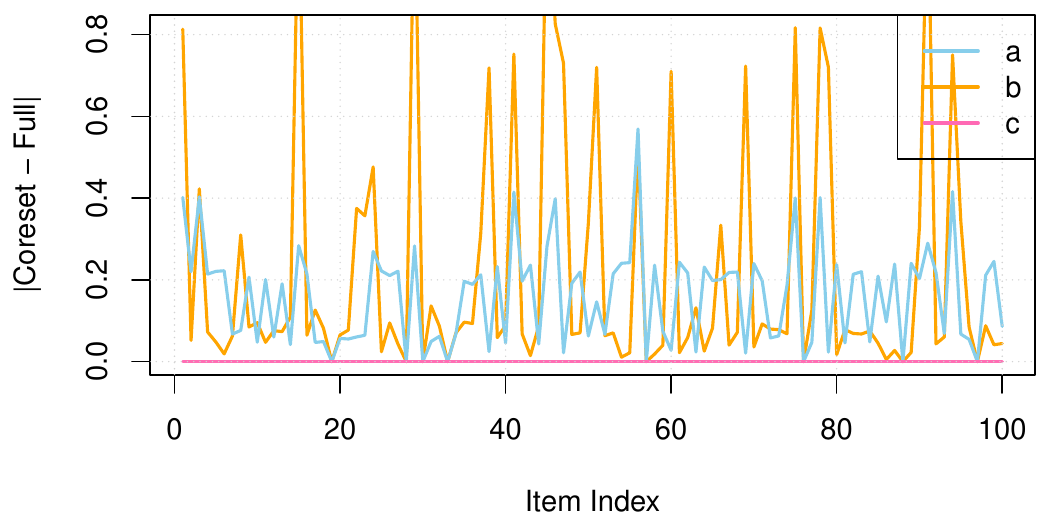}
\\
\includegraphics[width=0.312\linewidth]{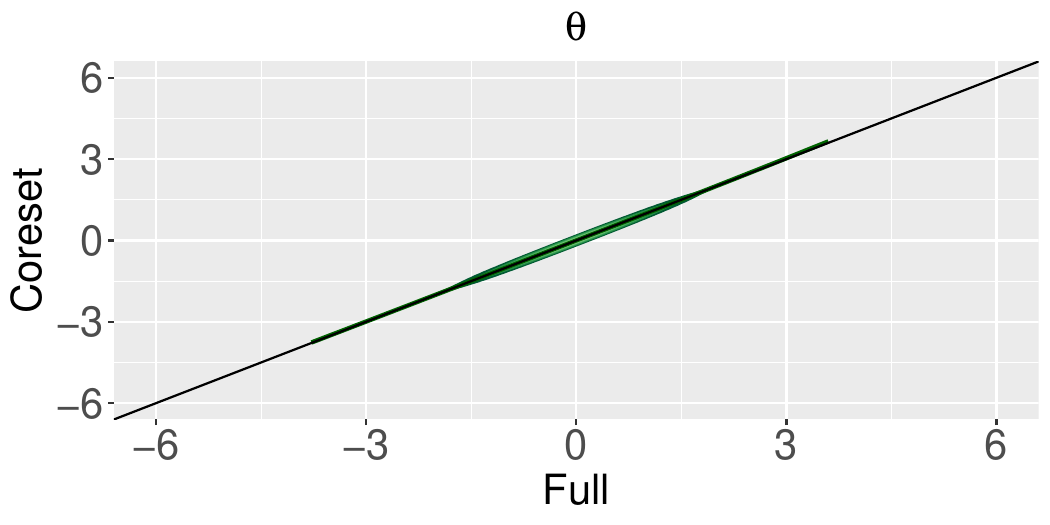}&
\includegraphics[width=0.312\linewidth]{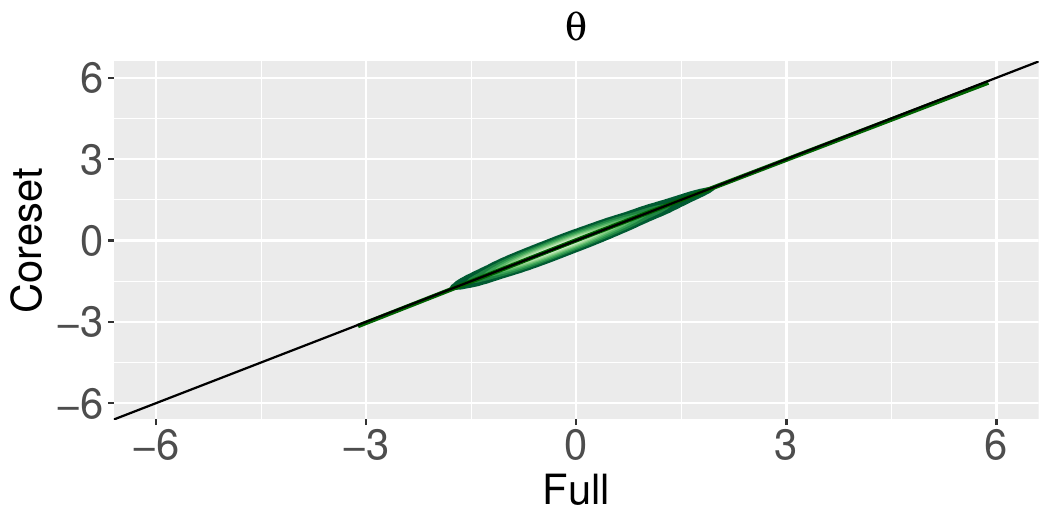}&
\includegraphics[width=0.312\linewidth]{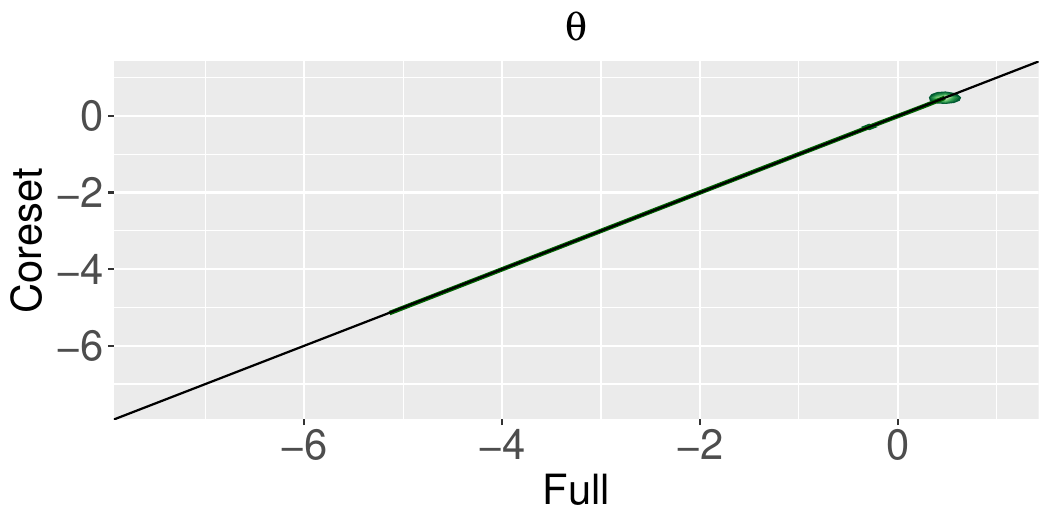}\\
\end{tabular}
\end{center}
\end{figure*}

\paragraph{Parameter Estimation Accuracy} Overall, we find that incorporating coresets leads to comparable estimates as on the full data set. The differences are larger for 3PL.
The bounded $\ell_1$ norm deviation 
(see \cref{thm:l1error}/\ref{thm:quality:coreset})
explains that either small errors are evenly distributed over many parameters, or large deviations affect only a few spikes. The accuracy clearly improves with increasing coreset size, cf.~\cref{tab:results,fig:param_exp_appendix_pareto:main}, and \cref{appendix:experiments}, especially~\cref{tab:results_appendix:b,fig:param_exp_appendix_pareto}. Our coresets compare favorably against the results obtained from uniform sampling, and clustering coresets as baselines, cf. \cref{appendix:uniform,appendix:clustering}. They also compare similarly to $\ell_1$ Lewis weights and $\ell_1$ leverage scores, see \cref{sec:app:L1scores}.

For the 2PL models, the bias for the parameters estimated on the coresets in comparison to the full data sets are small and negligible in comparison to the scale of the parameter, see \cref{fig:param_exp}. For the 3PL models, the bias is larger. This is because the item parameters of the 3PL model are not identifiable \citep{sanmartin2015} in the estimation approach, where even the sub-problems are non-convex. In this case, the coresets and the full data set (or, similarly, different starting values) may lead to different parameter estimates although they have a similar likelihood. Indeed, the close likelihood approximation provided by coresets not only mimics good model fit. Even when the model fits badly, it ensures that a proper diagnosis for detecting misspecification can be performed on coresets. For the ability parameters $\theta$ in 2PL models, the estimates are almost identical between the coresets and the full data. For 3PL, the estimates are bi-modal due to multiple local optima (\cref{fig:param_exp}, bottom right).

\section{CONCLUSIONS}

We develop coresets to facilitate scalable and efficient learning of large scale Item Response Theory models. Coresets enable significantly larger IRT studies and will hopefully motivate larger surveys. Our implementation and experiments illustrate that standard algorithms for IRT can greatly benefit from using coresets in the estimation process. We observe large computational savings as well as accurate parameter recovery on a small but carefully selected fraction of the large data. We note that in our experiments, estimates were recovered with negligible errors
when using coresets. Future research could incorporate coresets into state of the art IRT solvers that are more complicated than the standard approach but achieve much better estimation accuracy already on the original data. Further, it would be interesting to develop coresets for more general IRT models, including (ordered) categorical \citep{masters1982}, continuous \citep{ChenFPDF19}, multidimensional \citep{demars2016}, and multilevel \citep{adams1997} IRT models. Other interesting avenues are to extend to probit IRT models \citep{MunteanuOP22} or to incorporate sketching for logistic regression \citep{MunteanuOW21,MunteanuOW23,Munteanu23} such as to avoid storing the full latent parameter matrices.



\clearpage

\section*{Acknowledgements}
We thank the anonymous reviewers for their valuable comments. We thank Philipp Doebler for pointing us to IRT models. We thank Tim Novak and Rieke Möller-Ehmcke for their help with implementations and experiments.
The authors were supported by the project ``From Prediction to Agile Interventions in the Social Sciences (FAIR)'' funded by the Ministry of Culture and Science MKW.NRW, Germany.
Alexander Munteanu acknowledges additional support by the TU Dortmund - Center for Data Science and Simulation (DoDaS).

\bibliography{aistats2024_arxiv}


\clearpage
\appendix


\section{OMITTED PROOFS}
\label{appendix:theory}

\subsection{Technical Details on the Sensitivity Framework}
\label{sec:sensitivityframework}
 \begin{definition}[Coreset, cf. \citealp{FeldmanSS20}]
 \deflab{coresets}
 Let $X\in \REAL^{n\times d}$ be a set of points $\lbrace x_1,\ldots,x_n\rbrace$, weighted by $w\in \REAL_{>0}^n$. For any $\eta\in \REAL^d$, let the cost of $\eta$ w.r.t. the point $x_i$ be described by a function $w_i\cdot f\left(x_i \eta\right)$ mapping from $\REAL$ to $(0,\infty)$. Thus, the cost of $\eta$ w.r.t. the (weighted) set $X$ is $f_w\left( X \eta\right)=\sum_i w_i\cdot f\left(w_i \eta\right)$. Then a set $K\in\REAL^{k\times d}$, (re)weighted by $u\in\REAL_{>0}^k$ is a $(1+\varepsilon)$-coreset of $X$ for the function $f_w$ if $k\ll n$ and
 \begin{align*}
 \forall \eta\in\REAL^d\colon \left\lvert f_w\left(X \eta\right) - f_u\left(K \eta\right)\right\rvert\leq \varepsilon\cdot f_w\left(X \eta\right).
 \end{align*}
 \end{definition}

 In our analysis we use sampling based on so-called sensitivity scores, the range space induced by the set of functions, and the VC-dimension. We define these notions next.

\begin{definition}[Sensitivity, \citep{LangbergS10}]
  \deflab{sensitivity}
  Consider a family of functions $\mathcal{F}=\lbrace g_1,\ldots,g_n \rbrace$ mapping from $\REAL^d$ to $[0,\infty)$ and weighted by $w\in \REAL_{>0}^n$. The sensitivity of $g_\ell$ for the function $f_w(\eta)=\sum_{\ell\in [n]} w_\ell g_\ell(\eta)$, where $\eta\in \REAL^d$, is
  \begin{align}
  \varsigma_\ell = \sup \frac{w_\ell g_\ell(\eta)}{f_w(\eta)},
  \end{align}
  The total sensitivity is $\mathfrak{S}=\sum_{\ell\in [n]} \varsigma_\ell$.
  \end{definition}

 \begin{definition}[Range space; VC dimension]
 A range space is a pair $\mathfrak{R}=\left( \mathcal{F}, \verb|ranges|\right)$, where $\mathcal{F}$ is a set and $\verb|ranges|$ is a family of subsets of $\mathcal{F}$. The VC dimension $\Delta(\mathfrak{R})$ of $\mathfrak{R}$ is the size $|G|$ of the largest subset $G\subseteq \mathcal{F}$ such that $G$ is shattered by $\verb|ranges|$, i.e., $|\lbrace G\cap R:R\in\verb|ranges|\rbrace |=2^{|G|}$.
 \end{definition}

 \begin{definition}[Induced range space]
 Let $\mathcal{F}$ be a finite set of functions mapping from $\REAL^d$ to $\REAL_{\geq 0}$. For every $x\in\REAL^d$ and $r\in \REAL_{\geq 0}$, let $\verb|range|_{\mathcal{F}} (x,r) = \lbrace f\in \mathcal{F} : f(x)\geq r\rbrace$, and $\verb|ranges|(\mathcal{F}) = \lbrace \verb|range|_{\mathcal{F}}(x,r) : x\in \REAL^d, r\in\REAL_{\geq 0} \rbrace$. Let $\mathfrak{R}_\mathcal{F} = \left( \mathcal{F}, \verb|ranges|(\mathcal{F})\right)$ be the range space induced by $\mathcal{F}$.
 \end{definition}

 To construct coresets for the IRT models, we use a framework that combines sensitivity scores with the theory of VC dimension, originally proposed by \citet{BravermanFL16,Bravermanetal21a}. We employ a more recent and slightly modified version, stated in the following theorem.

\begin{theorem}[\citealp{FeldmanSS20}, Theorem 31]
 \thmlab{coreset:sensitivity}
 Consider a family of functions $\mathcal{F}=\lbrace f_1,\ldots,f_n\rbrace$ mapping from $\REAL^d$ to $[0,\infty]$ and a vector of weights $w\in\REAL_{>0}^n$. Let $\varepsilon,\delta\in (0,1/2)$. Let $s_i\geq \varsigma_i$. Let  $S=\sum_{i=1}^n s_i \geq \mathfrak{S}$. Given $s_i$ one can compute in time $\Oh{|\mathcal{F}|}$ a set $\mathcal{R} \subset \mathcal{F}$ of
 \begin{align*}
 \Oh{\frac{S}{\varepsilon^2} \left( \Delta \log S + \log\frac{1}{\delta}\right)}
 \end{align*}
 weighted functions such that with probability $1-\delta$ we have for all $\eta\in \REAL^d$ simultaneously 
 \begin{align*}
 \left\lvert \sum_{f\in \mathcal{F}} w_i f_i(\eta) - \sum_{f\in \mathcal{R}} u_i f_i(\eta) \right\rvert \leq \varepsilon \sum_{f\in \mathcal{F}} w_i f_i(\eta), 
 \end{align*}
 where each element of $\mathcal{R}$ is sampled i.i.d. with probability $p_j=\frac{s_j}{S}$ from $\mathcal{F}$, $u_i=\frac{S w_j}{|\mathcal{R}| s_j}$ denotes the weight of a function $f_i\in\mathcal{R}$ that corresponds to $f_j\in \mathcal{F}$, and where $\Delta$ is an upper bound on the VC dimension of the range space $\mathfrak{R}_{\mathcal{F}^*}$ induced by $\mathcal{F}^*$ that can be defined by defining $\mathcal{F}^*$ to be the set of functions $f_j\in \mathcal{F}$ where each function is scaled by $\frac{S w_j}{|\mathcal{R}| s_j}$.
 \end{theorem}

Note that \thmref{coreset:sensitivity} does not put additional requirements on the set of the functions $\mathcal{F}$ besides an upper bound on the sensitivities, and a bounded VC-dimension of the range space induced by those functions. 

\subsection{Omitted Proofs for the 2PL Model}

\begin{definition}[Leverage scores, cf. \citealp{DrineasMMW12}]
 \label{def:leverage:scores}
 Given an arbitrary matrix $X\in\REAL^{m\times d}$, with $m>d$, let $U$ denote the $m\times d$ matrix consisting of the $d$ left singular vectors of $X$, and let $u_i$ denote the $i$-th row of the matrix $U$ as a row vector, for all $i\in [m]$. The $i$-th leverage score corresponding to row $x_i$ of $X$ is given by
        \begin{align*}
            l_i = \| u_i \|_2^2.
        \end{align*}
\end{definition}

\begin{lemma}\lemlab{lem:leverageequiv}
    Let $X=U\Sigma V^T$ be the singular value decomposition of $X$. The three definitions are equivalent:
    \begin{enumerate}
    \item  The $i$-th leverage score (corresponding to row $x_i$) is given by
        \begin{align*}
            l_i = \| u_i\|_2^2.
        \end{align*}
    \item The $i$-th leverage score is
        given by
        \begin{align*}
            l_i = \sup_{\eta\in \REAL^d \setminus \lbrace 0\rbrace} \frac{|x_i \eta|^2}{\| X\eta\|_2^2}.
        \end{align*}
    \item The $i$-th leverage score is
        given by $$l_i = e_i^T X \left(X^T X\right)^{-1} X^T e_i$$ 
 \end{enumerate}
\end{lemma}
\begin{proof}
    Statement 1 is equivalent to \defref{leverage:scores} since the SVD yields $U$, which is exactly the matrix of the left singular vectors of $X$.

    Statement 2 is equivalent to Statement 1 since by a change of basis
    \begin{align*}
            l_i & = \sup_{\eta\in \REAL^d \setminus \lbrace 0\rbrace} \frac{|x_i \eta|^2}{\| X\eta\|_2^2} 
 = \sup_{\eta\in \REAL^d \setminus \lbrace 0\rbrace} \frac{|u_i \eta|^2}{\| U\eta\|_2^2} \stackrel{CSI}{\leq} \frac{\|u_i\|_2^2  \|\eta\|_2^2}{\| U\eta\|_2^2} = \frac{\|u_i\|_2^2  \|\eta\|_2^2}{\| \eta\|_2^2} = \| u_i\|_2^2.
        \end{align*}
        The conclusion follows from the Cauchy-Schwarz inequality (CSI) and the fact that $U$ is an orthonormal matrix. The inequality is tight due to the supremum over all $\eta\in \REAL^d$ and the existence of $\eta^* = u_i^T\in \REAL^d$ that realizes equality in CSI.

        Let $e_i$, for $i\in [m]$, be the standard basis vectors in $\REAL^{m}$ containing $1$ as $i$-th coordinate, and $0$ everywhere else.
        \begin{align*}
            l_i & = e_i^T X \left(X^T X\right)^{-1} X^T e_i \\
            & = e_i^T U\Sigma V^T \left(V\Sigma U^T U\Sigma V^T\right)^{-1} V\Sigma U^T  e_i = e_i^T U\Sigma V^T \left(V\Sigma^{2} V^T\right)^{-1} V\Sigma U^T  e_i \\
            & = e_i^T U\Sigma V^T V\Sigma^{-2} V^T V\Sigma U^T  e_i = e_i^T U\Sigma \Sigma^{-2} \Sigma U^T  e_i \\
            & = e_i^T U U^T  e_i = u_i u_i^T = \| u_i\|_2^2
        \end{align*}
    since $U$ and $V$ are orthonormal matrices, and $\Sigma$ is a square diagonal matrix.
\end{proof}

\begin{lemma}[Restatement of \lemref{leverage:invariant}]
Suppose we are given a matrix $X\in \REAL^{m\times n}$ (for any $m,n\in \mathbb{N}$) and an arbitrary diagonal matrix $D=(d_{i j})_{i\in [m], j\in [m]}$, with $d_{i j}\in \lbrace -1,1\rbrace$ if $i=j$, and $d_{i j}=0$ otherwise. Then 
the leverage scores of $X$ are the same as the leverage scores of $DX$.
\end{lemma}
\begin{proof}
Let $D=\mathrm{diag}(\{-1,1\}^m)$ be chosen as in the statement. Then it holds that $D^2=D^T D=I_m$. Further it holds that $e_i^T D = d_{i i} e_i^T$, where $e_i$ denotes the $i$th standard basis vector, i.e., the vector containing a $1$ as its $i$-th coordinate, and zeros everywhere else. The $i$-th leverage score of $X$ can be expressed as $\ell_i=e_i^T X \left(X^T X\right)^{-1} X^T e_i$ by \lemref{lem:leverageequiv} \citep[cf.][]{DrineasMMW12}. Similarly, for the $i$-th leverage score $\tilde{\ell}_i$ of $DX$ we have that
\begin{align*}
\tilde{\ell}_i &= e_i^T (DX) \left(X^T D^T DX\right)^{-1} (X^T D^T)\, e_i\\
&= \left(e_i^T D\right) X \left(X^T D^2 X\right)^{-1} X^T \left( D^T e_i\right)\\
&= d_{i i} e_i^T X \left(X^T X\right)^{-1} X^T e_i d_{i i} = d_{i i}^2 \cdot \ell_i = \ell_i,
\end{align*}
as we have claimed.
\end{proof}

\begin{theorem}[Restatement of \thmref{2pl:main:result}]
Let $X_{(i)}=( -Y_{ij} \beta_j^T)_{j\in[n]}\in\REAL^{n\times 2}$ be $\mu_1$-complex, for each $i\in[m]$. Let $\varepsilon\in (0,1/2)$. There exists a weighted set $K\in \REAL^{k\times 2}$ of size\footnote{We use the $\tilde{O}$ notation to omit $o(\log n)$ terms for a clean presentation. The full statements can be found in the proof.} $k\in \tilde O(\frac{\mu^3}{\varepsilon^4}(\log(n)^4 + \log(m))$, that is a $(1+\varepsilon)$-coreset simultaneously for all $X_{(i)}$, $i\in[m]$ for the 2PL IRT problem. The coreset can be constructed with constant probability and in $\tilde O(n)$ time. 
\end{theorem}
\begin{proof}
The proof is immediate from Theorem 19 from~\citep{MunteanuSSW18} for logistic regression in $d=2$ dimensions. Especially the reduced size $k$ follows directly from setting the dimension to constant, using $\mu_1\leq n$, and union bounding over the $i\in [m]$ iterations, which contributes the $\log(m)$ term. Further $O((\log\log(n))^4)$ terms, hidden in our $\tilde O$ notation, appear since the construction is applied recursively $O(\log\log n)$ times.

We further argue how the construction can be completed in $O(\nnz(X_{(i)})\log\log(n)) = \tilde{O}(n)$ time. The algorithm of Theorem 19 from~\citep{MunteanuSSW18} approximates the $\ell_2$-leverage scores using an $\ell_2$-subspace embedding using a CountSketch with constant distortion (say $\eps=1/10$) for a fast $QR$-decomposition, and a Gaussian matrix to approximate the row-norms of $Q$ by reducing from $d$ to $O(\log(n))$ dimensions, as in \citep{DrineasMMW12}. Further, they require an $O(\log(n))$ factor for reducing to $1/n^c$ error probability.

In our work, however, the dimension is only $d=2$, and so it is not necessary to reduce this. Further, since we aim at a constant failure probability, it is only necessary to boost the error probability of the CountSketch by a factor $O(\log\log(n))$ for a union bound over the recursive applications, which inflates its size by this exact amount. Thus, the running time for applying the CountSketch with a constant distortion remains bounded by $O(\nnz(X_{(i)})\log\log(n))=\tilde{O}(n)$ and the remaining steps all depend only on $O(\log\log(n))$, i.e., the size of the sketch.
\end{proof}

\subsection{Bounding the Sensitivities for the 3PL Model}

Let the functions $g_i$ and $h_i$ be defined as in \subsecref{3PL:model}. I.e., we let them be instances of the following form. 
\begin{align*}
    g_i(z)=&-\ln\left( \frac{1-c_i}{1+\exp( z )}\right) = \ln( 1+\exp( z )) - \ln( 1-c_i)\quad \text{ and }\\
h_i(z)=& - \ln\left( c_i + \frac{1-c_i}{1+\exp( -z )}\right).
\end{align*}
Throughout this subsection we will use the following fact.
\begin{lemma}
\lemlab{3pl:reg:linbound}
It holds for all values of $i\in[m]$ that $z\leq g_i(z)$ for all $z\geq 0$, and $g_i(z)\leq 2z$, for $z\geq \ln\left(1+\sqrt{3}\right)$.
\end{lemma}
\begin{proof}
The lower bound is valid for all $z\geq 0$, as $z\leq g_i(z) \Leftrightarrow e^z\leq 1+e^z\leq \left(1+e^z\right)/\left(1-c_i\right)$ for $c\in [0,0.5)$. For the upper bound we have that $g_i(z)\leq 2z \Leftrightarrow (1-c_i)\cdot e^{2z} - e^z-1\geq 0$. The quadratic expression is nonnegative for the values of $z$ that satisfy $e^z\geq 1+\sqrt{3}$, i.e., for $z\geq \ln\left(1+\sqrt{3}\right)\geq 1.005$.
\end{proof}

We use the sensitivity framework of \thmref{coreset:sensitivity}, where all input weights $w_\ell$ are set to 1. Let $f_1\left(X\beta_j\right)=\sum_{i\in[m], Y_{ij}=-1} g_i\left( x_i\beta_j\right)$. Let $f_2\left(X\beta_j\right) = 
\sum_{i\in[m], Y_{ij}=1} h_i\left(x_i \beta_j\right)$, as in Equation~(\ref{eqn:3pl:a}). 

Let $m'_{-}$ and $m'_{+}$ be the number of summands in Equation~(\ref{eqn:3pl:a}) with $Y_{ij}=-1$ and with $x_i\beta_j<0$ and $x_i\beta_j\geq 0$, respectively. Similarly, let $m''_{-}$ and $m''_{+}$ be the number of summands in Equation~(\ref{eqn:3pl:a}) with $Y_{ij}=1$ and with $x_i\beta_j<0$ and $x_i\beta_j\geq 0$, respectively. Let $m'=m'_{-} + m'_{+}$ and $m''=m''_{-} + m''_{+}$. For simplicity we rearrange the indices of summands within the functions $f_1$ and $f_2$ to $i\in[m']$ and $i\in[m'']$ respectively. In the following lemma we bound the relation between $m'$ and $m''$. Recall that we assumed that $a_i>0$ holds for all items $i\in[m]$.

\begin{lemma}
\lemlab{labels:bound}
Given the matrix $X_{(j)} = (-Y_{ij}\alpha_i^T)_{i\in[m]} \in \REAL^{m\times 2}$. Let $X'_{(j)}$ and $X''_{(j)}$ contain the $m'$ and $m''$ rows of $X_{(j)}$ that satisfy $Y_{ij}=-1$ and $Y_{ij}=1$, respectively. Let $X'_{(j)}$ and $X''_{(j)}$ be $\mu$-complex. Then it holds that $X_{(j)}$ is $2\mu$-complex, and that
\begin{align}
\label{mu:subset:bound}
\frac{m''}{2\mu_0}\leq m'\leq m'' \cdot 2\mu_0.
\end{align}
\end{lemma}
\begin{proof}
To see the first claim of the lemma, we note that for $p\in\{0,1\}$
\begin{align*}
    &\sup_{\eta\in\REAL^{2}\setminus \lbrace 0\rbrace} \frac{\| (X_{(j)}\eta)^+\|_p}{\| (X_{(j)}\eta)^-\|_p} = 
    \sup_{\eta\in\REAL^{2}\setminus \lbrace 0\rbrace} \frac{\| (X'_{(j)}\eta)^+\|_p + \|X''_{(j)}\eta)^+\|_p}{\| (X'_{(j)}\eta)^-\|_p + \| (X''_{(j)}\eta)^-\|_p} \leq\\
    & \leq \sup_{\eta\in\REAL^{2}\setminus \lbrace 0\rbrace} \frac{\| (X'_{(j)}\eta)^+\|_p}{\| (X'_{(j)}\eta)^-\|_p + \| (X''_{(j)}\eta)^-\|_p} + \sup_{\eta\in\REAL^{2}\setminus \lbrace 0\rbrace} \frac{\|(X''_{(j)}\eta)^+\|_p}{\| (X'_{(j)}\eta)^-\|_p + \| (X''_{(j)}\eta)^-\|_p}\\
    & \leq \sup_{\eta\in\REAL^{2}\setminus \lbrace 0\rbrace} \frac{\| (X'_{(j)}\eta)^+\|_p}{\| (X'_{(j)}\eta)^-\|_p} + \sup_{\eta\in\REAL^{2}\setminus \lbrace 0\rbrace} \frac{\|(X''_{(j)}\eta)^+\|_p}{\| (X''_{(j)}\eta)^-\|_p}\\
    & \leq \mu_p + \mu_p = 2\mu_p.
\end{align*}
For the second claim we use the properties of the space $\REAL^2$. Since $a_i >0$ for all $i\in[m]$, the original points $\alpha_i=(a_i,b_i)$ lie in the halfspace with positive first coordinate. 
By choosing $\hat{\eta}=(1,0)^T$, it holds that $x_i\hat{\eta} = -Y_{ij}a_i$, which is positive if $Y_{ij}=-1$ and negative if $Y_{ij}=1$. Thus, it follows that $\| (X_{(j)}\hat{\eta})^+\|_0 = m'$ and $\| (X_{(j)}\hat{\eta})^-\|_0 = m''$. The definition of the $2\mu_0$-complexity of $X_{(j)}$ implies that:
\begin{align*}
    2\mu_0 & \geq \sup_{\eta\in\REAL^{2}\setminus \lbrace 0\rbrace} \frac{\| (X_{(j)}\eta)^+\|_0}{\| (X_{(j)}\eta)^-\|_0} \geq \frac{\| (X_{(j)}\hat{\eta})^+\|_0}{\| (X_{(j)}\hat{\eta})^-\|_0} = \frac{m'}{m''}.
\end{align*}
The second bound of Equation~(\ref{mu:subset:bound}) can be obtained similarly using $\hat{\eta}=(-1,0)^T$. This concludes the proof.
\end{proof}

Unfortunately an analogous expression to \cref{mu:subset:bound} in $\ell_1$-norm does not follow verbatim. For technical reasons we thus need to assume that $\sup_{\eta\in\mathbb{R}\setminus\{0\}}\frac{\|X'_{(j)}\eta\|_1}{\|X''_{(j)}\eta\|_1} \leq 2\mu_1$.

The following three lemmas follow the approach of \citet{ClarksonW15b} and \citet{MunteanuSSW18}, adapted here to work for our different sets of functions $g_i$ and $h_i$, to bound the sensitivities for the first part of the sum defining $f(\beta_j\mid A,C)$, cf. Eq.~(\ref{eqn:3pl:a}). For the first two lemmas it suffices to assume that the matrices $X'$ and $X''$ are $\mu_1$-complex, thus, by \lemref{labels:bound} $X$ is $2\mu_1$-complex.

\begin{lemma}
\lemlab{3pl:lemma12}
Let $X'\in \REAL^{m'\times 2},X''\in \REAL^{m''\times 2}$ be $\mu_1$-complex. Let $U$ be an orthonormal basis for the columnspace of $X$. If for index $\ell$  
$\beta_j\in \REAL^2$ 
satisfies $1.005\leq x_\ell\beta_j$, then it holds that $g_\ell\left( x_\ell\beta_j\right) \leq 12\mu_1^2\cdot \|U_\ell\|_2\cdot f_1\left( X\beta_j\right)$.
\end{lemma}
\begin{proof}
Let $X=UR$, where $U$ is an orthonormal basis for the columnspace of $X$. Let $U_\ell$ be the $\ell$-th row of $U$. From Cauchy-Schwarz inequality (CSI), orthornomality of $U$, \lemref{3pl:reg:linbound}, $1.005\leq x_\ell\beta_j$, $\mu_1$-complexity of $X$, and the positivity of $g_\ell$ we have that 
\begin{align*}
g_\ell\left(x_\ell\beta_j\right)&= g_\ell\left(U_\ell R \beta_j\right) \stackrel{CSI}{\leq} g_\ell\left(\|U_\ell\|_2\cdot \|R \beta_j\|_2\right) = g_\ell\left(\|U_\ell\|_2\cdot \|UR \beta_j\|_2\right) \\ 
&= g_\ell\left(\|U_\ell\|_2\cdot \|X \beta_j\|_2\right) \leq 2 \cdot \|U_\ell\|_2\cdot \|X \beta_j\|_2 \leq 2 \cdot \|U_\ell\|_2\cdot \|X \beta_j\|_1 \\
&\leq 2 \cdot \|U_\ell\|_2\cdot (1+2\mu_1) \|X' \beta_j\|_1 \stackrel{(\ref{mu:property})}{\leq} 2 \cdot \|U_\ell\|_2\cdot 3\mu_1(1+\mu_1)\|(X' \beta_j)^+\|_1 \\
&\leq 12\mu_1^2 \cdot \|U_\ell\|_2\cdot \sum_{i\in [m']: x_i \beta_j \geq 0} |x_i \beta_j| \\
&\leq 12\mu_1^2 \cdot \|U_\ell\|_2\cdot \sum_{i\in [m']: x_i \beta_j \geq 0} g_i\left(x_i \beta_j\right) \leq 12\mu_1^2 \cdot \|U_\ell\|_2\cdot  f_1\left(X \beta_j\right).
\end{align*}
\end{proof}

\begin{lemma}
\lemlab{3pl:lemma13}
Let $X'\in \REAL^{m'\times 2}$ be $\mu_1$-complex. If for index $\ell$, 
$\beta_j\in \REAL^2$ 
satisfies $1.005\geq x_\ell\beta_j$, then it holds that $g_\ell\left(x_\ell\beta_j\right) \leq \left( 40+\frac{5\mu_1}{2}\right) \cdot \frac{1}{m'} \cdot f_1\left( X \beta_j\right)$.
\end{lemma}
\begin{proof}
Let $K^-=\lbrace i\in[m']: x_i\beta_j\leq -2 \rbrace$ and $K^+=\lbrace i\in[m']: x_i\beta_j> -2 \rbrace$. It holds for all $i$ that $g_i(-2)=\ln(1+\exp(-2))-\ln(1-c_i) \geq \ln(1+\exp(-2)) > \frac{1}{8}$, and $g_\ell(x_\ell \beta_j)\leq g_\ell(1.005)\leq \ln(1+\exp(1.005)) +\ln 2 <2.5$, due to the monotonicity of $g_\ell$ and our assumption that $c_i\in [0,0.5)$. It holds that $|K^-|+|K^+|=m'$.

In case that $K^+ \geq \frac{m'}{2}$ we have that
\begin{align*}
f_1\left(X \beta_j\right) &= \sum_{i\in K^+} g_i\left(x_i \beta_j\right) + \sum_{i\in K^-} g_i\left(x_i \beta_j\right) \geq \sum_{i\in K^+} g_i\left(x_i \beta_j\right) \\
&\geq \sum_{i\in K^+} g_i\left(-2\right) \geq \frac{m'}{2}\cdot \frac{1}{8} \geq \frac{m'}{40}\cdot 2.5 \geq \frac{m'}{40}\cdot g_\ell(x_\ell \beta_j).
\end{align*}
In case that $K^+ < \frac{m'}{2}$ it is $K^- \geq  \frac{m'}{2}$ and thus
\begin{align*}
f_1\left( X \beta_j\right) &\geq \sum_{i\in [m']:x_i \beta_j\geq 0} g_i\left(x_i \beta_j\right) \geq \sum_{i\in [m']:x_i \beta_j\geq 0} |x_i \beta_j| = \|\left( X' \beta_j\right)^+\|_1 \\
&\stackrel{(\ref{mu:property})}{\geq} \frac{\|\left( X' \beta_j\right)^-\|_1}{\mu_1} = \frac{1}{\mu_1} \sum_{i\in [m']:x_i \beta_j< 0} |x_i \beta_j| \\
&\geq \frac{1}{\mu_1} \sum_{i\in K^-} |x_i \beta_j| \geq \frac{|K^{-}|\cdot |-2|}{\mu_1} \geq \frac{m'}{2.5\mu_1}\cdot 2.5 \geq \frac{2m'}{5\mu_1} \cdot g_\ell(x_\ell \beta_j).
\end{align*}
The claim follows by summing the upper bounds for $g_\ell(x_\ell \beta_j)$ from both cases.
\end{proof}

We combine \lemref{3pl:lemma12} and \lemref{3pl:lemma13} to obtain the following result that provides upper bounds on the sensitivities of the functions $g_\ell$ regarding the combined function $f_1(X\beta)+f_2(X\beta)$, as well as an upper bound for the total sensitivity on the first part of the sum that defines $f(\beta_j\mid A,C)$.

\begin{lemma}
\lemlab{3pl:lemma14}
Let $X'\in \REAL^{m'\times 2},X''\in \REAL^{m''\times 2}$ be $\mu$-complex. Let $U$ be an orthonormal basis for the columnspace of $X$. For each $i\in[m']$ the sensitivity of $g_i\left(x_i \beta_j\right)$ for the function $f_1+f_2$ is bounded by $\varsigma'_i \leq s'_i = 42.5\mu_1^2 \cdot \left( \|U_i\|_2 + \frac{1}{m'} \right)$. The sum of sensitivities for $g_i, i\in[m']$  is bounded by $S' \leq 170 \mu_1^2 \sqrt{m'}$.
\end{lemma}
\begin{proof}
From \lemref{3pl:lemma12} and \lemref{3pl:lemma13} we have for each $i\in [m']$ that
\begin{align*}
\varsigma'_i& = \sup_{\beta_j} \frac{g_i\left(x_i \beta_j\right)}{f_1\left( X \beta_j\right)+ f_2\left( X \beta_j\right)} \leq \frac{g_i\left(x_i \beta_j\right)}{f_1\left( X \beta_j\right)} \leq 12\mu_1^2 \cdot \|U_i\|_2 + \left( 40+\frac{5\mu_1}{2}\right) \cdot \frac{1}{m'}\\ 
&\leq 42.5\mu_1^2 \cdot \left( \|U_i\|_2 +\frac{1}{m'}\right)  = s'_i.
\end{align*}

Since the Frobenius norm of the matrix $U$ is $\|U\|_F = \sqrt{\sum_{j\in[2]} \sum_{i\in[m]} |U_{ij}|^2}= \sqrt{\sum_{j\in[2]} 1} =\sqrt{2}$, due to the orthonormality of $U$, we have that
\begin{align*}
S' &= \sum_{i\in[m']} s'_i = 42.5\mu_1^2 \cdot \left( \sum_{i\in[m']}\|U_i\|_2 + \sum_{i\in[m']}\frac{1}{m'} \right)\\
&\stackrel{CSI}{\leq} 42.5\mu_1^2 \cdot \left( \|U\|_F^2 \cdot \sqrt{m'} + \frac{m'}{m'} \right)\\
& \leq 42.5\mu_1^2 \cdot \left( 2\sqrt{m'} + 1 \right) \leq 42.5\mu_1^2 \cdot 4\sqrt{m'}\\
&= 170 \mu_1^2 \sqrt{m'}.
\end{align*}
\end{proof}

The second part of the sum defining $f(\beta_j\mid A,C)$ contains the functions corresponding to labels $Y_{i j} =1$. The following lemma bounds their sensitivities. Let $E=\max \{ \ln(1/c_i)\mid {i\in[m]}\}$ (over the entire input). 

\begin{lemma}
\lemlab{3pl:lemma:sig}
Let $X''\in\REAL^{m''\times 2}$ be $\mu_0$-complex. For each $\ell\in[m'']$ the sensitivity of $h_\ell\left(x_\ell \beta_j\right)$ for the function $f_1+f_2$ is bounded by $\varsigma''_\ell \leq 3.5E\cdot(1+\mu_0)\cdot \frac{1}{m''} =s''_\ell$. The sum of sensitivities for $h_i, i\in[m'']$ is bounded by $S''\leq 3.5E\cdot(1+\mu_0)$.
\end{lemma}
\begin{proof}
Since each function $h_\ell$, $\ell\in[m'']$, satisfies $0< h_\ell(x_\ell\beta_j)< E$, we have that for each $\ell\in[m'']$,  
$\beta_j\in \REAL^2$ 
satisfies
\begin{align*}
f_2\left( X \beta_j\right) &= \sum_{i\in[m'']: x_i\beta_j\geq 0} h_i\left(x_i\beta_j\right) + \sum_{i\in[m'']: x_i\beta_j< 0} h_i\left(x_i\beta_j\right) \\
&\geq \sum_{i\in[m'']: x_i\beta_j\geq 0} h_i\left(x_i\beta_j\right) \geq \sum_{i\in[m'']: x_i\beta_j\geq 0} h_i\left(0\right)\\
& = \sum_{i\in[m'']: x_i\beta_j\geq 0} \ln\left(\frac{2}{1+c_i}\right) \geq \sum_{i\in[m'']: x_i\beta_j\geq 0} \ln\left(\frac{4}{3}\right) = m''_{+} \ln\left(\frac{4}{3}\right) \\
& \geq 
\frac{m''_{+}}{3.5 E}\cdot h_\ell\left(x_\ell \beta_j\right). 
\end{align*}
The sensitivity of $h_\ell\left(x_\ell\beta_j\right)$ regarding the function  $f_1+f_2$ is then bounded by 
\begin{align*}
\varsigma''_\ell =\sup_{\beta_j} \frac{h_\ell\left(x_\ell \beta_j\right)}{f_1\left( X \beta_j\right)+ f_2\left( X \beta_j\right)} \leq \frac{h_\ell\left(x_\ell \beta_j\right)}{f_2\left( X \beta_j\right)} \leq \frac{3.5E}{m''_{+}} \stackrel{(\ref{mu:property})}{\leq} \frac{3.5E\cdot(1+\mu_0)}{m''} =s''_\ell,
\end{align*}
while the sum of sensitivities of the functions $h_i, i\in[m'']$ regarding the function $f_1+f_2$ is bounded by
\begin{align*}
S''= \sum_{i\in[m'']}s''_i \leq \frac{3.5E\cdot(1+\mu_0)}{m''}\cdot m'' = 3.5E\cdot(1+\mu_0). 
\end{align*}
\end{proof}

\begin{lemma}
\label{lem:totalsensitivity}
The total sensitivity is bounded by $\mathfrak{S} \leq 170 \mu^2 \sqrt{m} + 7E\mu \in \Oh{\sqrt{m}}$.
\end{lemma}
\begin{proof}
\cref{lemma:3pl:lemma14,lemma:3pl:lemma:sig} can be combined to bound the total sensitivity in terms of $m',m''$, and we can relate the latter quantities to $m$ using \lemref{labels:bound}. This implies that the total sensitivity for the function $f_1+f_2$ is \begin{align*}
    \mathfrak{S} &\leq S=S'+S'' = 170 \mu_1^2 \sqrt{m'} + 3.5E(1+\mu_0) 
    \leq 170 \mu^2 \sqrt{m} + 7E\mu \in \Oh{\sqrt{m}}. 
\end{align*}
\end{proof}

\subsection{Bounding the VC Dimension for the 3PL Model}

In order to apply the sensitivity framework, we need to bound the VC dimension of the range spaces induced by the sets of (weighted) functions $g_i$ and $h_i$. Let $g_{i}(\eta)=g\left(x_i \eta\right)$ and  $h_{i}(\eta)=h_i\left(x_i \eta\right)$. The dimension of the domains of our functions is $d=2$ (in both cases where $\alpha_i$ or $\beta_j$ take the role of the variable $\eta$). 
We first bound the VC dimension in the case that all weights are fixed to the same (though arbitrarily chosen) positive constant $\rho$. This is dealt with in the following two lemmas:

\begin{lemma}
\lemlab{VC:logistic:const}
The range space induced by $\mathcal{G}_\rho=\lbrace \rho g_{(i)}\colon i\in[m]\rbrace$, $\rho\in\REAL_{>0}$, satisfies $\Delta\left( \mathfrak{R}_{\mathcal{G}_\rho} \right)\leq d+1=3$.
\end{lemma}
\begin{proof}
The function $g:\REAL\rightarrow\REAL_{\geq 0}$ is monotonically increasing and invertible. Let $G\subseteq \mathcal{G}_\rho$, $z\in\REAL$, and $r\in\REAL$. It holds that 
\begin{align*}
\mathtt{range}_G(\eta,r) = \lbrace \rho g_{i}\in\mathcal{G}_\rho \colon \rho g_{i}(\eta)\geq r\rbrace = \lbrace \rho g_{i}\in\mathcal{G}_\rho \colon x_i \eta\geq g^{-1}(r/\rho)\rbrace.
\end{align*}
Then it follows that 
\begin{align*}
&\left\lvert \lbrace \mathtt{range}_G(\eta,r) \colon \eta\in\REAL^2, r\in\REAL_{\geq 0}\rbrace \right\rvert\\ 
= &\left\lvert \lbrace \lbrace \rho g_{i}\in G \colon x_i \eta \geq g^{-1}(r/\rho) \rbrace \colon \eta\in\REAL^2, r\in\REAL_{\geq 0}\rbrace \right\rvert \\
= & \left\lvert \lbrace \lbrace g_{i}\in G \colon x_i \eta \geq \tau \rbrace \colon \eta\in\REAL^2, \tau\in\REAL\rbrace \right\rvert.
\end{align*}
Since each function $g_{i}$ is associated with the point $x_i$, the last set is the set of points shattered by the hyperplane classifier $x_i \mapsto \textbf{1}_{[x_i \eta -\tau \geq 0]}$. Its VC dimension is thus $d+1=3$ \citep{Kearns1994introduction}, implying that $\left\lvert \lbrace \mathtt{range}_G(\eta,r) \colon \eta\in\REAL^2, r\in\REAL_{\geq 0}\rbrace \right\rvert = 2^{|G|}$ can only hold if $|G|\leq d+1=3$. Therefore, the VC dimension of the range space induced by $\mathcal{G}_\rho$ is bounded by $d+1=3$.
\end{proof}

\begin{lemma}
\lemlab{VC:sigmoid:const}
The range space induced by $\mathcal{H}_\rho=\lbrace \rho h_{(i)}\colon i\in[m]\rbrace$, $\rho\in\REAL_{>0}$, satisfies $\Delta\left( \mathfrak{R}_{\mathcal{H}_\rho} \right)\leq d+1=3$.
\end{lemma}
\begin{proof}
The functions $h_{i}:\REAL\rightarrow \left(0,\ln\left(1/c_i\right)\right)$ are monotonically decreasing and invertible independent of the choice of $c_i$. Let $H\subseteq \mathcal{H}_\rho$, $\eta\in\REAL^2$, and $r\in\REAL$. For $r\geq \ln(1/c_i)/\rho$ we have $\mathtt{range}_H(\eta,r) = \emptyset$. Otherwise, it holds that $r<\ln(1/c_i)/\rho$ and
\begin{align*}
\mathtt{range}_H(\eta,r) = \lbrace \rho h_{i}\in\mathcal{H}_\rho \colon \rho h_{i}(\eta)\geq r\rbrace = \lbrace \rho h_{i}\in\mathcal{H}_\rho \colon x_i \eta\leq h^{-1}(r/\rho)\rbrace.
\end{align*}
It follows that 
\begin{align*}
&\left\lvert \lbrace \mathtt{range}_H(\eta,r) \colon \eta\in\REAL^2, r\in\REAL_{\geq 0}\rbrace \right\rvert\\
= & \left\lvert \lbrace \lbrace \rho h_{i}\in H \colon x_i \eta \leq h^{-1}(r/\rho) \rbrace \colon \eta\in\REAL^2, r \leq \ln(1/c_i)/\rho\rbrace \cup \lbrace\emptyset\rbrace \right\rvert \\
\leq &  \left\lvert \lbrace \lbrace \rho h_{i}\in H \colon x_i \eta \leq \tau \rbrace \colon \eta\in\REAL^2, \tau\in\REAL\rbrace \right\rvert.
\end{align*}
Since each function $h_{i}$ is associated with the point $x_i$, the last set is the set of points that is shattered by an affine classifier $x_i \mapsto \textbf{1}_{[x_i \eta -\tau \leq 0]}$. As before in \lemref{VC:logistic:const} we conclude that the VC dimension of the range space induced by $\mathcal{H}_\rho$ is at most $d+1=3$.
\end{proof}

\citet{BlumerEHW89} gave a general Theorem for bounding the VC dimension of the union or intersection of $t$ range spaces, each of bounded VC dimension at most $D$. Their result gives $O(tD\log t)$. Here, we give a bound of $O(tD)$ for the special case that the range spaces are disjoint.
\begin{lemma}
\label{lem:disjVCbound}
Let $\mathcal{F}$ be any family of functions. And let $F_1,\ldots,F_t\subseteq \mathcal{F}$, each non-empty, form a partition of $\mathcal{F}$, i.e., their disjoint union satisfies $\dot\bigcup_{i=1}^t F_i = \mathcal{F}$. Let the VC dimension of the range space induced by $F_i$ be bounded by $D$ for all $i\in[t]$. Then the VC dimension of the range space induced by $\mathcal{F}$ satisfies $\Delta\left( \mathfrak{R}_\mathcal{F} \right)\leq tD$.
\end{lemma}
\begin{proof}
    We prove the claim by contradiction. To this end suppose the VC dimension for $\mathcal F$ is strictly larger than $tD$. Then there exists a set $G$ of size $|G|>tD$ that is shattered by the ranges of $\mathfrak{R}_\mathcal{G}$. Consider its intersections $G_i= G\cap F_i, i\in[t]$ with the sets $F_i$. By their disjointness, $G_i$ must be shattered by the ranges of $\mathfrak{R}_{{F}_i}$. Note that at least one of them must therefore have $|G|/t > D$, which contradicts the assumption that their VC dimension is bounded by $D$. Our claim thus follows.
\end{proof}

\begin{corollary}
\corlab{VC:union}
Let $\mathcal{F}=\mathcal{G}\,\dot\cup\,\mathcal{H}$ be the set of functions in the 3PL IRT model where each function is either of type $g_i\in \mathcal{G}$ or $h_i\in \mathcal{H}$ and each function is weighted by $0 < w_i\in W := \lbrace u_1,\ldots,u_t\rbrace$.
The range spaces induced by $\mathcal{F}$ satisfies $\Delta\left( \mathfrak{R}_{\mathcal{F}}\right) \leq 6t$.
\end{corollary}
\begin{proof}
We partition $\mathcal{G}$, and $\mathcal{H}$ into disjoint subsets $\mathcal G_{u_1},\ldots,\mathcal  G_{u_t}\subseteq \mathcal G$, and $\mathcal H_{u_1},\ldots,\mathcal H_{u_t}\subseteq \mathcal H$ where the functions in any of those sets have the same weight. By the subset relation and using \cref{lemma:VC:logistic:const,lemma:VC:sigmoid:const}, the VC dimension induced by any of these sets is bounded above by $d+1=3$.
Further we have that $\mathcal F = \mathcal G\,\dot\cup\, \mathcal H = (\dot\bigcup_{i=1}^t \mathcal G_i) \;\dot\cup\; (\dot\bigcup_{i=1}^t \mathcal H_i)$ is a partition of $\mathcal F$ into $2t$ disjoint subsets by construction. The claim follows by invoking \cref{lem:disjVCbound}.
\end{proof}

\subsection{Putting Everything Together for the 3PL Model}
\begin{theorem}[Restatement of \thmref{3pl:main:result}]
\thmlab{3pl:main:result:rest}
Let each $X_{(j)}=( -Y_{ij} \alpha_i^T)_{i\in[m]}\in\REAL^{m\times 2}$. Let $X'_{(j)}$ contain the rows $i$ of $X_{(j)}$ where $Y_{ij}=-1$ and let $X''_{(j)}$ comprise the rows with $Y_{ij}=1$. Let $X'_{(j)}$ and $X''_{(j)}$ be $\mu$-complex.\\ 
Let  $\sup_{\eta\in\mathbb{R}\setminus\{0\}}{\|X'_{(j)}\eta\|_1}/{\|X''_{(j)}\eta\|_1} \leq 2\mu_1$ for each $j\in[n]$. Let $\varepsilon\in (0,1/2)$. There exists a weighted set $K\in \REAL^{k\times 2}$ of size 
$k\in O(\frac{\mu^2 \sqrt{m}}{\varepsilon^2} (\log(m)^2 + \log(n)))$, that is a $(1+\varepsilon)$-coreset for all $X_{(j)}$, $j\in[n]$ simultaneously for the 3PL IRT problem. The coreset can be constructed with constant probability and in $O(m)$ time.
\end{theorem}
\begin{proof}
For a single computation of $\beta_j$, say $\beta_1$, our input consists of a matrix $X_{(1)}$ and labels $Y_{i1}$, that define the function $f_1+f_2$. We want to apply \thmref{coreset:sensitivity} to the set of functions $g_i$ and $h_i$ that occur in their respective parts of $f_1+f_2$, and obtain a $(1+\varepsilon)$-coreset $K$ for the function $f_1+f_2$ on $X_{(1)}$.

\cref{lemma:3pl:lemma14,lemma:3pl:lemma:sig} bound the sensitivities of single functions $g_i$ and $h_i$, while \cref{lem:totalsensitivity} bounds the total sensitivity $S$. \corref{VC:union} yields an upper bound of $6t$ on the VC dimension $\Delta$ of the range space induced by the functions $g_i$ and $h_i$, where $t$ denotes the number of different weights. We discuss the choice of $t$ at the end of the proof.

The algorithm to compute the coreset $K$ requires to compute the upper bounds on the sensitivities of \lemref{3pl:lemma14} for the submatrix $X'_{(1)}$ (of $X_{(1)}$), that depend on an orthonormal basis of the columnspace of $X_{(1)}$. This enables the algorithm to sample the input points with probabilities proportional to the values $s_i$ (which equal either $s'_i$ or $s''_i$, depending on the function), divided by the total sensitivity.

This can be done by computing the QR-decomposition of $X_{(1)}=QR$, in time $\Oh{m d^2}=\Oh{m}$ \citep{Golub2013matrix}. $Q$ is an orthonormal basis for the columnspace of $X_{(1)}$. From $Q=U$ we compute the row-norms $\|U_i\|_2$, and thus the values of $s'_i$. Sampling the $|K|$ elements can be done using a weighted reservoir sampler \citep{Chao82} in linear time $\Oh{m}$. The total running time is thus $\Oh{m}$.

Although $X_{(1)}$ being in $\REAL^{m\times 2}$ enables a fast (linear time) QR-decomposition, it is advisable in practice to use a fast $QR$-decomposition as in \citep{DrineasMMW12}, since this reduces the constant factors (depending on $d=2$ in this paper). The idea is that we can obtain a fast constant factor approximation to the square root of the leverage scores $\|Q_i\|_2$, with success probability $1-\delta''=1-\delta/2$, and use these as the input to the reservoir samplers. Using CountSketch, i.e., the sketching techniques of \citet{ClarksonW13}, we reduce the size of the matrix to be decomposed to only $O(d^2)$, which is a small constant rather than $O(m)$.

As in the 2PL case, for any other coordinate $\beta_j$, $2\leq j\leq n$ within one iteration, the labels $Y_{ij}$ come from $\lbrace -1,1\rbrace^m$. \lemref{leverage:invariant} implies that the leverage scores of $X_{(1)}$, that have been used for the coreset construction for $\beta_1$, remain the same for all other $X_{(j)}$, and thus can be used for all other coordinates $\beta_j$, $2\leq j\leq n$ as well. Since the sensitivity scores remain the same, we can use the same coreset for the optimization of all $\beta_j$, $j\in[n]$.

To control the success probability of sensitivity sampling over all $\beta_j, j \in [n]$, let $\delta'=\delta/(2n)$. Then the total failure probability (for the approximation of the leverage scores and the coreset sampling) is at most $\delta''+n\cdot\delta' = \delta/2 + \delta/2=\delta$. 

It remains to bound the number of different weights used for the sampling, and in the VC-dimension bound of the involved range space. Each function $g_i$ and $h_i$ is sampled with probability proportional to $s'_i/(\sum s'_i + \sum s''_i)$ and $s''_i/(\sum s'_i + \sum s''_i)$ respectively. We can round the sensitivities $s'_i$ and $s''_i$ up to the next power of $2$, and obtain the values $\hat{s}'_i$ and $\hat{s}''_i$ respectively. It holds that $s'_i\leq \hat{s}'_i \leq 2s'_i$ and $s''_i\leq \hat{s}''_i \leq 2s''_i$, for all $i\in[m]$. Then, we can sample the functions $g_i$ and $h_i$ proportional to the probabilities $\hat{s}'_i/(\sum \hat{s}'_i + \sum \hat{s}''_i)$ and $\hat{s}''_i/(\sum \hat{s}'_i + \sum \hat{s}''_i)$, respectively. It holds that $\sum \hat{s}'_i + \sum \hat{s}''_i\leq 2(\sum s'_i + \sum s''_i) = O(\mu^2\sqrt{m})$, by \cref{lem:totalsensitivity}.

We observe that:
\begin{align}
   1&\geq \hat{s}'_i\geq s'_i \geq \sup_{\beta_j} \frac{g_i(x_i\beta_j)}{f_1(X\beta_j) + f_2(X\beta_j)} \stackrel{\beta_j=0}{\geq} \frac{g_i(0)}{f_1(0)+f_2(0)} \nonumber \\
   &= \frac{\ln(2)-\ln(1-c_i)}{\sum_{Y_{ij}=-1} (\ln(2)-\ln(1-c_i)) +\sum_{Y_{ij}=1}( -\ln(c_i + \frac{1-c_i}{2}))} \nonumber \\
   &\geq \frac{\ln(2)}{m'\cdot(\ln(2)-\ln(1-c_i)) +m''\cdot(\ln(\frac{2}{1-c_i}))} \geq \frac{\ln(2)}{2\ln(2)m' +\ln(4)m''} \nonumber\\
   &= \frac{1}{2m' +2m''} = \frac{1}{2m}
   \label{numberofpowersoftwo:g}
\end{align}
We can analogously conclude that
\begin{align}
    1\geq \hat{s}''_i\geq s''_i\geq \frac{h_i(0)}{f_1(0)+f_2(0)}\geq \frac{\ln(\frac{4}{3})}{2m}.
    \label{numberofpowersoftwo:h}
\end{align}
Equations~(\ref{numberofpowersoftwo:g}) and (\ref{numberofpowersoftwo:h}) imply that there can be at most $t=O(\log(m))$ values of $\hat{s}'_i$ and $\hat{s}''_i$, which implies that $\Delta = O(\log(m))$. Thus we can construct a single coreset $K$ of size 
\begin{align}
\label{eqn:3pl:coreset:1}
|K|&=\Oh{\frac{S}{\varepsilon^2} \left( \Delta \log S + \log\left(\frac{1}{\delta'}\right)\right)} \nonumber\\
&= \Oh{\frac{\mu^2\sqrt{m}}{\varepsilon^2}\cdot (\log(\mu^2\sqrt{m}) \log(m) + \log(n))} \nonumber\\
&\!\!\!\overset{\mu\leq m}{=} \Oh{\frac{\mu^2\sqrt{m}}{\varepsilon^2}\cdot (\log(m)^2 + \log(n))},
\end{align}
 for all $X_{(j)}$, $j\in[m]$, with constant success probability at least $1-\delta$ in time $O(m)$, as claimed.
\end{proof}

Finally, we need to address the differences between the coresets for $f(\beta_j\mid A,C)$ (claimed by \thmref{3pl:main:result}), and the coresets for $f(\alpha_i,c_i \mid B)$. In the 2PL case the two cases were interchangeable, since the function depended on one parameter only. Here, for $f(\alpha_i,c_i \mid B)$ function $g_i$ and $h_i$ are functions of two parameters, $\alpha_i$ and $c_i$. We need the following result that gives us a lower bound on the sum of the logistic loss functions.

\begin{lemma}[\citealp{MunteanuOW21}, Lemma 2.2]
\lemlab{logistic:sum:lowerbound}
Let $Z\in \REAL^{n\times d}$ be a $\mu_1$-complex matrix for bounded $\mu_1<\infty$, and let $z_i$ be its rows. For all $y\in \REAL^d$ it holds that
\begin{align*}
\sum_{i\in[n]} \ln\left( 1+\exp\left( z_i y\right)\right) \geq \frac{n}{2\mu_1} \left( 1+\ln(\mu_1)\right).
\end{align*}
\end{lemma}

We slightly adapt the notation of the functions $g_j$ and $h_j$ (we change of the index to emphasize that the fixed parameters encoded in the rows of $X$ are now $\beta_j,j\in [n]$). To keep in mind that these functions are functions of an additional variable $c_i$, we write
\begin{align*}
    g_j(z,c_i)=-\ln\left( \frac{1-c_i}{1+\exp( z )}\right) = \ln( 1+\exp( z )) - \ln( 1-c_i)
\end{align*}
and
\begin{align*}
    h_j(z,c_i)= - \ln\left( c_i + \frac{1-c_i}{1+\exp( -z )}\right).
\end{align*}

The following lemma claims that by increasing the value of $c_i$ by a small additive value, the sum of all functions will increase only by a small multiplicative error. Since the roles of $n$ and $m$ are reversed, we also let $n'$ and $n''$ take the role of $m'$ and $m''$ respectively.

\begin{lemma}
\lemlab{additiveerror}
Let 
\begin{align*}
f(X\alpha_i, c_i) &= f_1(X\alpha_i, c_i) + f_2(X\alpha_i, c_i) \\
&= \sum_{j\in[n], Y_{ij}=-1}g_j(x_j\alpha_i, c_i) + \sum_{j\in[n], Y_{ij}=1}h_j(x_j\alpha_i, c_i).    
\end{align*}
Then it holds that 
\begin{align*}
    \left|f\Big(X\alpha_i, c_i\Big) - f\left(X\alpha_i, c_i+\frac{\varepsilon}{\mu^2}\right)\right| \leq \varepsilon f(X\alpha_i, c_i).
\end{align*}
\end{lemma}
\begin{proof}
For the sigmoid functions $h_j$ we have that
\begin{align*}
    h_j(z,c_i) = -\ln\left( c_i+\frac{1-c_i}{1+\exp(-z)}\right) = \ln \left( \frac{1+\exp(-z)}{1+c_i \exp(-z)}\right). 
\end{align*}
Then using the fact that the functions $h_j$ and their differences are monotonic, we have that
\begin{align}
    h_j\Big(z,c_i\Big) - h_j\left(z,c_i+\frac{\varepsilon}{\mu^2}\right) & =  \ln \left( \frac{1+\exp(-z)}{1+c_i \exp(-z)}\right) - \ln \left( \frac{1+\exp(-z)}{1+(c_i+\frac{\varepsilon}{\mu^2}) \exp(-z)}\right) \nonumber\\
    & = \ln \left( \frac{1+(c_i+\frac{\varepsilon}{\mu^2}) \exp(-z)}{1+c_i \exp(-z)}\right) \nonumber \\
    & \leq \ln\left( \frac{c_i+\frac{\varepsilon}{\mu^2}}{c_i}\right) =  \ln\left( 1+ \frac{\varepsilon}{c_i \mu^2}\right) \leq \frac{\varepsilon}{c_i \mu^2} \leq \frac{\varepsilon \kappa}{\mu^2},
    \label{otherdirection:sigmoid}
\end{align}
where we assume that $1/\kappa$ is a constant lower bound for all $c_i$, see the discussion on parameters $c_i$ in \secref{preliminaries}.

For the logistic functions $g_j$ it holds that
\begin{align}
    g_j\Big(z,c_i+\frac{\varepsilon}{\mu^2}\Big) - g_j(z,c_i) & = -\ln\Big(1-c_i-\frac{\varepsilon}{\mu^2}\Big) + \ln(1-c_i) \nonumber\\
    &= \ln \left( 1+\frac{\frac{\varepsilon}{\mu^2}}{1-c_i-\frac{\varepsilon}{\mu^2}} \right) \leq \frac{\frac{\varepsilon}{\mu^2}}{1-c_i-\frac{\varepsilon}{\mu^2}} \leq \frac{4\varepsilon}{\mu^2},
    \label{otherdirection:logistic}
\end{align}
since $c\leq 1/2$ and $\varepsilon/\mu^2\leq 1/4$. We may assume that $\kappa \geq 4$. Then, Equations~(\ref{otherdirection:sigmoid}) and (\ref{otherdirection:logistic}) imply that
\begin{align*}
    \Big|f\Big(X\alpha_i, c_i\Big) - f\Big(X\alpha_i, c_i+\frac{\varepsilon}{\mu^2}\Big)\Big|  
    & \leq n'\cdot \frac{\varepsilon \kappa}{\mu^2} + n''\cdot \frac{4\varepsilon}{\mu^2} \leq \kappa\varepsilon \frac{n}{\mu^2} \\
    & \leq 2\kappa\varepsilon (1+2\mu_0)\frac{n' }{2\mu^2} \leq 6\kappa\cdot\varepsilon f(X\alpha_i, c_i),
\end{align*}
where the last two inequalities follow from \lemref{labels:bound} and  \lemref{logistic:sum:lowerbound} (since $\ln( 1+\exp( x_j \alpha_i)) \leq g_j(x_j\alpha_i, c_i)$). Rescaling $\varepsilon$ by the constant $6\kappa$ completes the proof.
\end{proof}

Then, we can obtain coresets for the case where we wish to optimize the item parameters on a reduced number of examinees using the following corollary.

\begin{corollary}
\corlab{3pl:main:otherdirection}
Let each $X_{(i)}=( -Y_{ij} \beta_j^T)_{j\in[n]}\in\REAL^{n\times 2}$. Let $X'_{(i)}$ contain the columns $j$ of $X_{(i)}$ where $Y_{ij}=-1$ and let $X''_{(i)}$ comprise the columns with $Y_{ij}=1$. Let $X'_{(i)}$ be $\mu$-complex and $X''_{(i)}$ be $\mu$-complex for each $i\in[m]$. Let $\varepsilon\in (0,1/4)$. There exists a weighted set $K\in \REAL^{k\times 2}$ of size 
$k\in O(\frac{\mu^4 \sqrt{n}}{\varepsilon^3} (\log(n)^2 + \log(m)))$, that is a $(1+\varepsilon)$-coreset for all $X_{(i)}$, $i\in[m]$ simultaneously for the 3PL IRT problem. The coreset can be constructed with constant probability and in $O(n)$ time.
\end{corollary}
\begin{proof}
The correctness and the running time of the corollary follow from \thmref{3pl:main:result} with reversed roles of $n$ and $m$, and with the following adaptations. 

The claims on the sensitivity bounds can be taken verbatim, since they hold uniformly for arbitrary values of $c_i\in [0,1/2)$.

To bound the VC dimension of the induced range spaces we divide the interval $[0,1/2)$ that contains all $c_i$ into a grid of $O(\mu^2/\varepsilon)$ segments of length no larger than $\varepsilon'=\varepsilon/(6\kappa\mu^2)$, and round up each $c_i$ to the closest point on the grid (cutting off at $1/2$). Hereby, each $c_i$ is approximated by an additive error of at most $\varepsilon'$, and the function $f(X\alpha_i,c_i)$ is approximated by a multiplicative error $1+\varepsilon$ using \lemref{additiveerror}.

Then we construct a partition into $O(\frac{\mu^2}{\varepsilon} \log(n) )$ classes, as in \cref{lem:disjVCbound,cor:VC:union}, such that the functions in each class have the same type $g_j$ or $h_j$, the same grid value $\hat{c}_i$ as a discretization of $c_i$, and the same weight. We obtain that the VC dimension of the induced range space is bounded by $O(\frac{\mu^2}{\varepsilon} \log(n) )$.

Rounding up the guessing parameters $c_i$ causes an additional multiplicative error $(1+\varepsilon)$. Since $(1+\varepsilon)^2\leq 1+3\varepsilon$, we rescale $\varepsilon''=\varepsilon/3$ to obtain the claim of the corollary.
\end{proof}

\subsection{On the Quality of the Solution Found on a Coreset}
\label{sec:quality:coreset}

\thmref{2pl:main:result} and \thmref{3pl:main:result}
guarantee that the values of the IRT loss functions evaluated on the whole input set and on the coreset, respectively, differ at most by an $\varepsilon$-fraction of the optimal value of the IRT loss function of the whole set. Here we show that the parameters that realize the optimal values of the loss function on the whole input and on the coreset are also close to each other.

To this end, for any given matrix $M\in \mathbb{R}^{n\times d}$, let $\sigma^{(1)}_{\min} (M)= \inf_{x\in \mathbb{R}^d\setminus\lbrace 0\rbrace} \frac{\|Mx\|_1}{\|x\|_1}$ \citep[cf.][]{Golub2013matrix}. Recall that the loss function $f(X\eta)$ for 3PL models is represented by the sum of different functions $g_i(z)$ and $h_i(z)$, where $g_i(z)$ was lower bounded by $z$ by \lemref{3pl:reg:linbound} for all $z\geq 0$. For 2PL models, we have $h_i(z)=g_i(z)$ since $c_i=0$ for all items. From \lemref{additiveerror}, we have that the coreset produces $c_i$ that are within $O(\frac{\varepsilon}{\mu^2})$ to the corresponding optimal value. The following theorem handles the remaining parameters, conditioned on an arbitrary choice of all other parameters, in particular also for the optimal set of parameters.

\begin{theorem}
\label{thm:quality:coreset}
Let $X$ be any matrix that satisfies the conditions and $\mu$-assumptions of Theorem 3.2 resp. 3.3, and let $K$, weighted by $u\in\REAL^k$ be any $(1+\varepsilon)$-coreset for $X$. Let $\eta_{\text{opt}}$ and $\eta_{\text{core}}$ be the minimizer of the IRT loss function $f(X\eta)$ and $f_u(K\eta)$, respectively. Let $\tau=1$ for the 2PL resp. $\tau=2$ for the 3PL model.
Then
\begin{align*}
    \|\eta_{\text{opt}} - \eta_{\text{core}} \|_1 \leq \frac{(1+\mu)^\tau(2+3\varepsilon)}{\sigma^{(1)}_{\min} (X)} \cdot f(X \eta_{\text{opt}}).
\end{align*}
\end{theorem}
\begin{proof}
The coreset definition implies that $f(X\eta_{\text{core}}) \leq (1+3\varepsilon) \cdot f(X\eta_{\text{opt}})$. Further, we have for the 3PL model that
\begin{align*}
    \sigma^{(1)}_{\min} (X) \cdot  \|\eta_{\text{opt}} - \eta_{\text{core}} \|_1 & \leq  \|X\eta_{\text{opt}} - X\eta_{\text{core}} \|_1\\
    & \leq  \|X\eta_{\text{opt}}\|_1 + \|X\eta_{\text{core}} \|_1\\
    & \leq  (1+\mu)\left( \|(X\eta_{\text{opt}})^+\|_1 +  \|(X\eta_{\text{core}})^+ \|_1 \right)\\
    & \stackrel{(*)}{\leq}  (1+\mu)^2\left( \|(X'\eta_{\text{opt}})^+\|_1 +  \|(X'\eta_{\text{core}})^+ \|_1 \right)\\
    & = (1+\mu)^2\cdot ( \sum_{x_i\in X', x_i\eta_{\text{opt}}>0} |x_i\eta_{\text{opt}}| +  \sum_{x_i\in X', x_i\eta_{\text{core}}>0} |x_i\eta_{\text{core}}| )\\
    & \leq (1+\mu)^2\cdot ( \sum_{x_i\in X', x_i\eta_{\text{opt}}>0} g_i(x_i\eta_{\text{opt}}) +  \sum_{x_i\in X', x_i\eta_{\text{core}}>0} g_i(x_i\eta_{\text{core}}) )\\
    & \leq (1+\mu)^2\cdot ( f(X\eta_{\text{opt}}) +   f(X\eta_{\text{core}}) )\\
    & \leq (1+\mu)^2\cdot ( f(X\eta_{\text{opt}}) +   (1+3\varepsilon) f(X\eta_{\text{opt}}) )\\
    & = (1+\mu)^2\cdot (2+3\varepsilon)\cdot f(X\eta_{\text{opt}}). 
\end{align*}
Finally, for the 2PL model, the additional factor of $(1+\mu)$ in the line tagged with $(*)$ is not necessary since $X=X'$. Thus, the claim holds in both cases.
\end{proof}

\begin{lemma}
\lemlab{coreset:error:approx}
    Let $K$, weighted by the non-negative weights $u\in\REAL^k$, be any coreset for $X$ for the function $f_w$. Let $\varepsilon\in (0,1/2)$. Let $\hat{\eta}\in \argmin_{\eta\in \REAL^d} f_u (K\eta)$. Then it holds that 
    \begin{align*}
        f_{w}(X\hat{\eta}) \leq (1+4\varepsilon) \min_{\eta\in \REAL^d} f_{w}(X\eta).
    \end{align*}
\end{lemma}

\begin{proof}
    Let $\eta^* \in \argmin_{\eta\in \REAL^d} f_{w}(X\eta)$. Then we have that
    \begin{align*}
        f_{w}(X\hat{\eta}) &\leq \frac{1}{1-\varepsilon}\cdot f_u (K\hat{\eta})  \leq \frac{1}{1-\varepsilon} \cdot f_u (K\eta^*)  \leq \frac{1+\varepsilon}{1-\varepsilon} \cdot f_{w} (X\eta^*) \leq (1+4\varepsilon) \cdot f_{w} (X\eta^*)  
    \end{align*}
    The first and the third inequality follow from the coreset property (\defref{coresets} and Eq.~(\ref{eqn:coreset:property})). The second inequality follows from the fact that $\hat{\eta}$ minimizes $f_u(K\eta)$ over all possible $\eta\in \REAL^d$. The last inequality follows from $\varepsilon\in (0,1/2)$.
\end{proof}

\clearpage
\section{ADDITIONAL EXPERIMENTAL RESULTS}
\label{appendix:experiments}

See \cref{tab:results_appendix,tab:results_appendix:b,tab:results_appendix3,tab:results_appendix3:b,tab:results_appendix4,tab:results_appendix4:b} and \cref{fig:param_exp_appendix,fig:param_exp_appendix2,fig:param_exp_appendix3,fig:param_exp_appendix4,fig:param_exp_appendix9,fig:param_exp_appendix_pareto,fig:param_exp_appendix5,fig:param_exp_appendix8,fig:param_exp_appendix6,fig:param_exp_appendix7} for additional experimental results on the parameter estimation accuracy along with the results already reported in the main paper.
\begin{table*}[hp!]
\caption{2PL Experiments on synthetic data: The means and standard deviations (std.) of running times, taken across $20$ repetitions. 
{In each repetition, the running time (in minutes) of 50 iterations of the main loop was measured}
per data set, and for different configurations of the data dimensions: the number of items $m$, the number of examinees $n$, and the coreset size $k$. The (relative) gain is defined as  $(1-\mathrm{mean}_{\sf coreset}/\mathrm{mean}_{\sf full})\cdot 100$ \%.
{The largest experiment was run only once, due to the large running time. Some measures thus do not apply, indicated by N/A values in the last row}.
}
	\label{tab:results_appendix}
	\begin{center}
		\begin{tabular}{ c r r r| r r| r r| r }
    &\multicolumn{3}{c}{}& \multicolumn{2}{c}{{{\bf Full data} (min)}}& \multicolumn{2}{c}{{{\bf Coresets} (min)}} & \multicolumn{1}{c}{ } \\
		\hline
		{\bf data} &
		{$\mathbf n$} & {$\mathbf m$} & {$\mathbf k$} &  {\bf mean} & {\bf std.} & {\bf mean} & {\bf std.} & {\bf gain} \\ \hline \hline
	2PL-Syn & $50\,000$ & $100$ & $100$ & $34.565$ & $5.220$ & $22.752$ & $3.692$ & $34.178$ \% \\ \hline
	2PL-Syn & $50\,000$ & $200$ & $500$ &  $65.745$ & $11.897$ & $30.121$ & $4.645$ & $54.185$ \% \\ \hline 
    2PL-Syn & $50\,000$ & $500$ & $500$ & $136.981$ & $12.556$ & $45.547$ & $3.863$ & $66.749$ \% \\ \hline\hline
    2PL-Syn & $100\,000$ & $100$ & $100$ & $75.135$ & $11.881$ & $51.029$ & $7.524$ & $32.084$ \% \\ \hline
    2PL-Syn & $100\,000$ & $200$ & $1\,000$ & $122.252$ & $12.043$ & $61.459$ & $10.654$ & $49.727$ \%\\ \hline
    2PL-Syn & $100\,000$ & $500$ & $1\,000$ & $231.276$ & $23.793$ & $80.861$ & $11.161$ & $65.037$ \% \\ \hline\hline
    2PL-Syn & $200\,000$ & $100$ & $1\,000$ & $155.053$ & $18.877$ & $99.352$ & $12.055$ & $35.924$ \% \\ \hline
    2PL-Syn & $200\,000$ & $200$ & $2\,000$ & $247.654$ & $34.069$ & $119.075$ & $13.717$ & $51.919$ \% \\ \hline
    2PL-Syn & $200\,000$ & $500$ & $2\,000$ & $466.832$ & $48.734$ & $169.494$ & $21.862$ & $63.693$ \%   \\ \hline\hline
    2PL-Syn & $500\,000$ & $100$ & $5\,000$ & $339.057$ & $115.382$ & $228.041$ & $75.920$ & $32.743$ \%\\ \hline
    2PL-Syn & $500\,000$ & $200$ & $5\,000$ & $518.274$ & $77.108$ & $291.678$ & $44.327$ & $43.721$ \% \\ \hline
    2PL-Syn & $500\,000$ & $500$ & $5\,000$ & $1\,278.845$ & $494.938$ & $591.878$ & $221.218$ & $53.718$ \% \\ \hline\hline 
    2PL-Syn & $500\,000$ & $5\,000$ & $5\,000$ & $9\,363.750$ & N/A & $5\,536.684$ & N/A & $40.871$ \% \\ \hline\hline
	\end{tabular}
	\end{center}
\end{table*}

\begin{table*}[hp!]
\caption{2PL Experiments on synthetic data: 
The quality of the solution found.
Let $f_{\sf full}$ and $f_{{\sf core}(j)}$ be the optimal values of the loss function on the input and on the coreset for the $j$-th repetition, respectively. Let $f_{\sf core} = \min_j f_{{\sf core}(j)}$.
Mean and standard deviation of the relative deviation $|f_{\sf core} - f_{{\sf core}(j)}| / f_{\sf core}$ (in $\%$): 
\textbf{mean dev} and \textbf{std. dev}. Relative error: \textbf{rel. error} $\hat{\varepsilon}=|f_{\sf core} - f_{\sf full}|/f_{\sf full}$ (cf. \lemref{coreset:error:approx}).
Mean Absolute Deviation: $\textbf{mad}(\alpha)=\frac{1}{n}\sum (|a_{\sf full}-a_{\sf core}| + |b_{\sf full}-b_{\sf core}|)$; $\textbf{mad}(\theta)=\frac{1}{m}\sum |\theta_{\sf full}-\theta_{\sf core}|$, evaluated on the parameters that attained the optimal $f_{\sf full}$ and $f_{\sf core}$. {The largest experiment was run only once, due to the large running time. Some measures thus do not apply, indicated by N/A in the last row}.
}
	\label{tab:results_appendix:b}
	\begin{center}
		\begin{tabular}{ c r r r| r r | r | c c}
		\hline
		{\bf data} &
		{$\mathbf n$} & {$\mathbf m$} & {$\mathbf k$} &  {\bf mean dev} & {\bf std. dev} & {\bf rel. error $\hat{\varepsilon}$}  & {\bf $\text{mad}(\alpha)$} & {\bf $\text{mad}(\theta)$}\\ \hline \hline
	2PL-Syn & $50\,000$ & $100$ & $100$ & $6.146$ \% & $2.178$ \% & $0.13452$ & $1.108$ & $0.045$ \\ \hline
	2PL-Syn & $50\,000$ & $200$ & $500$ & $2.241$ \% & $0.918$ \% & $0.05214$ & $0.508$ & $0.011$ \\ \hline 
    2PL-Syn & $50\,000$ & $500$ & $500$ & $1.533$ \% & $0.892$ \% & $0.04803$ & $0.525$ & $0.008$ \\ \hline\hline
    2PL-Syn & $100\,000$ & $100$ & $100$ & $7.203$ \% & $2.918$ \% & $0.14776$ & $0.970$ & $0.040$ \\ \hline
    2PL-Syn & $100\,000$ & $200$ & $1\,000$ & $1.086$ \% & $0.544$ \% & $0.03404$ & $0.379$ & $0.008$ \\ \hline
    2PL-Syn & $100\,000$ & $500$ & $1\,000$ & $0.999$ \% & $0.542$ \% & $0.03140$ & $0.345$ & $0.005$ \\ \hline\hline
    2PL-Syn & $200\,000$ & $100$ & $1\,000$ & $1.936$ \% & $0.849$ \% & $0.04400$ & $0.374$ & $0.008$ \\ \hline
    2PL-Syn & $200\,000$ & $200$ & $2\,000$ & $0.743$ \% & $0.411$ \% & $0.02375$ & $0.248$ & $0.003$ \\ \hline
    2PL-Syn & $200\,000$ & $500$ & $2\,000$ & $1.273$ \% & $0.565$ \% & $0.03013$ & $0.268$ & $0.002$ \\ \hline\hline
    2PL-Syn & $500\,000$ & $100$ & $5\,000$ & $0.551$ \% & $0.184$ \% & $0.01399$ & $0.142$ & $0.002$ \\ \hline
    2PL-Syn & $500\,000$ & $200$ & $5\,000$ & $0.731$ \% & $0.275$ \% & $0.01689$ & $0.180$ & $0.002$ \\ \hline
    2PL-Syn & $500\,000$ & $500$ & $5\,000$ & $0.473$ \% & $0.239$ \% & $0.01445$ & $0.171$ & $0.001$ \\ \hline\hline  
    2PL-Syn & $500\,000$ & $5\,000$ & $5\,000$ & N/A & N/A & $0.00076$ & $0.120$ & $0.013$ \\ \hline\hline
	\end{tabular}
	\end{center}
\end{table*}

\begin{figure*}[hp!]
\caption{2PL Experiments on synthetic data: Parameter estimates for the coresets compared to the full data sets. 
For each experiment the upper figure shows the bias for the item parameters $a$ and $b$. The lower figure shows a kernel density estimate for the ability parameters $\theta$ with a LOESS regression line in dark green.
The ability parameters were standardized to zero mean and unit variance. In all rows, the vertical axis is scaled such as to display $2\,{\mathrm{std.}}$ of the corresponding parameter estimate obtained from the full data set.}
\begin{center}
\begin{tabular}{ccc}
{\tiny{$\mathbf{n=50\,000,m=100,k=100}$}}&{\tiny{$\mathbf{n=50\,000,m=200,k=500}$}}&{\tiny{$\mathbf{n=50\,000,m=500,k=500}$}}
\\
\includegraphics[width=0.3\linewidth]{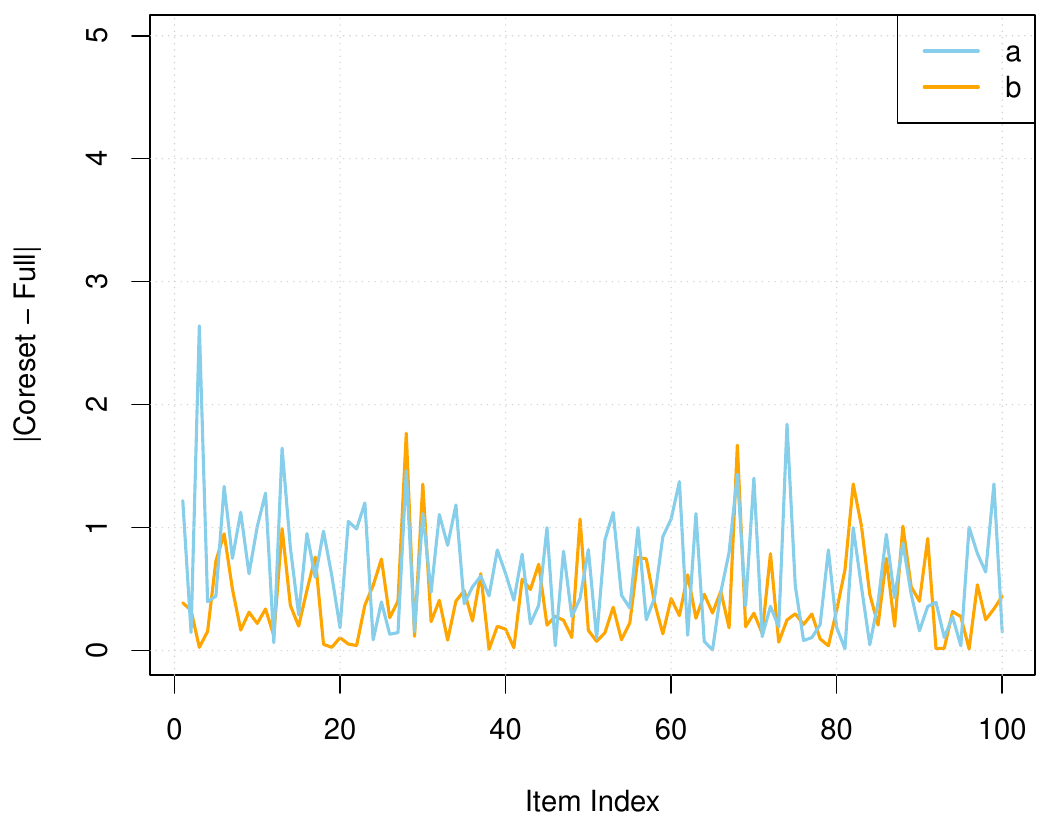}&
\includegraphics[width=0.3\linewidth]{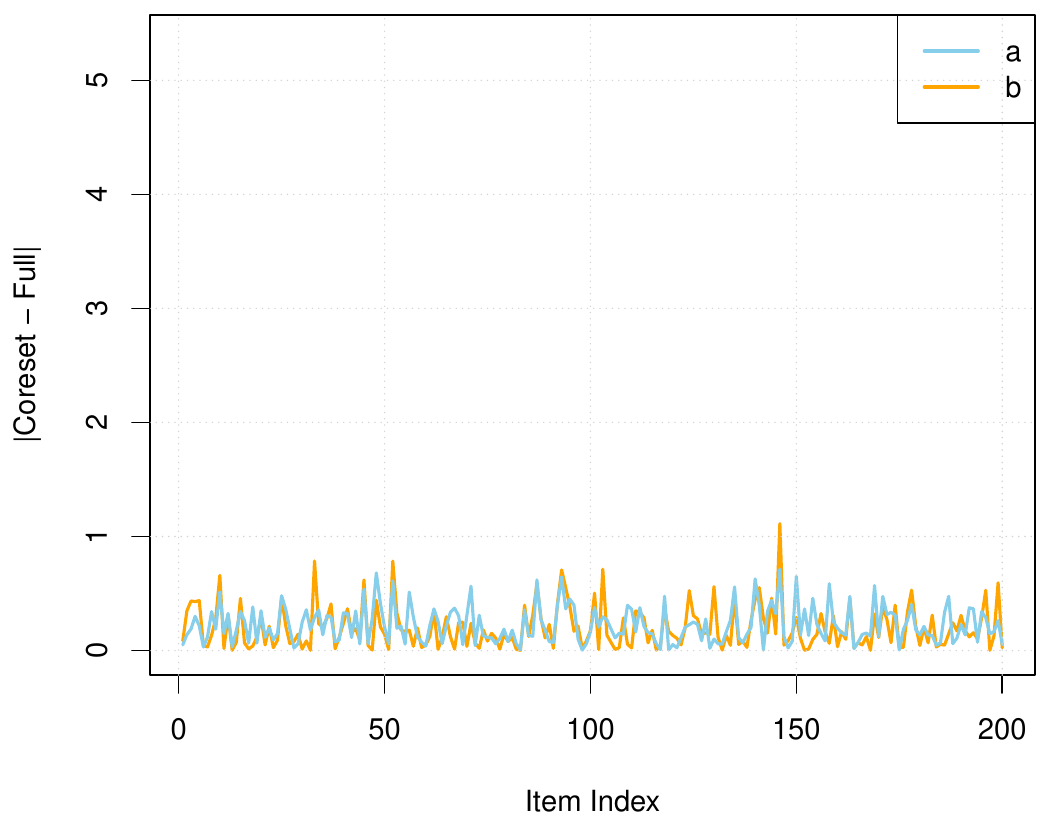}&
\includegraphics[width=0.3\linewidth]{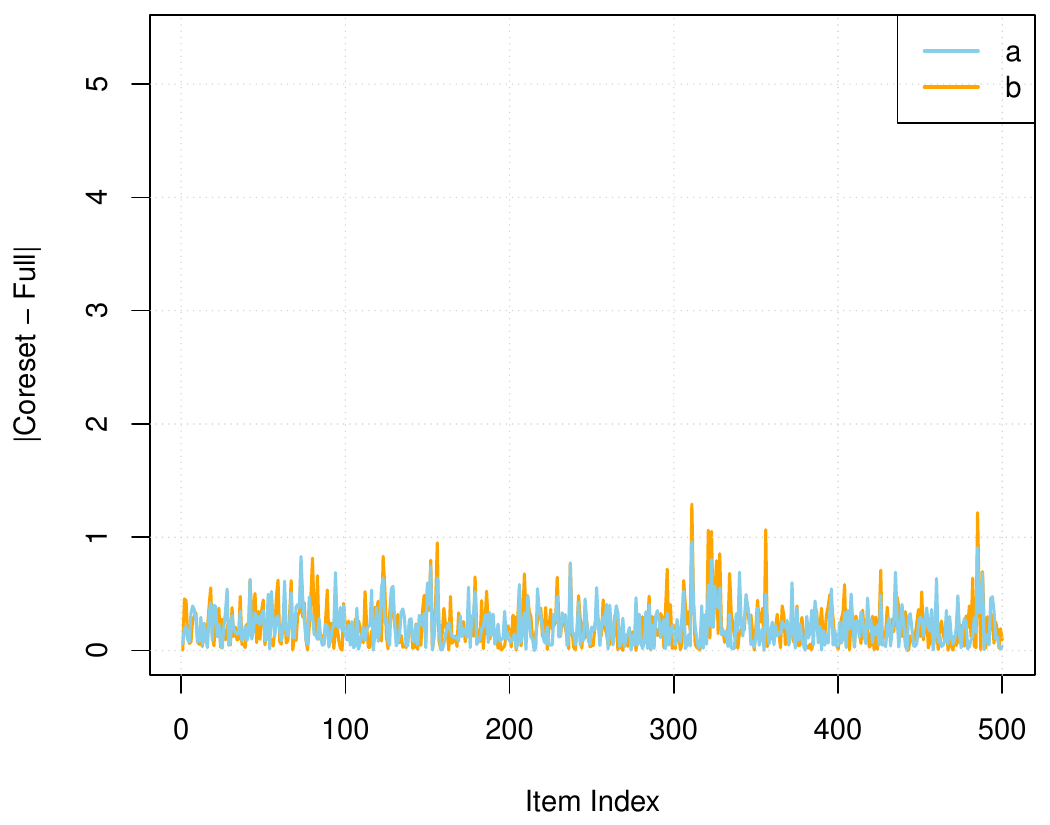}\\
\includegraphics[width=0.3\linewidth]{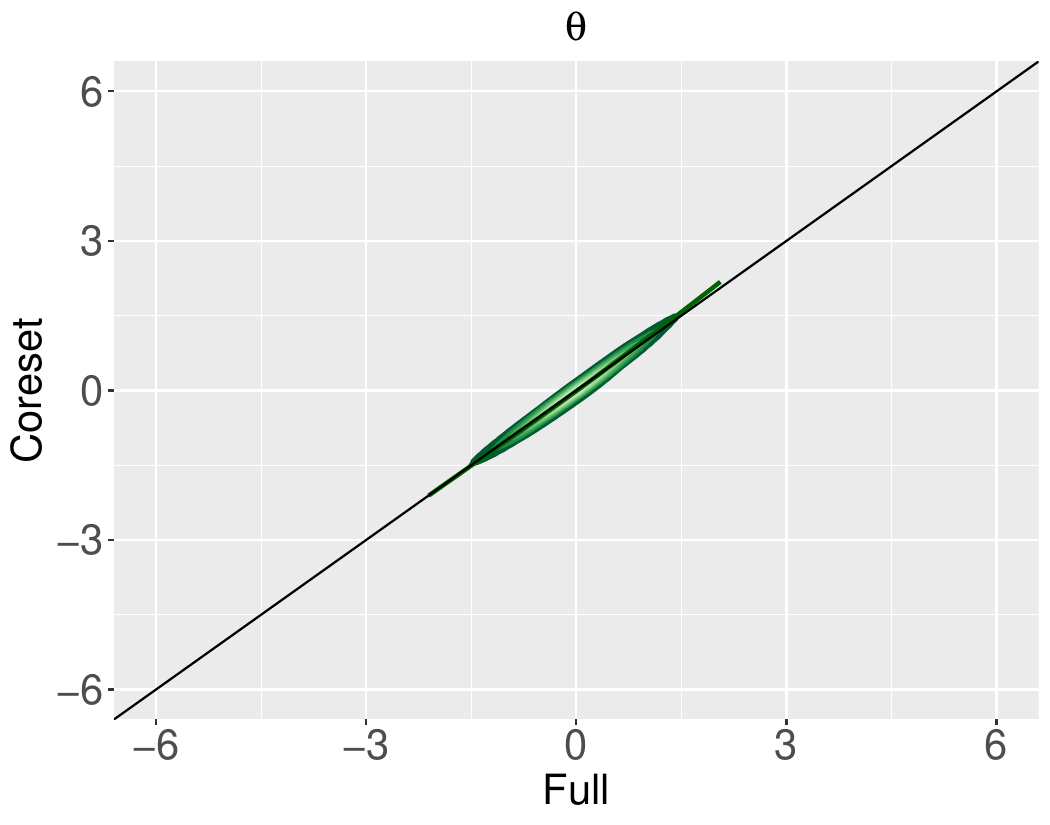}&
\includegraphics[width=0.3\linewidth]{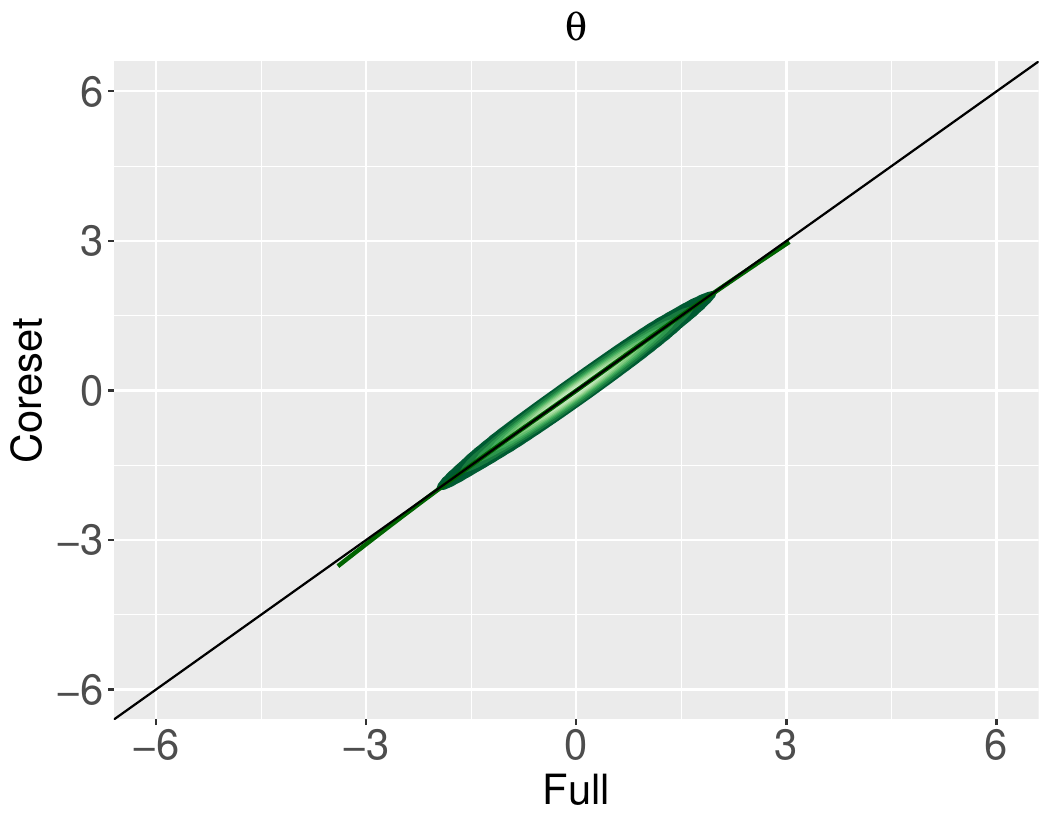}&
\includegraphics[width=0.3\linewidth]{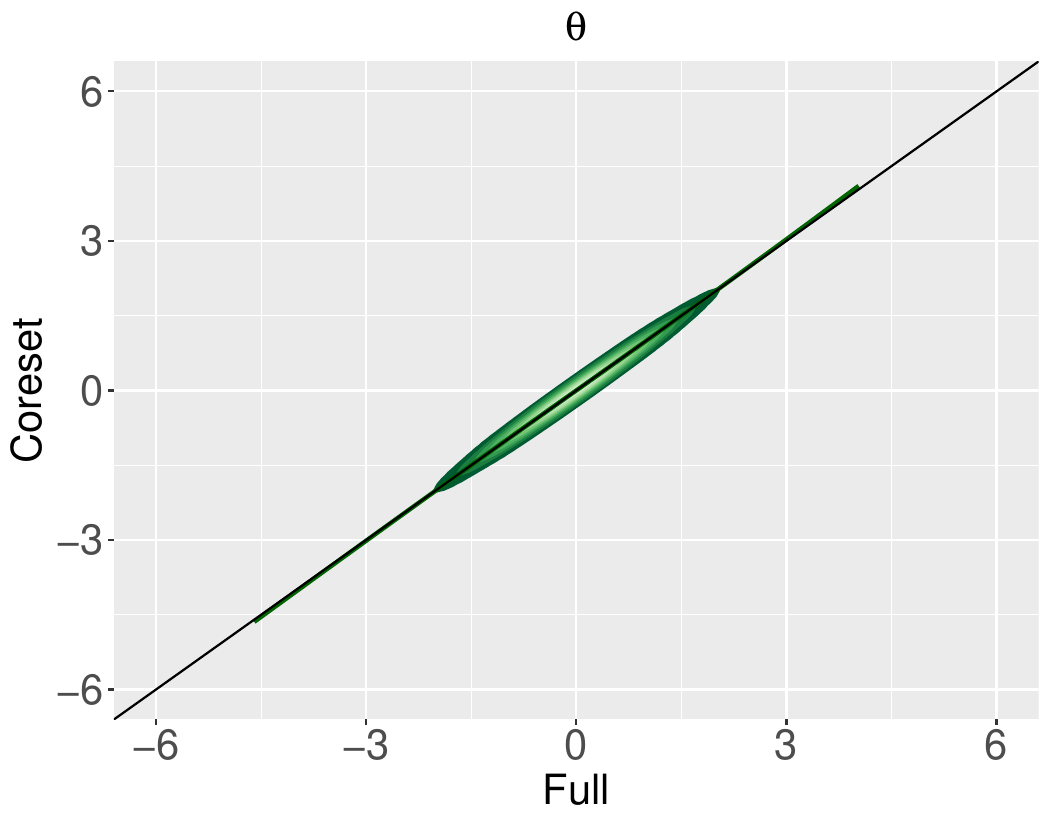}\\
{\tiny{$\mathbf{n=100\,000,m=500,k=500}$}}&{\tiny{$\mathbf{n=100\,000,m=200,k=1\,000}$}}&{\tiny{$\mathbf{n=100\,000,m=500,k=1\,000}$}}
\\
\includegraphics[width=0.3\linewidth]{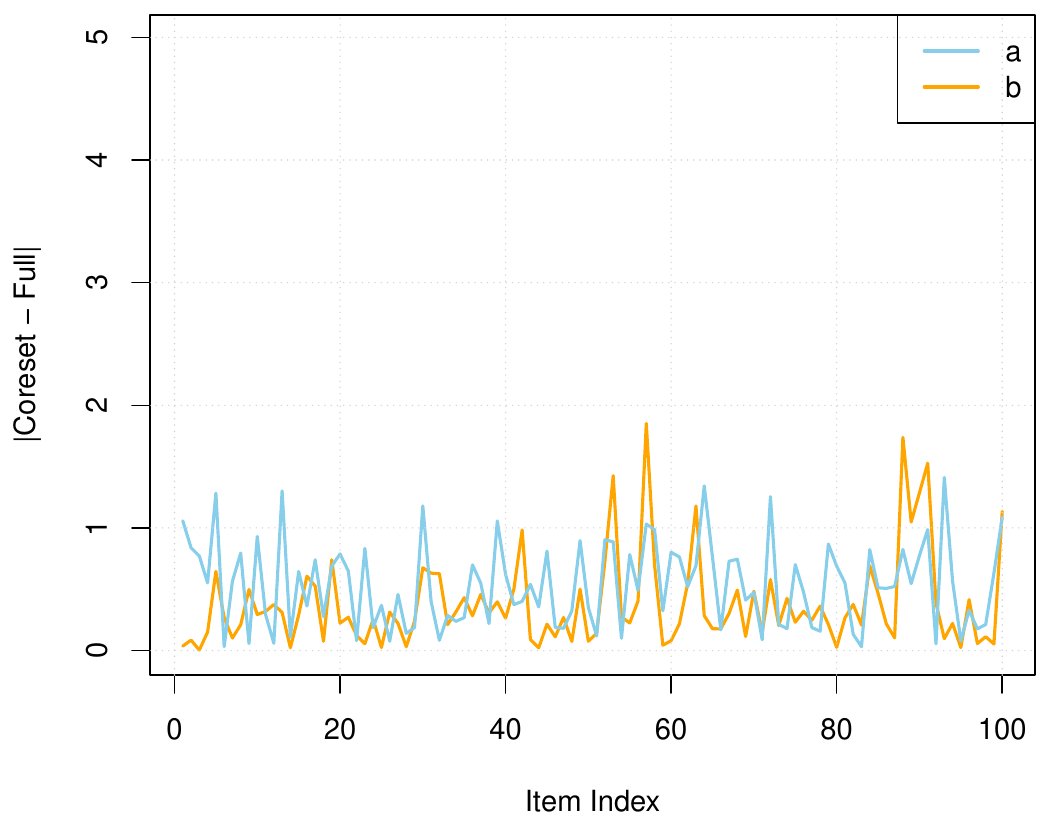}&
\includegraphics[width=0.3\linewidth]{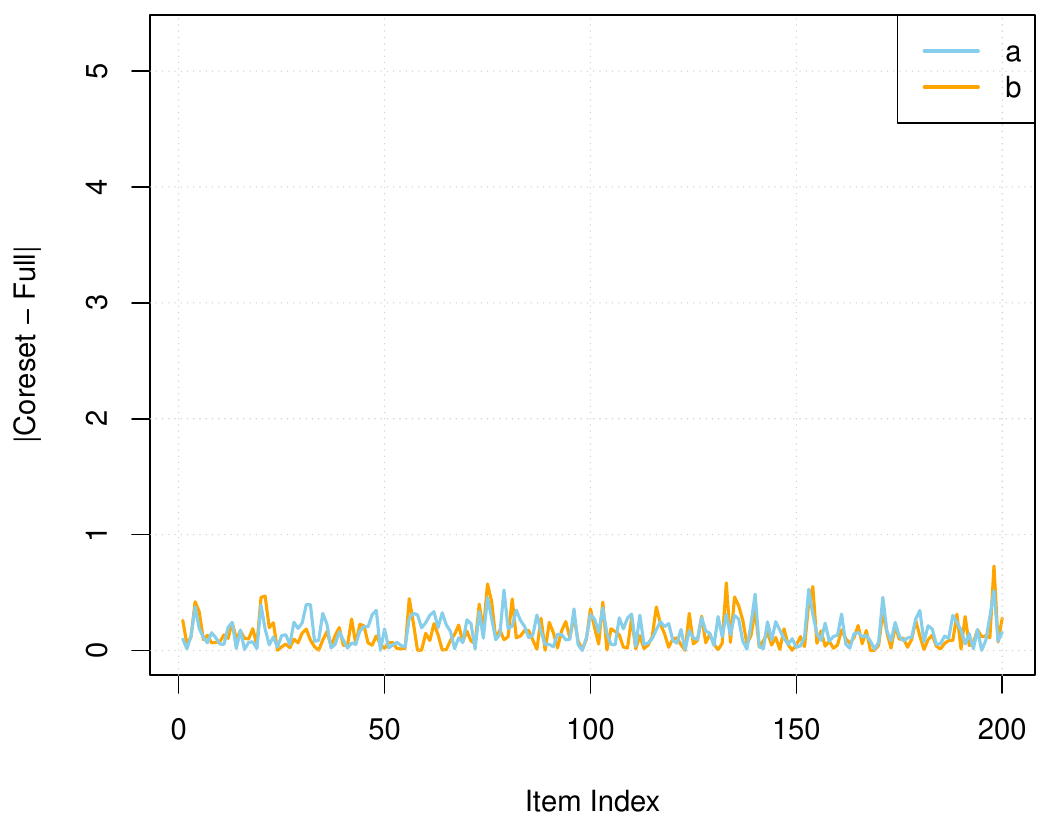}&
\includegraphics[width=0.3\linewidth]{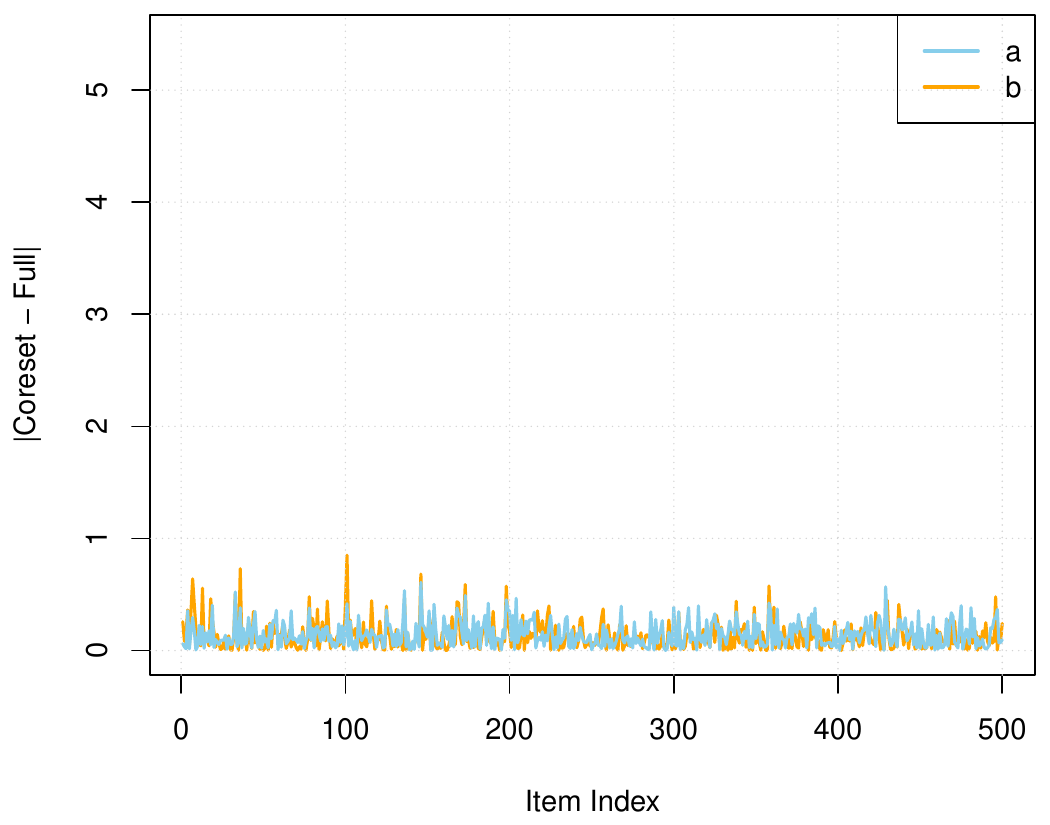}\\
\includegraphics[width=0.3\linewidth]{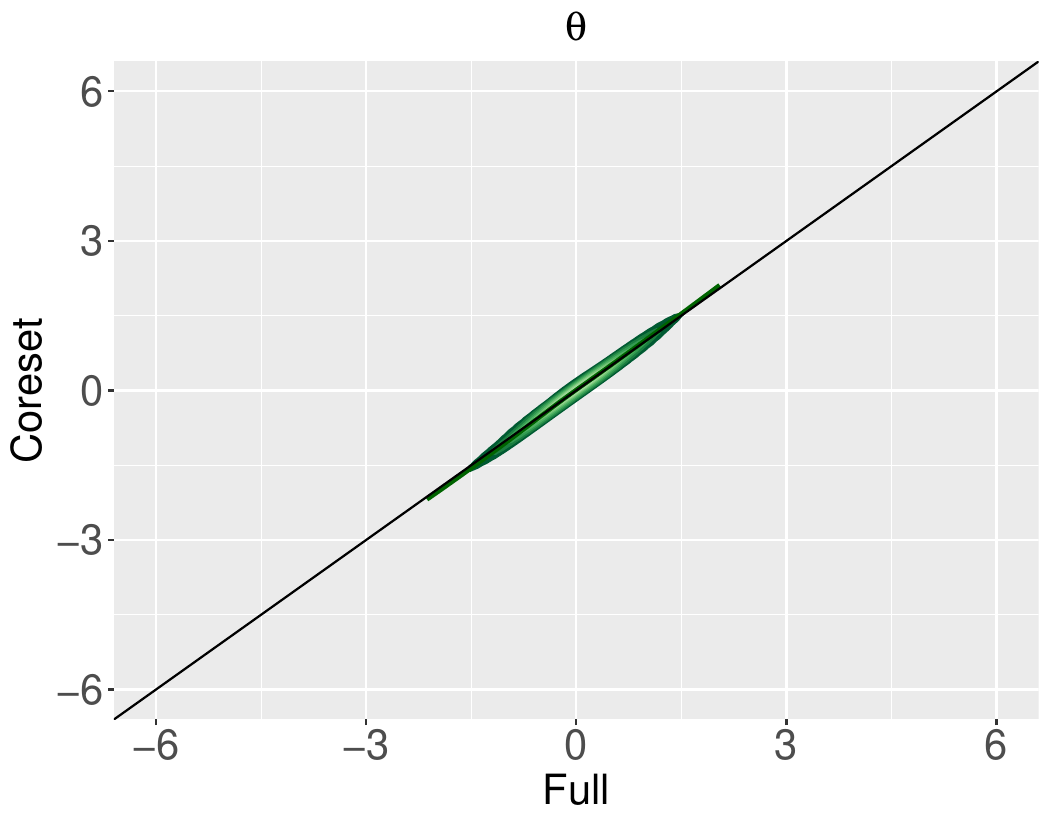}&
\includegraphics[width=0.3\linewidth]{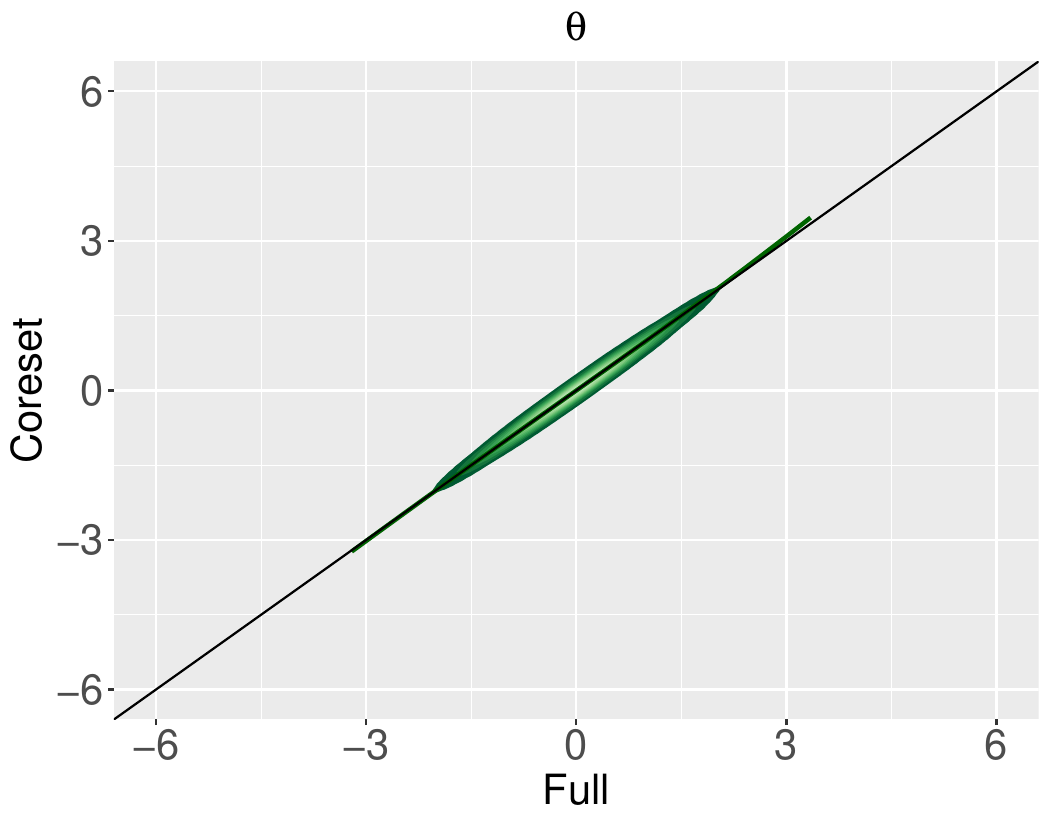}&
\includegraphics[width=0.3\linewidth]{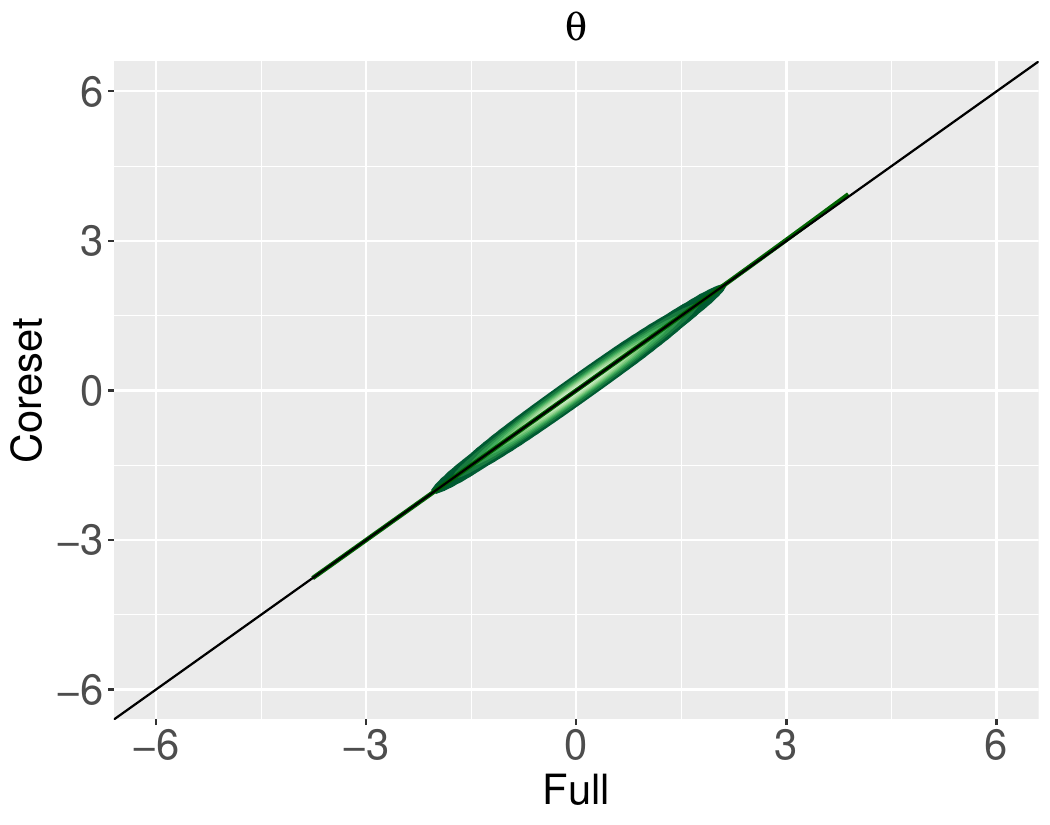}\\
\end{tabular}
\label{fig:param_exp_appendix}
\end{center}

\end{figure*}

\begin{figure*}[hp!]
\caption{2PL Experiments on synthetic data: Parameter estimates for the coresets compared to the full data sets. 
For each experiment the upper figure shows the bias for the item parameters $a$ and $b$. The lower figure shows a kernel density estimate for the ability parameters $\theta$ with a LOESS regression line in dark green.
The ability parameters were standardized to zero mean and unit variance. In all rows, the vertical axis is scaled such as to display $2\,{\mathrm{std.}}$ of the corresponding parameter estimate obtained from the full data set.}
\begin{center}
\begin{tabular}{ccc}
{\tiny{$\mathbf{n=200\,000,m=100,k=1\,000}$}}&{\tiny{$\mathbf{n=200\,000,m=200,k=2\,000}$}}&{\tiny{$\mathbf{n=200\,000,m=500,k=2\,000}$}}
\\
\includegraphics[width=0.3\linewidth]{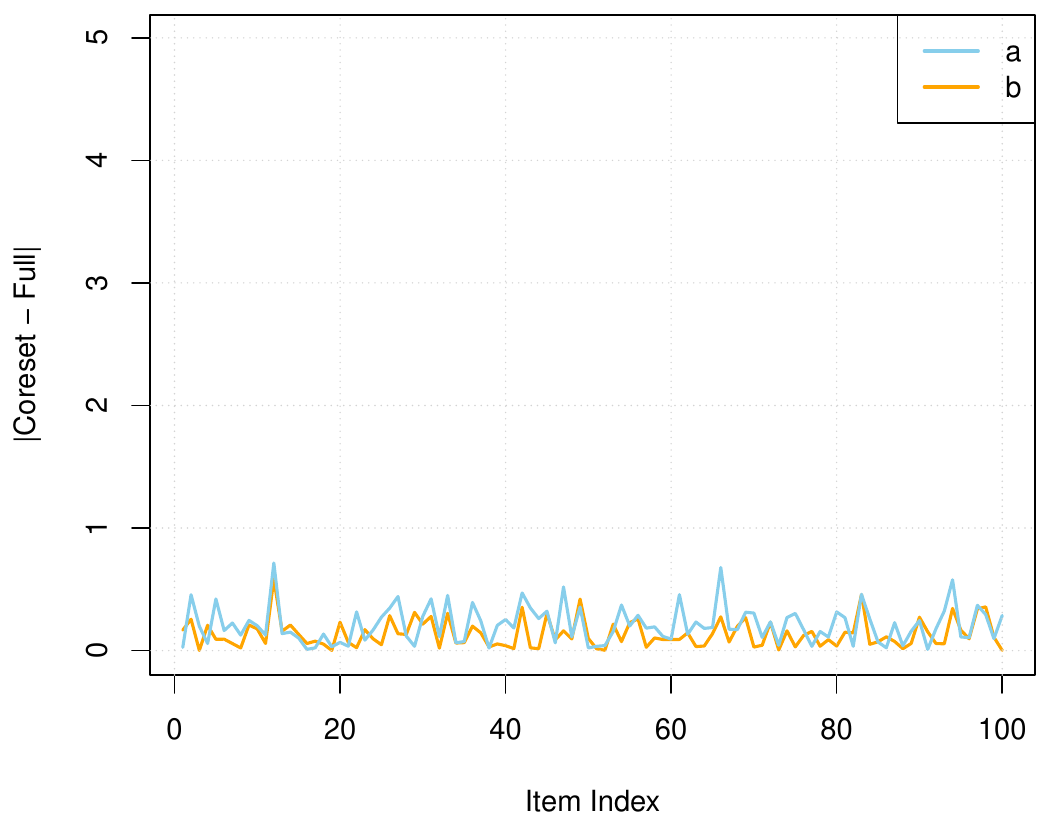}&
\includegraphics[width=0.3\linewidth]{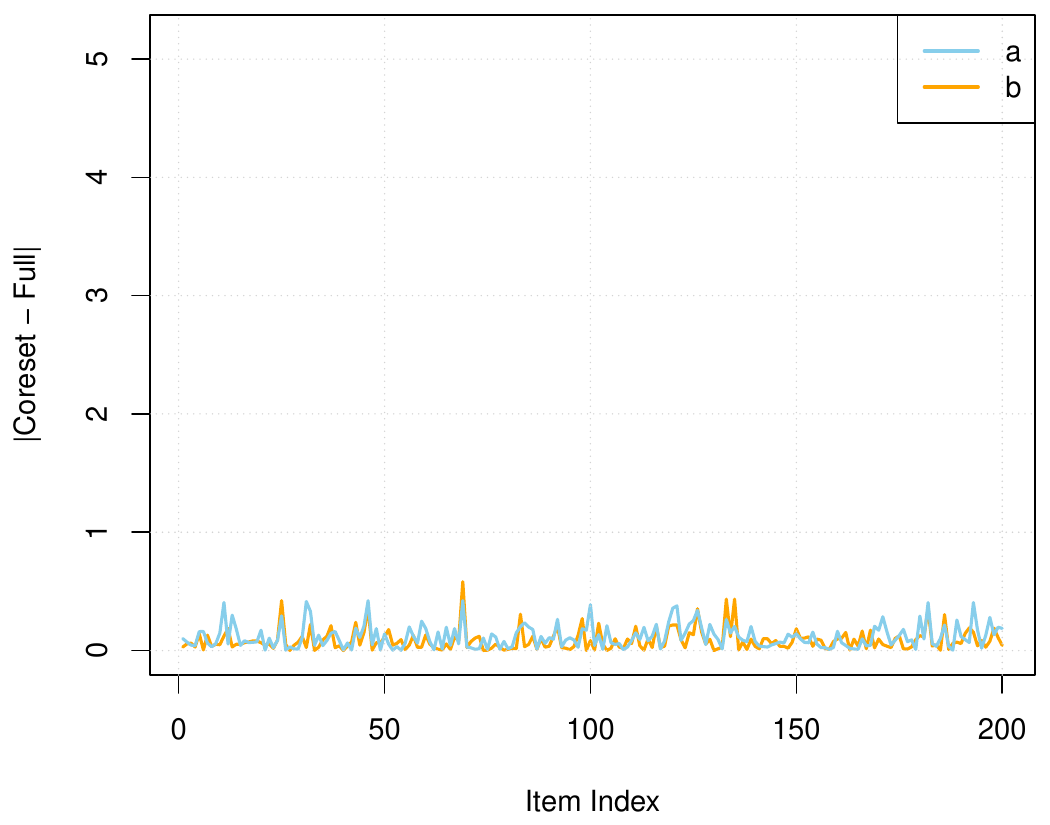}&
\includegraphics[width=0.3\linewidth]{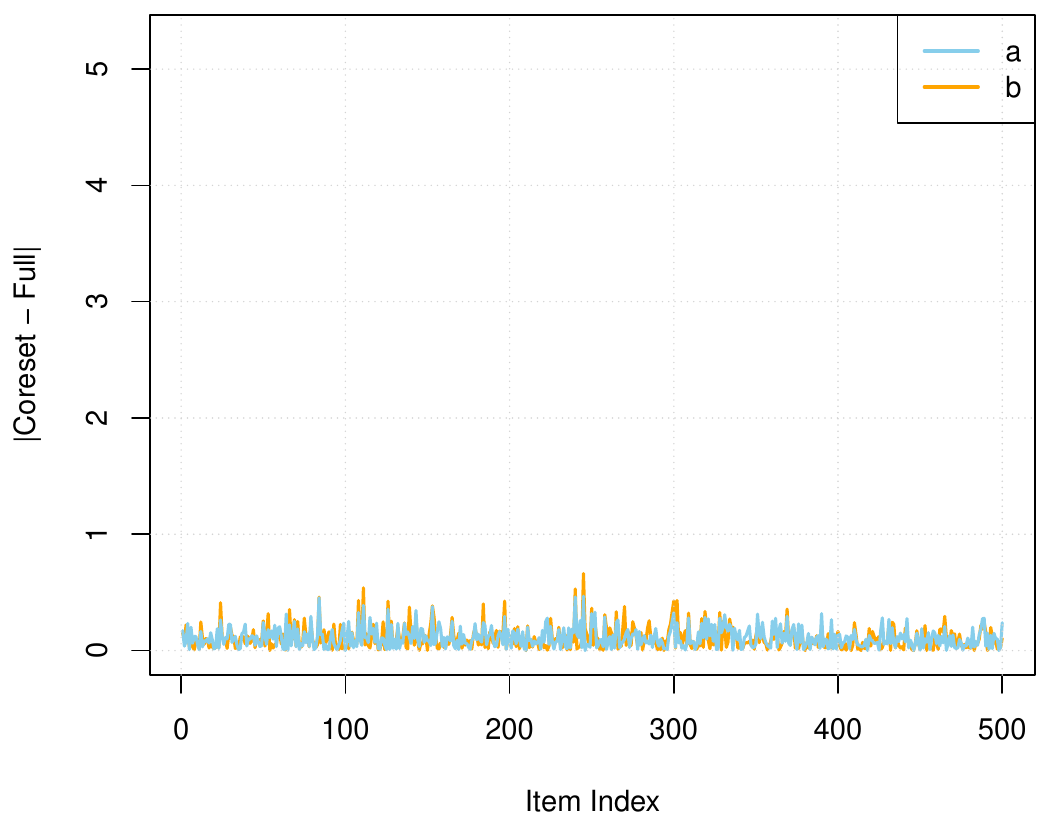}\\
\includegraphics[width=0.3\linewidth]{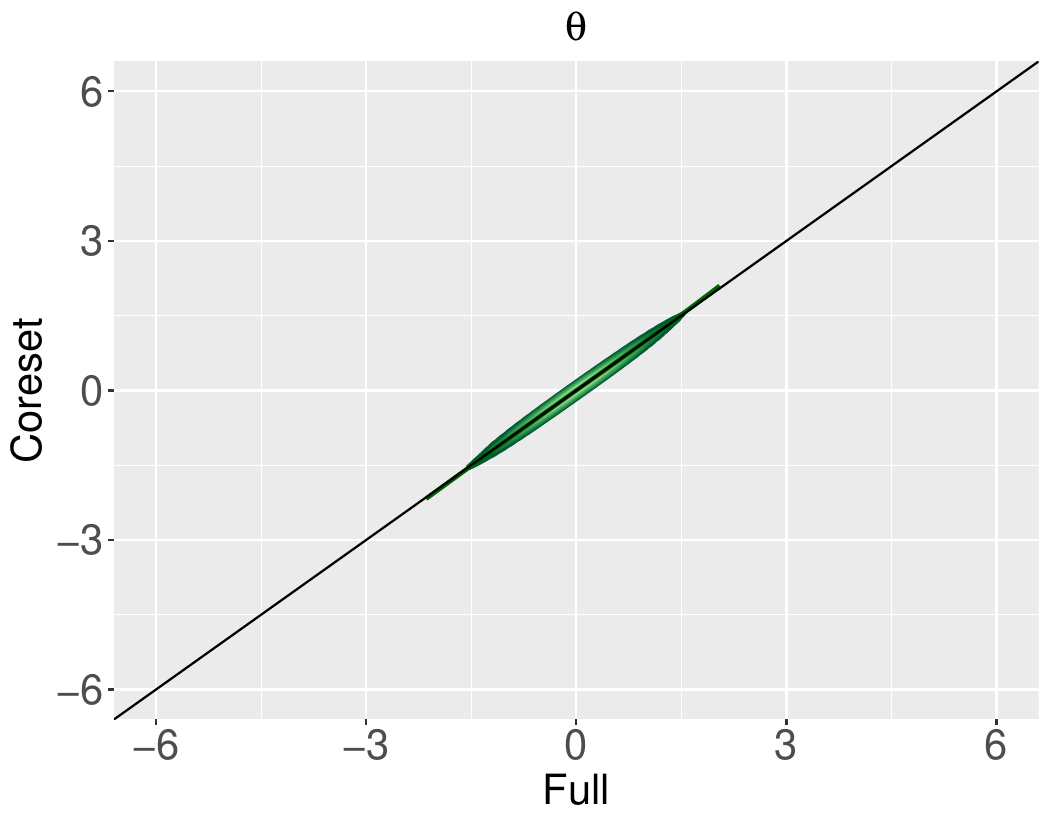}&
\includegraphics[width=0.3\linewidth]{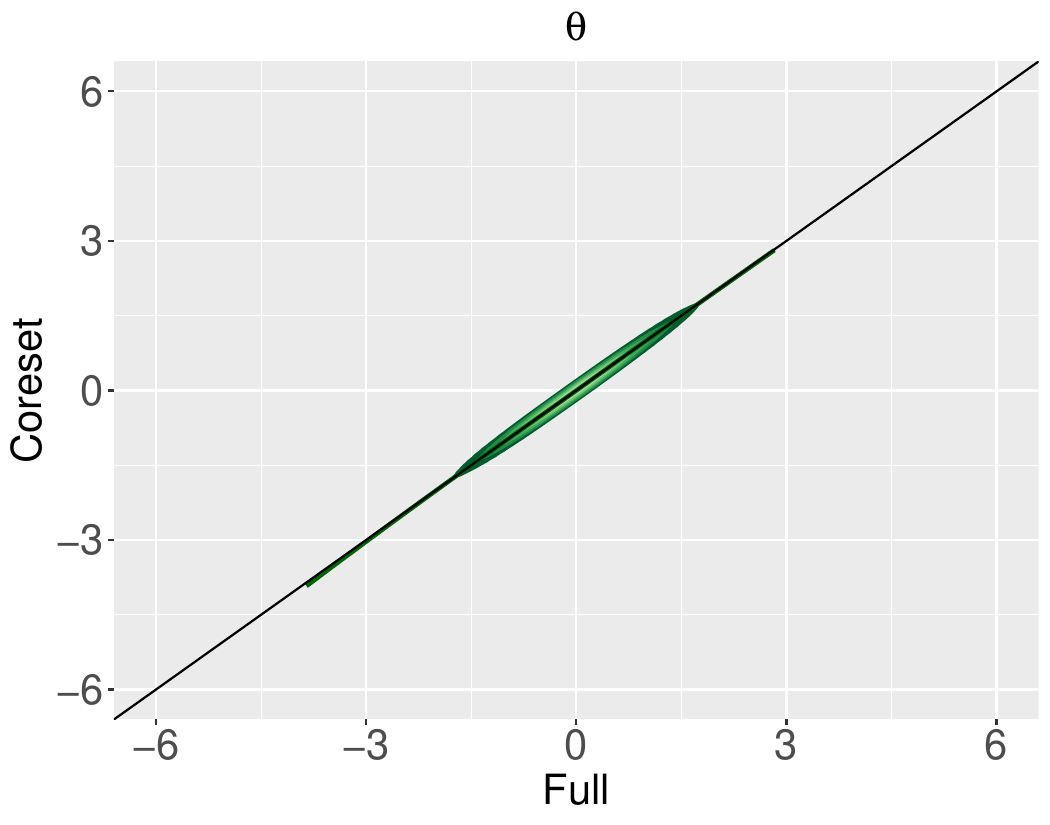}&
\includegraphics[width=0.3\linewidth]{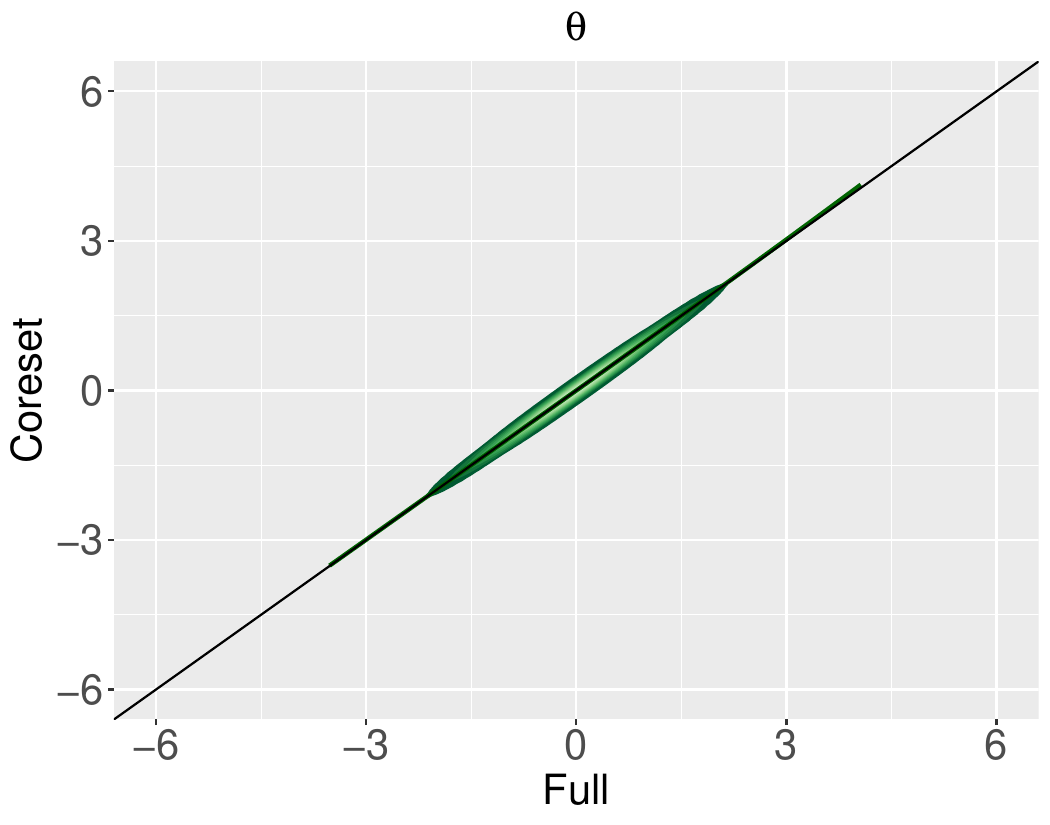}\\
{\tiny{$\mathbf{n=500\,000,m=100,k=5\,000}$}}&{\tiny{$\mathbf{n=500\,000,m=200,k=5\,000}$}}&{\tiny{$\mathbf{n=500\,000,m=500,k=5\,000}$}}
\\
\includegraphics[width=0.3\linewidth]{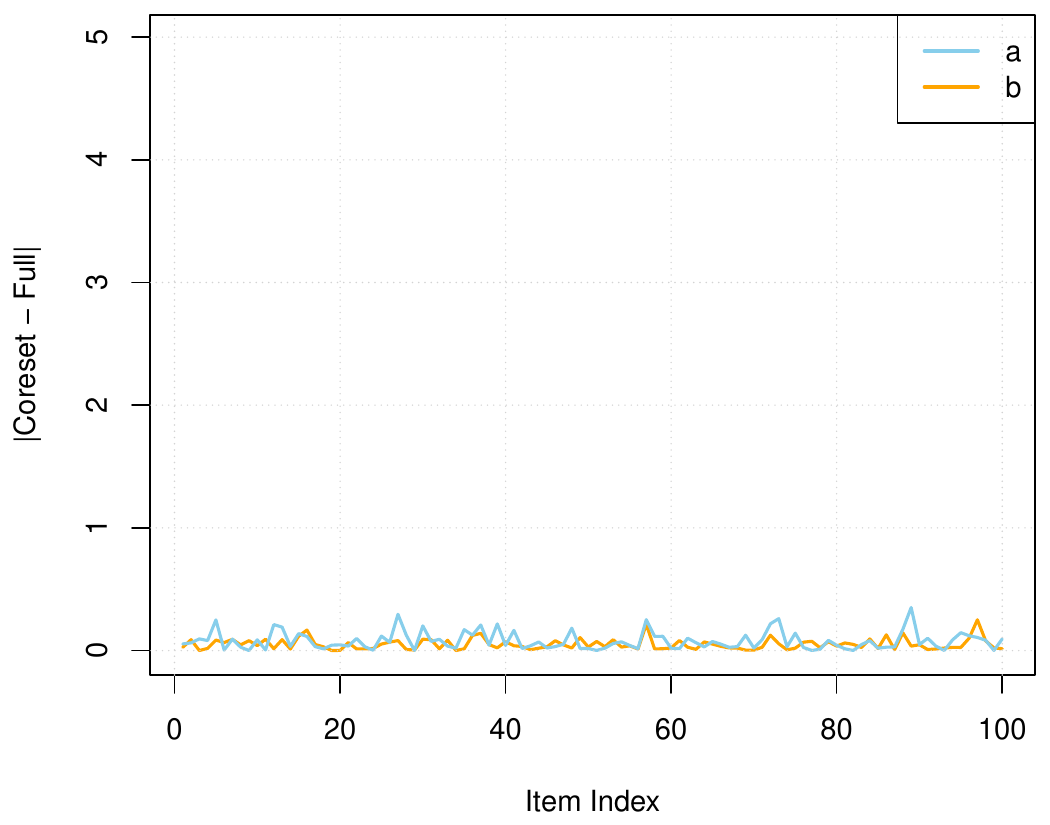}&
\includegraphics[width=0.3\linewidth]{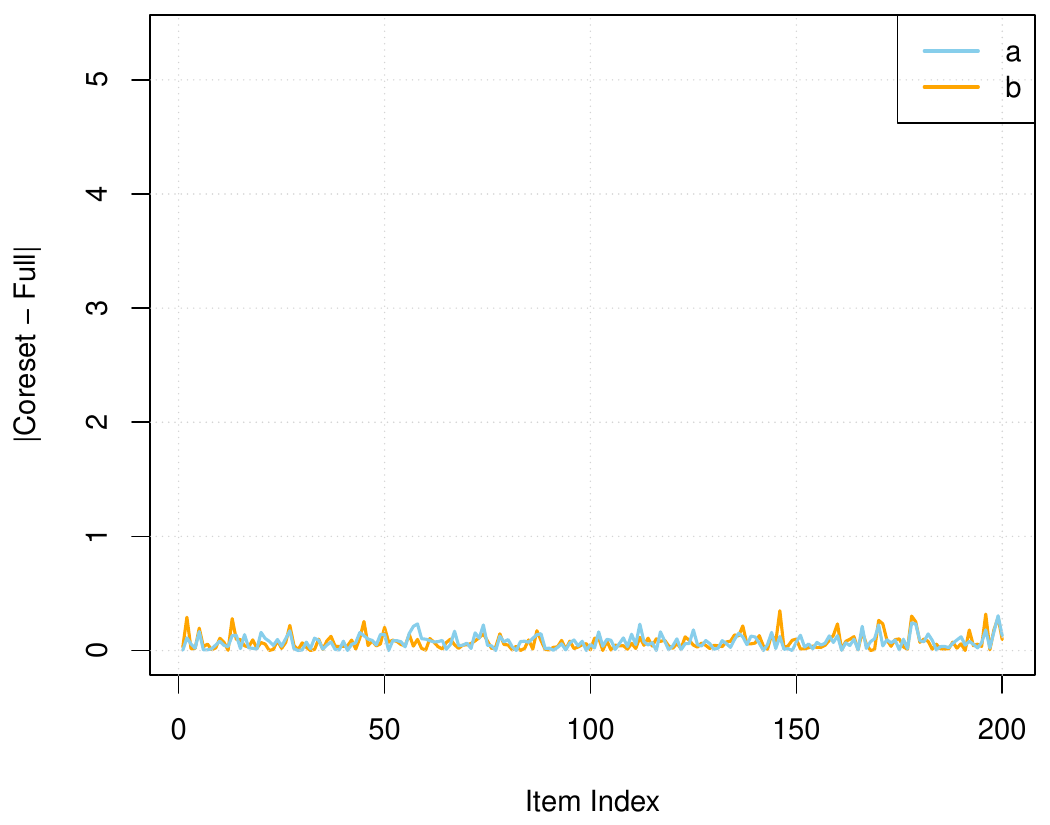}&
\includegraphics[width=0.3\linewidth]{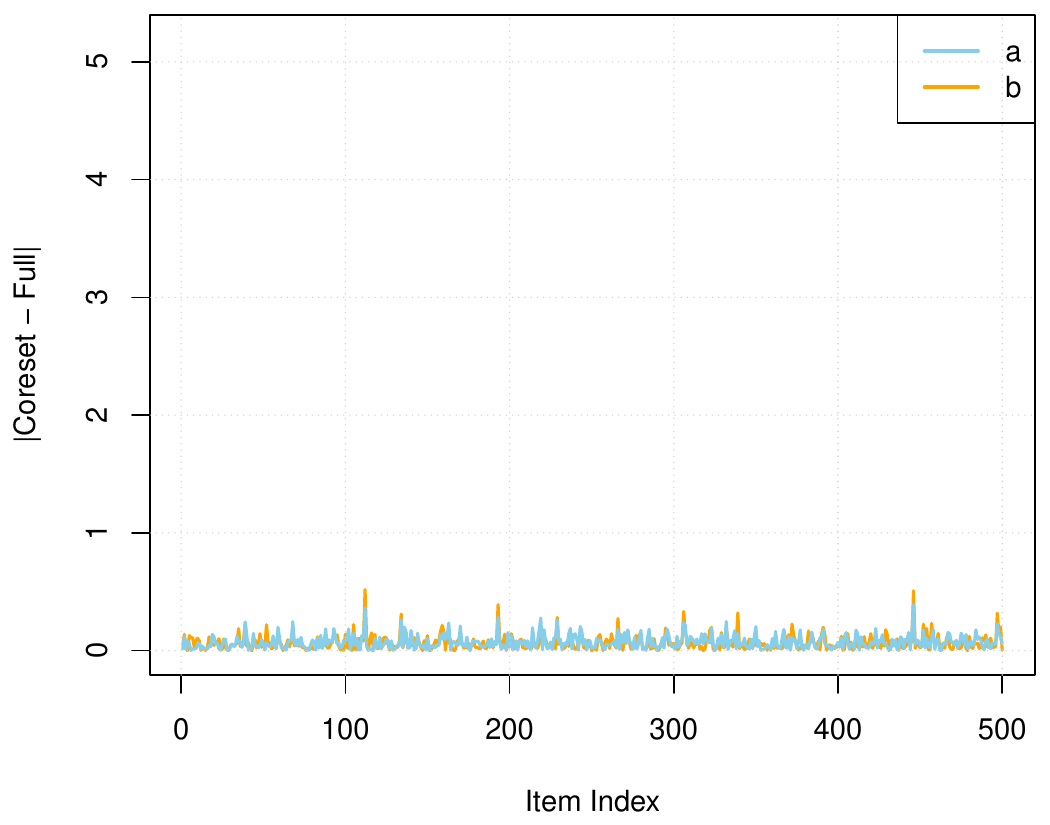}\\
\includegraphics[width=0.3\linewidth]{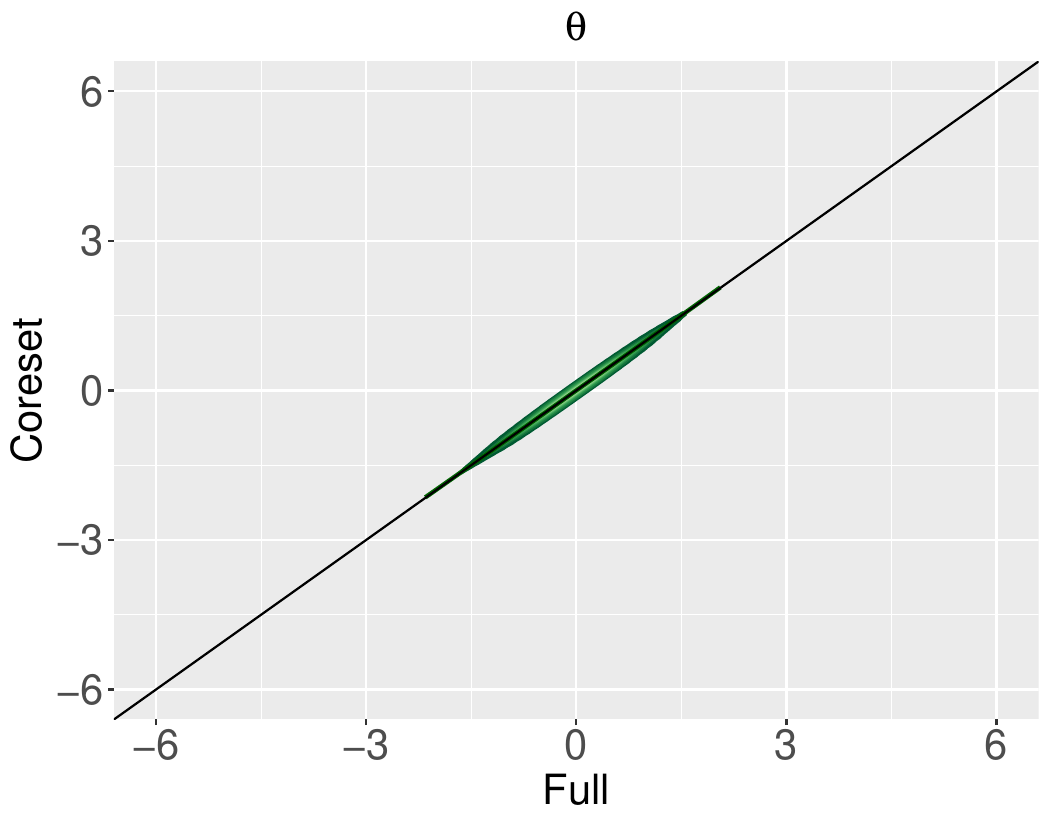}&
\includegraphics[width=0.3\linewidth]{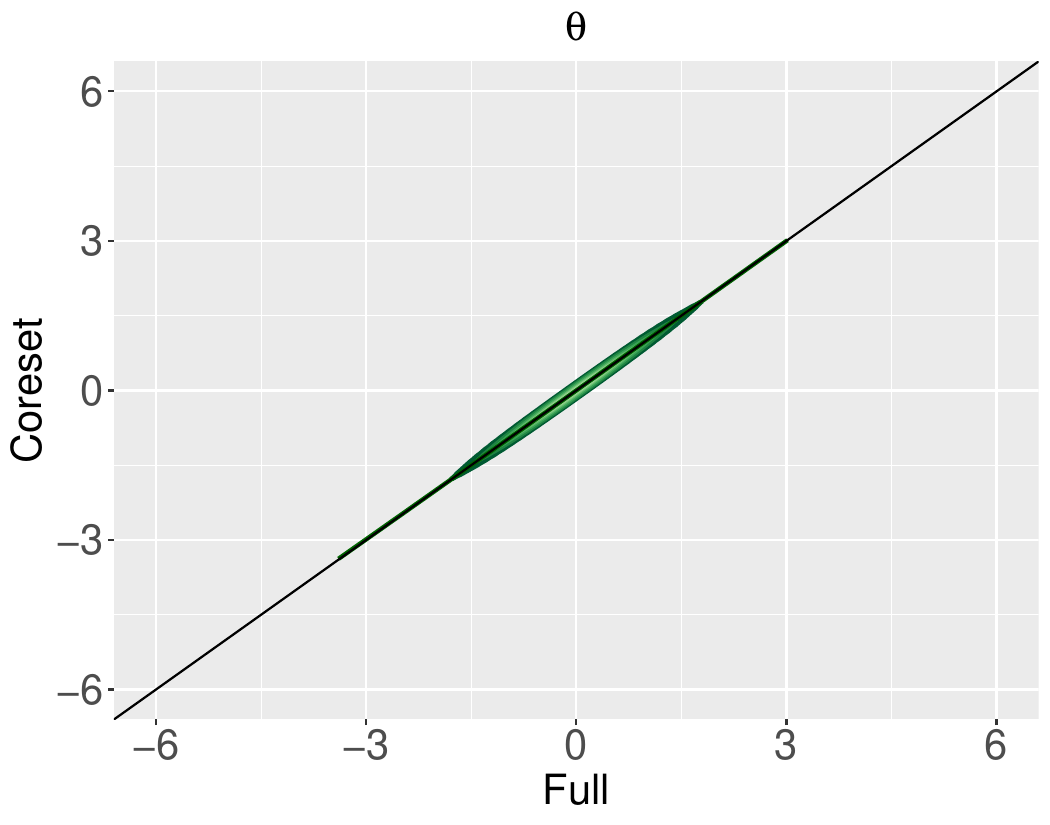}&
\includegraphics[width=0.3\linewidth]{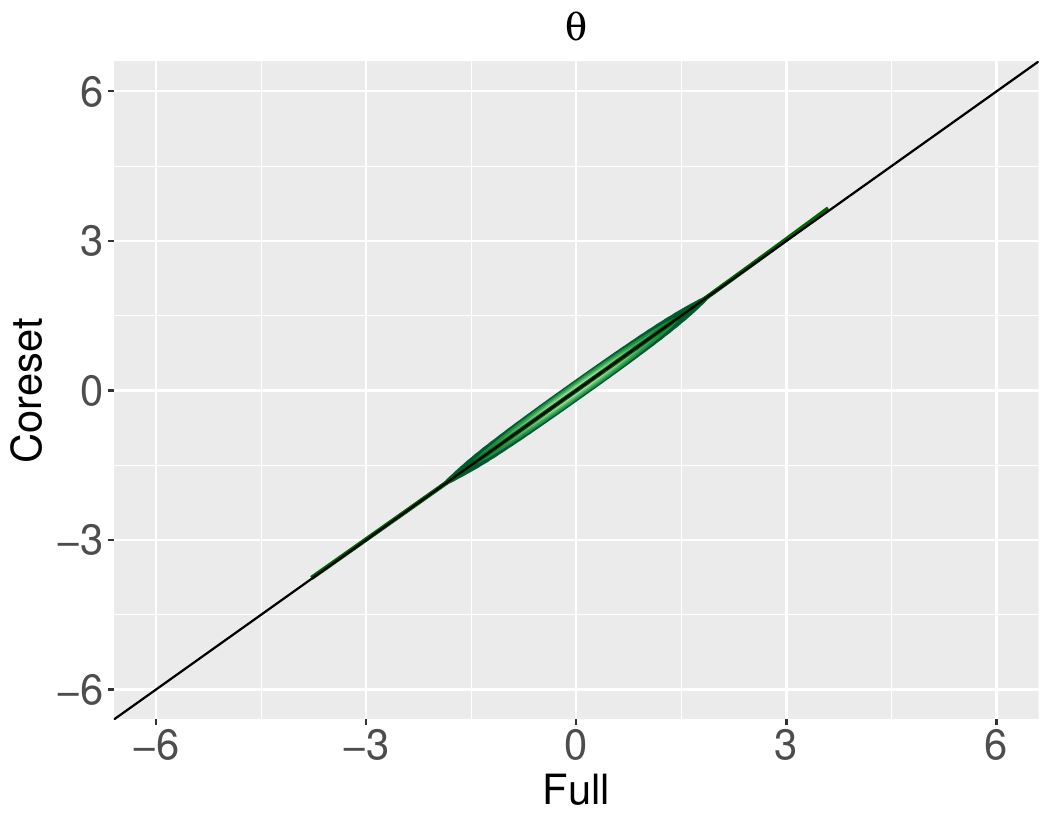}\\
\end{tabular}
\label{fig:param_exp_appendix2}
\end{center}
\end{figure*}

\begin{figure*}[hp!]
\caption{2PL Experiments on synthetic data: Parameter estimates for the coresets compared to the full data set on the largest generated set with $n=500\,000$ and $m=5\,000$. For the experiment the left figure shows the bias for the item parameters $a$ and $b$. The right figure shows a kernel density estimate for the ability parameters $\theta$ with a LOESS regression line in dark green.
The ability parameters were standardized to zero mean and unit variance. The vertical axis is scaled such as to display $2\,{\mathrm{std.}}$ of the corresponding parameter estimate obtained from the full data set.}
\begin{center}
\begin{tabular}{ccc}
{\tiny{$\mathbf{n=500\,000,m=5\,000,k=5\,000}$}}& &
\\
\includegraphics[width=0.3\linewidth]{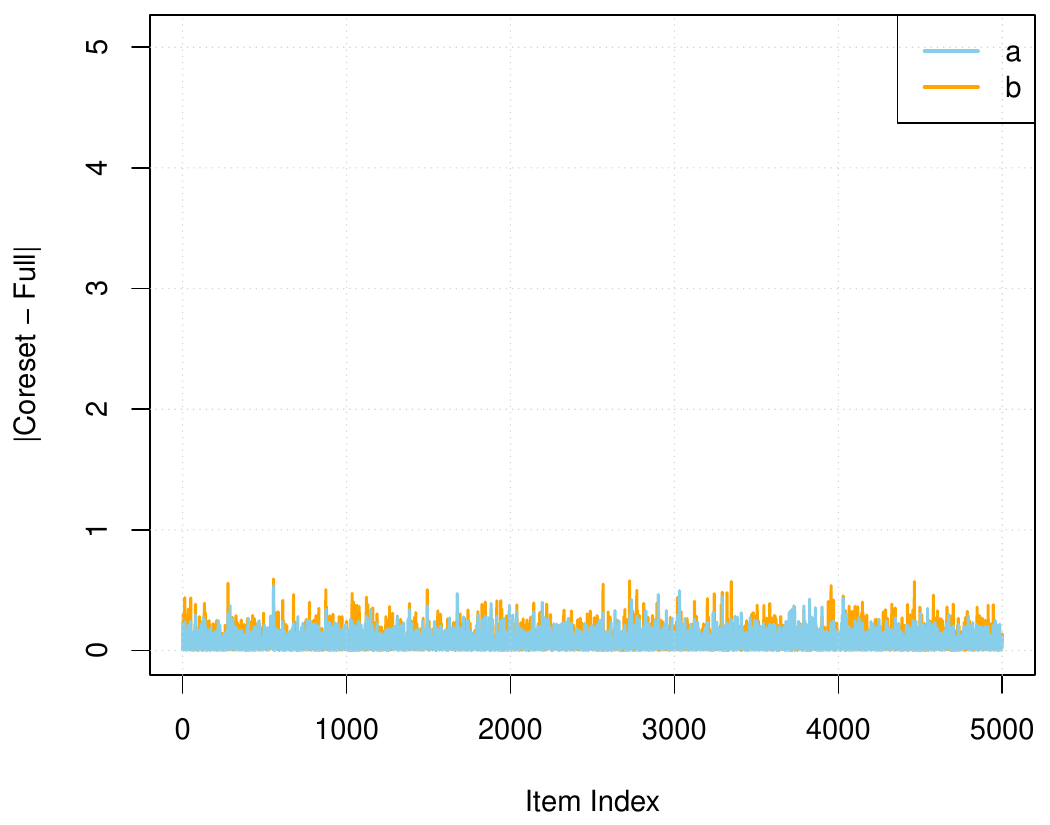}&
\includegraphics[width=0.3\linewidth]{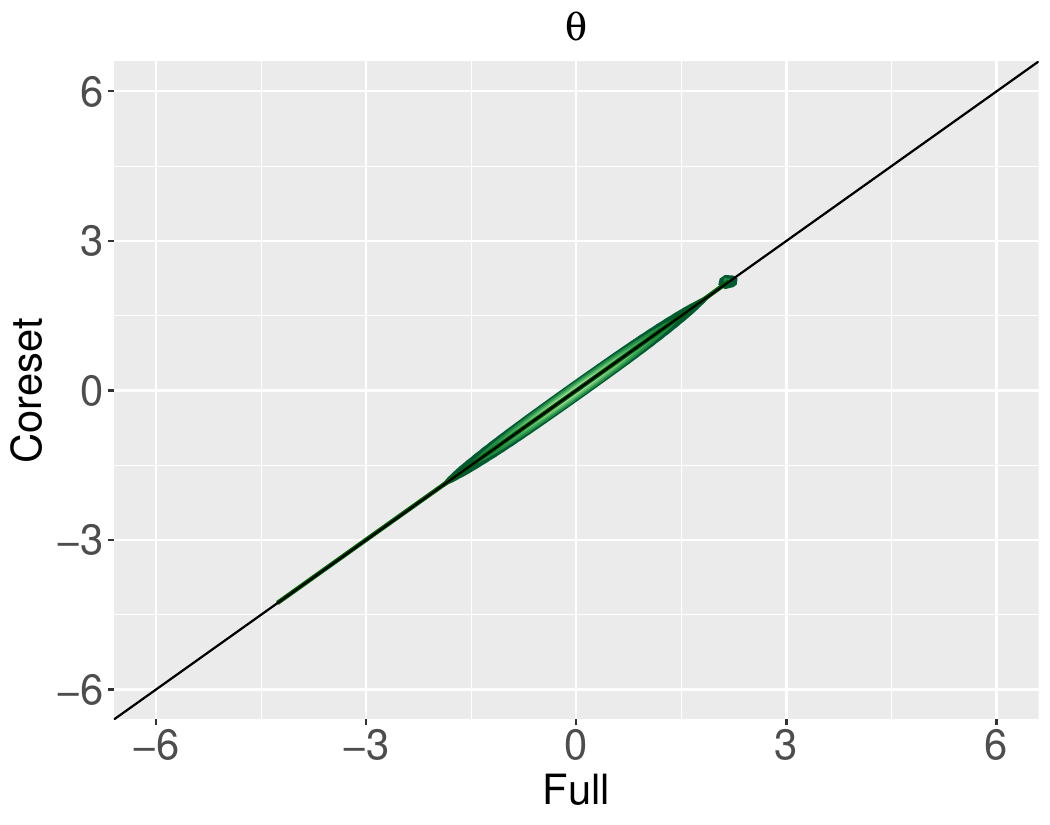}&
 \\
\end{tabular}
\label{fig:param_exp_appendix3}
\end{center}
\end{figure*}

\begin{table*}[hp!]
\caption{2PL Experiments on real world SHARE~\citep{boersch2022survey} and NEPS~\citep{NEPS-SC4} data: 
The quality of the solution found.
Let $f_{\sf full}$ and $f_{{\sf core}(j)}$ be the optimal values of the loss function on the input and on the coreset for the $j$-th repetition, respectively. Let $f_{\sf core} = \min_j f_{{\sf core}(j)}$.
Mean and standard deviation of the relative deviation $|f_{\sf core} - f_{{\sf core}(j)}| / f_{\sf core}$ (in $\%$): 
\textbf{mean dev} and \textbf{std. dev}. Relative error: \textbf{rel. error} $\hat{\varepsilon}=|f_{\sf core} - f_{\sf full}|/f_{\sf full}$ (cf. \lemref{coreset:error:approx}).
Mean Absolute Deviation: $\textbf{mad}(\alpha)=\frac{1}{n}\sum (|a_{\sf full}-a_{\sf core}| + |b_{\sf full}-b_{\sf core}|)$; $\textbf{mad}(\theta)=\frac{1}{m}\sum |\theta_{\sf full}-\theta_{\sf core}|$, evaluated on the parameters that attained the optimal $f_{\sf full}$ and $f_{\sf core}$.
}
	\label{tab:results_appendix3:b}
	\begin{center}
	\begin{tabular}{ c r r r| r r | r | c c}
		\hline
		{\bf data} &
		{$\mathbf n$} & {$\mathbf m$} & {$\mathbf k$} &  {\bf mean dev} & {\bf std. dev} & {\bf rel. error $\hat{\varepsilon}$}  & {\bf $\text{mad}(\alpha)$} & {\bf $\text{mad}(\theta)$}\\ \hline \hline
		SHARE & $138\,997$ & $10$ & $500$ & $5.335$ \% & $2.098$ \% & $0.11347$ & $0.770$ & $0.090$
		\\ \hline
		SHARE & $138\,997$ & $10$ & $1\,000$ & $1.682$ \% & $0.930$ \% & $0.06193$ & $0.307$ & $0.040$
		\\ \hline
		SHARE & $138\,997$ & $10$ & $2\,000$ & $1.251$ \% & $0.820$ \% & $0.04263$ & $0.129$ & $0.015$
		\\ \hline
		SHARE & $138\,997$ & $10$ & $4\,000$ & $0.686$ \% & $0.414$ \% & $0.02791$ & $0.108$ & $0.013$
		\\ \hline
		SHARE & $138\,997$ & $10$ & $6\,000$ & $1.930$ \% & $0.611$ \% & $0.03546$ & $0.095$ & $0.007$
		\\ \hline
		SHARE & $138\,997$ & $10$ & $8\,000$ & $0.600$ \% & $0.252$ \% & $0.01935$ & $0.061$ & $0.007$
		\\ \hline
		SHARE & $138\,997$ & $10$ & $10\,000$ & $1.557$ \% & $0.407$ \% & $0.02713$ & $0.092$ & $0.014$
		\\ \hline
		SHARE & $138\,997$ & $10$ & $20\,000$ & $0.356$ \% & $0.168$ \% & $0.01415$ & $0.045$ & $0.003$
		\\ \hline\hline
		NEPS & $11\,532$ & $88$ & $100$ & $4.363$ \% & $2.176$ \% & $0.09335$ & $1.477$ & $0.171$
		\\ \hline
		NEPS & $11\,532$ & $88$ & $200$ & $3.324$ \% & $1.480$ \% & $0.07134$ & $0.930$ & $0.142$
		\\ \hline
		NEPS & $11\,532$ & $88$ & $500$ & $1.969$ \% & $0.657$ \% & $0.03795$ & $0.499$ & $0.075$
		\\ \hline
		NEPS & $11\,532$ & $88$ & $750$ & $1.478$ \% & $0.524$ \% & $0.02675$ & $0.432$ & $0.062$
		\\ \hline
		NEPS & $11\,532$ & $88$ & $1\,000$ & $1.191$ \% & $0.395$ \% & $0.02007$ & $0.320$ & $0.045$
		\\ \hline
		NEPS & $11\,532$ & $88$ & $2\,000$ & $0.352$ \% & $0.120$ \% & $0.00506$ & $0.182$ & $0.026$
		\\ \hline
		NEPS & $11\,532$ & $88$ & $5\,000$ & $0.220$ \% & $0.169$ \% & $0.00147$ & $0.101$ & $0.015$
		\\ \hline
		NEPS & $11\,532$ & $88$ & $10\,000$ & $0.301$ \% & $0.200$ \% & $0.00094$ & $0.071$ & $0.012$
		\\ \hline\hline
	\end{tabular}
	\end{center}
\end{table*}

\begin{table*}[hp!]
\caption{2PL Experiments on real world SHARE~\citep{boersch2022survey} and NEPS~\citep{NEPS-SC4} data: The means and standard deviations (std.) of running times, taken across $20$ repetitions. 
{In each repetition, the running time (in minutes) of 50 iterations of the main loop was measured}
per data set for different configurations of the data dimensions: the number of items $m$, the number of examinees $n$, and the coreset size $k$. The (relative) gain is defined as  $(1-\mathrm{mean}_\mathrm{coreset}/\mathrm{mean}_\mathrm{full})\cdot 100$ \%.
}
	\label{tab:results_appendix3}
	\begin{center}
	\begin{tabular}{ c r r r| r r| r r| r }
    &\multicolumn{3}{c}{}& \multicolumn{2}{c}{{{\bf Full data} (min)}}& \multicolumn{2}{c}{{{\bf Coresets} (min)}} & \multicolumn{1}{c}{ } \\
		\hline
		{\bf data} &
		{$\mathbf n$} & {$\mathbf m$} & {$\mathbf k$} &  {\bf mean} & {\bf std.} & {\bf mean} & {\bf std.} & {\bf gain} \\ \hline \hline
        SHARE & $138\,997$ & $10$ & $500$ & $28.853$ & $1.618$ & $30.436$ & $1.451$ & $-5.484$ \% 
		\\ \hline
		SHARE & $138\,997$ & $10$ & $1\,000$ & $28.853$ & $1.618$ & $29.649$ & $1.375$ & $-2.758$ \% 
		\\ \hline
		SHARE & $138\,997$ & $10$ & $2\,000$ & $28.853$ & $1.618$ & $28.578$ & $0.195$ & $0.953$ \%
		\\ \hline
		SHARE & $138\,997$ & $10$ & $4\,000$ & $28.853$ & $1.618$ & $27.861$ & $0.070$ & $3.439$ \%
		\\ \hline
		SHARE & $138\,997$ & $10$ & $6\,000$ & $28.853$ & $1.618$ & $27.746$ & $0.080$ & $3.837$ \% 
		\\ \hline
		SHARE & $138\,997$ & $10$ & $8\,000$ & $28.853$ & $1.618$ & $27.637$ & $0.085$ & $4.216$ \%
		\\ \hline
		SHARE & $138\,997$ & $10$ & $10\,000$ & $28.853$ & $1.618$ & $27.560$ & $0.082$ & $4.481$ \% 
		\\ \hline
		SHARE & $138\,997$ & $10$ & $20\,000$ & $28.853$ & $1.618$ & $27.525$ & $0.085$ & $4.603$ \%
		\\ \hline\hline
		NEPS & $11\,532$ & $88$ & $100$ & $5.968$ & $0.061$ & $4.020$ & $0.010$ & $32.640$ \%
		\\ \hline
		NEPS & $11\,532$ & $88$ & $200$ & $5.968$ & $0.061$ & $4.113$ & $0.257$ & $31.084$ \%
		\\ \hline
		NEPS & $11\,532$ & $88$ & $500$ & $5.968$ & $0.061$ & $4.402$ & $0.333$ & $26.237$ \% 
		\\ \hline
		NEPS & $11\,532$ & $88$ & $750$ & $5.968$ & $0.061$ & $4.036$ & $0.014$ & $32.373$ \% 
		\\ \hline
		NEPS & $11\,532$ & $88$ & $1\,000$ & $5.968$ & $0.061$ & $4.009$ & $0.016$ & $32.829$ \%
		\\ \hline
		NEPS & $11\,532$ & $88$ & $2\,000$ & $5.968$ & $0.061$ & $3.940$ & $0.057$ & $33.983$ \% 
		\\ \hline
		NEPS & $11\,532$ & $88$ & $5\,000$ & $5.968$ & $0.061$ & $4.779$ & $0.105$ & $19.920$ \%
		\\ \hline
		NEPS & $11\,532$ & $88$ & $10\,000$ & $5.968$ & $0.061$ & $5.849$ & $0.064$ & $2.003$ \%
		\\ \hline\hline
	\end{tabular}
	\end{center}
\end{table*}

\begin{figure*}[hp!]
\caption{2PL Experiments on the real world SHARE data~\citep{boersch2022survey}.  Parameter estimates for the coresets compared to the full data sets. 
For each experiment the upper figure shows the bias for the item parameters $a$ and $b$. The lower figure shows a kernel density estimate for the ability parameters $\theta$ with a LOESS regression line in dark green.
The ability parameters were standardized to zero mean and unit variance. In all rows, the vertical axis is scaled such as to display $2\,{\mathrm{std.}}$ of the corresponding parameter estimate obtained from the full data set.}
\begin{center}
\begin{tabular}{ccc}
{\tiny{$\mathbf{n=138\,997,m=10,k=500}$}}&{\tiny{$\mathbf{n=138\,997,m=10,k=1\,000}$}}&{\tiny{$\mathbf{n=138\,997,m=10,k=2\,000}$}}
\\
\includegraphics[width=0.3\linewidth]{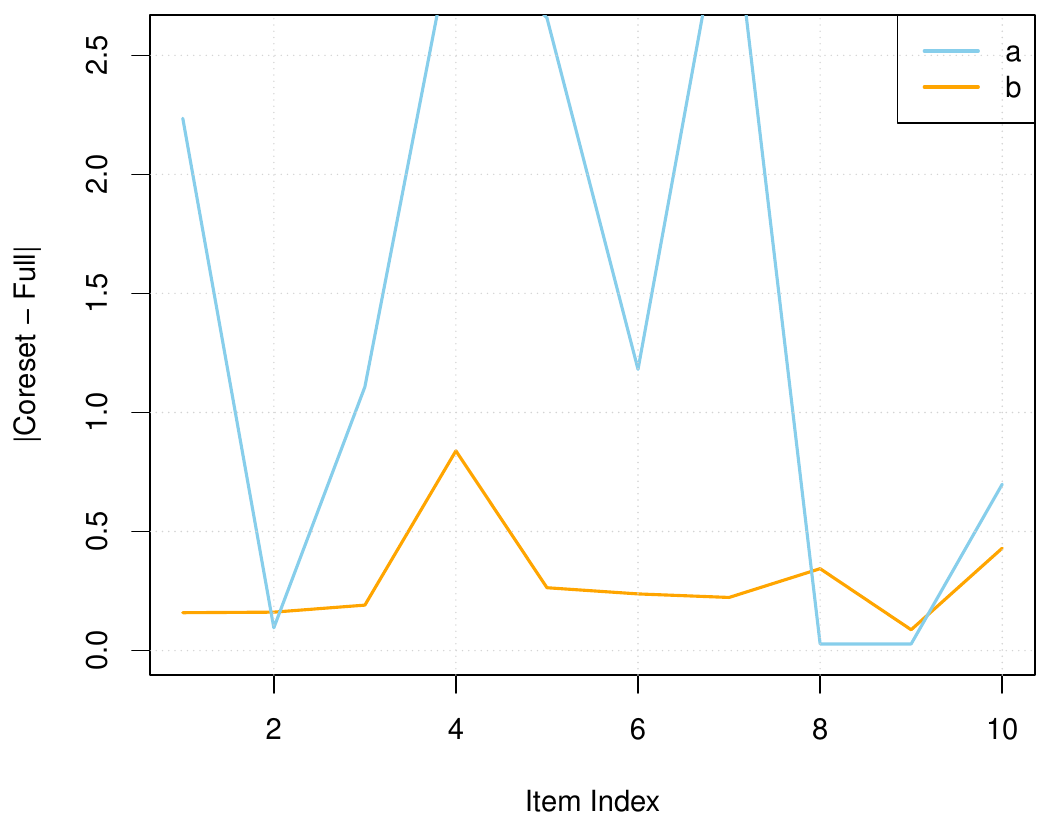}&
\includegraphics[width=0.3\linewidth]{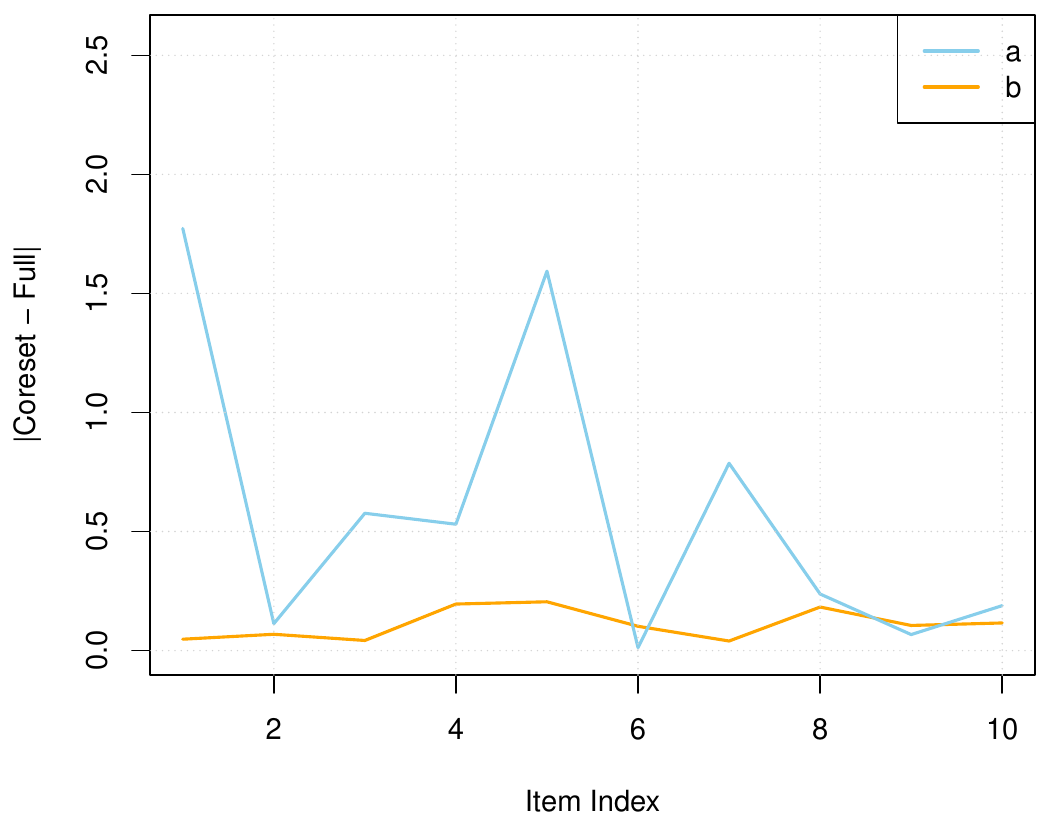}&
\includegraphics[width=0.3\linewidth]{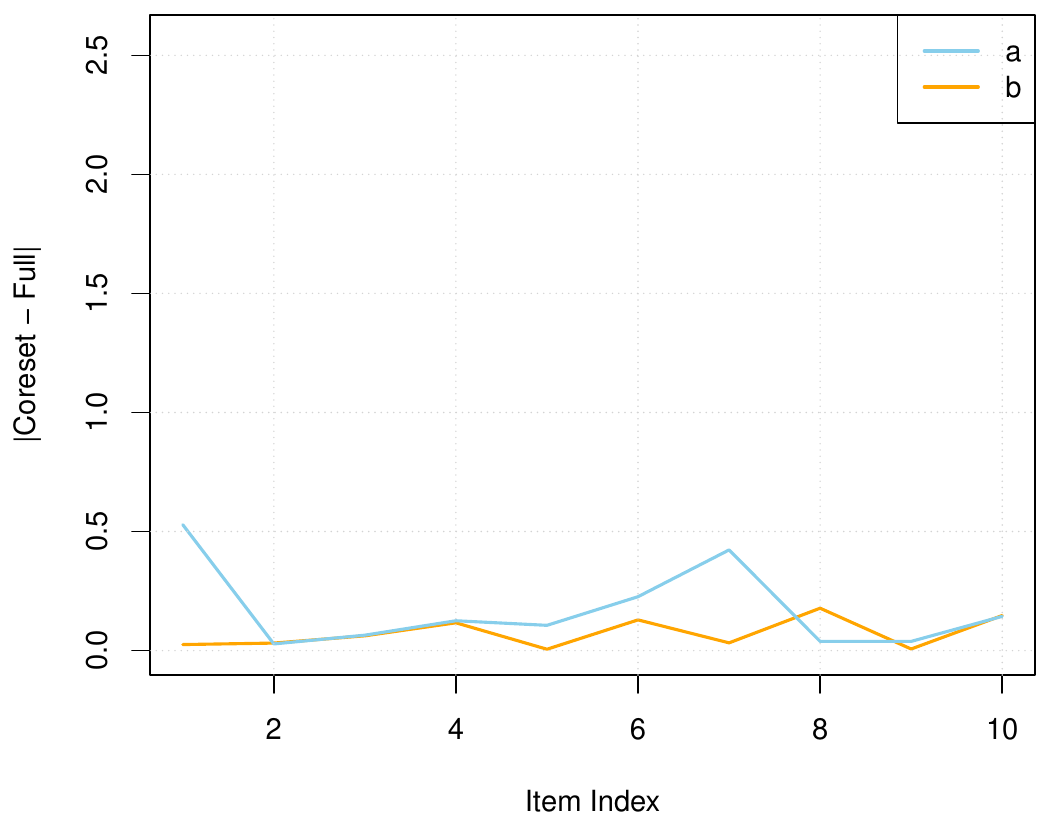}\\
\includegraphics[width=0.3\linewidth]{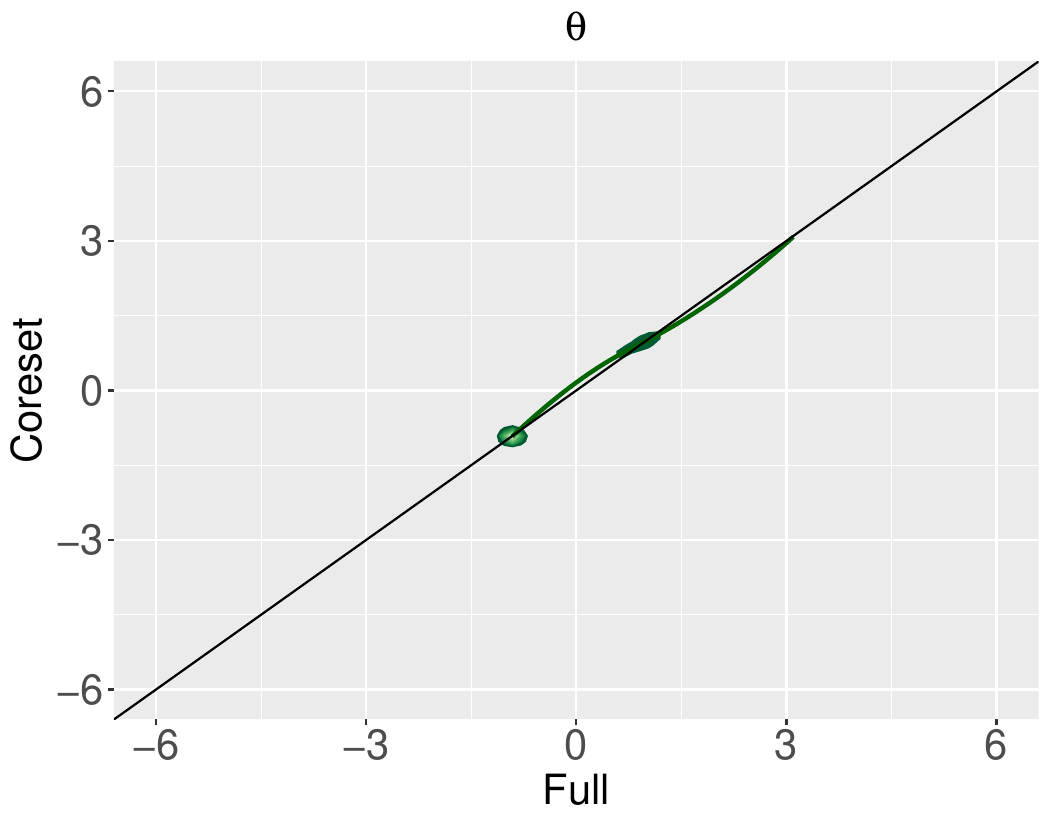}&
\includegraphics[width=0.3\linewidth]{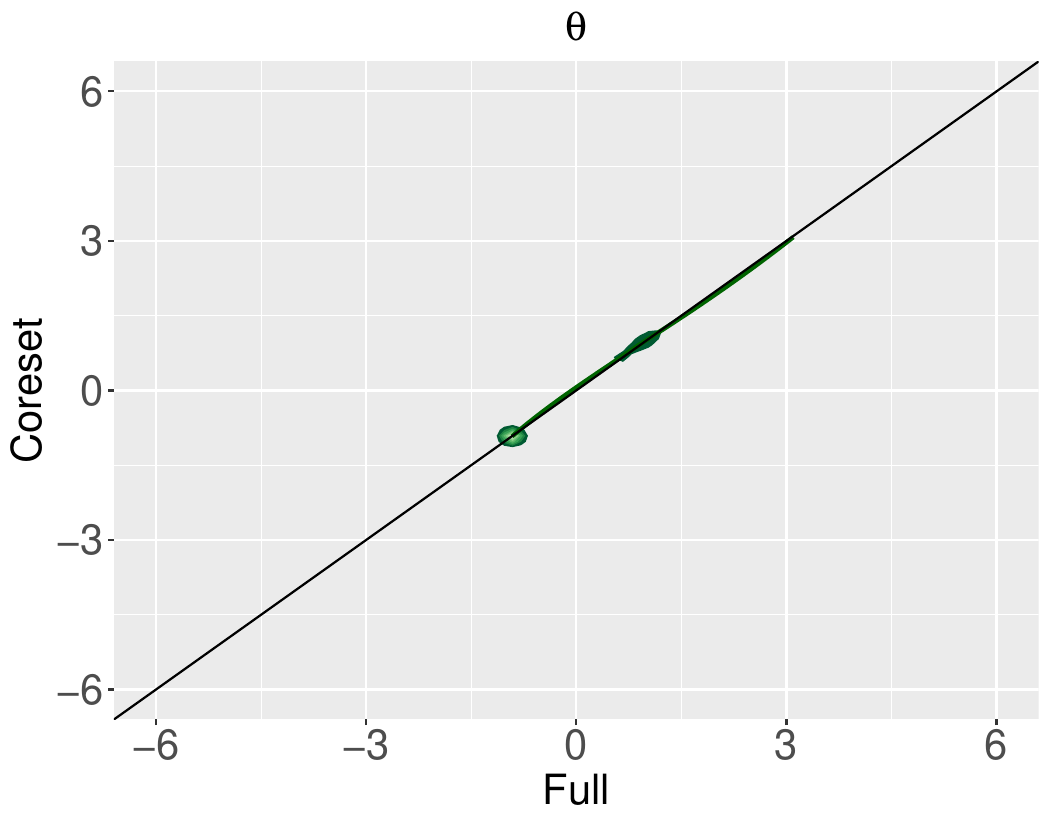}&
\includegraphics[width=0.3\linewidth]{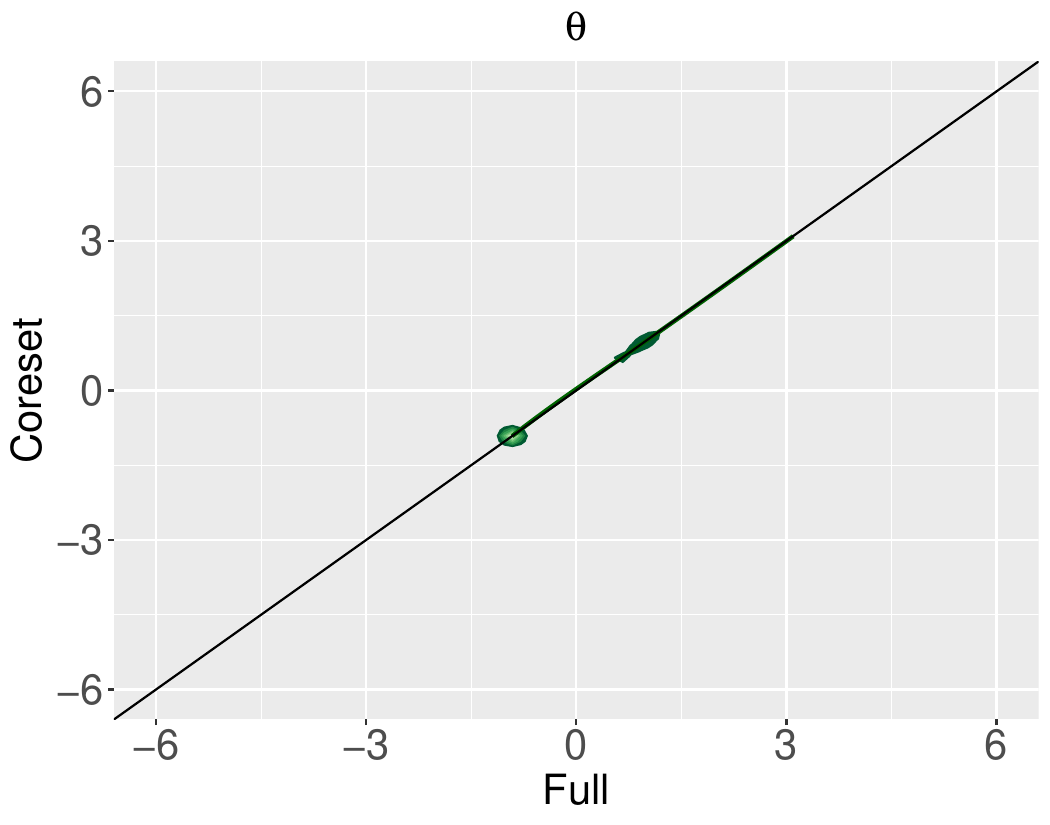}\\
\end{tabular}
\label{fig:param_exp_appendix4}
\end{center}
\end{figure*}

\begin{figure*}[hp!]
\caption{2PL Experiments on the real world SHARE data~\citep{boersch2022survey}.  Parameter estimates for the coresets compared to the full data sets. 
For each experiment the upper figure shows the bias for the item parameters $a$ and $b$. The lower figure shows a kernel density estimate for the ability parameters $\theta$ with a LOESS regression line in dark green.
The ability parameters were standardized to zero mean and unit variance. In all rows, the vertical axis is scaled such as to display $2\,{\mathrm{std.}}$ of the corresponding parameter estimate obtained from the full data set.}
\begin{center}
\begin{tabular}{ccc}
{\tiny{$\mathbf{n=138\,997,m=10,k=4\,000}$}}&{\tiny{$\mathbf{n=138\,997,m=10,k=6\,000}$}}&{\tiny{$\mathbf{n=138\,997,m=10,k=8\,000}$}}
\\
\includegraphics[width=0.3\linewidth]{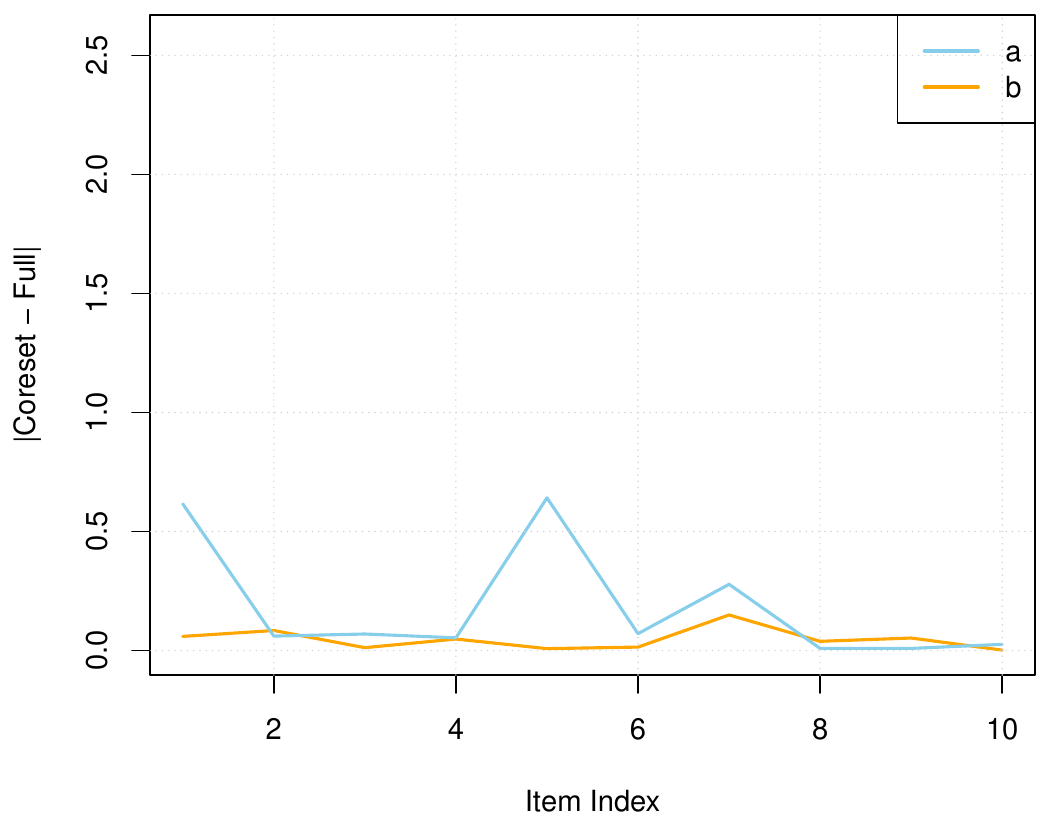}&
\includegraphics[width=0.3\linewidth]{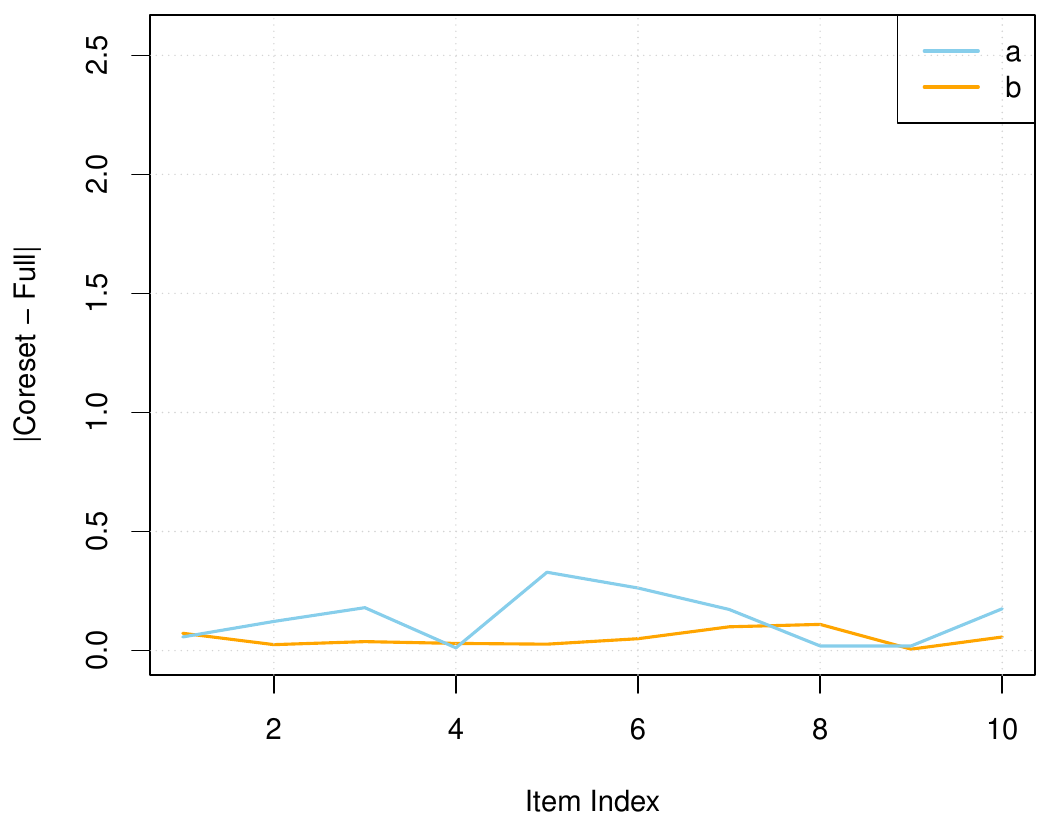}&
\includegraphics[width=0.3\linewidth]{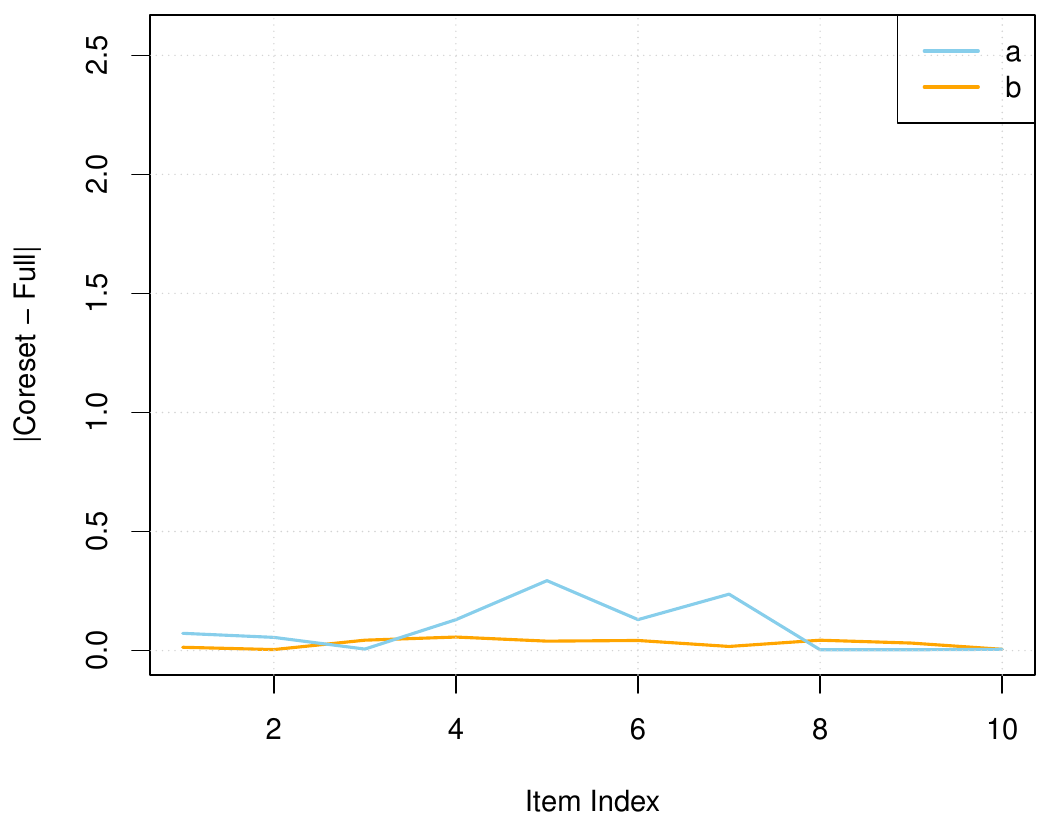}\\
\includegraphics[width=0.3\linewidth]{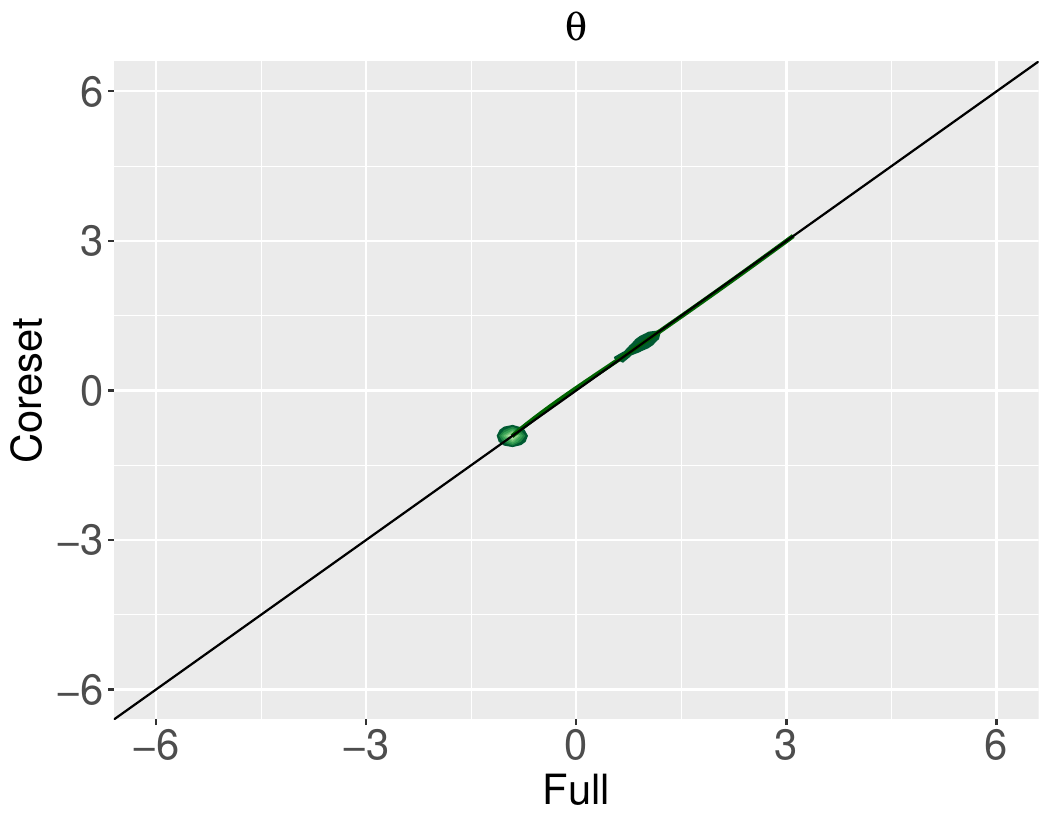}&
\includegraphics[width=0.3\linewidth]{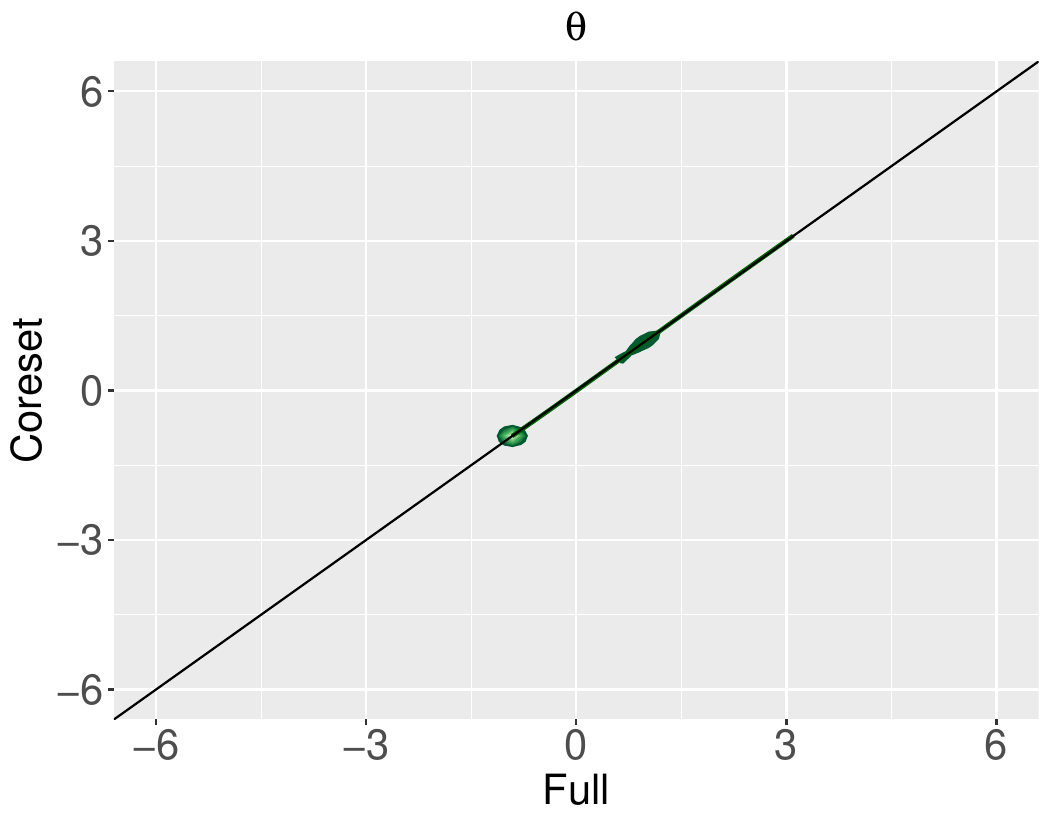}&
\includegraphics[width=0.3\linewidth]{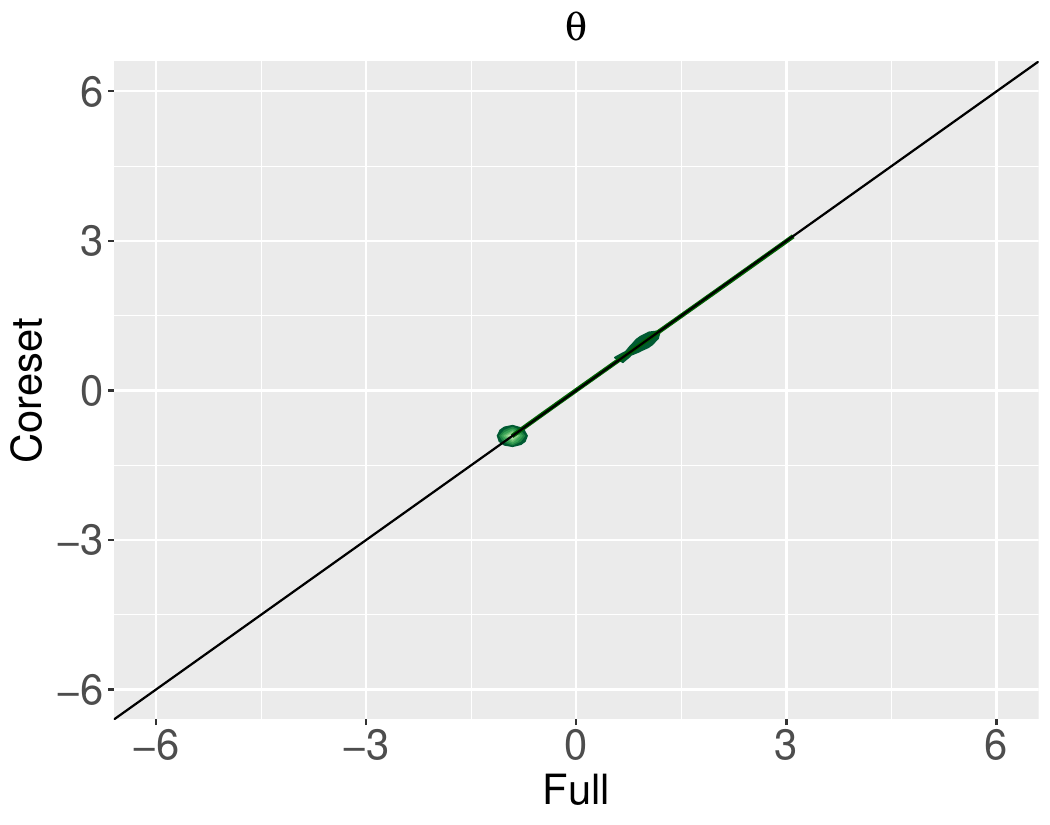}
\\
{\tiny{$\mathbf{n=138\,997,m=10,k=10\,000}$}}&{\tiny{$\mathbf{n=138\,997,m=10,k=20\,000}$}}&
\\
\includegraphics[width=0.3\linewidth]{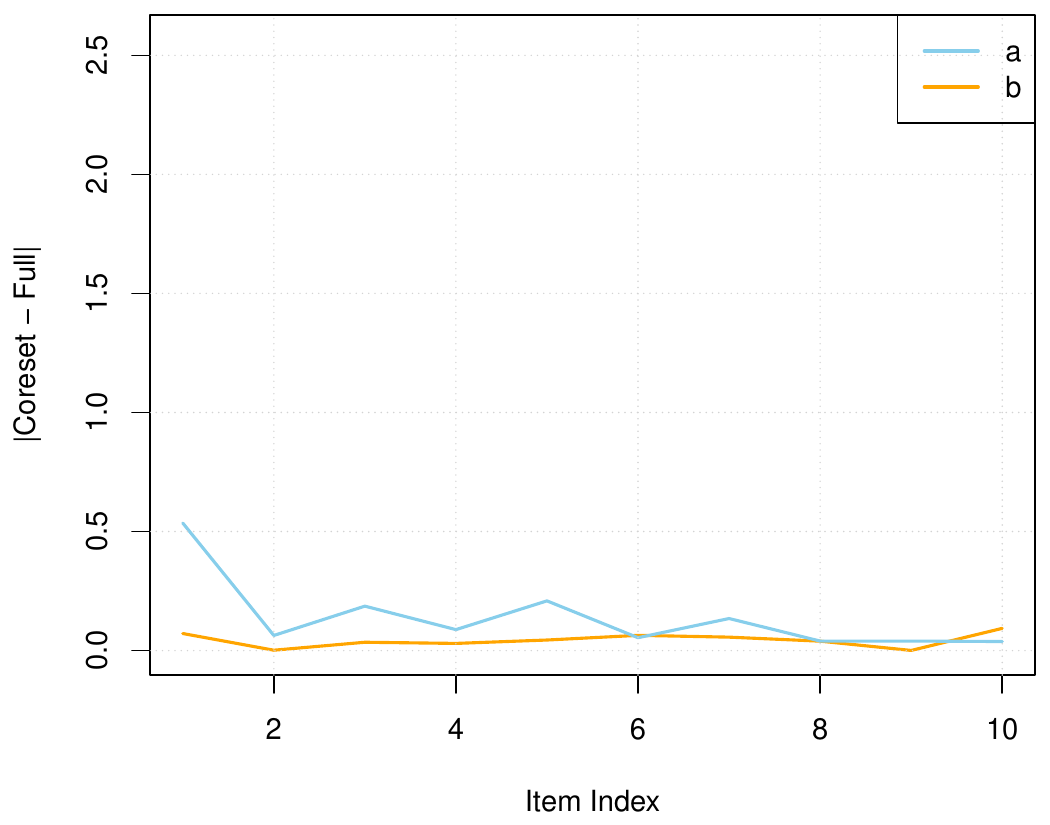}&
\includegraphics[width=0.3\linewidth]{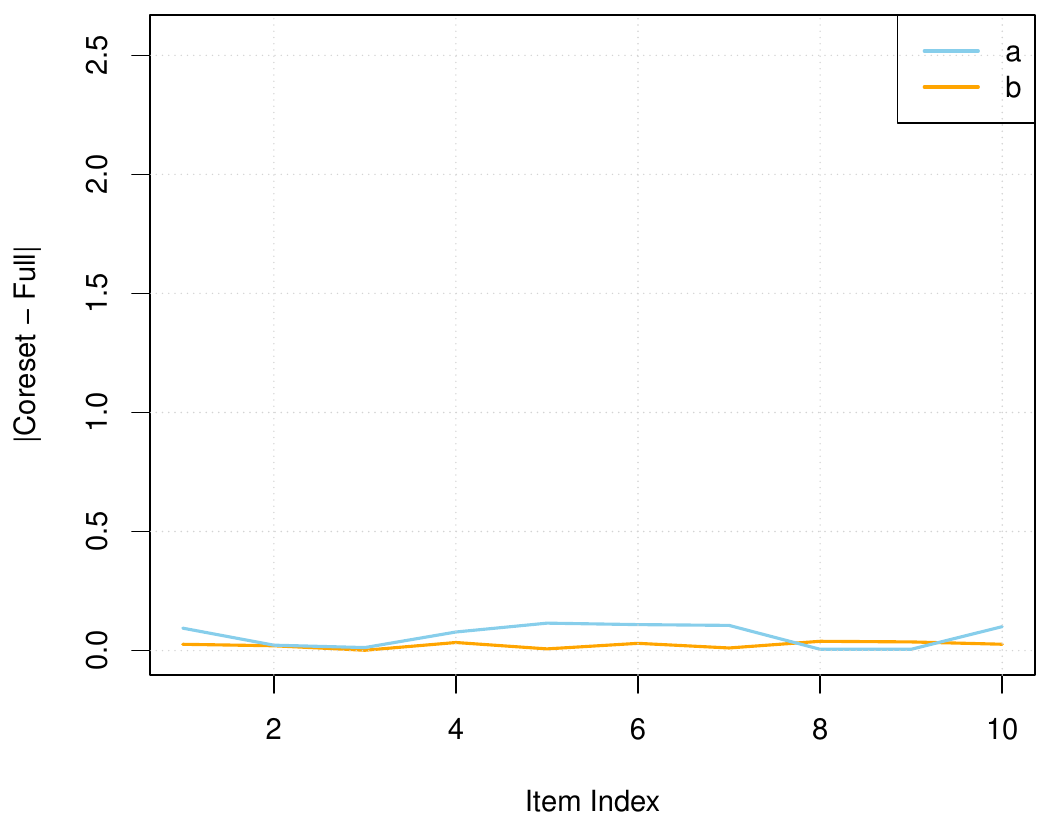}&
\\
\includegraphics[width=0.3\linewidth]{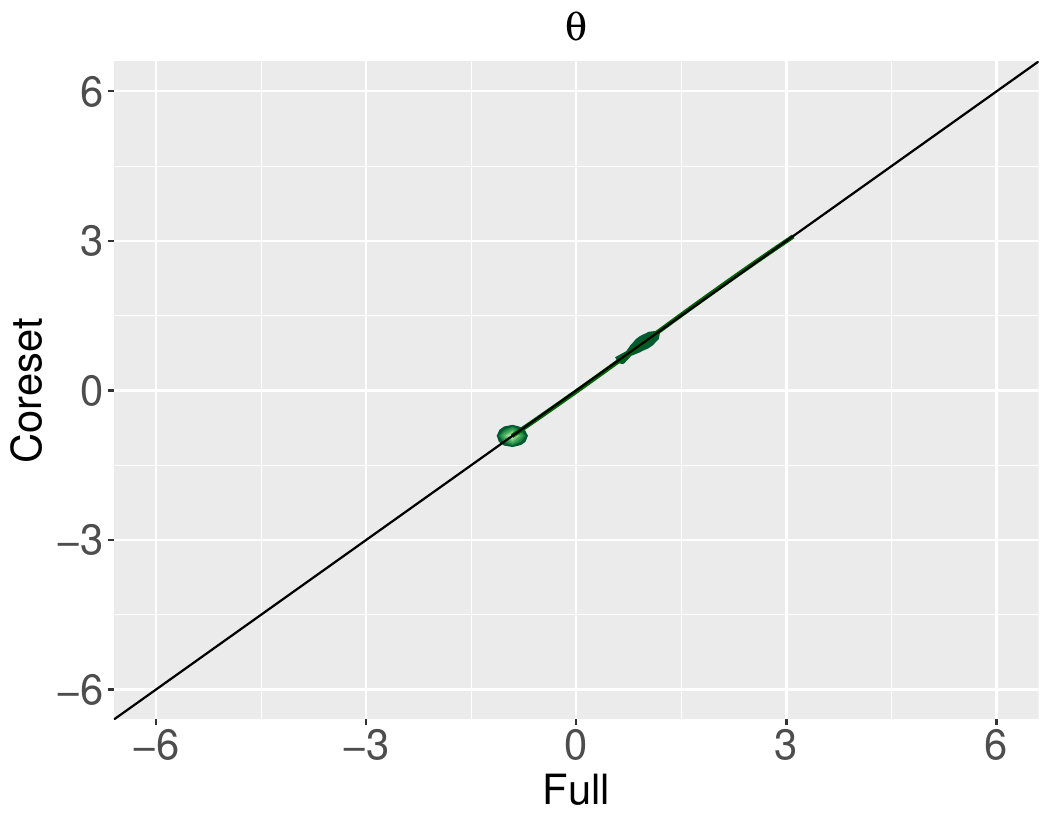}&
\includegraphics[width=0.3\linewidth]{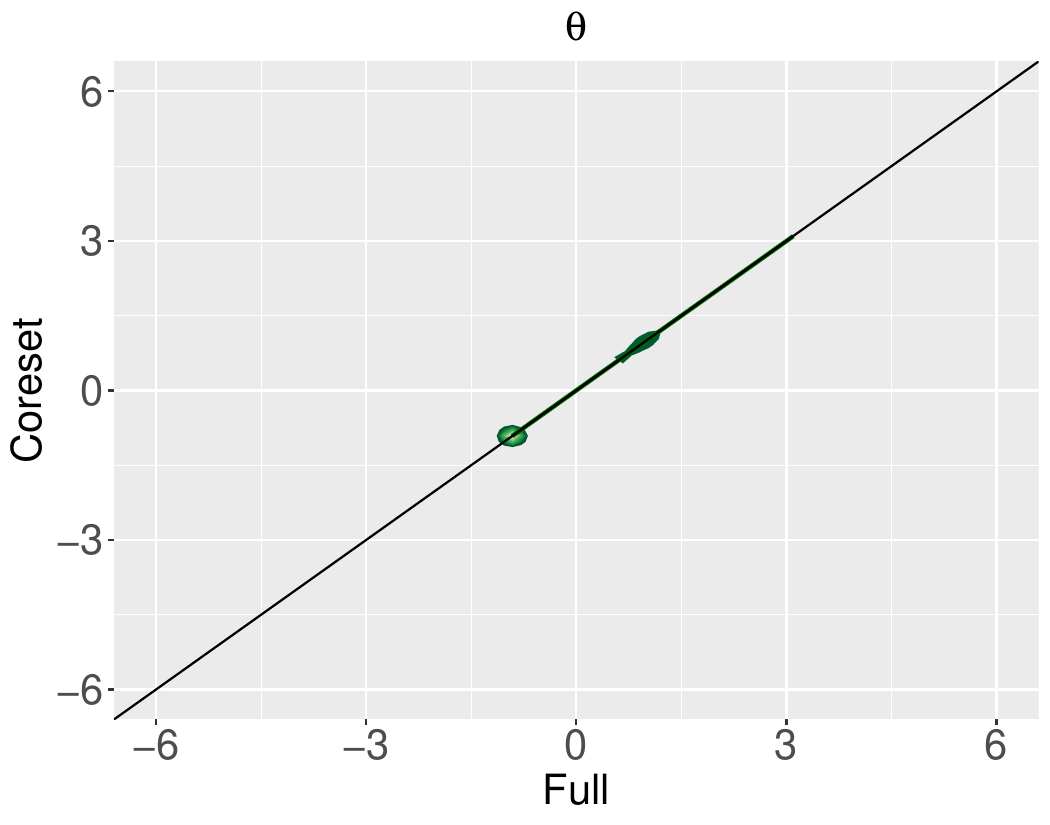}&
\end{tabular}
\label{fig:param_exp_appendix9}
\end{center}
\end{figure*}

\begin{figure*}[hp!]
\caption{2PL Experiments on real world SHARE~\citep{boersch2022survey} and NEPS data~\citep{NEPS-SC4}: A comparison between the coreset sizes and the the quality of the solution found, by the relative error and the mean absolute deviation ($\alpha$), cf.~\cref{tab:results_appendix3:b}.
Let $f_{\sf full}$ and $f_{{\sf core}(j)}$ be the optimal values of the loss function on the input and on the coreset for the $j$-th repetition, respectively. Let $f_{\sf core} = \min_j f_{{\sf core}(j)}$.
Relative error: \textbf{rel. error} $\hat{\varepsilon}=|f_{\sf core} - f_{\sf full}|/f_{\sf full}$ (cf. \lemref{coreset:error:approx}).
Mean Absolute Deviation: $\textbf{mad}(\alpha)=\frac{1}{n}\sum (|a_{\text{full}}-a_{\text{core}}| + |b_{\text{full}}-b_{\text{core}}|)$, 
evaluated on the parameters that attained the optimal $f_{\sf full}$ and $f_{\sf core}$. 
The coreset sizes for the NEPS data end at $10\,000$, to not exceed the input data size. 
}
\begin{center}
\includegraphics[width=0.8\linewidth]{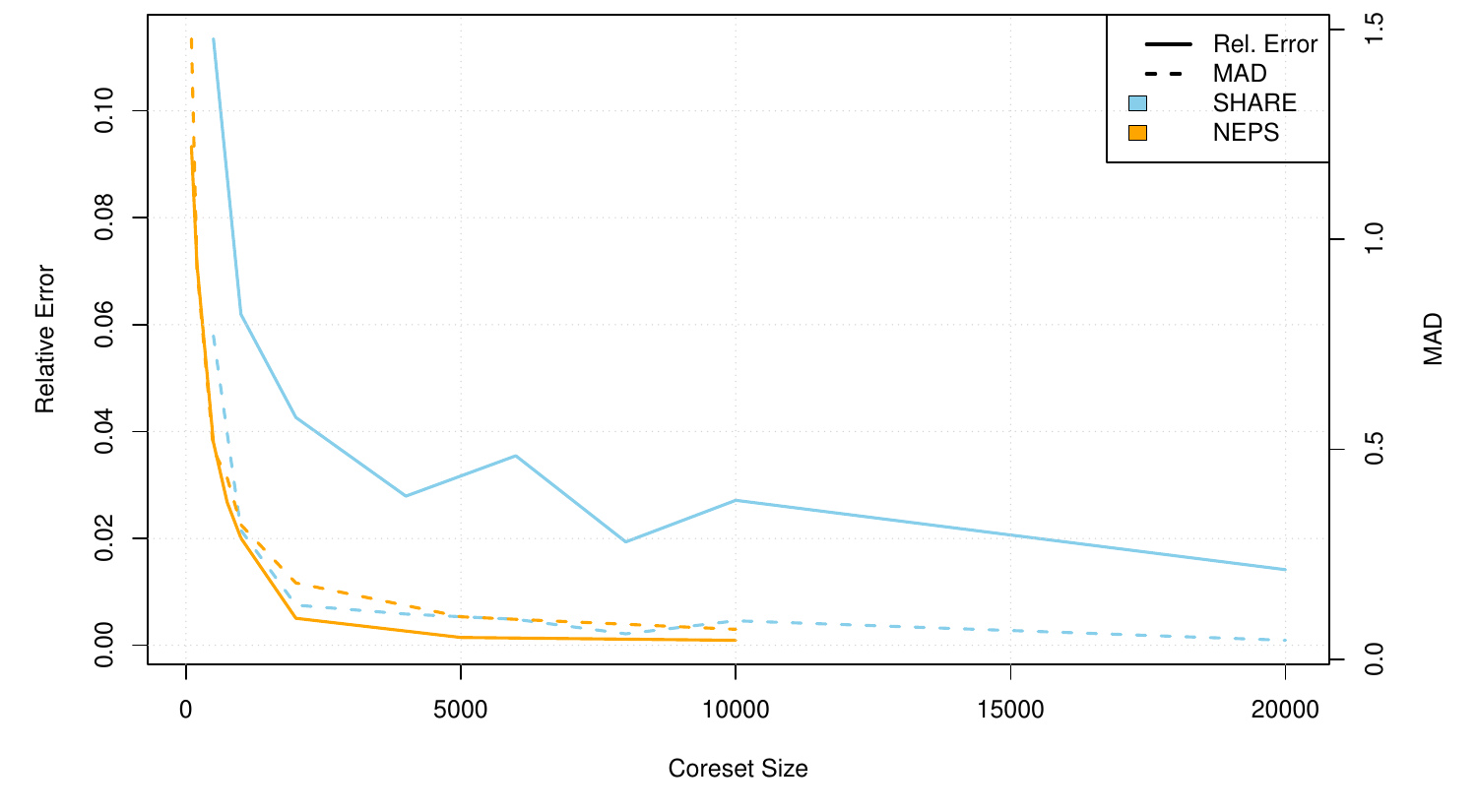} 
\label{fig:param_exp_appendix_pareto}
\end{center}
\end{figure*}

\begin{figure*}[hp!]
\caption{2PL Experiments on real world NEPS data~\citep{NEPS-SC4}: Parameter estimates for the coresets compared to the full data sets. 
For each experiment the upper figure shows the bias for the item parameters $a$ and $b$. The lower figure shows a kernel density estimate for the ability parameters $\theta$ with a LOESS regression line in dark green. 
The ability parameters were standardized to zero mean and unit variance. In all rows, the vertical axis is scaled such as to display $2\,{\mathrm{std.}}$ of the corresponding parameter estimate obtained from the full data set.}
\begin{center}
\begin{tabular}{ccc}
{\tiny{$\mathbf{n=11\,532, m=88, k=100}$}}&{\tiny{$\mathbf{n=11\,532, m=88, k=200}$}}&{\tiny{$\mathbf{n=11\,532, m=88, k=500}$}}
\\
\includegraphics[width=0.3\linewidth]{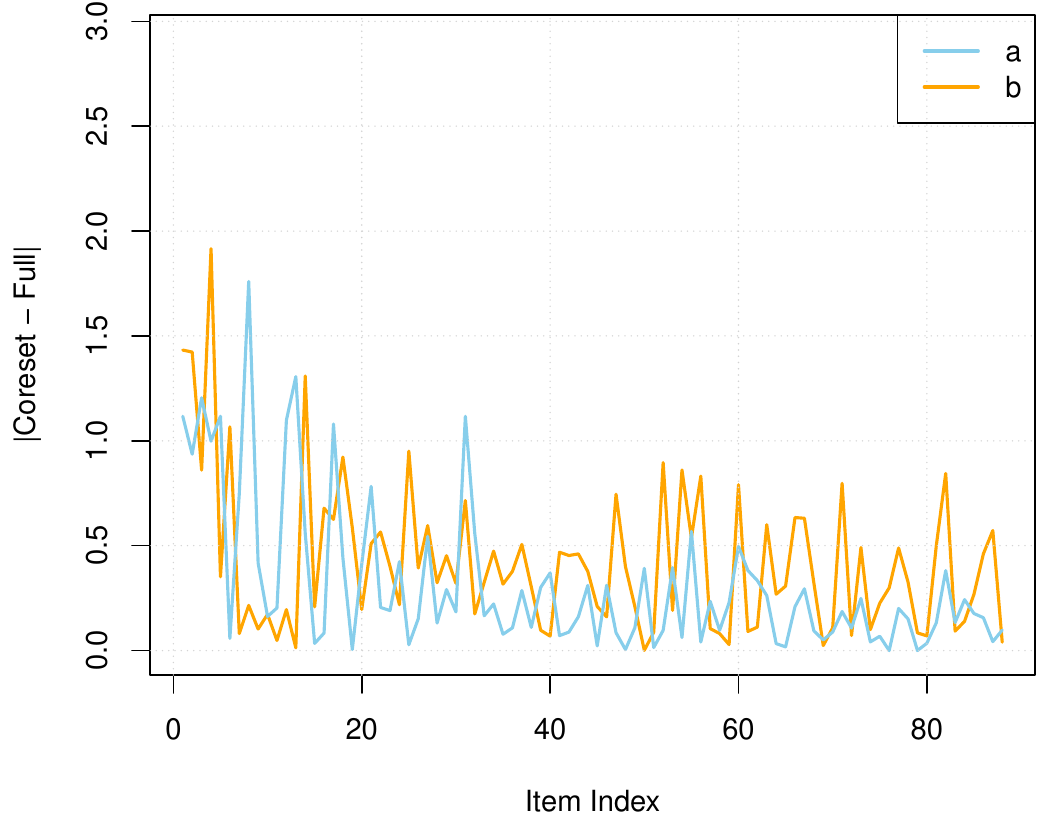} &
\includegraphics[width=0.3\linewidth]{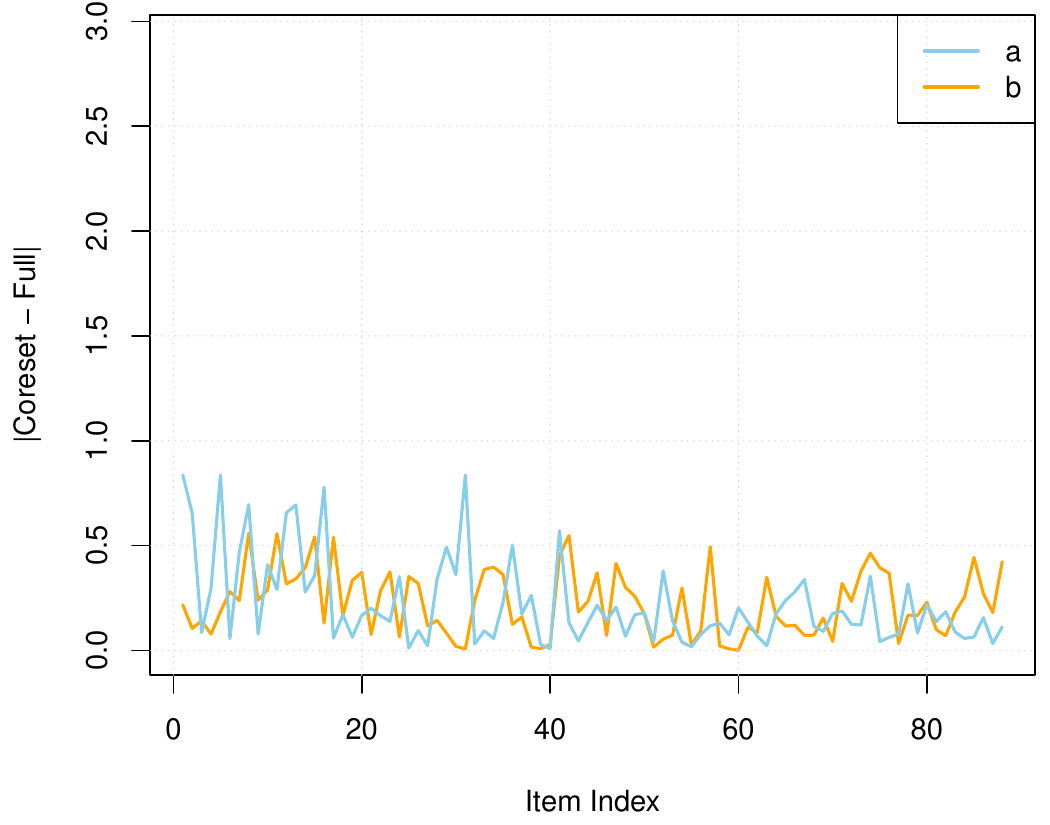} &
\includegraphics[width=0.3\linewidth]{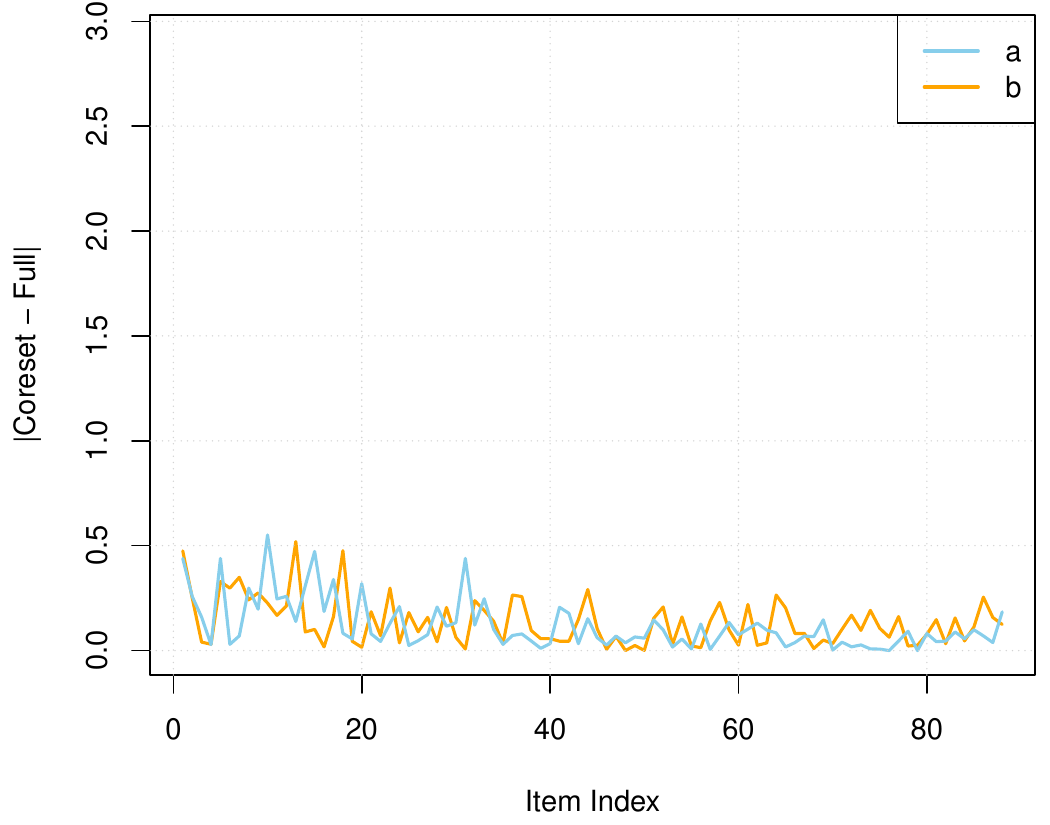}
 \\
 \includegraphics[width=0.3\linewidth]{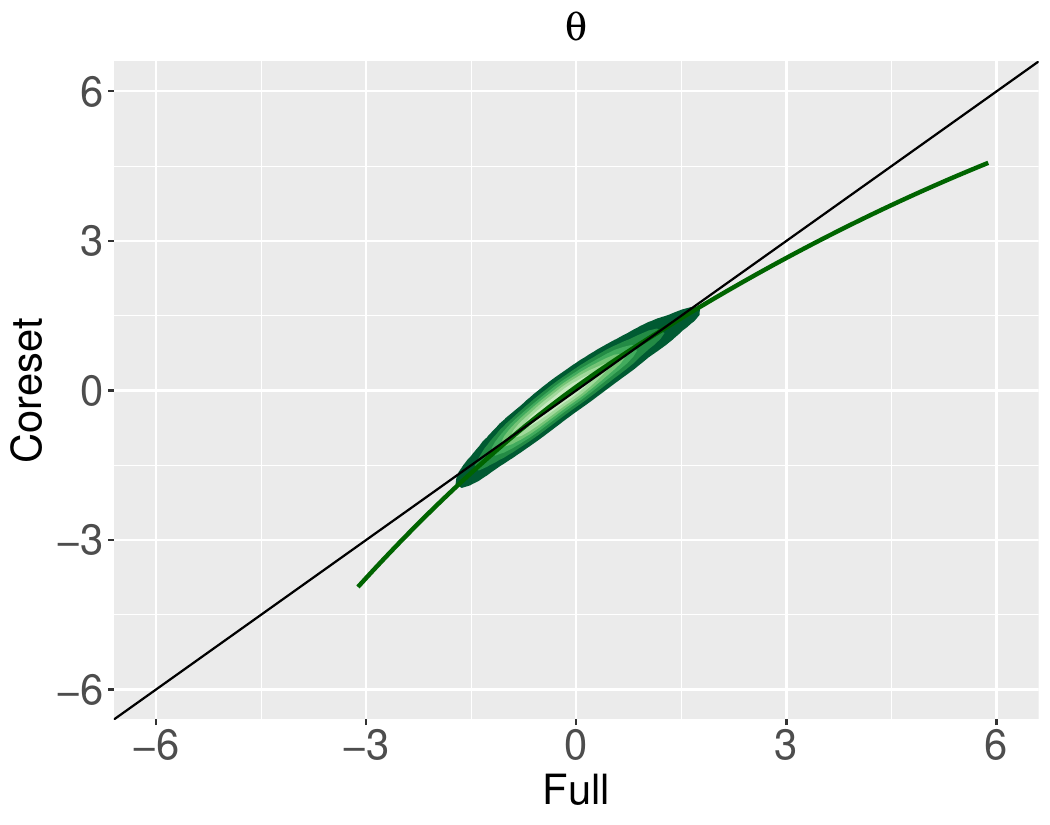} &
\includegraphics[width=0.3\linewidth]{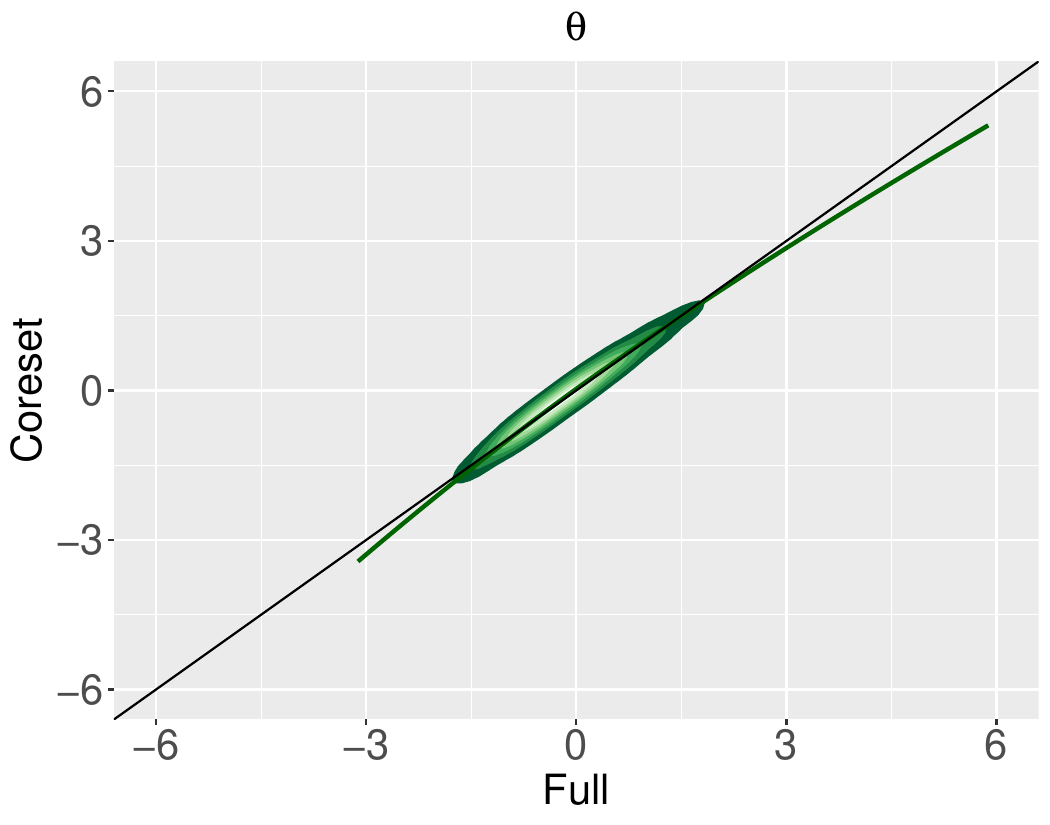} &
\includegraphics[width=0.3\linewidth]{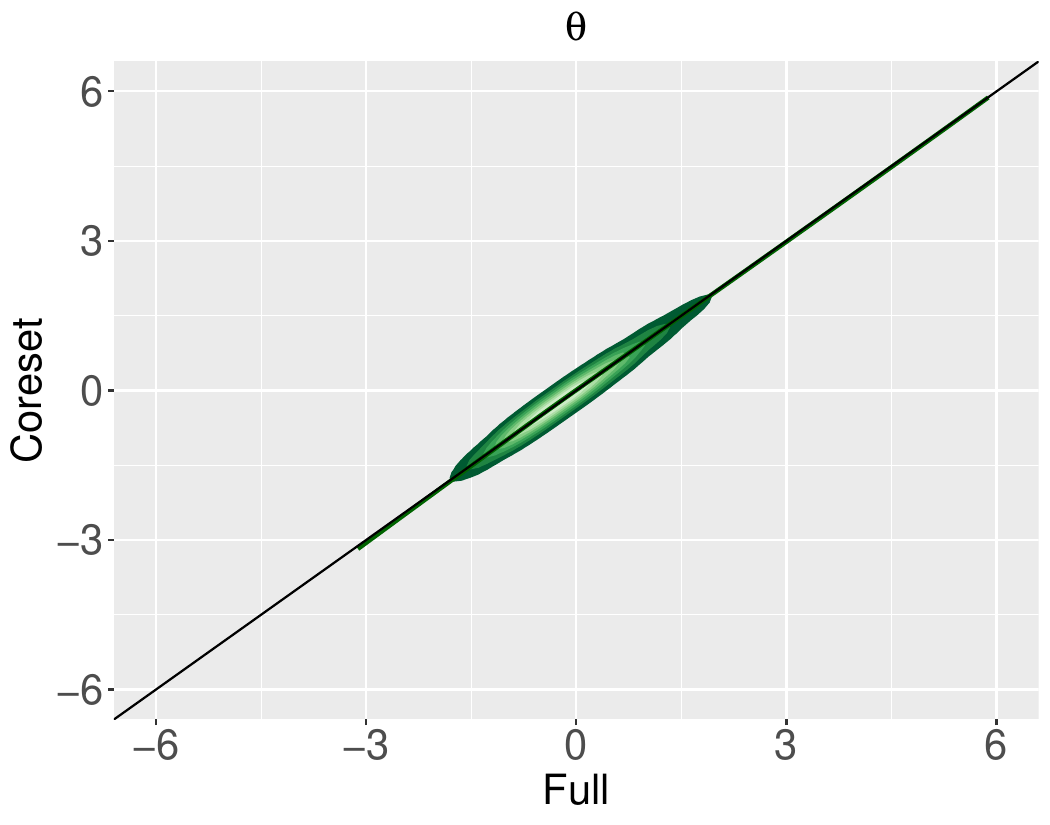}
\end{tabular}
\label{fig:param_exp_appendix5}
\end{center}
\end{figure*}

\begin{figure*}[hp!]
\caption{2PL Experiments on real world NEPS data~\citep{NEPS-SC4}: Parameter estimates for the coresets compared to the full data sets. 
For each experiment the upper figure shows the bias for the item parameters $a$ and $b$. The lower figure shows a kernel density estimate for the ability parameters $\theta$ with a LOESS regression line in dark green. 
The ability parameters were standardized to zero mean and unit variance. In all rows, the vertical axis is scaled such as to display $2\,{\mathrm{std.}}$ of the corresponding parameter estimate obtained from the full data set.}
\begin{center}
\begin{tabular}{ccc}
{\tiny{$\mathbf{n=11\,532, m=88, k=750}$}}&{\tiny{$\mathbf{n=11\,532, m=88, k=1\,000}$}}&{\tiny{$\mathbf{n=11\,532, m=88, k=2\,000}$}}
\\
\includegraphics[width=0.3\linewidth]{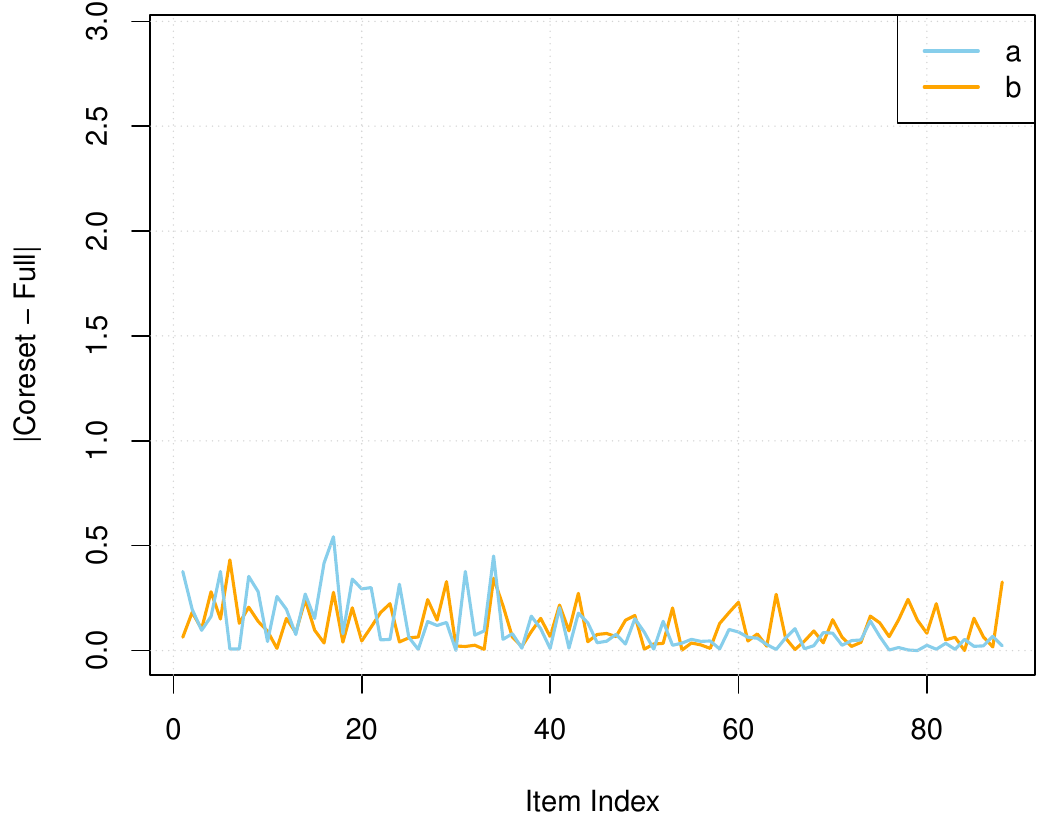} &
\includegraphics[width=0.3\linewidth]{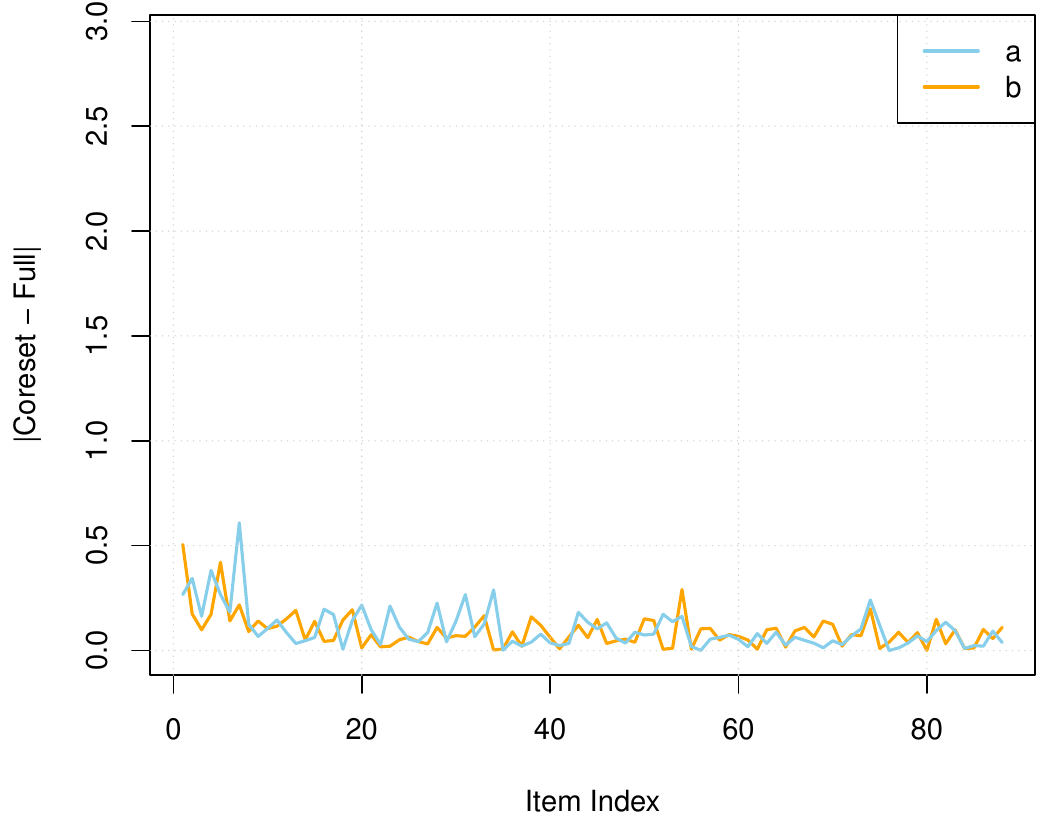} &
\includegraphics[width=0.3\linewidth]{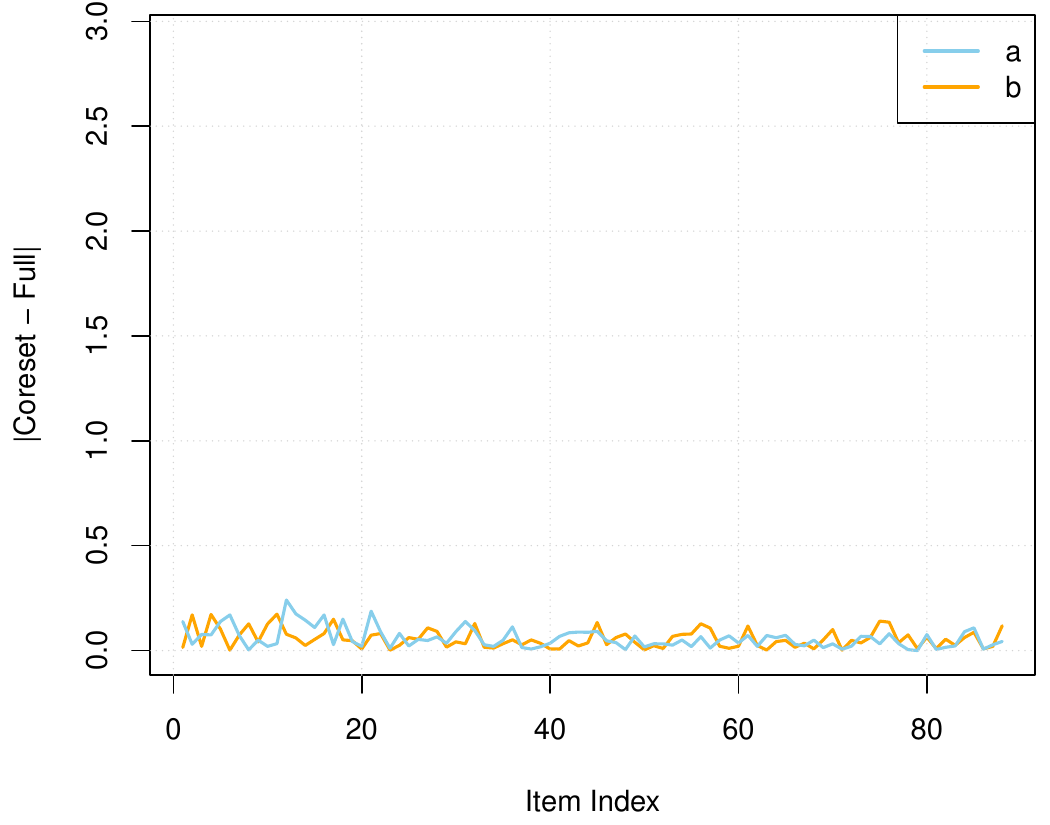}
 \\
 \includegraphics[width=0.3\linewidth]{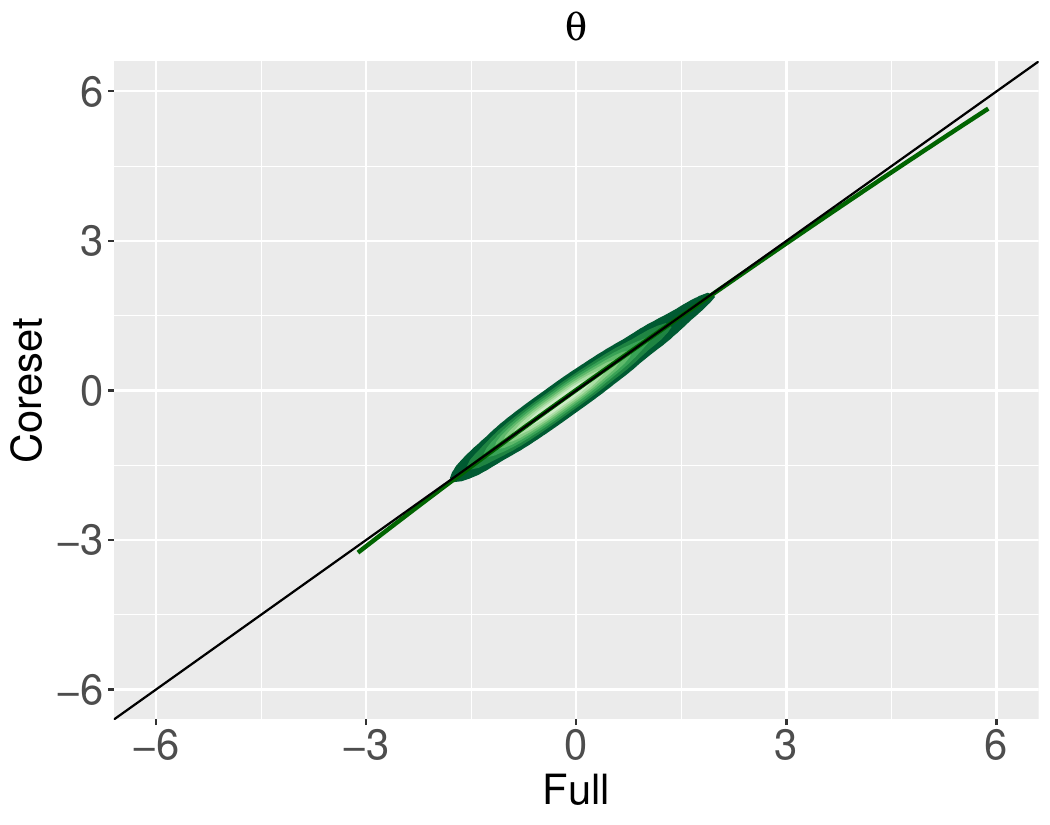} &
\includegraphics[width=0.3\linewidth]{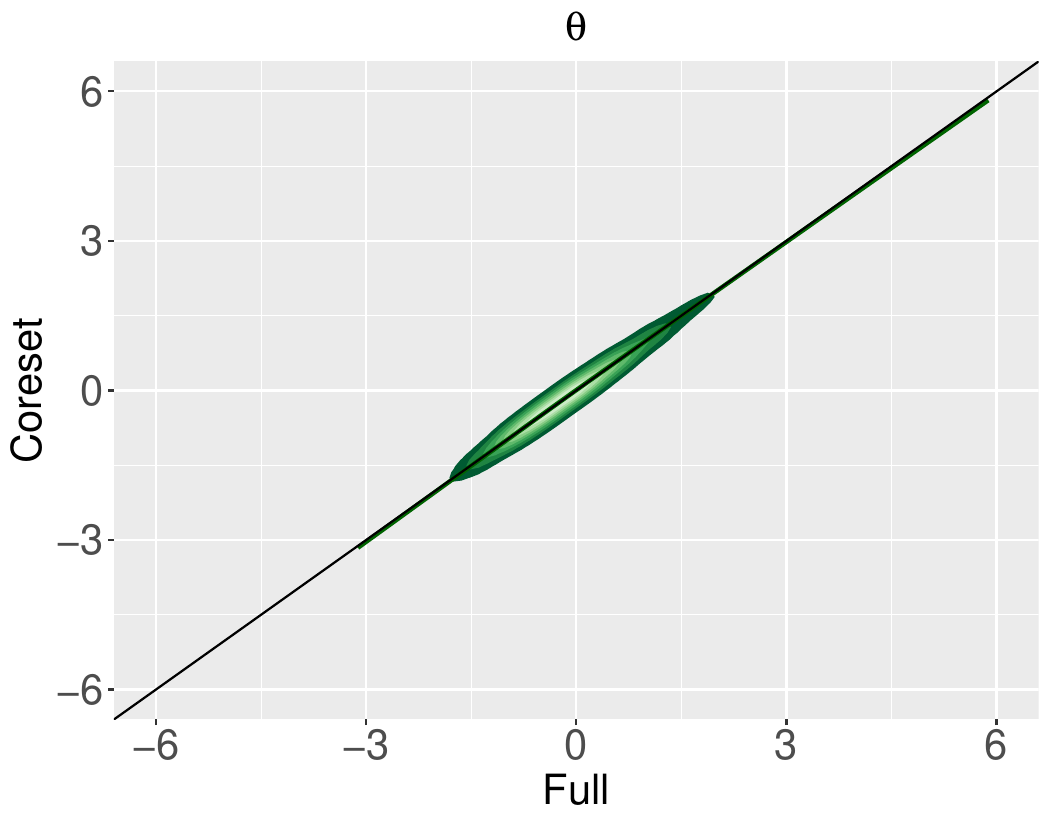} &
\includegraphics[width=0.3\linewidth]{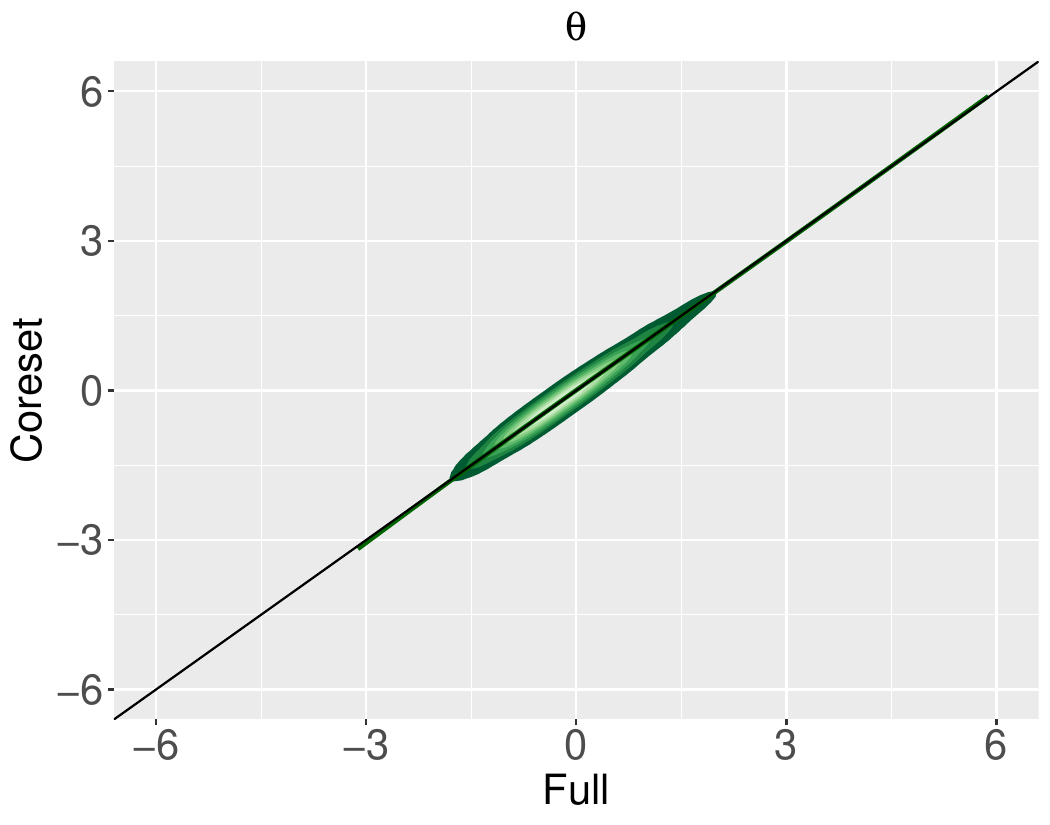}
\\
{\tiny{$\mathbf{n=11\,532, m=88, k=5\,000}$}}&{\tiny{$\mathbf{n=11\,532, m=88, k=10\,000}$}}&
\\
\includegraphics[width=0.3\linewidth]{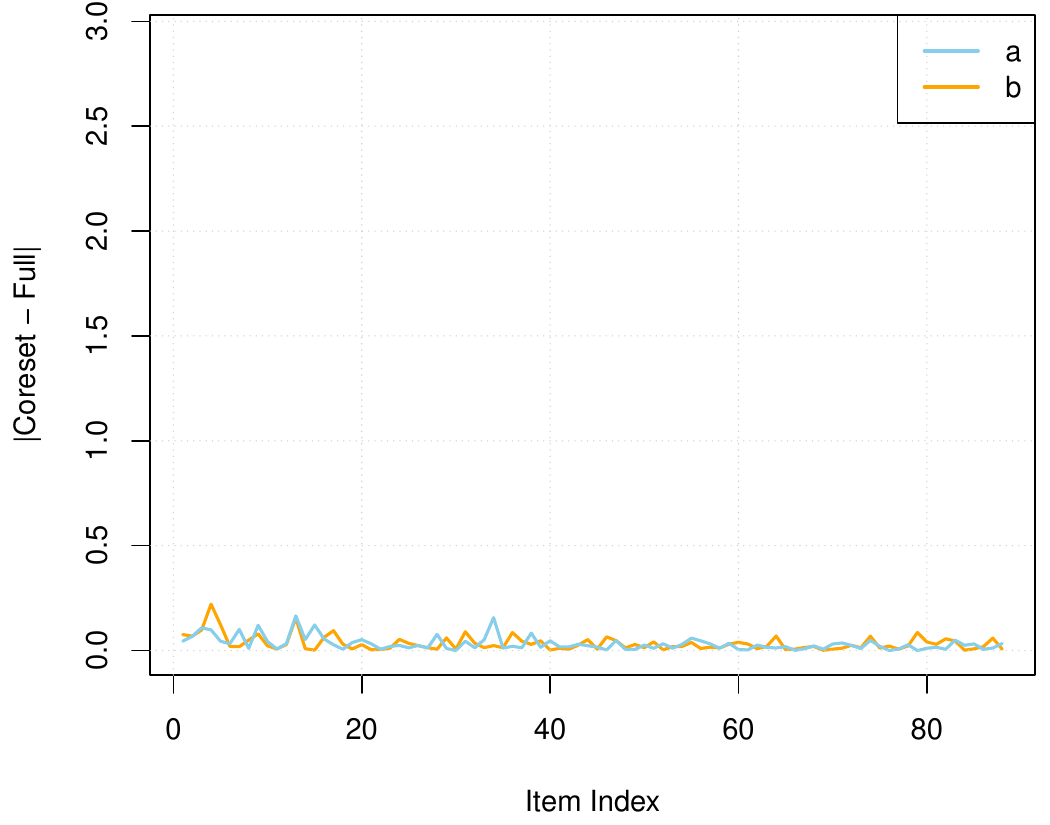} &
\includegraphics[width=0.3\linewidth]{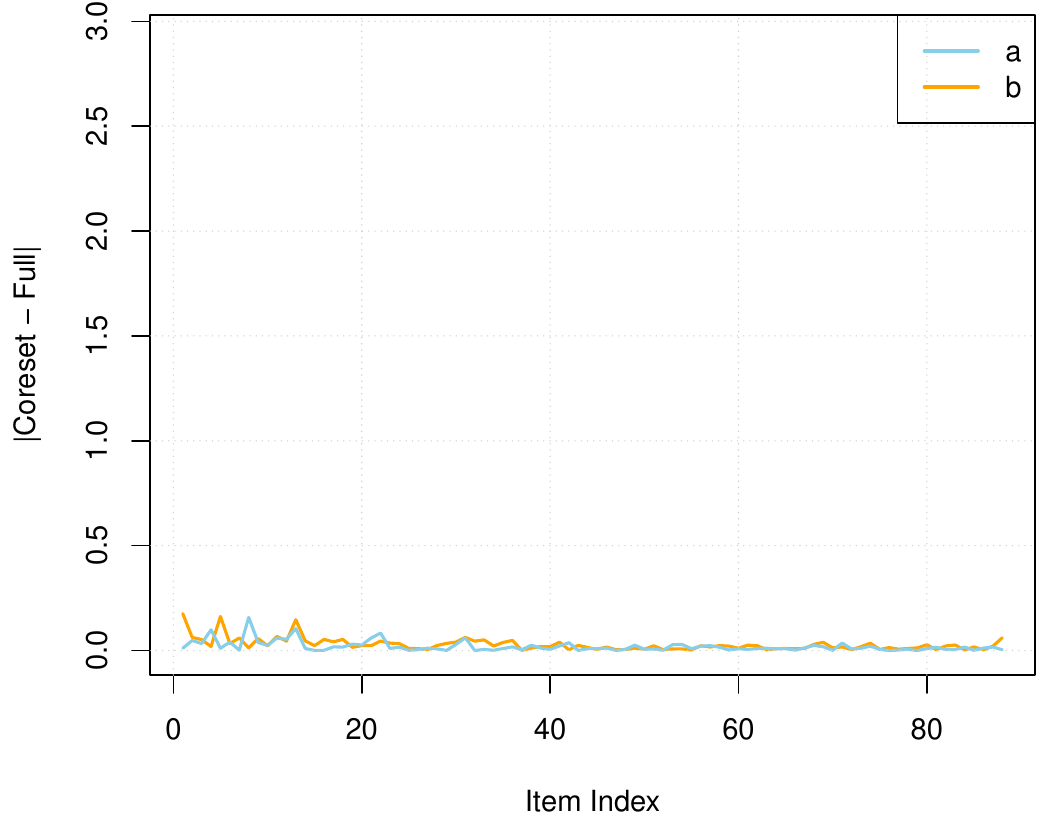} &
 \\
 \includegraphics[width=0.3\linewidth]{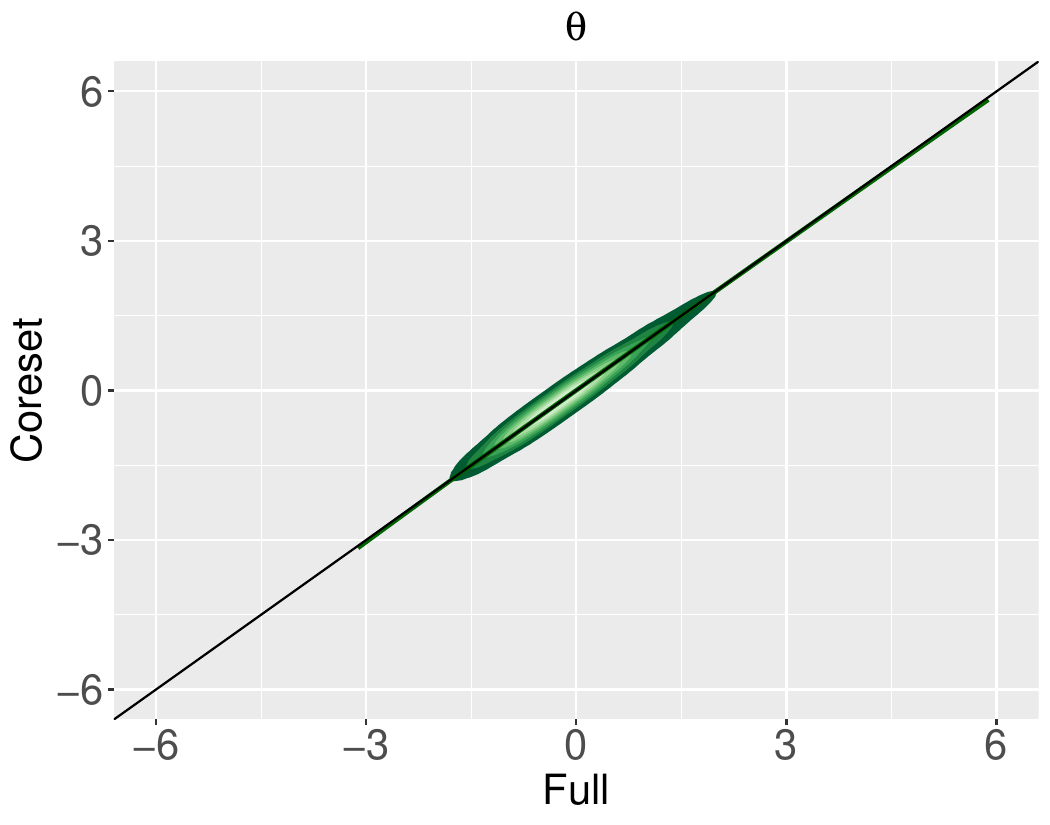} &
\includegraphics[width=0.3\linewidth]{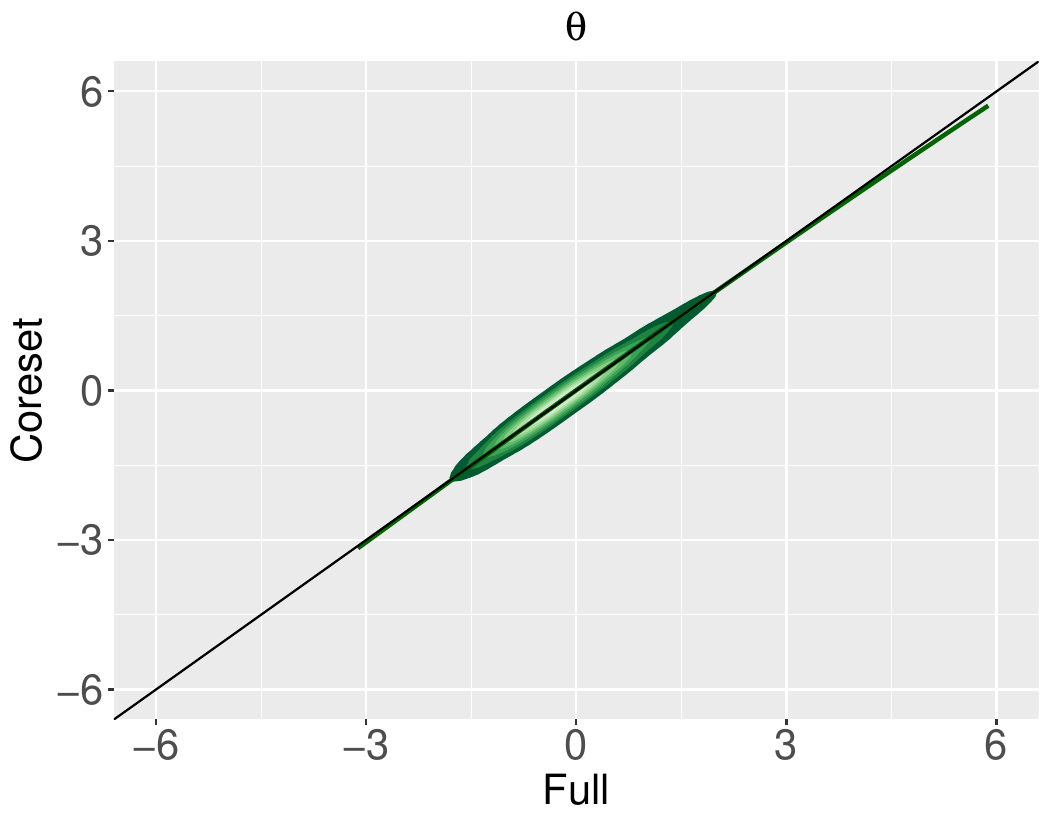} &
\end{tabular}
\label{fig:param_exp_appendix8}
\end{center}
\end{figure*}

\begin{table*}[hp!]
\caption{3PL Experiments on synthetic data: The means and standard deviations (std.) of running times, taken across $20$ repetitions. 
{In each repetition, the running time (in minutes) of 50 iterations of the main loop was measured}
per data set for different configurations of the data dimensions: the number of items $m$, the number of examinees $n$, and the coreset size $k$. The (relative) gain is defined as  $(1-\mathrm{mean}_{\sf coreset}/\mathrm{mean}_{\sf full})\cdot 100$ \%.
}
	\label{tab:results_appendix4}
	\begin{center}
	\begin{tabular}{ c r r r| r r| r r| r }
    &\multicolumn{3}{c}{}& \multicolumn{2}{c}{{{\bf Full data} (min)}}& \multicolumn{2}{c}{{{\bf Coresets} (min)}} & \multicolumn{1}{c}{ } \\
		\hline
		{\bf data} &
		{$\mathbf n$} & {$\mathbf m$} & {$\mathbf k$} &  {\bf mean} & {\bf std.} & {\bf mean} & {\bf std.} & {\bf gain} \\ \hline \hline
    3PL-Syn & $50\,000$ & $100$ & $2\,000$ & $211.468$ & $31.355$ & $41.648$ & $5.197$ & $80.305$ \%  \\ \hline
    3PL-Syn & $50\,000$ & $100$ & $5\,000$ & $211.468$ & $31.355$ & $90.243$ & $12.134$ & $57.325$ \% \\ \hline
    3PL-Syn & $50\,000$ & $100$ & $10\,000$ & $211.468$ & $31.355$ & $93.780$ & $13.929$ & $55.653$ \%  \\ \hline\hline
    3PL-Syn & $50\,000$ & $200$ & $2\,000$ & $369.816$ & $36.676$ & $50.588$ & $1.962$ & $86.321$ \%  \\ \hline
    3PL-Syn & $50\,000$ & $200$ & $5\,000$ & $369.816$ & $36.676$ & $89.274$ & $30.368$ & $75.860$ \%  \\ \hline
    3PL-Syn & $50\,000$ & $200$ & $10\,000$ &  $369.816$ & $36.676$ & $145.674$ & $25.702$ & $60.609$ \% \\ \hline\hline
    3PL-Syn & $100\,000$ & $100$ & $5\,000$ &  $412.616$ & $65.389$ & $125.407$ & $15.408$ & $69.607$ \% \\ \hline
    3PL-Syn & $100\,000$ & $200$ & $5\,000$ &  $722.319$ & $118.262$ & $150.164$ & $26.767$ & $79.211$ \%  \\ \hline
    3PL-Syn & $200\,000$ & $100$ & $10\,000$ &  $893.183$ & $112.257$ & $196.802$ & $14.608$ & $77.966$ \% \\ \hline\hline
	\end{tabular}
	\end{center}
\end{table*}

\begin{table*}[hp!]
\caption{3PL Experiments on synthetic data: 
The quality of the solution found.
Let $f_{\sf full}$ and $f_{{\sf core}(j)}$ be the optimal values of the loss function on the input and on the coreset for the $j$-th repetition, respectively. Let $f_{\sf core} = \min_j f_{{\sf core}(j)}$.
Mean and standard deviation of the relative deviation $|f_{\sf core} - f_{{\sf core}(j)}| / f_{\sf core}$ (in $\%$): 
\textbf{mean dev} and \textbf{std. dev}. Relative error: \textbf{rel. error} $\hat{\varepsilon}=|f_{\sf core} - f_{\sf full}|/f_{\sf full}$ (cf. \lemref{coreset:error:approx}).
Mean Absolute Deviation: $\textbf{mad}(\alpha)=\frac{1}{n}\sum (|a_{\sf full}-a_{\sf core}| + |b_{\sf full}-b_{\sf core}| + |c_{\sf full}-c_{\sf core}| )$; $\textbf{mad}(\theta)=\frac{1}{m}\sum |\theta_{\sf full}-\theta_{\sf core}|$, evaluated on the parameters that attained the optimal $f_{\sf full}$ and $f_{\sf core}$.
}
	\label{tab:results_appendix4:b}
	\begin{center}
	\begin{tabular}{ c r r r| r r | r | c c}
		\hline
		{\bf data} &
		{$\mathbf n$} & {$\mathbf m$} & {$\mathbf k$} &  {\bf mean dev} & {\bf std. dev} & {\bf rel. error $\hat{\varepsilon}$}  & {\bf $\text{mad}(\alpha)$} & {\bf $\text{mad}(\theta)$}\\ \hline \hline
	3PL-Syn & $50\,000$ & $100$ & $2\,000$ & $4.495$ \% & $2.392$ \% & $0.45212$ & $2.820$ & $0.625$ \\ \hline
    3PL-Syn & $50\,000$ & $100$ & $5\,000$ & $2.061$ \% & $1.935$ \% & $0.03228$ & $0.968$ & $0.048$ \\ \hline
    3PL-Syn & $50\,000$ & $100$ & $10\,000$ & $2.237$ \% & $2.417$ \% & $0.00212$ & $0.384$ & $0.010$ \\ \hline\hline
    3PL-Syn & $50\,000$ & $200$ & $2\,000$ & $5.280$ \% & $3.065$ \% & $0.43784$ & $2.832$ & $0.649$ \\ \hline
    3PL-Syn & $50\,000$ & $200$ & $5\,000$ & $4.536$ \% & $2.615$ \% & $0.01662$ & $0.906$ & $0.037$ \\ \hline
    3PL-Syn & $50\,000$ & $200$ & $10\,000$ & $3.306$ \% & $1.459$ \% & $0.02186$ & $0.488$ & $0.001$ \\ \hline\hline
    3PL-Syn & $100\,000$ & $100$ & $5\,000$ & $8.370$ \% & $3.944$ \% & $0.02065$ & $1.375$ & $0.101$ \\ \hline
    3PL-Syn & $100\,000$ & $200$ & $5\,000$ & $4.819$ \% & $1.784$ \% & $0.06281$ & $1.545$ & $0.140$\\ \hline
    3PL-Syn & $200\,000$ & $100$ & $10\,000$ & $3.413$ \% & $2.529$ \% & $0.01789$ & $0.524$ & $0.003$ \\ \hline\hline
	\end{tabular}
	\end{center}
\end{table*}

\begin{figure*}[hp!]
\caption{3PL Experiments on synthetic data. Parameter estimates for the coresets compared to the full data sets. 
For each experiment the upper figure shows the bias for the item parameters $a$ and $b$. The lower figure shows a kernel density estimate for the ability parameters $\theta$ with a LOESS regression line in dark green. 
The ability parameters were standardized to zero mean and unit variance. In all rows, the vertical axis is scaled such as to display $4\,{\mathrm{std.}}$ of the corresponding parameter estimate obtained from the full data set.}
\begin{center}
\begin{tabular}{ccc}
{\tiny{$\mathbf{n=50\,000,m=100,k=2\,000}$}}&{\tiny{$\mathbf{n=50\,000,m=100,k=5\,000}$}}&{\tiny{$\mathbf{n=50\,000,m=100,k=10\,000}$}}
\\
\includegraphics[width=0.3\linewidth]{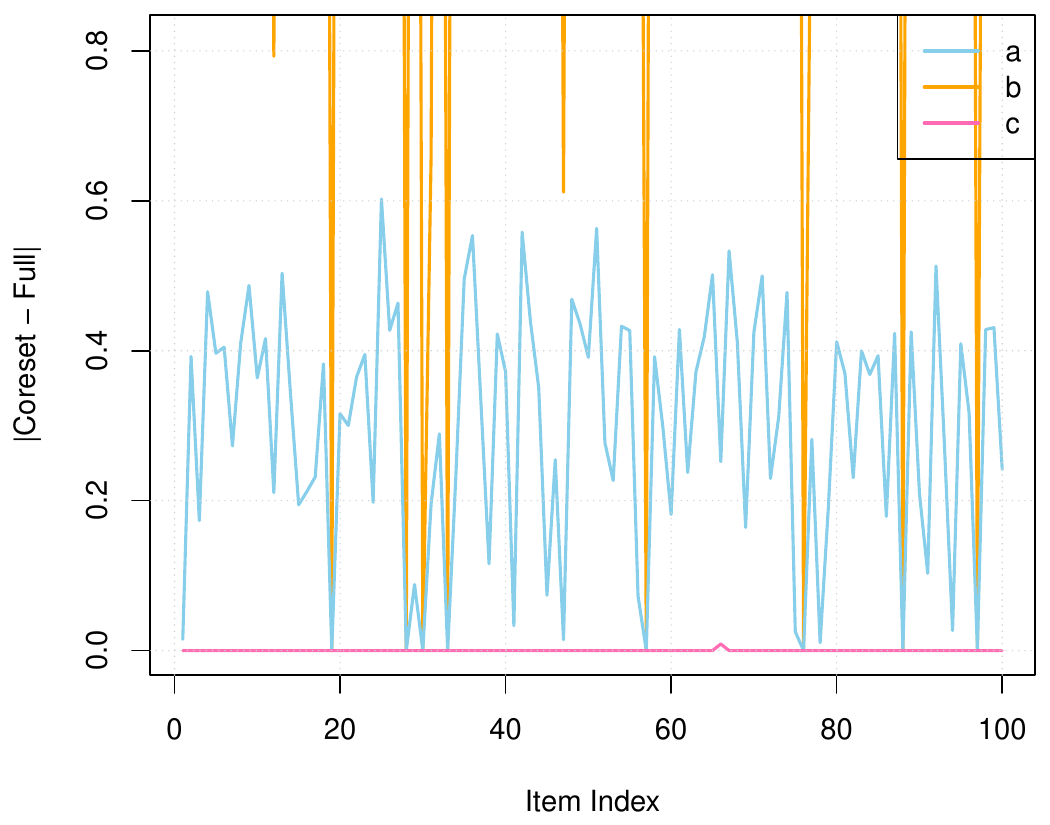}&
\includegraphics[width=0.3\linewidth]{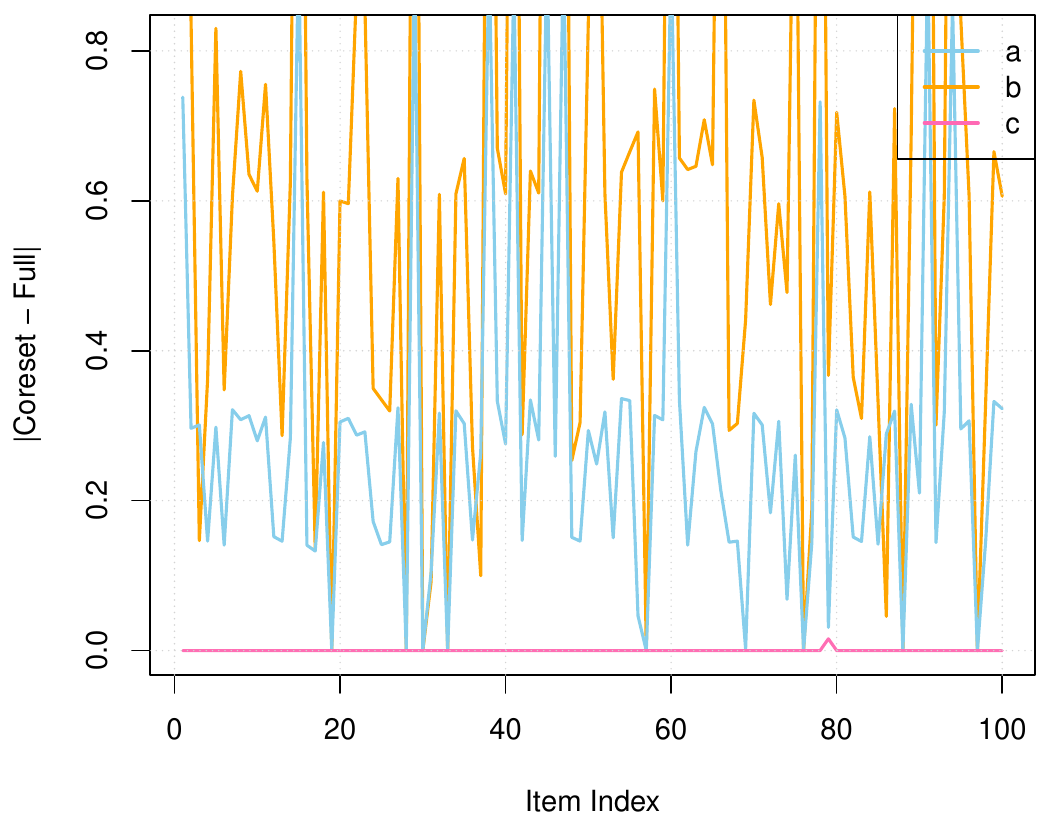}
 &
 \includegraphics[width=0.3\linewidth]{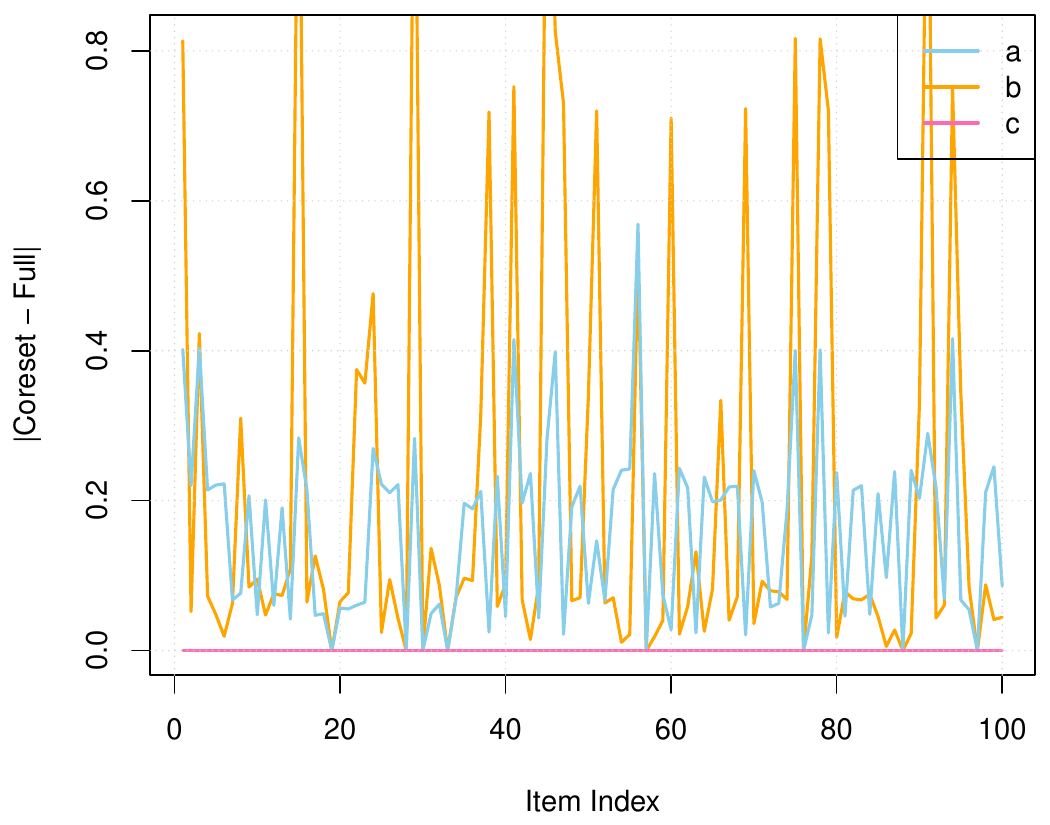}
 \\
\includegraphics[width=0.3\linewidth]{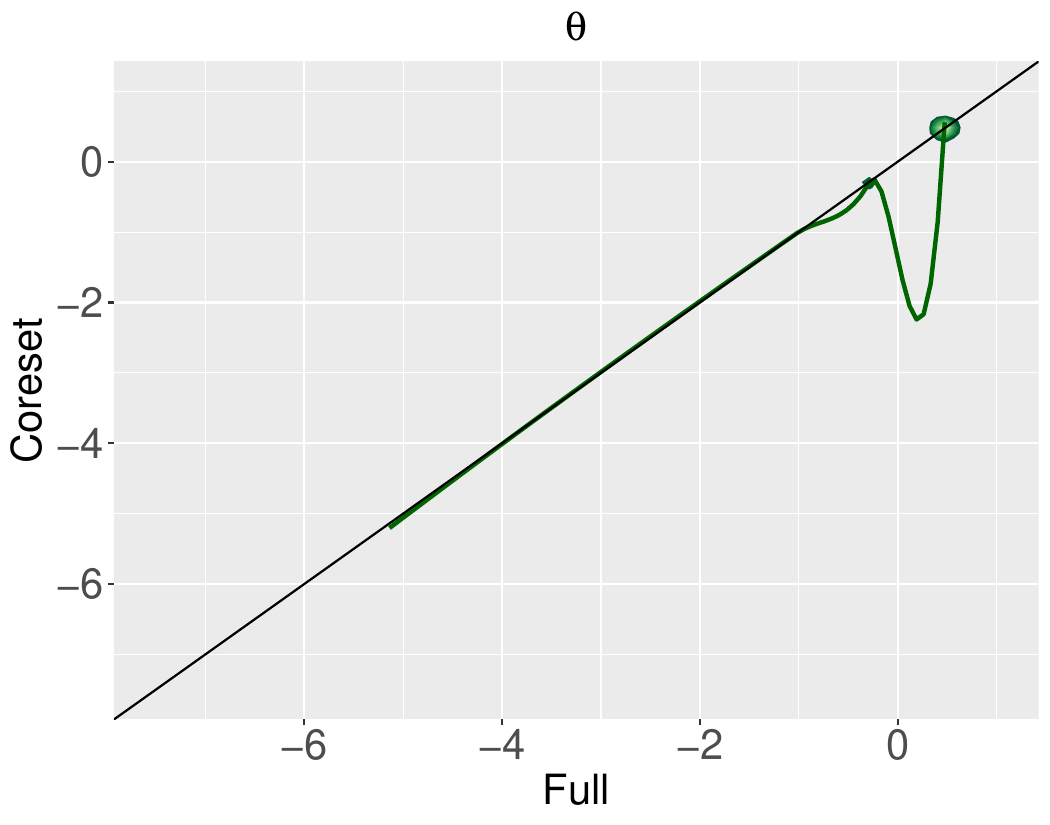}&
\includegraphics[width=0.3\linewidth]{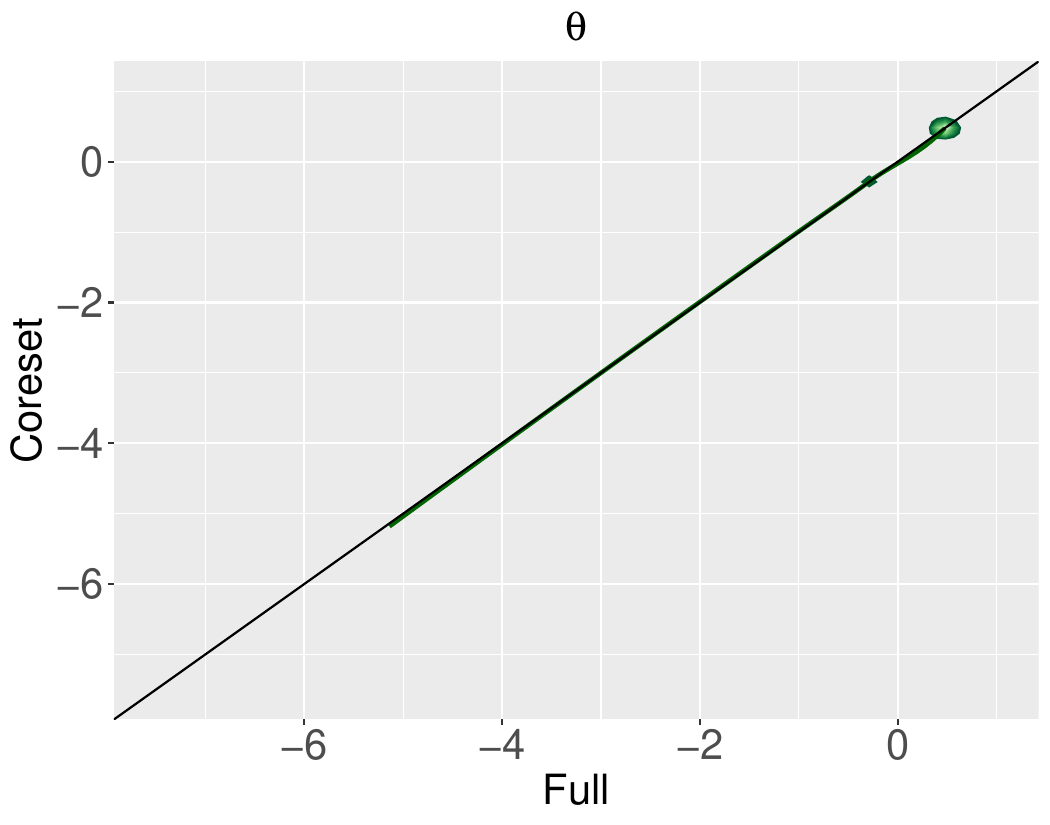} &
\includegraphics[width=0.3\linewidth]{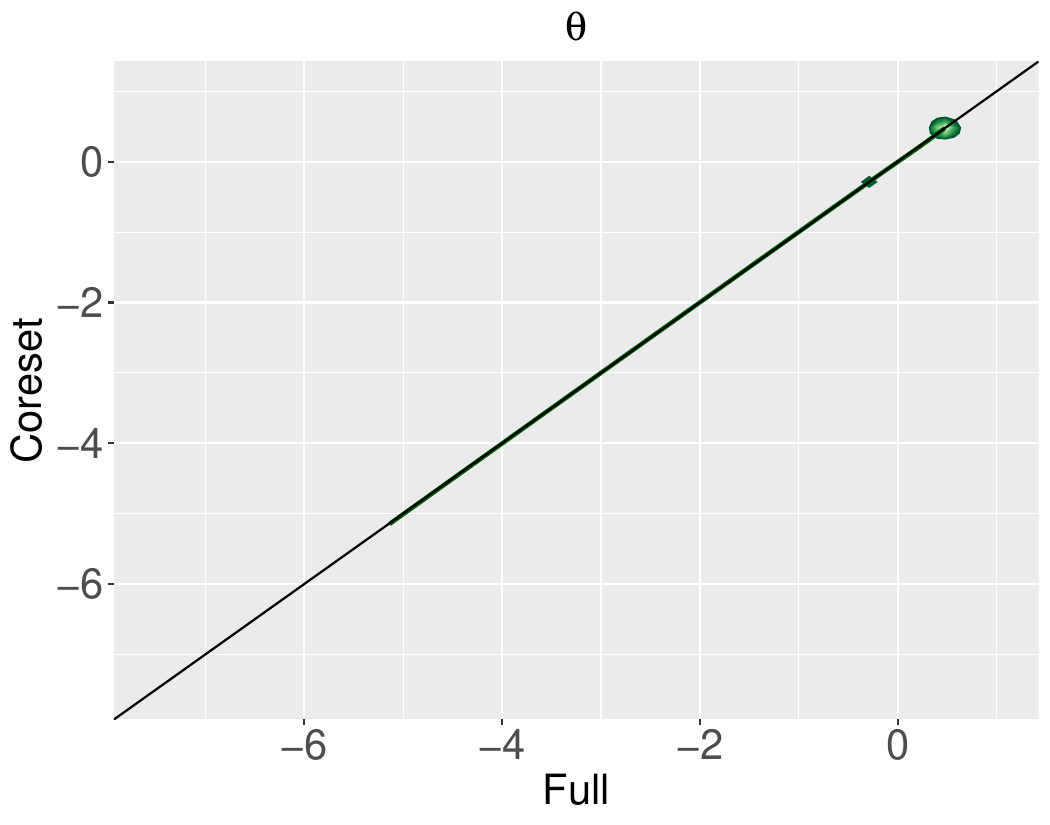}
\\
{\tiny{$\mathbf{n=50\,000,m=200,k=2\,000}$}}&{\tiny{$\mathbf{n=50\,000,m=200,k=5\,000}$}}&{\tiny{$\mathbf{n=50\,000,m=200,k=10\,000}$}}
\\
\includegraphics[width=0.3\linewidth]{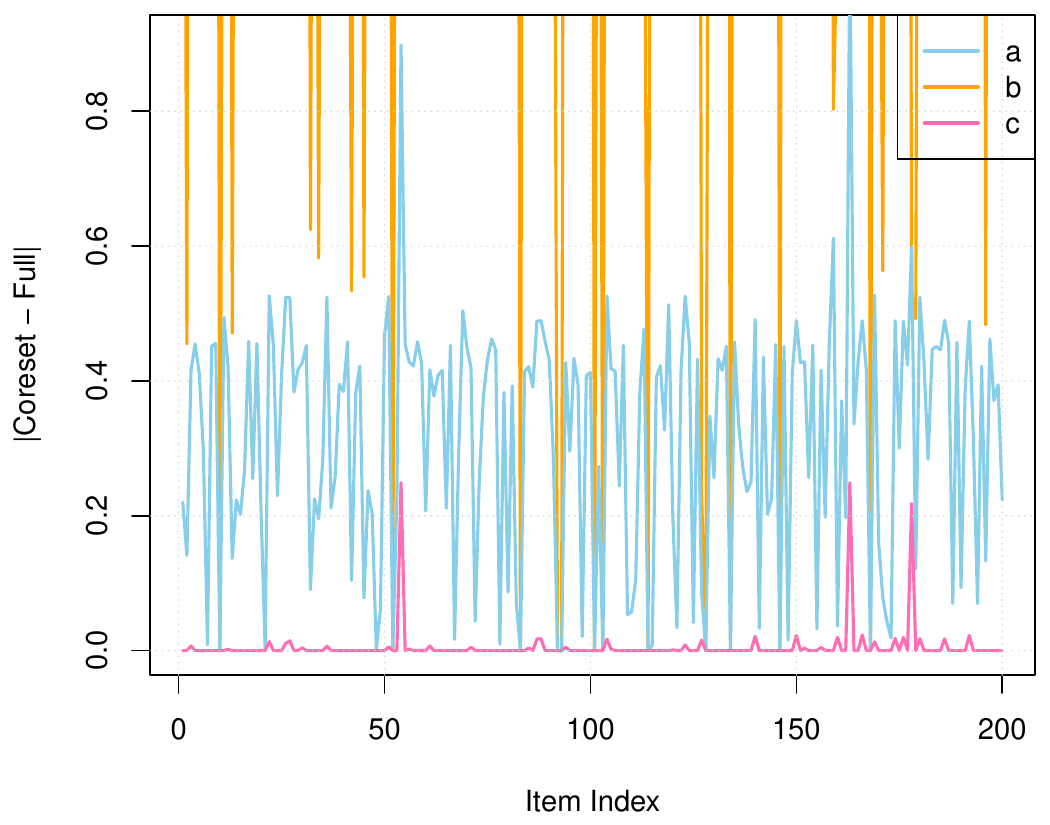}&
\includegraphics[width=0.3\linewidth]{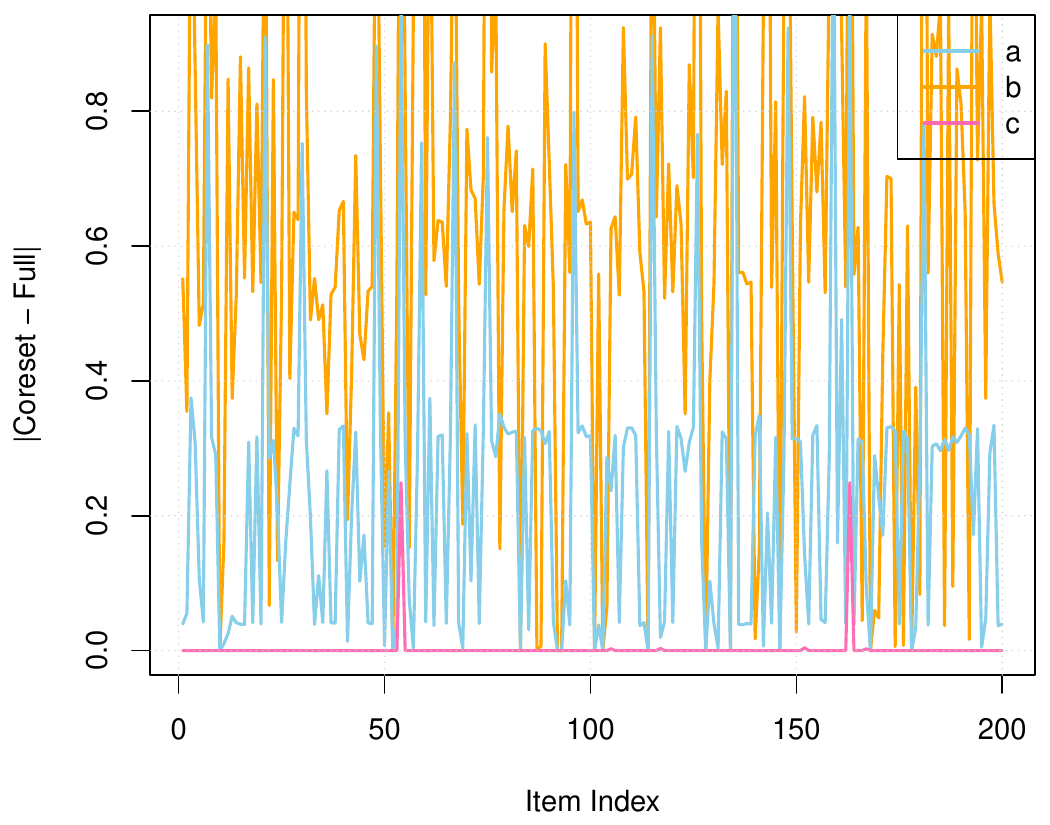}
 & \includegraphics[width=0.3\linewidth]{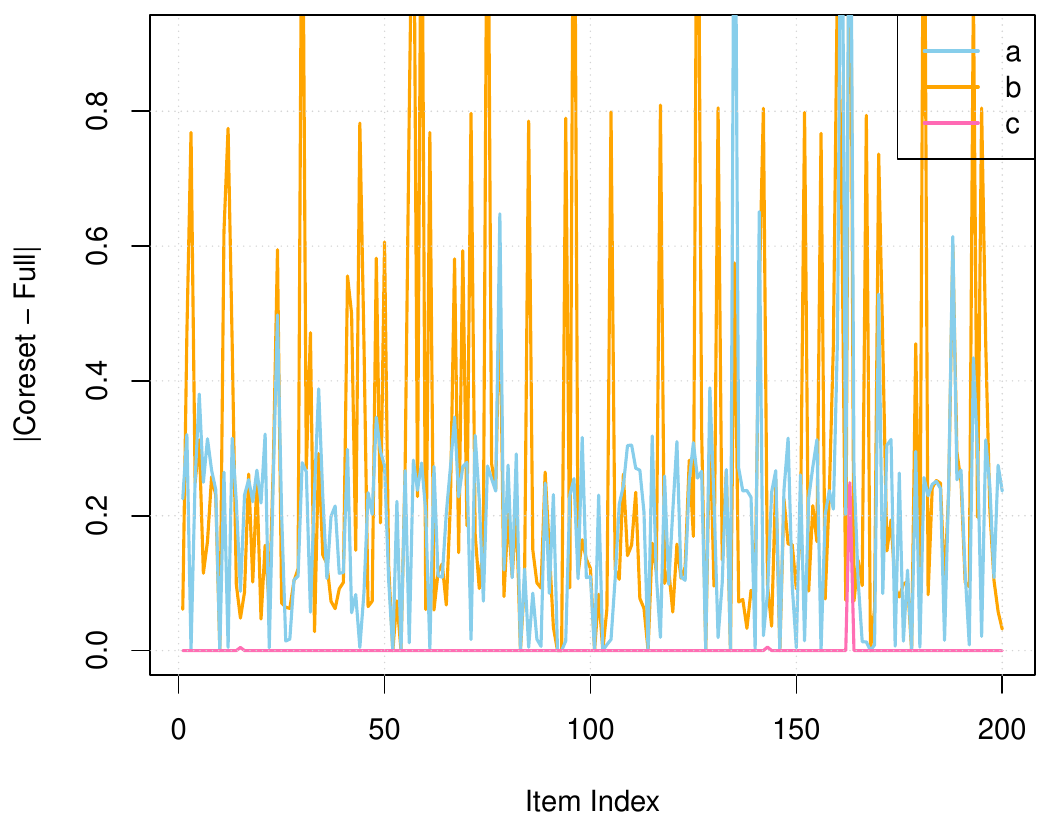}
 \\
\includegraphics[width=0.3\linewidth]{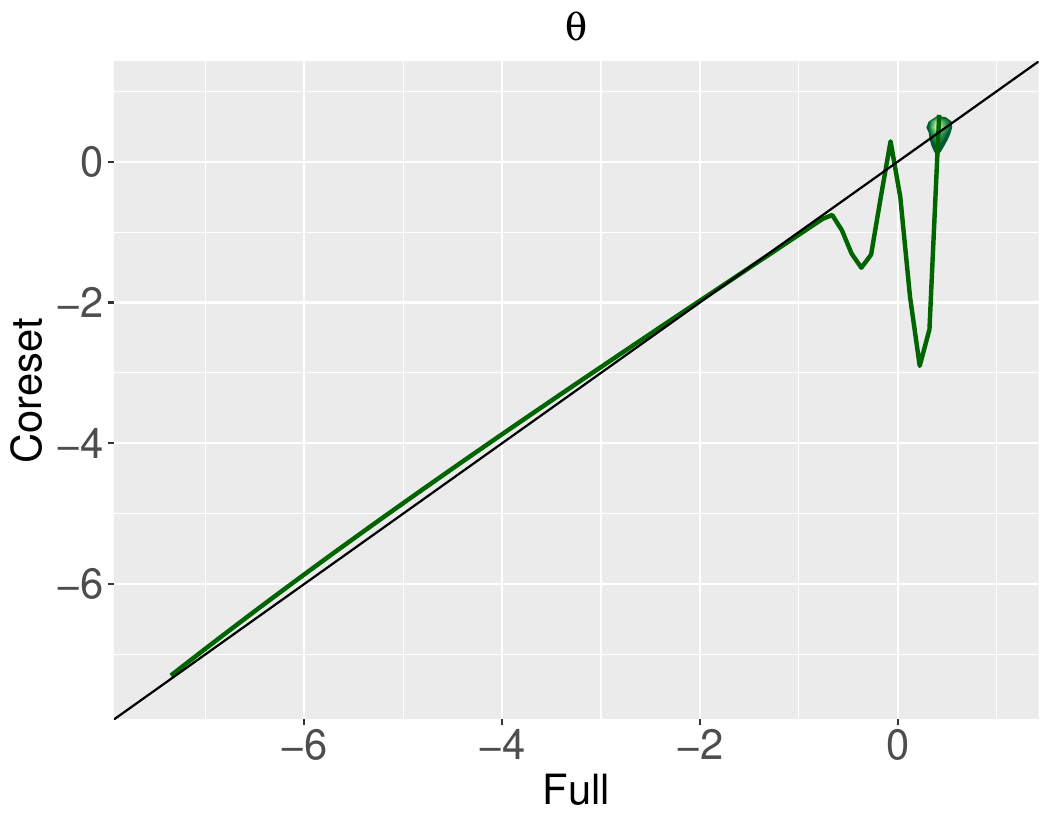}&
\includegraphics[width=0.3\linewidth]{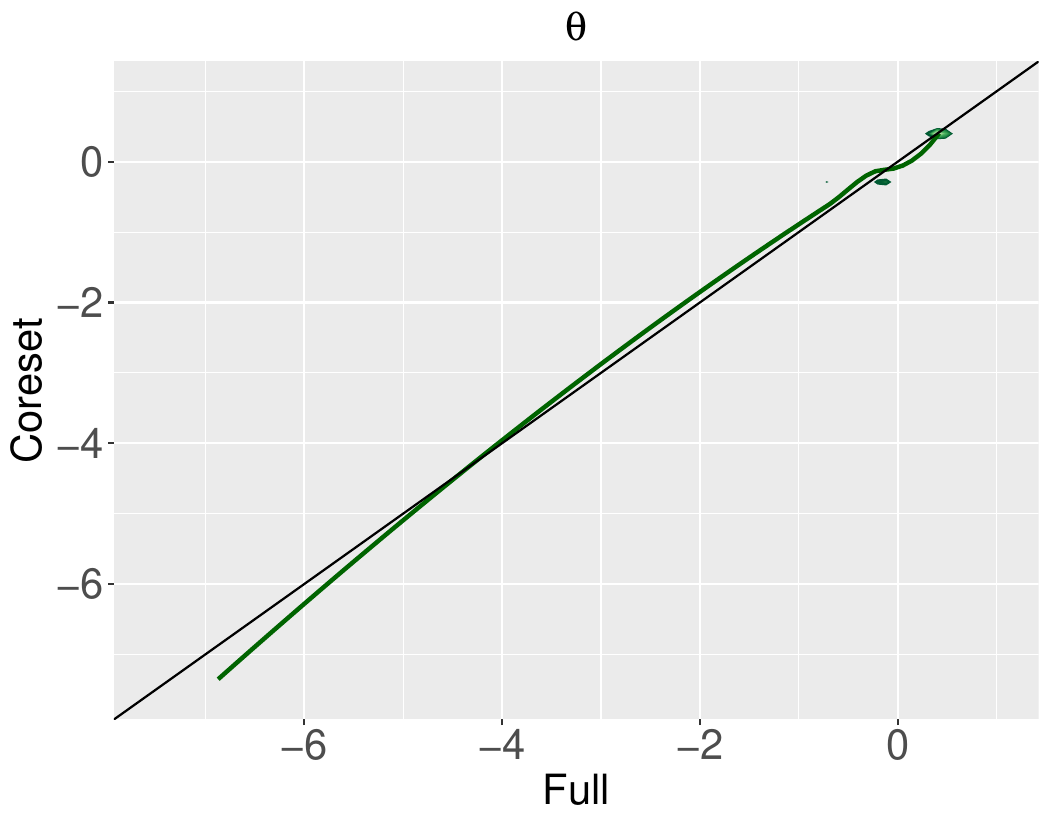}
 &
 \includegraphics[width=0.3\linewidth]{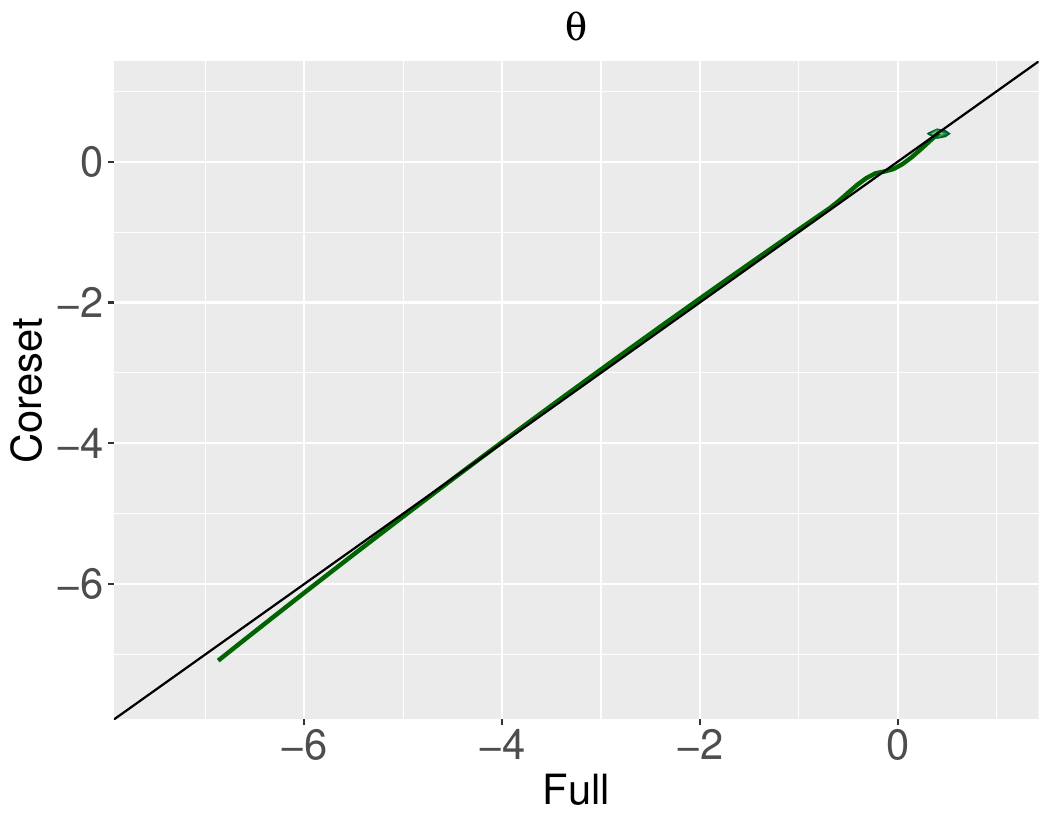}
 \\ 
\end{tabular}
\label{fig:param_exp_appendix6}
\end{center}
\end{figure*}

\begin{figure*}[hp!]
\caption{3PL Experiments on synthetic data. Parameter estimates for the coresets compared to the full data sets. 
For each experiment the upper figure shows the bias for the item parameters $a$ and $b$. The lower figure shows a kernel density estimate for the ability parameters $\theta$ with a LOESS regression line in dark green. 
The ability parameters were standardized to zero mean and unit variance. In all rows, the vertical axis is scaled such as to display $4\,{\mathrm{std.}}$ of the corresponding parameter estimate obtained from the full data set.}
\begin{center}
\begin{tabular}{ccc}
{\tiny{$\mathbf{n=100\,000,m=100,k=5\,000}$}}&{\tiny{$\mathbf{n=100\,000,m=200,k=5\,000}$}}&{\tiny{$\mathbf{n=200\,000,m=100,k=10\,000}$}}
\\
\includegraphics[width=0.3\linewidth]{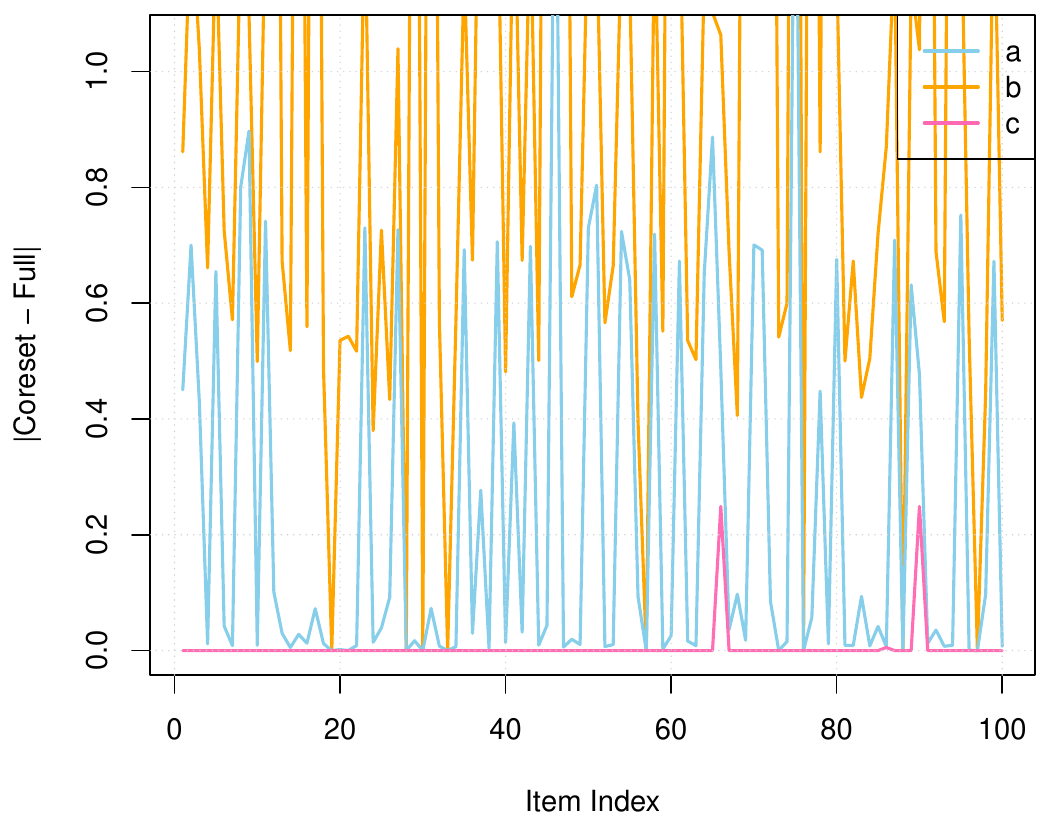}&
\includegraphics[width=0.3\linewidth]{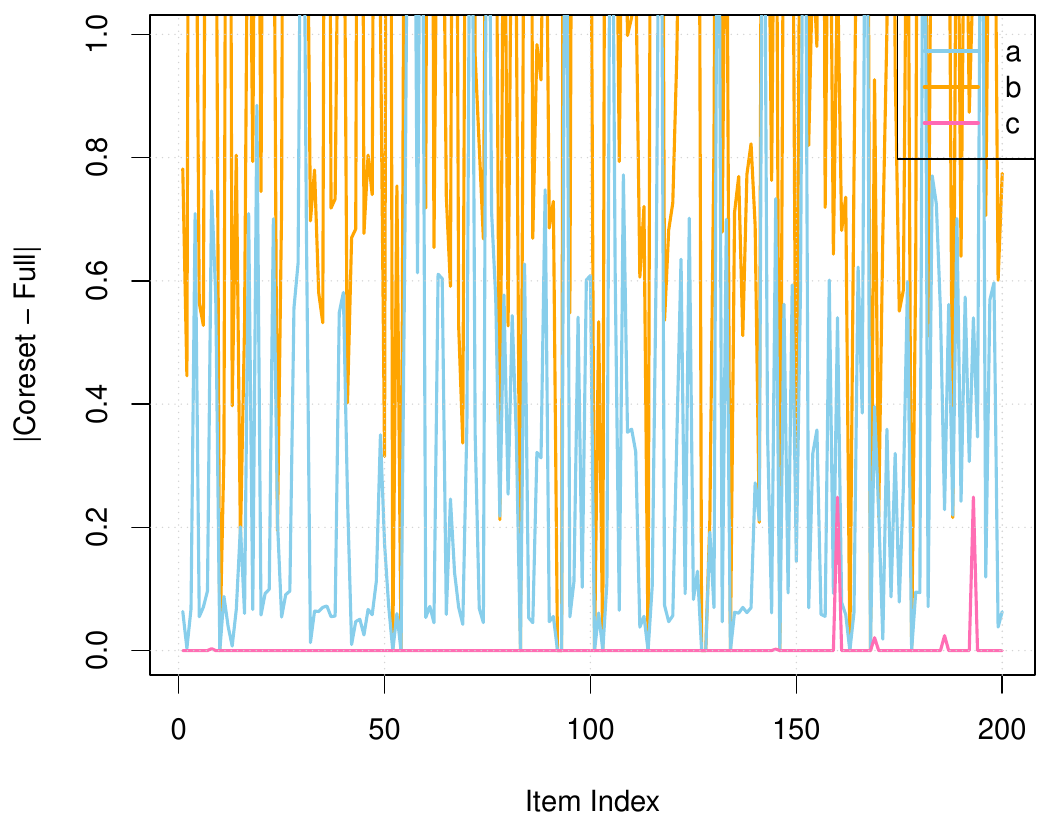}
 &
\includegraphics[width=0.3\linewidth]{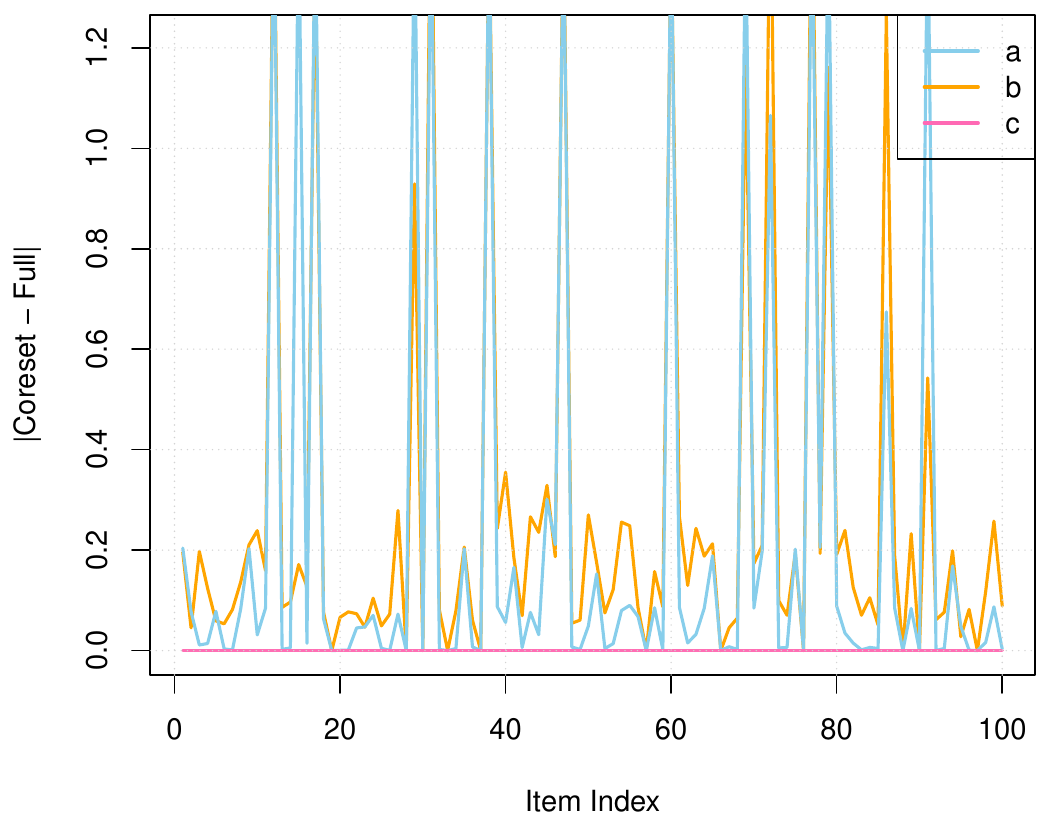}
 \\
\includegraphics[width=0.3\linewidth]{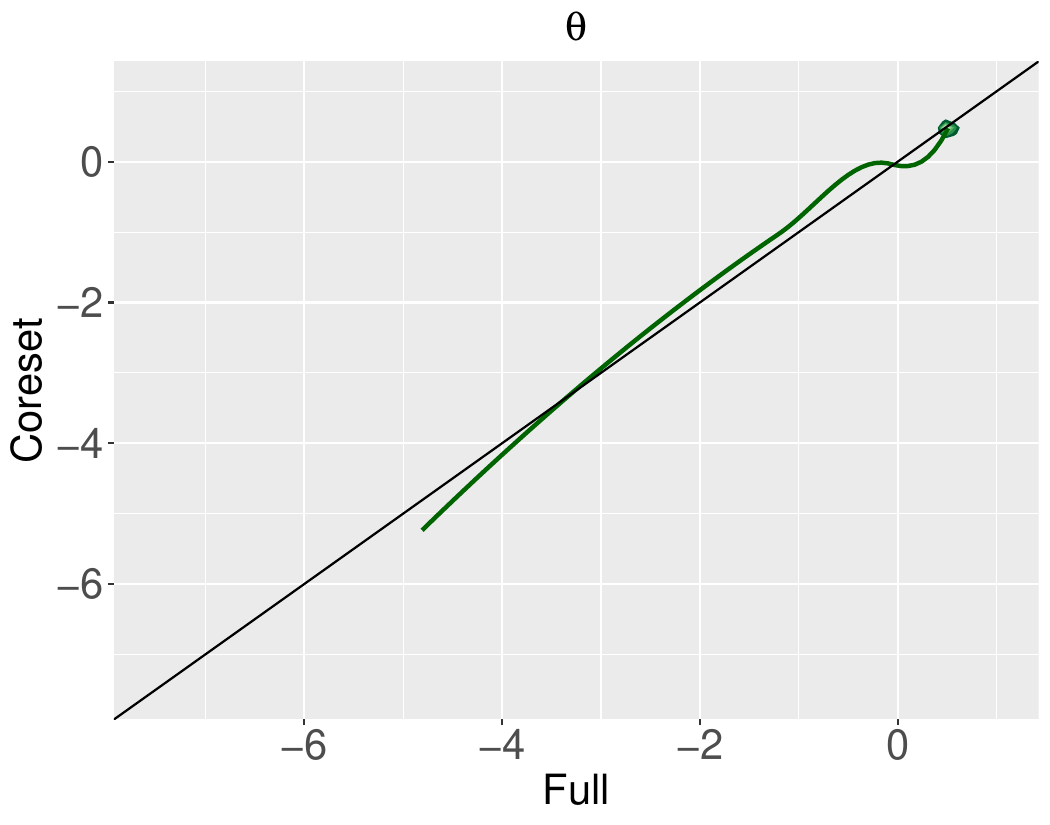}&
\includegraphics[width=0.3\linewidth]{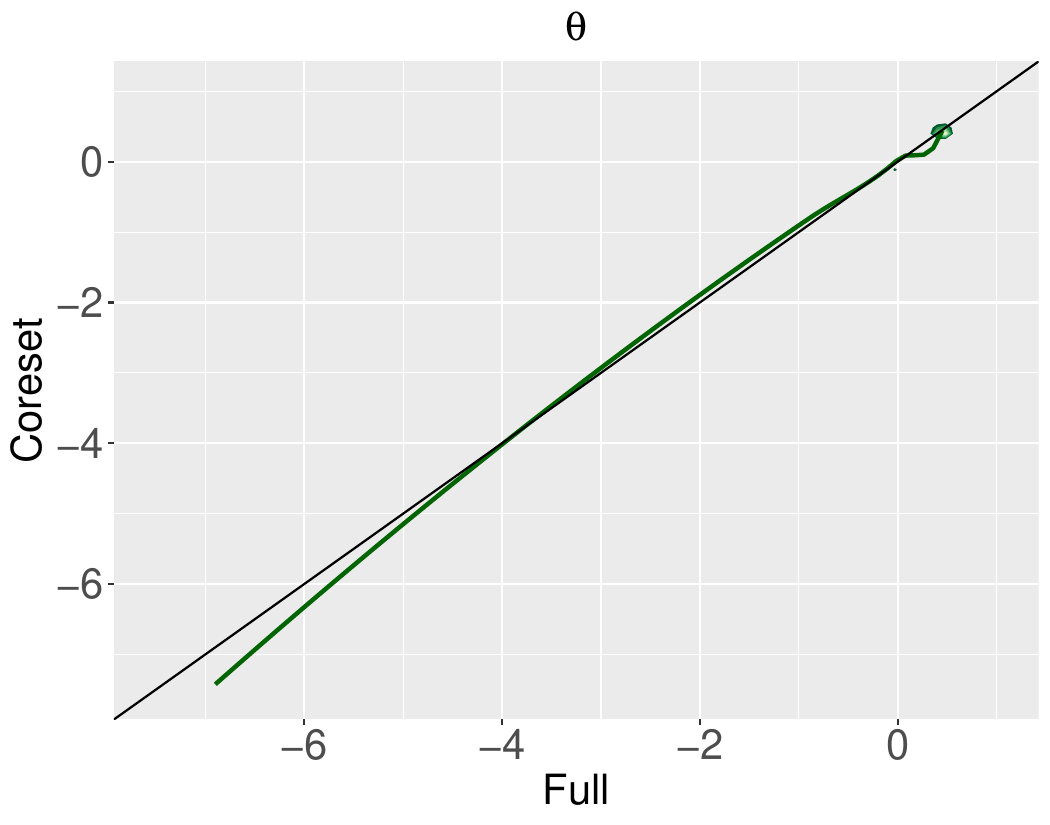} &
\includegraphics[width=0.3\linewidth]{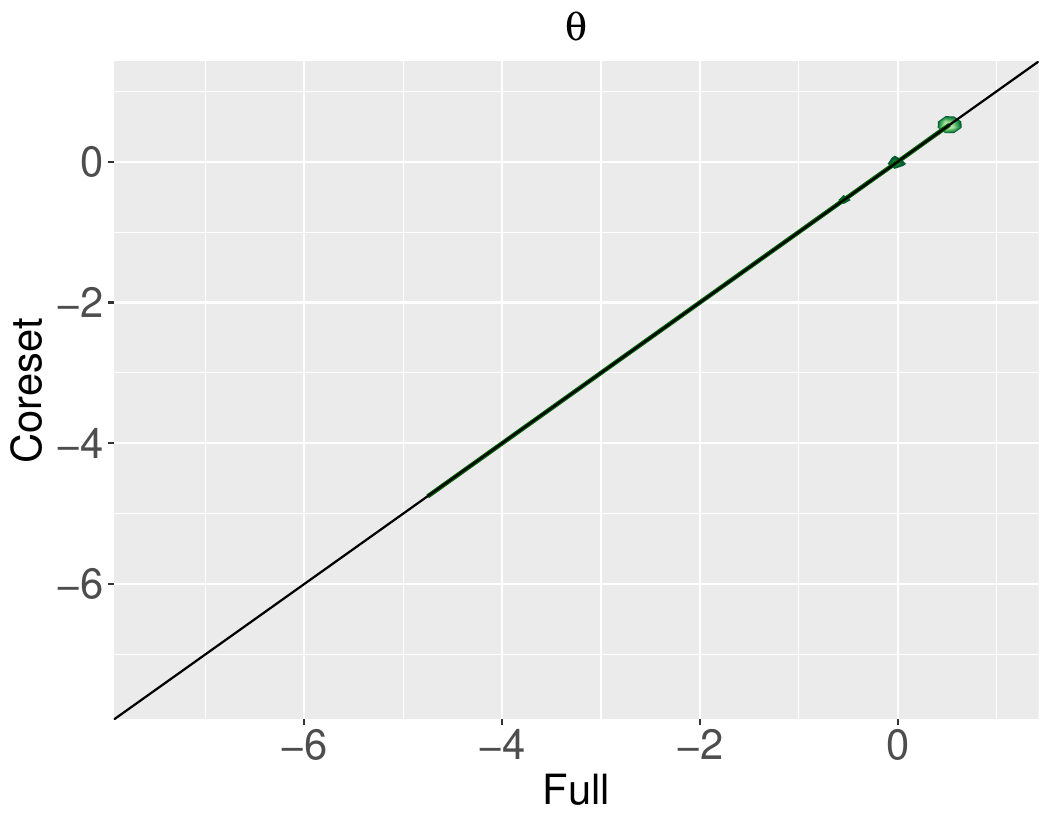}
 \\
\end{tabular}
\label{fig:param_exp_appendix7}
\end{center}
\end{figure*}

\clearpage
\section{COMPARISON TO UNIFORM SAMPLING}
\label{appendix:uniform}

The interested reader may ask why not to simply use uniformly sampled subsets of the input instead of coresets, as this is arguably the de facto standard baseline used for estimating IRT models from subsamples. For instance, \citet{karadavut2016comparison} showed in an extensive comparison that uniform sampling works better than standard $\ell_2$-leverage score methods (note that we use \emph{square root} $\ell_2$-leverage scores, which makes a large difference). Further, uniform sampling is commonly used for constructing training data by subsampling from the complete data space $\{-1,1\}^{m\times n}$~\citep{bonifay2017complexity}.

However, it is well known that uniform samples of sublinear size cannot yield strong multiplicative approximation guarantees, even for mild data with $\mu=1$. This also holds for other techniques that rely on uniform subsampling, such as stochastic gradient descent (SGD) as the authors demonstrate theoretically, and practically in \citep{MunteanuSSW18}. Coresets, in contrast, are designed to provably approximate the loss to within a $(1+\varepsilon)$ factor with  sublinear sample size in the natural case where $\mu$ is bounded.

To corroborate this in the context of IRT models, we compared between the approximation achieved by uniformly sampled subsets of the input and our coresets, after 50 iterations for 2PL IRT models on synthetic data (generated as described in the main body) and on real-world SHARE~\citep{boersch2022survey} and NEPS data~\citep{NEPS-SC4}.
The results are measured for both methods in terms of mean absolute deviations of calculated estimates from the actual item parameters and from the actual ability parameter, as well in terms of the relative error of the objective function, cf. \lemref{coreset:error:approx}, summarized in~\cref{tab:results_uniform,tab:results_uniform:share,tab:results_uniform:neps}.

Initial experiments showed that the uniform samples were consistently less accurate by (at least) one order of magnitude regarding the Mean Absolute Deviation (MAD). To get an impression of the best performance of the two methods, we repeat both experiments using uniform samples and the coresets 20 times independently and compare the best result for each method to one another. Note that the information on which repetition gave the best result is not available in practice, so this is an overly optimistic scenario.

Indeed, for the best performing repetition, the parameter estimates from uniform samples w.r.t MAD are comparable up to a negligible amount. But the relative error of the objective function approximation using uniform samples is very large. For the synthetic data, the relative error is always around $50\,\%$, while for the real-world data, we see that the error actually decreases as the data sample size grows. However, to get a result of comparable quality to the coresets, the uniform sample needs to comprise almost the whole input, while our coresets achieve the same error using a tiny fraction of the input (cf.~\cref{tab:results_uniform:neps}). 

We also note that downstream tasks, such as calculating gradients, uncertainty quantification measures, Hessian, Fisher information etc. require a close approximation of the objective function. We thus conclude that coresets are better suited than uniform sampling, even in optimistic situations where the latter yields accurate point estimation results.

\begin{table*}[ht!]
\caption{2PL experiments on synthetic data for uniformly sampled subsets vs. coresets. 
Comparison of the best solutions found taken across $20$ repetitions (each running $50$ iterations of the main loop) per data set for different configurations of the data dimensions: the number of items $m$, the number of examinees $n$, and the uniform sample/coreset size $k$. 
Let $f_{\sf full}$, $f_{{\sf unif}(j)}$ and $f_{{\sf core}(j)}$ be the optimal values of the loss function on the input, on the uniform sample for the $j$-th repetition, and on the coreset for the $j$-th repetition, respectively. Let $f_{\sf unif} = \min_j f_{{\sf unif}(j)}$, and $f_{\sf core} = \min_j f_{{\sf core}(j)}$.
Comparison made w.r.t. Relative errors: {\bf r.err. $\hat{\varepsilon}_{\sf unif}$}$=|f_{\sf unif} - f_{\sf full}|/f_{\sf full}$, {\bf r.err. $\hat{\varepsilon}_{\sf core}$}$=|f_{\sf core} - f_{\sf full}|/f_{\sf full}$ (cf. \lemref{coreset:error:approx}), and Mean Absolute Deviations (MAD): $\text{mad}_{\sf core}(\alpha)=\frac{1}{n}\sum (|a_{\sf full}-a_{\sf core}| + |b_{\sf full}-b_{\sf core}|)$; $\text{mad}_{\sf core}(\theta)=\frac{1}{m}\sum |\theta_{\sf full}-\theta_{\sf core}|$. $\text{mad}_{\sf unif}(\alpha)=\frac{1}{n}\sum (|a_{\sf full}-a_{\sf unif}| + |b_{\sf full}-b_{\sf unif}|)$; $\text{mad}_{\sf unif}(\theta)=\frac{1}{m}\sum |\theta_{\sf full}-\theta_{\sf unif}|$.
}
	\label{tab:results_uniform}
	\begin{center}
		\begin{tabular}{ r r r| c c r | c c r}
    \multicolumn{3}{c}{}& \multicolumn{3}{c}{{\bf Uniform sampling}}& \multicolumn{3}{c}{{\bf Coresets}} \\
		\hline
		{$\mathbf n$} & {$\mathbf m$} & {$\mathbf k$} &  {\bf $\text{mad}_{\sf unif}(\alpha)$} & {\bf $\text{mad}_{\sf unif}(\theta)$} & {\bf r.err. $\hat{\varepsilon}_{\sf unif}$} & {\bf $\text{mad}_{\sf core}(\alpha)$} & {\bf $\text{mad}_{\sf core}(\theta)$}  & {\bf r.err. $\hat{\varepsilon}_{\sf core}$} \\ \hline \hline
		$50\,000$ & $100$ & $100$ & $1.023$ & $0.029$ & $0.49127$ & $1.108$ & $0.045$ & $0.13452$ \\ \hline
		$50\,000$ & $200$ & $500$ & $0.475$ & $0.009$ & $0.49284$ & $0.508$ & $0.011$ & $0.05214$ \\ \hline
		$50\,000$ & $500$ & $500$ & $0.450$ & $0.004$ & $0.49262$ & $0.525$ & $0.008$ & $0.04803$  \\ \hline\hline
		$100\,000$ & $100$ & $100$ & $0.975$ & $0.077$ & $0.49173$ & $0.970$ & $0.040$ & $0.14776$  \\ \hline
		$100\,000$ & $200$ & $1\,000$ & $0.318$ & $0.007$ & $0.49389$ & $0.379$ & $0.008$ & $0.03404$  \\ \hline
		$100\,000$ & $500$ & $1\,000$ & $0.351$ & $0.002$ & $0.49377$ & $0.345$ & $0.005$ & $0.03140$ \\ \hline\hline
		$200\,000$ & $100$ & $1\,000$ & $0.331$ & $0.005$ & $0.49643$ & $0.374$ & $0.008$ & $0.04400$ \\ \hline
		$200\,000$ & $200$ & $2\,000$ & $0.241$ & $0.003$ & $0.49442$ & $0.248$ & $0.003$ & $0.02375$ \\ \hline
		$200\,000$ & $500$ & $2\,000$ & $0.239$ & $0.002$ & $0.49436$ & $0.268$ & $0.002$ & $0.03013$ \\ \hline\hline
		$500\,000$ & $100$ & $5\,000$ & $0.146$ & $0.002$ & $0.49479$ & $0.142$ & $0.002$ & $0.01399$ \\ \hline
		$500\,000$ & $200$ & $5\,000$ & $0.157$ & $0.002$ & $0.49478$ & $0.180$ & $0.002$ & $0.01689$ \\ \hline
		$500\,000$ & $500$ & $5\,000$ & $0.167$ & $0.001$ & $0.49477$ & $0.171$ & $0.001$ & $0.01445$ \\ \hline\hline
	\end{tabular}
	\end{center}
\end{table*}

\begin{table*}[ht!]
\caption{{
2PL experiments on real-world SHARE data~\citep{boersch2022survey} for uniformly sampled subsets vs. coresets. 
Comparison of the best solutions found taken across $20$ repetitions (each running $50$ iterations of the main loop) per data set for different configurations of the data dimensions: the number of items $m$, the number of examinees $n$, and the uniform sample/coreset size $k$. 
Let $f_{\sf full}$, $f_{{\sf unif}(j)}$ and $f_{{\sf core}(j)}$ be the optimal values of the loss function on the input, on the uniform sample for the $j$-th repetition, and on the coreset for the $j$-th repetition, respectively. Let $f_{\sf unif} = \min_j f_{{\sf unif}(j)}$, and $f_{\sf core} = \min_j f_{{\sf core}(j)}$.
Comparison made w.r.t. Relative errors: {\bf r.err. $\hat{\varepsilon}_{\sf unif}$}$=|f_{\sf unif} - f_{\sf full}|/f_{\sf full}$, {\bf r.err. $\hat{\varepsilon}_{\sf core}$}$=|f_{\sf core} - f_{\sf full}|/f_{\sf full}$ (cf. \lemref{coreset:error:approx}), and Mean Absolute Deviations (MAD): $\text{mad}_{\sf core}(\alpha)=\frac{1}{n}\sum (|a_{\sf full}-a_{\sf core}| + |b_{\sf full}-b_{\sf core}|)$; $\text{mad}_{\sf core}(\theta)=\frac{1}{m}\sum |\theta_{\sf full}-\theta_{\sf core}|$. $\text{mad}_{\sf unif}(\alpha)=\frac{1}{n}\sum (|a_{\sf full}-a_{\sf unif}| + |b_{\sf full}-b_{\sf unif}|)$; $\text{mad}_{\sf unif}(\theta)=\frac{1}{m}\sum |\theta_{\sf full}-\theta_{\sf unif}|$. 
}
}
	\label{tab:results_uniform:share}
	\begin{center}
		\begin{tabular}{ r r r| c c r | c c r}
    \multicolumn{3}{c}{}& \multicolumn{3}{c}{{\bf Uniform sampling}}& \multicolumn{3}{c}{{\bf Coresets}} \\
		\hline
		{$\mathbf n$} & {$\mathbf m$} & {$\mathbf k$} &  {\bf $\text{mad}_{\sf unif}(\alpha)$} & {\bf $\text{mad}_{\sf unif}(\theta)$} & {\bf r.err. $\hat{\varepsilon}_{\sf unif}$} & {\bf $\text{mad}_{\sf core}(\alpha)$} & {\bf $\text{mad}_{\sf core}(\theta)$}  & {\bf r.err. $\hat{\varepsilon}_{\sf core}$} \\ \hline \hline
$138\,997$ & $10$ & $500$ & $0.722$ & $0.071$ & $0.49618$ & $0.770$ & $0.090$ & $0.11347$
		\\ \hline
$138\,997$ & $10$ & $1\,000$  & $0.232$ & $0.034$ & $0.49534$ & $0.307$ & $0.040$ & $0.06193$
		\\ \hline
$138\,997$ & $10$ & $2\,000$ & $0.179$ & $0.020$ & $0.49255$ & $0.129$ & $0.015$ & $0.04263$
		\\ \hline
$138\,997$ & $10$ & $4\,000$ & $0.083$ & $0.004$ & $0.48608$ & $0.108$ & $0.013$ & $0.02791$
		\\ \hline
$138\,997$ & $10$ & $6\,000$ & $0.086$ & $0.005$ & $0.47939$ & $0.095$ & $0.007$ & $0.03546$
		\\ \hline
$138\,997$ & $10$ & $8\,000$ & $0.082$ & $0.006$ & $0.47202$ & $0.061$ & $0.007$ & $0.01935$
		\\ \hline
$138\,997$ & $10$ & $10\,000$ & $0.059$ & $0.008$ & $0.46502$ & $0.092$ & $0.014$ & $0.02713$
		\\ \hline
$138\,997$ & $10$ & $20\,000$ & $0.058$ & $0.010$ & $0.42961$ & $0.045$ & $0.003$ & $0.01415$
		\\ \hline
\end{tabular}
	\end{center}
\end{table*}

\begin{table*}[ht!]
\caption{{
2PL experiments on real-world NEPS data~\citep{NEPS-SC4} for uniformly sampled subsets vs. coresets. 
Comparison of the best solutions found taken across $20$ repetitions (each running $50$ iterations of the main loop) per data set for different configurations of the data dimensions: the number of items $m$, the number of examinees $n$, and the uniform sample/coreset size $k$. 
Let $f_{\sf full}$, $f_{{\sf unif}(j)}$ and $f_{{\sf core}(j)}$ be the optimal values of the loss function on the input, on the uniform sample for the $j$-th repetition, and on the coreset for the $j$-th repetition, respectively. Let $f_{\sf unif} = \min_j f_{{\sf unif}(j)}$, and $f_{\sf core} = \min_j f_{{\sf core}(j)}$.
Comparison made w.r.t. Relative errors: {\bf r.err. $\hat{\varepsilon}_{\sf unif}$}$=|f_{\sf unif} - f_{\sf full}|/f_{\sf full}$, {\bf r.err. $\hat{\varepsilon}_{\sf core}$}$=|f_{\sf core} - f_{\sf full}|/f_{\sf full}$ (cf. \lemref{coreset:error:approx}), and Mean Absolute Deviations (MAD): $\text{mad}_{\sf core}(\alpha)=\frac{1}{n}\sum (|a_{\sf full}-a_{\sf core}| + |b_{\sf full}-b_{\sf core}|)$; $\text{mad}_{\sf core}(\theta)=\frac{1}{m}\sum |\theta_{\sf full}-\theta_{\sf core}|$. $\text{mad}_{\sf unif}(\alpha)=\frac{1}{n}\sum (|a_{\sf full}-a_{\sf unif}| + |b_{\sf full}-b_{\sf unif}|)$; $\text{mad}_{\sf unif}(\theta)=\frac{1}{m}\sum |\theta_{\sf full}-\theta_{\sf unif}|$.
}
}
	\label{tab:results_uniform:neps}
	\begin{center}
		\begin{tabular}{ r r r| c c r | c c r}
    \multicolumn{3}{c}{}& \multicolumn{3}{c}{{\bf Uniform sampling}}& \multicolumn{3}{c}{{\bf Coresets}} \\
		\hline
		{$\mathbf n$} & {$\mathbf m$} & {$\mathbf k$} &  {\bf $\text{mad}_{\sf unif}(\alpha)$} & {\bf $\text{mad}_{\sf unif}(\theta)$} & {\bf r.err. $\hat{\varepsilon}_{\sf unif}$} & {\bf $\text{mad}_{\sf core}(\alpha)$} & {\bf $\text{mad}_{\sf core}(\theta)$}  & {\bf r.err. $\hat{\varepsilon}_{\sf core}$} \\ \hline \hline
$11\,532$ & $88$ & $100$ & $1.561$ & $0.185$ & $0.48878$ & $1.477$ & $0.171$ & $0.09335$
		\\ \hline
$11\,532$ & $88$ & $200$ & $1.056$ & $0.131$ & $0.48762$ & $0.930$ & $0.142$ & $0.07134$
		\\ \hline
$11\,532$ & $88$ & $500$ & $0.635$ & $0.096$ & $0.47713$ & $0.499$ & $0.075$ & $0.03795$
		\\ \hline
$11\,532$ & $88$ & $750$ & $0.486$ & $0.068$ & $0.46702$ & $0.432$ & $0.062$ & $0.02675$
		\\ \hline
$11\,532$ & $88$ & $1\,000$ & $0.393$ & $0.053$ & $0.45664$ & $0.320$ & $0.045$ & $0.02007$
		\\ \hline
$11\,532$ & $88$ & $2\,000$ & $0.227$ & $0.030$ & $0.41390$ & $0.182$ & $0.026$ & $0.00506$
		\\ \hline
$11\,532$ & $88$ & $5\,000$ & $0.107$ & $0.011$ & $0.28429$ & $0.101$ & $0.015$ & $0.00147$
		\\ \hline
$11\,532$ & $88$ & $10\,000$ & $0.029$ & $0.002$ & $0.06711$ & $0.071$ & $0.012$ & $0.00094$
		\\ \hline
\end{tabular}
	\end{center}
\end{table*}

\clearpage
\section{COMPARISON TO CORESETS FOR CLUSTERING}
\label{appendix:clustering}

The interested reader may find that the alternating optimization algorithm resembles some kind of EM-type algorithm, akin to the popular Lloyd's algorithm for the $k$-means clustering problem. One crucial difference, however, is that in the IRT context, both sets of parameters to be estimated are unknown latent variables, while for the $k$-means problem, one set of 'parameters', is implicitly given by the data, and the task reduces to finding the other set (the $k$ cluster centers). We also note that in the IRT problem, the desired output is an explicit description of $m$ ability parameters, and $n$ item parameters. One can thus not hope to reduce one (or both) of the dimensions only once and work only on this single reduced coreset, as is possible for $k$-means. 

Despite the above mentioned differences, the interested reader may ask why we should construct new coresets for the IRT models, if already existing solutions from a plethora of coresets designed for clustering problems would serve as well. 

Recently, \citet{SchwiegelshohnS22} provided an extensive empirical comparison of various coreset constructions. The best performing coresets in practice were generated by `distance sampling', which is based on sensitivity sampling \citep{FeldmanL11,LangbergS10}, the same coreset design pattern that we also used for our coreset construction. In the case of clustering problems, first an initial (and rough) bi-criteria approximation is computed. Then, subsampling is performed proportionally to the squared Euclidean distance of input points to their closest center from this approximation. This coreset construction consistently outperformed all competitors in \citep{SchwiegelshohnS22}, even the relatively new group sampling technique that achieves optimal theoretical bounds \citep{CohenAddadSS21}.

Thus, we compare our coresets to the winning `distance sampling' in terms of their approximation quality when applied to IRT models. The results are given in \cref{tab:results_sensitivity,tab:results:sensitivity:share,tab:results_sensitivity:neps}.

For all data sets, our coresets outperform the distance sampling coresets in terms of their approximation quality, for both, mean absolute deviation (MAD) and the relative error. The MAD obtained from distance sampling coresets is at least twice as large as the MAD on our coresets. The relative error of the distance sampling coresets is at least $20\,\%$ larger than using our coresets, sometimes as much as two or three times larger, or even worse on the real-world data sets. Indeed, on the real-world SHARE data set~\citep{boersch2022survey}, which is very sparse, the distance sampling coresets cannot approximate the loss function well enough (the relative error remains $\hat{\varepsilon}>0.30$), even if we allow $72\,\%$ of the input ($100\,000$ examinees) to be selected into the coresets. In comparison, our coresets approximate the loss function up to relative error $\hat{\varepsilon}<0.03$ by taking a coreset that comprises only $6\,\%$ of the input set ($8\,000$ examinees). Our construction seems much more robust to this sparse data setting.

We conclude that the distance sampling coresets can in some settings provide good approximations that are competitive to our coresets, but their performance deteriorates in the presence of sparse data. Only coresets that are specifically tailored for IRT models provide an approximation of guaranteed quality.

\begin{table*}[ht!]
\caption{2PL experiments on synthetic data for distance sampling coresets, based on sensitivity sampling, vs. IRT coresets. 
Comparison of the best solutions found taken across $10$ repetitions (each running $50$ iterations of the main loop) per data set for different configurations of the data dimensions: the number of items $m$, the number of examinees $n$, and the distance sample/coreset size $k$. 
Let $f_{\sf full}$, $f_{{\sf dist}(j)}$ and $f_{{\sf core}(j)}$ be the optimal values of the loss function on the input, on the distance sample for the $j$-th repetition, and on the coreset for the $j$-th repetition, respectively. Let $f_{\sf dist} = \min_j f_{{\sf dist}(j)}$, and $f_{\sf core} = \min_j f_{{\sf core}(j)}$.
Comparison made w.r.t. Relative errors: {\bf r.err. $\hat{\varepsilon}_{\sf dist}$}$=|f_{\sf dist} - f_{\sf full}|/f_{\sf full}$, {\bf r.err. $\hat{\varepsilon}_{\sf core}$}$=|f_{\sf core} - f_{\sf full}|/f_{\sf full}$ (cf. \lemref{coreset:error:approx}), and Mean Absolute Deviations (MAD): $\text{mad}_{\sf core}(\alpha)=\frac{1}{n}\sum (|a_{\sf full}-a_{\sf core}| + |b_{\sf full}-b_{\sf core}|)$; $\text{mad}_{\sf core}(\theta)=\frac{1}{m}\sum |\theta_{\sf full}-\theta_{\sf core}|$. $\text{mad}_{\sf dist}(\alpha)=\frac{1}{n}\sum (|a_{\sf full}-a_{\sf dist}| + |b_{\sf full}-b_{\sf dist}|)$; $\text{mad}_{\sf dist}(\theta)=\frac{1}{m}\sum |\theta_{\sf full}-\theta_{\sf dist}|$.
}
	\label{tab:results_sensitivity}
	\begin{center}
		\begin{tabular}{ r r r| c c r | c c r}
    \multicolumn{3}{c}{}& \multicolumn{3}{c}{{\bf Distance sampling coresets}}& \multicolumn{3}{c}{{\bf IRT coresets}} \\
		\hline
		{$\mathbf n$} & {$\mathbf m$} & {$\mathbf k$} &  { $\textbf{mad}_{\sf dist}(\alpha)$} & {$\textbf{mad}_{\sf dist}(\theta)$} & {\textbf{r.err.} $\hat{\varepsilon}_{\sf dist}$} & {$\textbf{mad}_{\sf core}(\alpha)$} & { $\textbf{mad}_{\sf core}(\theta)$}  & {\textbf{r.err.} $\hat{\varepsilon}_{\sf core}$} \\ \hline \hline
		$50\,000$ & $100$ & $100$ & $1.146$ & $0.058$ & $0.15496$ & $1.108$ & $0.045$ & $0.13452$ \\ \hline
		$50\,000$ & $200$ & $500$ & $0.659$ & $0.013$ & $0.08284$ & $0.508$ & $0.011$ & $0.05214$ \\ \hline
		$50\,000$ & $500$ & $500$ & $0.609$ & $0.013$ & $0.08582$ & $0.525$ & $0.008$ & $0.04803$  \\ \hline\hline
		$100\,000$ & $100$ & $100$ & $1.149$ & $0.027$ & $0.14136$ & $0.970$ & $0.040$ & $0.14776$  \\ \hline
		$100\,000$ & $200$ & $1\,000$ & $0.760$ & $0.009$ & $0.05923$ & $0.379$ & $0.008$ & $0.03404$  \\ \hline
		$100\,000$ & $500$ & $1\,000$ & $0.448$ & $0.011$ & $0.07641$ & $0.345$ & $0.005$ & $0.03140$ \\ \hline\hline
		$200\,000$ & $100$ & $1\,000$ & $0.543$ & $0.022$ & $0.06787$ & $0.374$ & $0.008$ & $0.04400$ \\ \hline
		$200\,000$ & $200$ & $2\,000$ & $0.343$ & $0.005$ & $0.04916$ & $0.248$ & $0.003$ & $0.02375$ \\ \hline
		$200\,000$ & $500$ & $2\,000$ & $0.354$ & $0.005$ & $0.04667$ & $0.268$ & $0.002$ & $0.03013$ \\ \hline\hline
		$500\,000$ & $100$ & $5\,000$ & $0.252$ & $0.013$ & $0.03292$ & $0.142$ & $0.002$ & $0.01399$ \\ \hline
		$500\,000$ & $200$ & $5\,000$ & $0.295$ & $0.005$ & $0.03394$ & $0.180$ & $0.002$ & $0.01689$ \\ \hline
		$500\,000$ & $500$ & $5\,000$ & $0.259$ & $0.003$ & $0.03424$ & $0.171$ & $0.001$ & $0.01445$ \\ \hline\hline
	\end{tabular}
	\end{center}
\end{table*}

\begin{table*}[ht!]
\caption{{
2PL experiments on real-world SHARE data for distance sampling coresets, based on sensitivity sampling, vs. IRT coresets. 
Comparison of the best solutions found taken across $10$ repetitions (each running $50$ iterations of the main loop) per data set for different configurations of the data dimensions: the number of items $m$, the number of examinees $n$, and the distance sample/coreset size $k$. 
Let $f_{\sf full}$, $f_{{\sf dist}(j)}$ and $f_{{\sf core}(j)}$ be the optimal values of the loss function on the input, on the distance sample for the $j$-th repetition, and on the coreset for the $j$-th repetition, respectively. Let $f_{\sf dist} = \min_j f_{{\sf dist}(j)}$, and $f_{\sf core} = \min_j f_{{\sf core}(j)}$.
Comparison made w.r.t. Relative errors: {\textbf{r.err.} $\hat{\varepsilon}_{\sf dist}$}$=|f_{\sf dist} - f_{\sf full}|/f_{\sf full}$, {\textbf{r.err.} $\hat{\varepsilon}_{\sf core}$}$=|f_{\sf core} - f_{\sf full}|/f_{\sf full}$ (cf. \lemref{coreset:error:approx}), and Mean Absolute Deviations (MAD): $\text{mad}_{\sf core}(\alpha)=\frac{1}{n}\sum (|a_{\sf full}-a_{\sf core}| + |b_{\sf full}-b_{\sf core}|)$; $\text{mad}_{\sf core}(\theta)=\frac{1}{m}\sum |\theta_{\sf full}-\theta_{\sf core}|$. $\text{mad}_{\sf dist}(\alpha)=\frac{1}{n}\sum (|a_{\sf full}-a_{\sf dist}| + |b_{\sf full}-b_{\sf dist}|)$; $\text{mad}_{\sf dist}(\theta)=\frac{1}{m}\sum |\theta_{\sf full}-\theta_{\sf dist}|$.
}
}
	\label{tab:results:sensitivity:share}
	\begin{center}
		\begin{tabular}{ r r r| c c r | c c r}
    \multicolumn{3}{c}{}
    & \multicolumn{3}{c}{{\bf Distance sampling coresets}}& \multicolumn{3}{c}{{\bf IRT coresets}} \\
		\hline
		{$\mathbf n$} &  {$\mathbf m$} &  {$\mathbf k$} &  { $\textbf{mad}_{\sf dist}(\alpha)$} & { $\textbf{mad}_{\sf dist}(\theta)$} & {\textbf{r.err.} $\hat{\varepsilon}_{\sf dist}$} & { $\textbf{mad}_{\sf core}(\alpha)$} & { $\textbf{mad}_{\sf core}(\theta)$}  & { 
  $\hat{\varepsilon}_{\sf core}$} \\ \hline \hline
$138\,997$ & $10$ & $500$ & $3.581$ & $0.329$ & $0.3843766731629$ & $0.770$ & $0.090$ & $0.11347$ 
		\\ \hline
$138\,997$ & $10$ & $1\,000$  & $3.580$ & $0.328$ & $0.3843766731630$ & $0.307$ & $0.040$ & $0.06193$ 
		\\ \hline
$138\,997$ & $10$ & $2\,000$ & $3.586$ & $0.330$ & $0.3843766731634$ & $0.129$ & $0.015$ &  $0.04263$ 
		\\ \hline
$138\,997$ & $10$ & $4\,000$ & $3.579$ & $0.328$ & $0.3843766731618$ & $0.108$ & $0.013$ & $0.02791$ 
		\\ \hline
$138\,997$ & $10$ & $6\,000$ & $3.581$ & $0.328$ & $0.3843766731613$ & $0.095$ & $0.007$ & $0.03546$ 
		\\ \hline
$138\,997$ & $10$ & $8\,000$ & $3.580$ & $0.328$ & $0.3843766731606$ & $0.061$ & $0.007$ & $0.01935$ 
		\\ \hline
$138\,997$ & $10$ & $10\,000$ & $3.581$ & $0.328$ & $0.3843766731605$ & $0.092$ & $0.014$ & $0.02713$ 
		\\ \hline
$138\,997$ & $10$ & $20\,000$ & $3.580$ & $0.328$ & $0.3843766731608$ & $0.045$ & $0.003$ & $0.01415$ 
        \\ \hline
\end{tabular}
	\end{center}
\end{table*}

\begin{table*}[ht!]
\caption{{
2PL experiments on real-world NEPS data for distance sampling coresets, based on sensitivity sampling, vs. IRT coresets. 
Comparison of the best solutions found taken across $10$ repetitions (each running $50$ iterations of the main loop) per data set for different configurations of the data dimensions: the number of items $m$, the number of examinees $n$, and the distance sample/coreset size $k$. 
Let $f_{\sf full}$, $f_{{\sf dist}(j)}$ and $f_{{\sf core}(j)}$ be the optimal values of the loss function on the input, on the distance sample for the $j$-th repetition, and on the coreset for the $j$-th repetition, respectively. Let $f_{\sf dist} = \min_j f_{{\sf dist}(j)}$, and $f_{\sf core} = \min_j f_{{\sf core}(j)}$.
Comparison made w.r.t. Relative errors: {\bf r.err. $\hat{\varepsilon}_{\sf dist}$}$=|f_{\sf dist} - f_{\sf full}|/f_{\sf full}$, {\bf r.err. $\hat{\varepsilon}_{\sf core}$}$=|f_{\sf core} - f_{\sf full}|/f_{\sf full}$ (cf. \lemref{coreset:error:approx}), and Mean Absolute Deviations (MAD): $\text{mad}_{\sf core}(\alpha)=\frac{1}{n}\sum (|a_{\sf full}-a_{\sf core}| + |b_{\sf full}-b_{\sf core}|)$; $\text{mad}_{\sf core}(\theta)=\frac{1}{m}\sum |\theta_{\sf full}-\theta_{\sf core}|$. $\text{mad}_{\sf dist}(\alpha)=\frac{1}{n}\sum (|a_{\sf full}-a_{\sf dist}| + |b_{\sf full}-b_{\sf dist}|)$; $\text{mad}_{\sf dist}(\theta)=\frac{1}{m}\sum |\theta_{\sf full}-\theta_{\sf dist}|$.
}
}
	\label{tab:results_sensitivity:neps}
	\begin{center}
		\begin{tabular}{ r r r| c c r | c c r}
    \multicolumn{3}{c}{}& \multicolumn{3}{c}{{\bf Distance sampling coresets}}& \multicolumn{3}{c}{{\bf IRT coresets}} \\
		\hline
		{$\mathbf n$} & {$\mathbf m$} & {$\mathbf k$} &  { $\textbf{mad}_{\sf dist}(\alpha)$} & { $\textbf{mad}_{\sf dist}(\theta)$} & {\textbf{r.err.} $\hat{\varepsilon}_{\sf dist}$} & {$\textbf{mad}_{\sf core}(\alpha)$} & { $\textbf{mad}_{\sf core}(\theta)$}  & {\textbf{r.err.} $\hat{\varepsilon}_{\sf core}$} \\ \hline \hline
$11\,532$ & $88$ & $100$ & $2.244$ & $0.433$ & $0.12674$ & $1.477$ & $0.171$ & $0.09335$
		\\ \hline
$11\,532$ & $88$ & $200$ & $1.818$ & $0.198$ & $0.11617$ & $0.930$ & $0.142$ & $0.07134$
		\\ \hline
$11\,532$ & $88$ & $500$ & $0.959$ & $0.138$ & $0.07654$ & $0.499$ & $0.075$ & $0.03795$
		\\ \hline
$11\,532$ & $88$ & $750$ & $0.432$ & $0.103$ & $0.07988$ & $0.432$ & $0.062$ & $0.02675$
		\\ \hline
$11\,532$ & $88$ & $1\,000$ & $0.654$ & $0.101$ & $0.06035$ & $0.320$ & $0.045$ & $0.02007$
		\\ \hline
$11\,532$ & $88$ & $2\,000$ & $0.490$ & $0.068$ & $0.06295$ & $0.182$ & $0.026$ & $0.00506$
		\\ \hline
$11\,532$ & $88$ & $5\,000$ & $0.101$ & $0.043$ & $0.04319$ & $0.101$ & $0.015$ & $0.00147$
		\\ \hline
$11\,532$ & $88$ & $10\,000$ & $0.301$ & $0.031$ & $0.04802$ & $0.071$ & $0.012$ & $0.00094$
		\\ \hline
\end{tabular}
	\end{center}
\end{table*}

\clearpage

\section{COMPARISON TO \texorpdfstring{$\ell_1$}{L1} 
LEVERAGE SCORES AND \texorpdfstring{$\ell_1$}{L1} 
LEWIS WEIGHTS }
\label{sec:app:L1scores}

Further baselines for subsampling the input that are used in the literature to approximate the objective functions related to logistic regression, are sampling proportional to $\ell_1$-leverage scores \citep{MunteanuOP22}, resp. to $\ell_1$-Lewis weights \citep{MaiMR21}.

We compared our coresets to the $\ell_1$-leverage scores, resp. $\ell_1$-Lewis weights, in terms of their approximation quality, their mean absolute deviation (MAD) and their relative error.

Our IRT coresets show very similar, and often slightly better performance compared to both alternative subsampling techniques, when applied to the synthetic, and the real-world data instances for the 2PL IRT model.

See~\cref{tab:results_l1score,tab:results_l1score:share,tab:results_l1score:neps} below for the comparison of our coresets to sampling based on $\ell_1$-leverage scores. Also, see \cref{tab:results_lewisweights,tab:results_lewisweights:share,tab:results_lewisweights:neps} below for the comparison of our coresets to sampling based on $\ell_1$-Lewis weights.

\subsection{\texorpdfstring{$\ell_1$}{L1}-Leverage Score Sampling}
\label{app:additional:l1leverage}

\begin{table*}[ht!]
\caption{2PL experiments on synthetic data for $\ell_1$-leverage score sampling coresets, vs. IRT coresets. 
Comparison of the best solutions found taken across $10$ repetitions (each running $50$ iterations of the main loop) per data set for different configurations of the data dimensions: the number of items $m$, the number of examinees $n$, and the $\ell_1$-leverage score sample/coreset size $k$. 
Let $f_{\sf full}$, $f_{{\sf L1s}(j)}$ and $f_{{\sf core}(j)}$ be the optimal values of the loss function on the input, on the distance sample for the $j$-th repetition, and on the coreset for the $j$-th repetition, respectively. Let $f_{\sf L1s} = \min_j f_{{\sf L1s}(j)}$, and $f_{\sf core} = \min_j f_{{\sf core}(j)}$.
Comparison made w.r.t. Relative errors: {\bf r.err. $\hat{\varepsilon}_{\sf L1s}$}$=|f_{\sf L1s} - f_{\sf full}|/f_{\sf full}$, {\bf r.err. $\hat{\varepsilon}_{\sf core}$}$=|f_{\sf core} - f_{\sf full}|/f_{\sf full}$ (cf. \lemref{coreset:error:approx}), and Mean Absolute Deviations (MAD): $\text{mad}_{\sf core}(\alpha)=\frac{1}{n}\sum (|a_{\sf full}-a_{\sf core}| + |b_{\sf full}-b_{\sf core}|)$; $\text{mad}_{\sf core}(\theta)=\frac{1}{m}\sum |\theta_{\sf full}-\theta_{\sf core}|$. $\text{mad}_{\sf L1s}(\alpha)=\frac{1}{n}\sum (|a_{\sf full}-a_{\sf L1s}| + |b_{\sf full}-b_{\sf L1s}|)$; $\text{mad}_{\sf L1s}(\theta)=\frac{1}{m}\sum |\theta_{\sf full}-\theta_{\sf L1s}|$.
}
	\label{tab:results_l1score}
	\begin{center}
		\begin{tabular}{ r r r| c c r | c c r}
    \multicolumn{3}{c}{}& \multicolumn{3}{c}{{\bf $\ell_1$-score sampling coresets}}& \multicolumn{3}{c}{{\bf IRT coresets}} \\
		\hline
		{$\mathbf n$} & {$\mathbf m$} & {$\mathbf k$} &  { $\textbf{mad}_{\sf L1s}(\alpha)$} & {$\textbf{mad}_{\sf L1s}(\theta)$} & {\textbf{r.err.} $\hat{\varepsilon}_{\sf L1s}$} & {$\textbf{mad}_{\sf core}(\alpha)$} & { $\textbf{mad}_{\sf core}(\theta)$}  & {\textbf{r.err.} $\hat{\varepsilon}_{\sf core}$} \\ \hline \hline
		$50\,000$ & $100$ & $100$ & $1.150$ & $0.045$ & $0.12357$ & $1.108$ & $0.045$ & $0.13452$ \\ \hline
		$50\,000$ & $200$ & $500$ & $0.466$ & $0.009$ & $0.04835$ & $0.508$ & $0.011$ & $0.05214$ \\ \hline
		$50\,000$ & $500$ & $500$ & $0.494$ & $0.005$ & $0.04893$ & $0.525$ & $0.008$ & $0.04803$  \\ \hline\hline
		$100\,000$ & $100$ & $100$ & $1.149$ & $0.036$ & $0.10821$ & $0.970$ & $0.040$ & $0.14776$  \\ \hline
		$100\,000$ & $200$ & $1\,000$ & $0.377$ & $0.009$ & $0.03051$ & $0.379$ & $0.008$ & $0.03404$  \\ \hline
		$100\,000$ & $500$ & $1\,000$ & $0.353$ & $0.005$ & $0.03865$ & $0.345$ & $0.005$ & $0.03140$ \\ \hline\hline
		$200\,000$ & $100$ & $1\,000$ & $0.323$ & $0.006$ & $0.03437$ & $0.374$ & $0.008$ & $0.04400$ \\ \hline
		$200\,000$ & $200$ & $2\,000$ & $0.290$ & $0.003$ & $0.02033$ & $0.248$ & $0.003$ & $0.02375$ \\ \hline
		$200\,000$ & $500$ & $2\,000$ & $0.252$ & $0.002$ & $0.02683$ & $0.268$ & $0.002$ & $0.03013$ \\ \hline\hline
		$500\,000$ & $100$ & $5\,000$ & $0.183$ & $0.002$ & $0.01142$ & $0.142$ & $0.002$ & $0.01399$ \\ \hline
		$500\,000$ & $200$ & $5\,000$ & $0.169$ & $0.002$ & $0.01371$ & $0.180$ & $0.002$ & $0.01689$ \\ \hline
		$500\,000$ & $500$ & $5\,000$ & $0.166$ & $0.001$ & $0.01265$ & $0.171$ & $0.001$ & $0.01445$ \\ \hline\hline
	\end{tabular}
	\end{center}
\end{table*}

\begin{table*}[ht!]
\caption{
2PL experiments on real-world SHARE data for $\ell_1$-leverage score sampling coresets, vs. IRT coresets. 
Comparison of the best solutions found taken across $10$ repetitions (each running $50$ iterations of the main loop) per data set for different configurations of the data dimensions: the number of items $m$, the number of examinees $n$, and the $\ell_1$-leverage score sample/coreset size $k$.
Let $f_{\sf full}$, $f_{{\sf L1s}(j)}$ and $f_{{\sf core}(j)}$ be the optimal values of the loss function on the input, on the distance sample for the $j$-th repetition, and on the coreset for the $j$-th repetition, respectively. Let $f_{\sf L1s} = \min_j f_{{\sf L1s}(j)}$, and $f_{\sf core} = \min_j f_{{\sf core}(j)}$.
Comparison made w.r.t. Relative errors: {\bf r.err. $\hat{\varepsilon}_{\sf L1s}$}$=|f_{\sf L1s} - f_{\sf full}|/f_{\sf full}$, {\bf r.err. $\hat{\varepsilon}_{\sf core}$}$=|f_{\sf core} - f_{\sf full}|/f_{\sf full}$ (cf. \lemref{coreset:error:approx}), and Mean Absolute Deviations (MAD): $\text{mad}_{\sf core}(\alpha)=\frac{1}{n}\sum (|a_{\sf full}-a_{\sf core}| + |b_{\sf full}-b_{\sf core}|)$; $\text{mad}_{\sf core}(\theta)=\frac{1}{m}\sum |\theta_{\sf full}-\theta_{\sf core}|$. $\text{mad}_{\sf L1s}(\alpha)=\frac{1}{n}\sum (|a_{\sf full}-a_{\sf L1s}| + |b_{\sf full}-b_{\sf L1s}|)$; $\text{mad}_{\sf L1s}(\theta)=\frac{1}{m}\sum |\theta_{\sf full}-\theta_{\sf L1s}|$.
}
	\label{tab:results_l1score:share}
	\begin{center}
		\begin{tabular}{ r r r| c c r | c c r}
    \multicolumn{3}{c}{}& \multicolumn{3}{c}{{\bf $\ell_1$-score sampling coresets}}& \multicolumn{3}{c}{{\bf IRT coresets}} \\
		\hline
		{$\mathbf n$} & {$\mathbf m$} & {$\mathbf k$} &  { $\textbf{mad}_{\sf L1s}(\alpha)$} & {$\textbf{mad}_{\sf L1s}(\theta)$} & {\textbf{r.err.} $\hat{\varepsilon}_{\sf L1s}$} & {$\textbf{mad}_{\sf core}(\alpha)$} & { $\textbf{mad}_{\sf core}(\theta)$}  & {\textbf{r.err.} $\hat{\varepsilon}_{\sf core}$} \\ \hline \hline
$138\,997$ & $10$ & $500$ & $0.875$ & $0.107$ & $0.13267$ & $0.770$ & $0.090$ & $0.11347$ 
		\\ \hline
$138\,997$ & $10$ & $1\,000$  & $0.320$ & $0.030$ & $0.09216$ & $0.307$ & $0.040$ & $0.06193$ 
		\\ \hline
$138\,997$ & $10$ & $2\,000$ & $0.172$ & $0.023$ & $0.04204$ & $0.129$ & $0.015$ &  $0.04263$ 
		\\ \hline
$138\,997$ & $10$ & $4\,000$ & $0.179$ & $0.027$ & $0.02958$ & $0.108$ & $0.013$ & $0.02791$ 
		\\ \hline
$138\,997$ & $10$ & $6\,000$ & $0.083$ & $0.010$ & $0.02851$ & $0.095$ & $0.007$ & $0.03546$ 
		\\ \hline
$138\,997$ & $10$ & $8\,000$ & $0.080$ & $0.005$ & $0.01958$ & $0.061$ & $0.007$ & $0.01935$ 
		\\ \hline
$138\,997$ & $10$ & $10\,000$ & $0.070$ & $0.008$ & $0.01386$ & $0.092$ & $0.014$ & $0.02713$ 
		\\ \hline
$138\,997$ & $10$ & $20\,000$ & $0.044$ & $0.004$ & $0.01200$ & $0.045$ & $0.003$ & $0.01415$ 
        \\ \hline
\end{tabular}
	\end{center}
\end{table*}

\begin{table*}[ht!]
\caption{
2PL experiments on real-world NEPS data for $\ell_1$-leverage score sampling coresets, vs. IRT coresets. 
Comparison of the best solutions found taken across $10$ repetitions (each running $50$ iterations of the main loop) per data set for different configurations of the data dimensions: the number of items $m$, the number of examinees $n$, and the $\ell_1$-leverage score sample/coreset size $k$.
Let $f_{\sf full}$, $f_{{\sf L1s}(j)}$ and $f_{{\sf core}(j)}$ be the optimal values of the loss function on the input, on the distance sample for the $j$-th repetition, and on the coreset for the $j$-th repetition, respectively. Let $f_{\sf L1s} = \min_j f_{{\sf L1s}(j)}$, and $f_{\sf core} = \min_j f_{{\sf core}(j)}$.
Comparison made w.r.t. Relative errors: {\bf r.err. $\hat{\varepsilon}_{\sf L1s}$}$=|f_{\sf L1s} - f_{\sf full}|/f_{\sf full}$, {\bf r.err. $\hat{\varepsilon}_{\sf core}$}$=|f_{\sf core} - f_{\sf full}|/f_{\sf full}$ (cf. \lemref{coreset:error:approx}), and Mean Absolute Deviations (MAD): $\text{mad}_{\sf core}(\alpha)=\frac{1}{n}\sum (|a_{\sf full}-a_{\sf core}| + |b_{\sf full}-b_{\sf core}|)$; $\text{mad}_{\sf core}(\theta)=\frac{1}{m}\sum |\theta_{\sf full}-\theta_{\sf core}|$. $\text{mad}_{\sf L1s}(\alpha)=\frac{1}{n}\sum (|a_{\sf full}-a_{\sf L1s}| + |b_{\sf full}-b_{\sf L1s}|)$; $\text{mad}_{\sf L1s}(\theta)=\frac{1}{m}\sum |\theta_{\sf full}-\theta_{\sf L1s}|$.
}
	\label{tab:results_l1score:neps}
	\begin{center}
		\begin{tabular}{ r r r| c c r | c c r}
    \multicolumn{3}{c}{}& \multicolumn{3}{c}{{\bf $\ell_1$-score sampling coresets}}& \multicolumn{3}{c}{{\bf IRT coresets}} \\
		\hline
		{$\mathbf n$} & {$\mathbf m$} & {$\mathbf k$} &  { $\textbf{mad}_{\sf L1s}(\alpha)$} & {$\textbf{mad}_{\sf L1s}(\theta)$} & {\textbf{r.err.} $\hat{\varepsilon}_{\sf L1s}$} & {$\textbf{mad}_{\sf core}(\alpha)$} & { $\textbf{mad}_{\sf core}(\theta)$}  & {\textbf{r.err.} $\hat{\varepsilon}_{\sf core}$} \\ \hline \hline
$11\,532$ & $88$ & $100$ & $1.388$ & $0.191$ & $0.06670$ & $1.477$ & $0.171$ & $0.09335$
		\\ \hline
$11\,532$ & $88$ & $200$ & $1.040$ & $0.132$ & $0.05428$ & $0.930$ & $0.142$ & $0.07134$
		\\ \hline
$11\,532$ & $88$ & $500$ & $0.559$ & $0.082$ & $0.02556$ & $0.499$ & $0.075$ & $0.03795$
		\\ \hline
$11\,532$ & $88$ & $750$ & $0.503$ & $0.061$ & $0.01956$ & $0.432$ & $0.062$ & $0.02675$
		\\ \hline
$11\,532$ & $88$ & $1\,000$ & $0.316$ & $0.040$ & $0.02133$ & $0.320$ & $0.045$ & $0.02007$
		\\ \hline
$11\,532$ & $88$ & $2\,000$ & $0.207$ & $0.023$ & $0.00468$ & $0.182$ & $0.026$ & $0.00506$
		\\ \hline
$11\,532$ & $88$ & $5\,000$ & $0.097$ & $0.006$ & $0.00162$ & $0.101$ & $0.015$ & $0.00147$
		\\ \hline
$11\,532$ & $88$ & $10\,000$ & $0.077$ & $0.010$ & $0.00194$ & $0.071$ & $0.012$ & $0.00094$
		\\ \hline
\end{tabular}
	\end{center}
\end{table*}

\clearpage 

\subsection{\texorpdfstring{$\ell_1$}{L1}-Lewis Weight Sampling}
\label{app:additional:l1lewis}

\begin{table*}[ht!]
\caption{2PL experiments on synthetic data for Lewis weights sampling coresets, vs. IRT coresets. 
Comparison of the best solutions found taken across $10$ repetitions (each running $50$ iterations of the main loop) per data set for different configurations of the data dimensions: the number of items $m$, the number of examinees $n$, and the Lewis weights sample/coreset size $k$. 
Let $f_{\sf full}$, $f_{{\sf lewis}(j)}$ and $f_{{\sf core}(j)}$ be the optimal values of the loss function on the input, on the distance sample for the $j$-th repetition, and on the coreset for the $j$-th repetition, respectively. Let $f_{\sf lewis} = \min_j f_{{\sf lewis}(j)}$, and $f_{\sf core} = \min_j f_{{\sf core}(j)}$.
Comparison made w.r.t. Relative errors: {\bf r.err. $\hat{\varepsilon}_{\sf lewis}$}$=|f_{\sf lewis} - f_{\sf full}|/f_{\sf full}$, {\bf r.err. $\hat{\varepsilon}_{\sf core}$}$=|f_{\sf core} - f_{\sf full}|/f_{\sf full}$ (cf. \lemref{coreset:error:approx}), and Mean Absolute Deviations (MAD): $\text{mad}_{\sf core}(\alpha)=\frac{1}{n}\sum (|a_{\sf full}-a_{\sf core}| + |b_{\sf full}-b_{\sf core}|)$; $\text{mad}_{\sf core}(\theta)=\frac{1}{m}\sum |\theta_{\sf full}-\theta_{\sf core}|$. $\text{mad}_{\sf lewis}(\alpha)=\frac{1}{n}\sum (|a_{\sf full}-a_{\sf lewis}| + |b_{\sf full}-b_{\sf lewis}|)$; $\text{mad}_{\sf lewis}(\theta)=\frac{1}{m}\sum |\theta_{\sf full}-\theta_{\sf lewis}|$.
}
	\label{tab:results_lewisweights}
	\begin{center}
		\begin{tabular}{ r r r| c c r | c c r}
    \multicolumn{3}{c}{}& \multicolumn{3}{c}{{\bf Lewis weights sampling coresets}}& \multicolumn{3}{c}{{\bf IRT coresets}} \\
		\hline
		{$\mathbf n$} & {$\mathbf m$} & {$\mathbf k$} &  { $\textbf{mad}_{\sf lewis}(\alpha)$} & {$\textbf{mad}_{\sf lewis}(\theta)$} & {\textbf{r.err.} $\hat{\varepsilon}_{\sf lewis}$} & {$\textbf{mad}_{\sf core}(\alpha)$} & { $\textbf{mad}_{\sf core}(\theta)$}  & {\textbf{r.err.} $\hat{\varepsilon}_{\sf core}$} \\ \hline \hline
		$50\,000$ & $100$ & $100$ & $1.011$ & $0.038$ & $0.10458$ & $1.108$ & $0.045$ & $0.13452$ \\ \hline
		$50\,000$ & $200$ & $500$ & $0.515$ & $0.011$ & $0.05234$ & $0.508$ & $0.011$ & $0.05214$ \\ \hline
		$50\,000$ & $500$ & $500$ & $0.481$ & $0.008$ & $0.05444$ & $0.525$ & $0.008$ & $0.04803$  \\ \hline\hline
		$100\,000$ & $100$ & $100$ & $1.149$ & $0.043$ & $0.09635$ & $0.970$ & $0.040$ & $0.14776$  \\ \hline
		$100\,000$ & $200$ & $1\,000$ & $0.342$ & $0.008$ & $0.02718$ & $0.379$ & $0.008$ & $0.03404$  \\ \hline
		$100\,000$ & $500$ & $1\,000$ & $0.338$ & $0.005$ & $0.03687$ & $0.345$ & $0.005$ & $0.03140$ \\ \hline\hline	
		$200\,000$ & $100$ & $1\,000$ & $0.378$ & $0.007$ & $0.03894$ & $0.374$ & $0.008$ & $0.04400$ \\ \hline
		$200\,000$ & $200$ & $2\,000$ & $0.311$ & $0.003$ & $0.02077$ & $0.248$ & $0.003$ & $0.02375$ \\ \hline
		$200\,000$ & $500$ & $2\,000$ & $0.257$ & $0.003$ & $0.02620$ & $0.268$ & $0.002$ & $0.03013$ \\ \hline\hline	
		$500\,000$ & $100$ & $5\,000$ & $0.169$ & $0.002$ & $0.01121$ & $0.142$ & $0.002$ & $0.01399$ \\ \hline
		$500\,000$ & $200$ & $5\,000$ & $0.164$ & $0.002$ & $0.01438$ & $0.180$ & $0.002$ & $0.01689$ \\ \hline
		$500\,000$ & $500$ & $5\,000$ & $0.165$ & $0.001$ & $0.01518$ & $0.171$ & $0.001$ & $0.01445$ \\ \hline\hline
	\end{tabular}
	\end{center}
\end{table*}

\begin{table*}[ht!]
\caption{
2PL experiments on real-world SHARE data for Lewis weights sampling coresets, vs. IRT coresets. 
Comparison of the best solutions found taken across $10$ repetitions (each running $50$ iterations of the main loop) per data set for different configurations of the data dimensions: the number of items $m$, the number of examinees $n$, and the Lewis weights sample/coreset size $k$. 
Let $f_{\sf full}$, $f_{{\sf lewis}(j)}$ and $f_{{\sf core}(j)}$ be the optimal values of the loss function on the input, on the distance sample for the $j$-th repetition, and on the coreset for the $j$-th repetition, respectively. Let $f_{\sf lewis} = \min_j f_{{\sf lewis}(j)}$, and $f_{\sf core} = \min_j f_{{\sf core}(j)}$.
Comparison made w.r.t. Relative errors: {\bf r.err. $\hat{\varepsilon}_{\sf lewis}$}$=|f_{\sf lewis} - f_{\sf full}|/f_{\sf full}$, {\bf r.err. $\hat{\varepsilon}_{\sf core}$}$=|f_{\sf core} - f_{\sf full}|/f_{\sf full}$ (cf. \lemref{coreset:error:approx}), and Mean Absolute Deviations (MAD): $\text{mad}_{\sf core}(\alpha)=\frac{1}{n}\sum (|a_{\sf full}-a_{\sf core}| + |b_{\sf full}-b_{\sf core}|)$; $\text{mad}_{\sf core}(\theta)=\frac{1}{m}\sum |\theta_{\sf full}-\theta_{\sf core}|$. $\text{mad}_{\sf lewis}(\alpha)=\frac{1}{n}\sum (|a_{\sf full}-a_{\sf lewis}| + |b_{\sf full}-b_{\sf lewis}|)$; $\text{mad}_{\sf lewis}(\theta)=\frac{1}{m}\sum |\theta_{\sf full}-\theta_{\sf lewis}|$.
}
	\label{tab:results_lewisweights:share}
	\begin{center}
		\begin{tabular}{ r r r| c c r | c c r}
    \multicolumn{3}{c}{}& \multicolumn{3}{c}{{\bf Lewis weights sampling coresets}}& \multicolumn{3}{c}{{\bf IRT coresets}} \\
		\hline
		{$\mathbf n$} & {$\mathbf m$} & {$\mathbf k$} &  { $\textbf{mad}_{\sf lewis}(\alpha)$} & {$\textbf{mad}_{\sf lewis}(\theta)$} & {\textbf{r.err.} $\hat{\varepsilon}_{\sf lewis}$} & {$\textbf{mad}_{\sf core}(\alpha)$} & { $\textbf{mad}_{\sf core}(\theta)$}  & {\textbf{r.err.} $\hat{\varepsilon}_{\sf core}$} \\ \hline \hline
$138\,997$ & $10$ & $500$ & $0.400$ & $0.057$ & $0.07814$ & $0.770$ & $0.090$ & $0.11347$ 
		\\ \hline
$138\,997$ & $10$ & $1\,000$  & $0.277$ & $0.019$ & $0.10915$ & $0.307$ & $0.040$ & $0.06193$ 
		\\ \hline
$138\,997$ & $10$ & $2\,000$ & $0.467$ & $0.053$ & $0.03697$ & $0.129$ & $0.015$ &  $0.04263$ 
		\\ \hline
$138\,997$ & $10$ & $4\,000$ & $0.147$ & $0.015$ & $0.02871$ & $0.108$ & $0.013$ & $0.02791$ 
		\\ \hline
$138\,997$ & $10$ & $6\,000$ & $0.119$ & $0.011$ & $0.02210$ & $0.095$ & $0.007$ & $0.03546$ 
		\\ \hline
$138\,997$ & $10$ & $8\,000$ & $0.086$ & $0.011$ & $0.01785$ & $0.061$ & $0.007$ & $0.01935$ 
		\\ \hline
$138\,997$ & $10$ & $10\,000$ & $0.053$ & $0.005$ & $0.01543$ & $0.092$ & $0.014$ & $0.02713$ 
		\\ \hline
$138\,997$ & $10$ & $20\,000$ & $0.045$ & $0.007$ & $0.01398$ & $0.045$ & $0.003$ & $0.01415$ 
        \\ \hline
\end{tabular}
	\end{center}
\end{table*}

\begin{table*}[ht!]
\caption{
2PL experiments on real-world NEPS data for Lewis weights sampling coresets, vs. IRT coresets. 
Comparison of the best solutions found taken across $10$ repetitions (each running $50$ iterations of the main loop) per data set for different configurations of the data dimensions: the number of items $m$, the number of examinees $n$, and the Lewis weights sample/coreset size $k$. 
Let $f_{\sf full}$, $f_{{\sf lewis}(j)}$ and $f_{{\sf core}(j)}$ be the optimal values of the loss function on the input, on the distance sample for the $j$-th repetition, and on the coreset for the $j$-th repetition, respectively. Let $f_{\sf lewis} = \min_j f_{{\sf lewis}(j)}$, and $f_{\sf core} = \min_j f_{{\sf core}(j)}$.
Comparison made w.r.t. Relative errors: {\bf r.err. $\hat{\varepsilon}_{\sf lewis}$}$=|f_{\sf lewis} - f_{\sf full}|/f_{\sf full}$, {\bf r.err. $\hat{\varepsilon}_{\sf core}$}$=|f_{\sf core} - f_{\sf full}|/f_{\sf full}$ (cf. \lemref{coreset:error:approx}), and Mean Absolute Deviations (MAD): $\text{mad}_{\sf core}(\alpha)=\frac{1}{n}\sum (|a_{\sf full}-a_{\sf core}| + |b_{\sf full}-b_{\sf core}|)$; $\text{mad}_{\sf core}(\theta)=\frac{1}{m}\sum |\theta_{\sf full}-\theta_{\sf core}|$. $\text{mad}_{\sf lewis}(\alpha)=\frac{1}{n}\sum (|a_{\sf full}-a_{\sf lewis}| + |b_{\sf full}-b_{\sf lewis}|)$; $\text{mad}_{\sf lewis}(\theta)=\frac{1}{m}\sum |\theta_{\sf full}-\theta_{\sf lewis}|$.
}
	\label{tab:results_lewisweights:neps}
	\begin{center}
		\begin{tabular}{ r r r| c c r | c c r}
    \multicolumn{3}{c}{}& \multicolumn{3}{c}{{\bf Lewis weights sampling coresets}}& \multicolumn{3}{c}{{\bf IRT coresets}} \\
		\hline
		{$\mathbf n$} & {$\mathbf m$} & {$\mathbf k$} &  { $\textbf{mad}_{\sf lewis}(\alpha)$} & {$\textbf{mad}_{\sf lewis}(\theta)$} & {\textbf{r.err.} $\hat{\varepsilon}_{\sf lewis}$} & {$\textbf{mad}_{\sf core}(\alpha)$} & { $\textbf{mad}_{\sf core}(\theta)$}  & {\textbf{r.err.} $\hat{\varepsilon}_{\sf core}$} \\ \hline \hline
$11\,532$ & $88$ & $100$ & $1.276$ & $0.165$ & $0.09161$ & $1.477$ & $0.171$ & $0.09335$
		\\ \hline
$11\,532$ & $88$ & $200$ & $0.916$ & $0.163$ & $0.05222$ & $0.930$ & $0.142$ & $0.07134$
		\\ \hline
$11\,532$ & $88$ & $500$ & $0.563$ & $0.082$ & $0.02108$ & $0.499$ & $0.075$ & $0.03795$
		\\ \hline
$11\,532$ & $88$ & $750$ & $0.465$ & $0.070$ & $0.02639$ & $0.432$ & $0.062$ & $0.02675$
		\\ \hline
$11\,532$ & $88$ & $1\,000$ & $0.323$ & $0.051$ & $0.01581$ & $0.320$ & $0.045$ & $0.02007$
		\\ \hline
$11\,532$ & $88$ & $2\,000$ & $0.213$ & $0.025$ & $0.00563$ & $0.182$ & $0.026$ & $0.00506$
		\\ \hline
$11\,532$ & $88$ & $5\,000$ & $0.105$ & $0.008$ & $0.00011$ & $0.101$ & $0.015$ & $0.00147$
		\\ \hline
$11\,532$ & $88$ & $10\,000$ & $0.063$ & $0.013$ & $0.00174$ & $0.071$ & $0.012$ & $0.00094$
		\\ \hline
\end{tabular}
	\end{center}

\end{table*}

\clearpage

\section{ON THE \texorpdfstring{$\mu$}{μ}-COMPLEXITY OF THE INPUT}
\label{appendix:estimation:mu} 

In the theoretical part of this paper, we assumed the $\mu$-complexity parameter to be a constant. An interested reader could ask: how large is this constant in reality, i.e., in the data sets we used to perform our experiments? 

In \citep{MunteanuSSW18} the value of $\mu_1$ was approximated up to a factor $\text{poly}(d)$, in time polynomial in $n$ and $d$ using linear programming, where $d$ is the dimension of the parameter space.
Recently, \citet{DexterKRD2023} showed how to compute $\mu_1$ exactly using linear programming in polynomial time.

In this work, we have $d=2$ for both, 2PL and 3PL models, and the definition of $\mu$ is extended to be the maximum of $\mu_0$ and $\mu_1$ across a wide sequence of iterations (cf.~\secref{preliminaries}). Calculating $\mu$ would thus require to solve a huge number of LPs, which is not viable in our setting. 

A good and fast approximation for $\mu$ can be obtained by evaluating it on the optimal solutions which need to be calculated anyway for the sake of comparison.
Intuitively, this works since logistic regression is tending to minimize the positive part (which corresponds to misclassifications), and to maximize the negative part (which corresponds to correct classifications). This heuristic approach is useful and in practice but it gives only a lower bound for $\mu$ which can in principle be far from the actual value.

Since we use coresets only to reduce the number of examinees in our experiments (cf.~Equation~(\ref{eqn:2pl:logistic:b}) for the 2PL model, resp. optimizing $f(\alpha_i,c_i\mid B)$ in the 3PL model, cf.~\corref{3pl:main:otherdirection}), we report only the values of $\mu_0$ and $\mu_1$ for this case. That is, when $X$ in the definition of $\mu$ depends on the labels $Y$ and the ability parameters $B$ of the complete input, while the supremum is taken over the item parameter vectors in $A$.

We present in \cref{tab:results_appendix:mu} our estimates on $\mu$: the median, the mean and the maximum over all possible items. On average the values of $\mu_0$ and $\mu_1$ are small constants ranging between $2$ and $30$. Only in rare cases $\mu$ takes large maximum values for some label vectors. We checked the corresponding labels, and found that the large values occur only in degenerate cases, in which the maximum likelihood estimator of the model is undefined, for example when an item is solved by all or none of the students. 

\begin{table*}[hp!]
\caption{The approximated values of the parameters $\mu_0$ and $\mu_1$: the mean, median and maximum values over all items $i\in[m]$, where the abilities of $n$ examinees and the respective labels are used as the input and for each $i$ the supremum is taken over item parameters $\alpha_i\in\mathbb{R}^2$.}
	\label{tab:results_appendix:mu}
	\begin{center}
	\begin{tabular}{ l r r || r r | r r | r r }
\multicolumn{3}{c}{}    & \multicolumn{2}{c}{{\bf Mean}} & \multicolumn{2}{c}{{\bf Median}} & \multicolumn{2}{c}{{\bf Maximum}} \\ \hline
Experiment& $n$ & $m$ & $\mu_0$ & $\mu_1$ & $\mu_0$ & $\mu_1$ & $\mu_0$ & $\mu_1$ \\ \hline \hline
2PL-Synt& $50\,000$ & $100$ & $7.85$ & $25.65$ & $5.80$ & $18.09$ & $48.21$ & $165.28$ \\ \hline
2PL-Synt& $50\,000$ & $200$ & $9.56$ & $29.87$ & $6.03$ & $17.57$ & $134.50$ & $377.11$ \\ \hline
2PL-Synt& $50\,000$ & $500$ & $10.41$ & $31.31$ & $5.95$ & $17.20$ & $305.75$ & $703.92$ \\ \hline
2PL-Synt& $100\,000$ & $100$ & $7.86$ & $25.79$ & $5.85$ & $18.16$ & $48.48$ & $164.74$ \\ \hline
2PL-Synt& $100\,000$ & $200$ & $9.41$ & $28.57$ & $5.79$ & $16.90$ & $124.16$ & $329.13$ \\ \hline
2PL-Synt& $100\,000$ & $500$ & $9.65$ & $29.99$ & $5.90$ & $17.12$ & $119.34$ & $296.16$ \\ \hline
2PL-Synt& $200\,000$ & $100$ & $7.84$ & $25.70$ & $5.75$ & $17.72$ & $48.75$ & $164.23$ \\ \hline
2PL-Synt& $200\,000$ & $200$ & $10.05$ & $29.10$ & $5.95$ & $17.50$ & $372.83$ & $715.77$ \\ \hline
2PL-Synt& $200\,000$ & $500$ & $8.98$ & $27.38$ & $5.84$ & $16.95$ & $282.29$ & $557.48$ \\ \hline
2PL-Synt& $500\,000$ & $100$ & $7.83$ & $25.67$ & $5.76$ & $17.82$ & $47.79$ & $161.16$ \\ \hline
2PL-Synt& $500\,000$ & $200$ & $9.65$ & $29.94$ & $6.02$ & $17.60$ & $140.80$ & $383.50$ \\ \hline
2PL-Synt& $500\,000$ & $500$ & $8.90$ & $27.61$ & $5.82$ & $16.76$ & $148.79$ & $427.87$ \\ \hline 
2PL-Synt& $500\,000$ & $5\,000$	& $11.22$ & $34.18$ & $6.29$ & $19.10$ & $1\,765.78$ & $2\,503.41$ \\ \hline
\hline
2PL-SHARE& $138\,997$ & $10$ & $12.87$ & $121.51$ & $11.86$ & $63.58$ & $33.21$ & $382.89$ \\ \hline
2PL-NEPS& $11\,532$ & $88$ & $7.05$ & $14.02$ & $3.02$ & $5.16$ & $58.14$ & $153.18$ \\ \hline
\hline
3PL-Synt& $50\,000$ & $100$ & $3.39$ & $3.36$& $2.00$ & $2.01$ & $38.00$ & $37.03$ \\ \hline
3PL-Synt& $50\,000$ & $200$ & $5.23$ & $5.19$ & $2.15$ & $2.15$ & $120.95$ & $118.47$ \\ \hline
3PL-Synt& $100\,000$ & $100$ & $3.38$ & $3.35$ &$2.00$ & $2.00$ & $37.99$ & $36.90$ \\ \hline
3PL-Synt& $100\,000$ & $200$ & $5.30$ & $5.25$ & $2.19$ & $2.19$ & $136.93$ & $133.64$ \\ \hline
3PL-Synt& $200\,000$ & $100$ & $3.40$ & $3.37$ & $2.01$ & $2.01$ & $38.38$ & $37.24$ \\ \hline
\hline
\end{tabular}
	\end{center}
\end{table*}

\end{document}